\documentclass[11pt]{article}
\usepackage{xcolor}
\usepackage{ulem}
\usepackage{graphicx,wrapfig,lipsum}
\usepackage{float}    % For tables and other floats
\usepackage{verbatim} % For comments and other
\usepackage{amsmath, amssymb, amsthm,mathtools}  % For math
\usepackage{subfig}   % For subfigures
\usepackage[colorlinks=true,citecolor=blue,linkcolor=blue,urlcolor=blue]{hyperref}
\usepackage{bookmark}
\usepackage{fullpage}
\usepackage{enumerate}
\usepackage{paralist}
\usepackage{xspace}
\usepackage{caption}
\usepackage{bbm}
\usepackage{url}
\usepackage{mathrsfs}
\usepackage{lipsum}
\usepackage{algorithm}
\usepackage{algorithmic}
\usepackage{color}
\usepackage{microtype}
\usepackage[numbers]{natbib}
\usepackage[section]{placeins}
\usepackage{lscape}
\usepackage{multirow}
\newcommand{\stoptocwriting}{%
	\addtocontents{toc}{\protect\setcounter{tocdepth}{-5}}}
\newcommand{\resumetocwriting}{%
	\addtocontents{toc}{\protect\setcounter{tocdepth}{\arabic{tocdepth}}}}

\newcommand{\vertiii}[1]{{\left\vert\kern-0.25ex\left\vert\kern-0.25ex\left\vert #1
		\right\vert\kern-0.25ex\right\vert\kern-0.25ex\right\vert}}

\makeatletter

\newcommand{\R}{\mathbb{R}}
\newcommand{\E}{\mathbb{E}}
\newcommand{\Deltatsecond}{\Delta_t}
\newcommand{\sign}{\mathrm{Sign}}

\DeclareMathOperator{\Cov}{Cov}

\newcommand{\err}{\varepsilon}

\newcommand{\rank}{\mathrm{rank}}

\newcommand{\proj}{\mathsf{P}}

\newcommand{\inner}[2]{\left\langle #1, #2 \right\rangle}

\newcommand{\norm}[1]{\left\lVert#1\right\rVert}

\newcommand{\del}{R_t}

\newcommand{\lowerbound}{\zeta}

{\medskip}

{\hspace*{\fill}$\Box$\par\vspace{4mm}}
{\hspace*{\fill}$\Box$\par}

\newcommand{\cA}{\mathcal{A}}

\newcommand{\cG}{\mathcal{G}}
\newcommand{\cH}{\mathcal{H}}

\newcommand{\cN}{\mathcal{N}}
\newcommand{\cO}{\mathcal{O}}

\newcommand{\cQ}{\mathcal{Q}}

\newcommand{\cS}{\mathcal{S}}

\newcommand{\cX}{\mathcal{X}}

\newcommand{\measurementnumber}{m}

\newcommand{\bE}{\mathbb{E}}

\newcommand{\bN}{\mathbb{N}}

\newcommand{\bP}{\mathbb{P}}

\newcommand{\bS}{\mathbb{S}}
\newcommand{\bT}{\mathbb{T}}

\newcommand{\deltaassumptionfinite}{\delta\lesssim {1}/\left({\sqrt{r}\kappa^2\bar\varphi^4 \log^2\left(1/\alpha\right)}\right)}

\newcommand{\bareta}{\bar{\eta}_{t}}
\newcommand{\scale}{\varphi}

\newtheorem{theorem}{Theorem}
\newtheorem{definition}{Definition}
\newtheorem{lemma}{Lemma}

\newtheorem{remark}{Remark}
\newtheorem{corollary}{Corollary}
\newtheorem{proposition}{Proposition}

\begin{document}

	\title{ Global Convergence of Sub-gradient Method for Robust Matrix Recovery: Small Initialization, Noisy Measurements, and Over-parameterization}
	\author{Jianhao Ma$^*$ and Salar Fattahi$^+$\vspace{2mm}\\
		Department of Industrial and Operations Engineering\\
		University of Michigan, Ann Arbor\\
		$^*$\href{mailto:jianhao@umich.edu}{jianhao@umich.edu}, $^+$\href{mailto:fattahi@umich.edu}{fattahi@umich.edu}
	}
	
	\maketitle

	\begin{abstract}

		In this work, we study the performance of sub-gradient method (SubGM) on a natural nonconvex and nonsmooth formulation of \textit{low-rank matrix recovery} with $\ell_1$-loss, where the goal is to recover a low-rank matrix from a limited number of measurements, a subset of which may be grossly corrupted with noise. We study a scenario where the rank of the true solution is unknown and over-estimated instead. The over-estimation of the rank gives rise to an over-parameterized model in which there are more degrees of freedom than needed. Such over-parameterization may lead to overfitting, or adversely affect the performance of the algorithm.  We prove that a simple SubGM with small initialization is \textit{agnostic} to both over-parameterization and noise in the measurements. 
		In particular, we show that small initialization nullifies the effect of over-parameterization on the performance of SubGM, leading to an exponential improvement in its convergence rate. Moreover, we provide the first unifying framework for analyzing the behavior of SubGM under both outlier and Gaussian noise models, showing that SubGM converges to the true solution, even under arbitrarily large and arbitrarily dense noise values, and---perhaps surprisingly---even if the globally optimal solutions \textit{do not} correspond to the ground truth. At the core of our results is a robust variant of restricted isometry property, called Sign-RIP, which controls the deviation of the sub-differential of the $\ell_1$-loss from that of an ideal, expected loss. As a byproduct of our results, we consider a subclass of robust low-rank matrix recovery with Gaussian measurements, and show that the number of required samples to guarantee the global convergence of SubGM is \textit{independent} of the over-parameterized rank.

	\end{abstract}
	
	\stoptocwriting	
	\section{Introduction}
	We study the problem of \textit{robust matrix recovery}, where the goal is to recover a low-rank positive semidefinite matrix $X^\star\in\mathbb{R}^{d\times d}$ from a limited number of linear measurements of the form $\mathbf{y} = \mathcal{A}(X^\star)+\mathbf{s}$, where $\mathbf{y}=[y_1,y_2,\dots,y_m]^\top$ is a vector of measurements, $\mathcal{A}$ is a linear operator defined as $\mathcal{A}(\cdot) = [\langle A_1,\cdot\rangle, \langle A_2,\cdot\rangle,\dots, \langle A_m,\cdot\rangle]^\top$ with measurement matrices $\{A_i\}_{i=1}^m$, and $\mathbf{s} = [s_1,s_2,\dots,s_m]^\top$ is a noise vector. More formally, the robust matrix recovery is defined as
	\begin{equation}
		\label{eq_RMLR}
		\begin{aligned}
			\text{find} & \ \ X^\star\ \
			\qquad \text{subject to:} \quad \  \mathbf{y} = \mathcal{A}(X^\star)+\mathbf{s},\ \ \rank(X^\star) = r,
			% \ \ X^*\in\mathcal{C},
		\end{aligned}
	\end{equation}
	where $r\leq d$ is the rank of $X^\star$. Robust matrix recovery plays a central role in many contemporary machine learning problems, including motion detection in video frames~\cite{bouwmans2014robust}, face recognition~\cite{luan2014extracting}, and collaborative filtering in recommender systems~\cite{luo2014efficient}. Despite its widespread applications, it is well-known that solving~\eqref{eq_RMLR} is a daunting task since it amounts to an NP-hard problem in its worst case~\cite{natarajan1995sparse, recht2010guaranteed}. What make this problem particularly difficult is the {nonconvexity} stemming from the rank constraint. The classical methods for solving low-rank matrix recovery problem are based on {convexification} techniques, which suffer from notoriously high computational cost. To alleviate this issue, a far more practical approach is to resort to the following natural \textit{nonconvex} and \textit{nonsmooth} formulation
	\begin{equation}
		\label{eq_l1}
		\begin{aligned}
			\min_{U\in\mathbb{R}^{d\times r'}} & \ \ f_{\ell_1}(U) := \frac{1}{m}\left\|\mathbf{y}-\mathcal{A}\left(UU^\top\right)\right\|_{1},
			% \quad\ \text{subject to:} \ \ (U,V)\in\mathcal{C},
		\end{aligned}
	\end{equation}
	where $r'\geq r$ is the search rank. The $\ell_1$-loss is used to robustify the solution against noisy measurements. The above formulation is inspired by the celebrated Burer-Monteiro approach~\cite{burer2003nonlinear}, which circumvents the explicit rank constraint by optimizing directly over the factorized model $X^\star = UU^\top$.
	
	Perhaps the most significant breakthrough result in this line of research was presented by~\citet{bhojanapalli2016global}, showing that, when the rank of the true solution is known and the measurements are noiseless, the nonconvex formulation of the problem with a smooth $\ell_2$-loss has a \textit{benign landscape}, i.e., it is devoid of undesirable local solutions; as a result, simple local-search algorithms are guaranteed to converge to the globally optimal solution. Such benign landscape seems to be omnipresent in other variants of low-rank matrix recovery, including matrix completion~\cite{ge2017no, ge2016matrix}, robust PCA~\cite{ge2017no, fattahi2020exact}, sparse dictionary learning~\cite{sun2016complete, qu2019analysis}, linear neural networks~\cite{kawaguchi2016deep}, among others; see recent survey papers~\cite{chi2019nonconvex, zhang2020symmetry}.

	A recurring assumption for the absence of spurious local minima is the {exact parameterization} of the rank: it is often presumed that the \textit{exact} rank of the true solution is known \textit{a priori}. However, the rank of the true solution is rarely known in many applications. 
	Therefore, it is reasonable to choose the rank of $UU^\top$ conservatively as $r'>r$, leading to an \textit{over-parameterized} model. This challenge is further compounded in the noisy regime, where the injected noise in the measurements can be ``absorbed'' as a part of the solution, due to the additional degrees of freedom in the model. Evidently, the existing proof techniques face major breakdowns in this setting, {\it as the problem may no longer enjoy a benign landscape}. Moreover,  over-parameterization may lead to a dramatic, exponential slow-down of the local-search algorithms---both theoretically and practically~\cite{zhuo2021computational, zhang2021preconditioned}. 
	
	In this work, we study the performance of a simple sub-gradient method (SubGM) on $f_{\ell_1}(U)$. We prove that small initialization nullifies the effect of over-parameterization on its performance---as if the search rank $r'$ were set to the true (but unknown) rank $r$. Moreover, we show that SubGM converges to the ground truth at a near-linear rate even if local, or even global, spurious minima exist. Our proposed overarching framework is based on a novel signal-residual decomposition of the the solution trajectory: we decompose the iterations of SubGM into \textit{low-rank (signal)} and \textit{residual} terms, and show that small initialization keeps the residual term small throughout the solution trajectory, while enabling the low-rank term to converge to the ground truth exponentially fast. 
	
	\subsection{Power of Small Initialization}\label{subsec:subgm}
	
	In this section, we shed light on the power of small initialization on the performance of SubGM for the robust matrix recovery.
	Given an initial point $U_0$ and at every iteration $t$, SubGM selects an arbitrary direction $D_t$ from the (Clarke) sub-differential~\cite{clarke1990optimization} of the $\ell_1$-loss function $\partial f_{\ell_1}(U_t)$. Due to local Lipschitzness of the  $\ell_1$-loss, the Clarke sub-differential exists and can be obtained via chain rule (see~\cite{clarke1990optimization}):
	\begin{align}\label{eq_subdif}
		\partial f_{\ell_1}(U_t) = \frac{1}{m}\sum_{i=1}^m\sign\left(\langle A_i,U_t U_t^{\top}-X^{\star}\rangle\right)\left({A_i+A_i^\top}\right) U_t.
	\end{align}
	At every iteration, SubGM updates the solution by moving towards $-D_t$---for an arbitrary choice of $D_t\in\partial f_{\ell_1}(U_t)$---with a step-size $\eta_t$. To showcase the effect of small initialization on the performance of SubGM, we consider an instance of robust matrix recovery, where the true solution $X^*$ is a randomly generated matrix with rank $r=3$ and dimension $d=20$. Furthermore, we consider $m=500$ measurements, where the measurement matrices $\{A_i\}_{i=1}^m$ have i.i.d. standard Gaussian entries.
	\vspace{2mm}
	
	\begin{figure*}
		% \vskip 0.2in
		\begin{center}
			\subfloat[]{
				{\includegraphics[width=5.6cm]{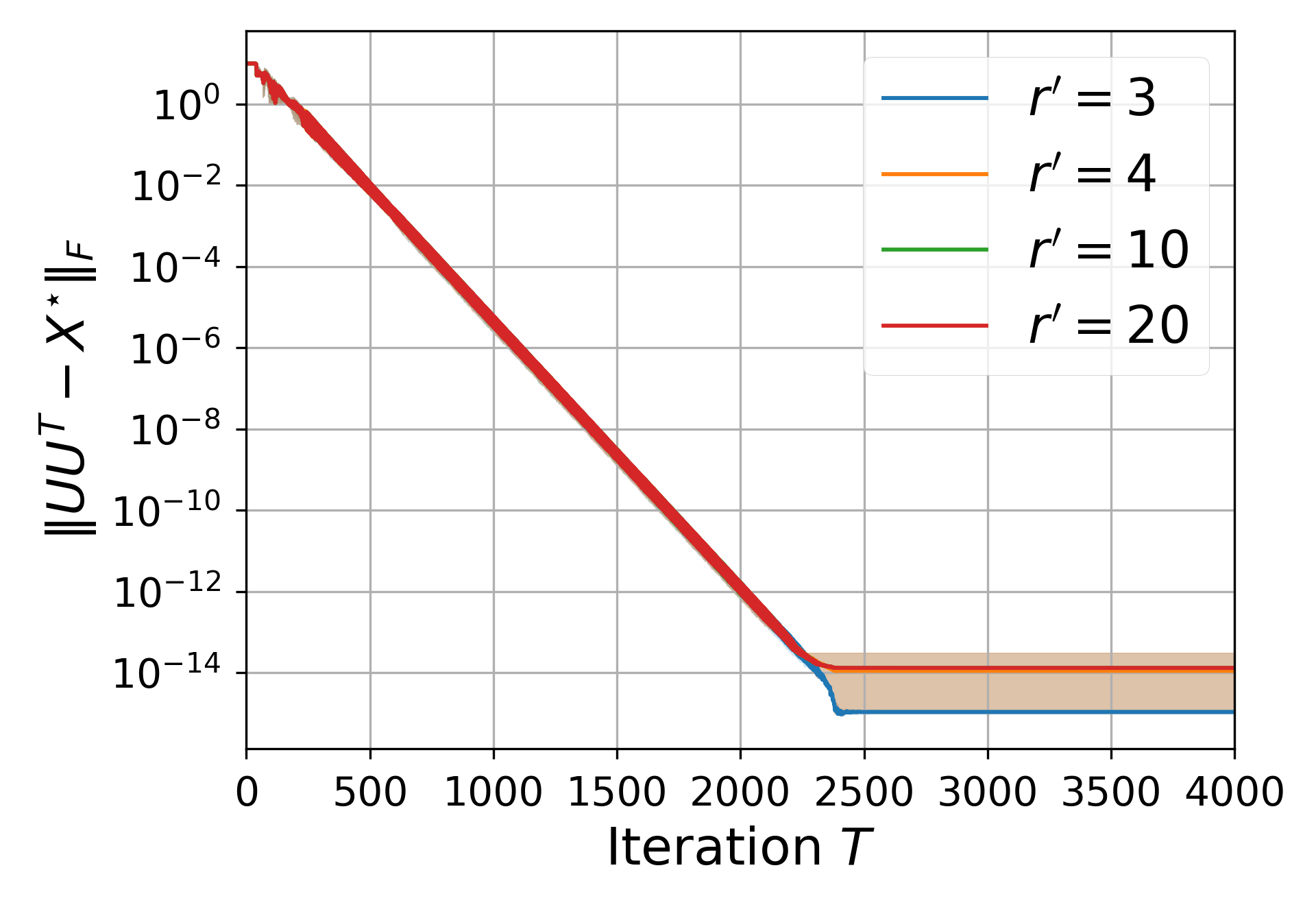}}\label{fig::rank}}\hspace{-5mm}
			\subfloat[]{
				{\includegraphics[width=5.6cm]{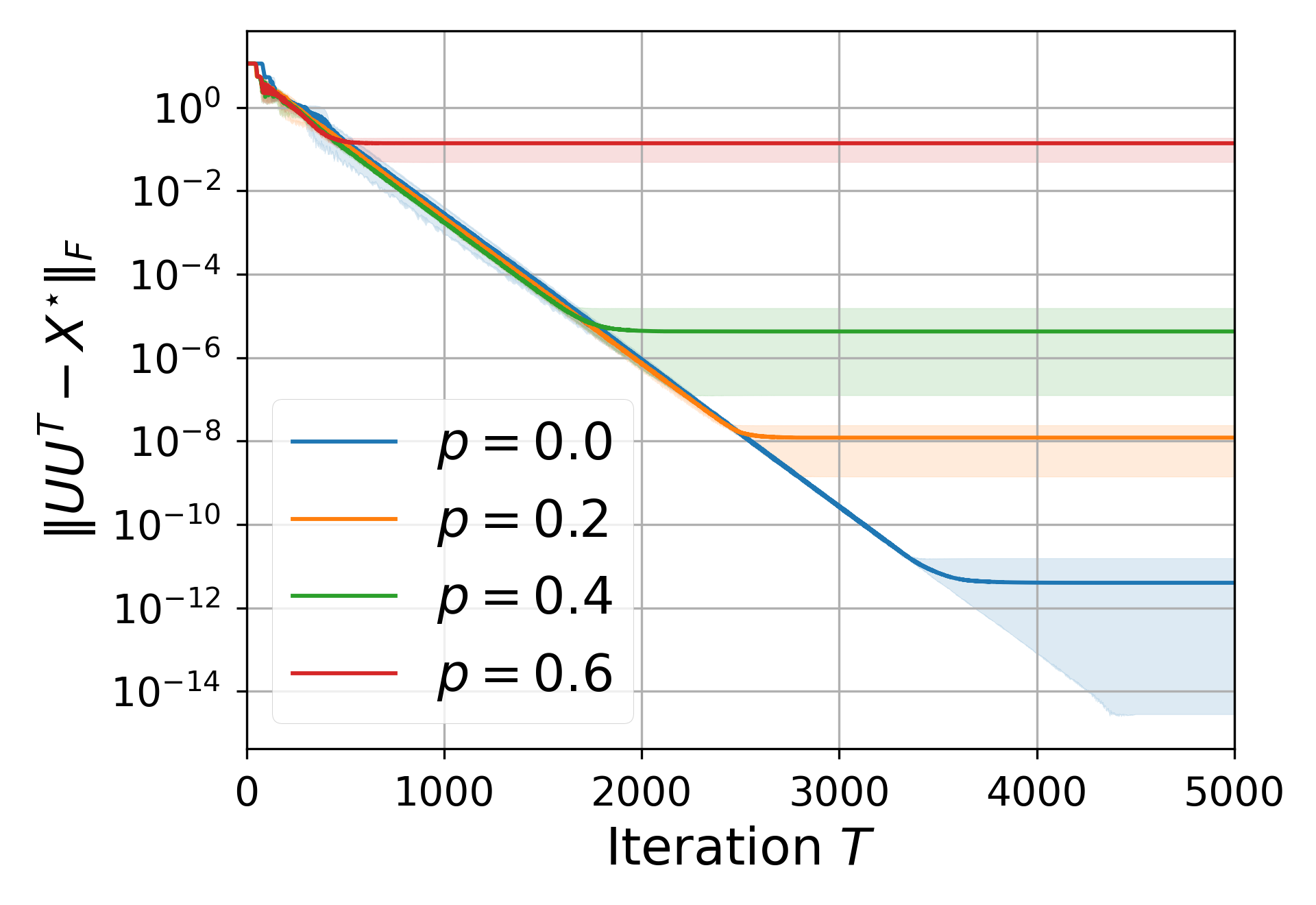}}\label{fig::noise-level}}\hspace{-5mm}
			\subfloat[]{
				{\includegraphics[width=5.6cm]{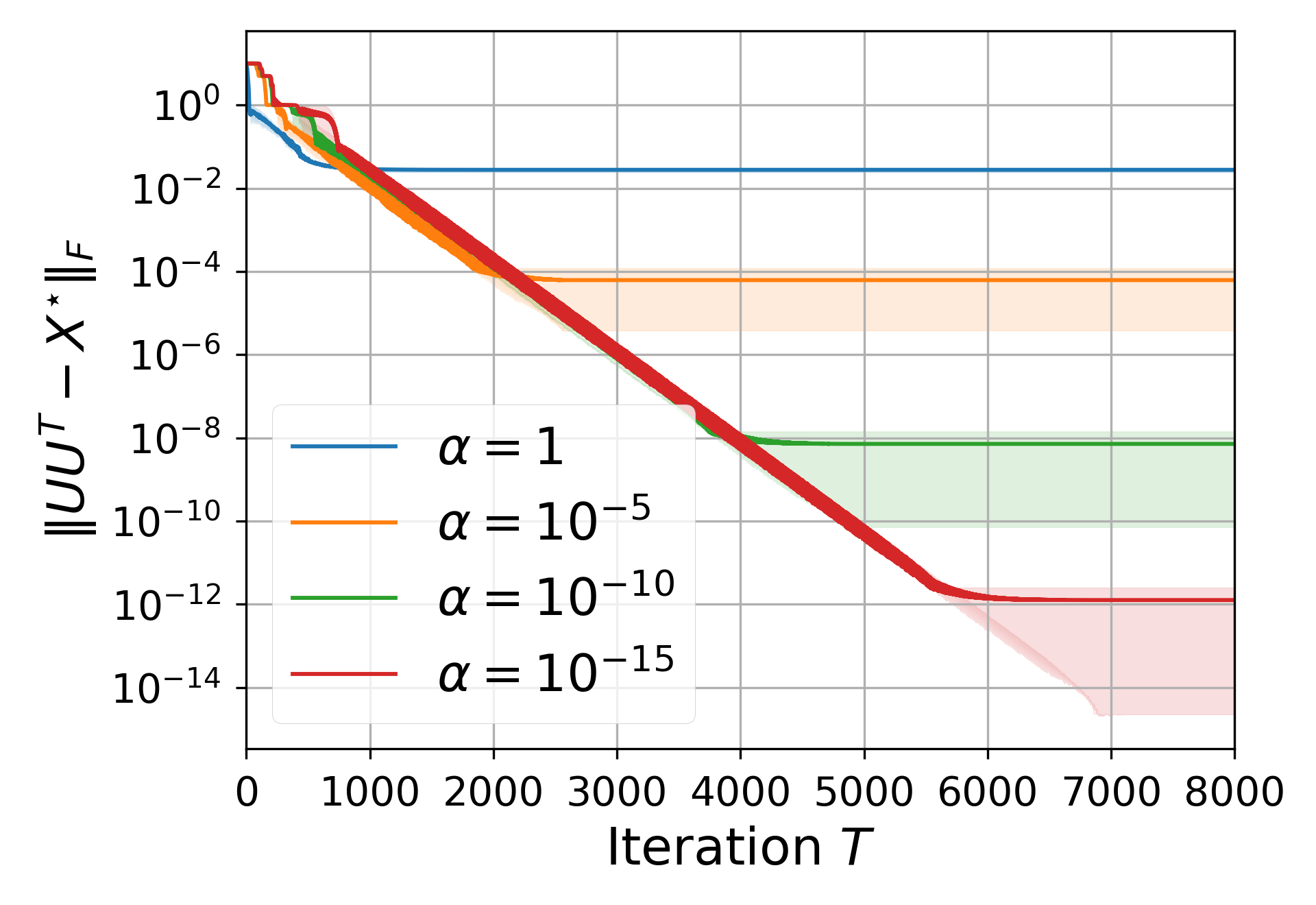}}\label{fig::init}}
		\end{center}
		\caption{\footnotesize (a) The performance of SubGM for different search ranks $r'$. (b) The performance of SubGM for different corruption probabilities. (c) The performance of SubGM for different values of the initialization scale $\alpha$. In all of the simulations, the initial point $U_0$ is chosen as $\alpha B$, where $B$ is obtained from Algorithm~\ref{alg::spectral-initialization}.}
	\end{figure*}
	
	\noindent{\bf Property 1: Small initialization makes SubGM agnostic to over-parameterization.} Figure~\ref{fig::rank} shows the performance of SubGM with small initialization for both exact ($r'=3$) and over-parameterized ($r'>3$) settings, where $10\%$ of the measurements are grossly corrupted with noise. Our simulations uncover an intriguing property of small initialization: neither the convergence rate nor the final error of SubGM is affected by the over-estimation of the rank. 
	Moreover, Figure~\ref{fig::noise-level} depicts the performance of SubGM for the fully over-parameterized problem (i.e., $r'=d=20$) with different levels of corruption probability (i.e., the fraction of measurements that are corrupted with large noise values). It can be seen that, even in the fully over-parameterized setting, SubGM is robust against large corruption probabilities.\vspace{2mm}
	
	\noindent{\bf Property 2: Small initialization improves convergence.} It is known that different variants of (sub-)gradient method converge linearly to the true solution, provided that the search rank coincides with the true rank ($r'=r$)~\cite{tu2016low, zheng2015convergent, tong2021accelerating, li2020nonconvex}. However, these methods suffer from a dramatic, exponential slow-down in over-parameterized models with noisy measurements~\cite{zhuo2021computational}. Our simulations reveal that small initialization can restore the convergence back to linear, even in the over-parameterized and noisy settings. Figure~\ref{fig::init} shows that SubGM converges linearly to an error that is proportional to the norm of the initial point: smaller initial points lead to more accurate solutions at the expense of slightly larger number of iterations.\vspace{2mm}
	
	\begin{figure*}
		% \vskip 0.2in
		\begin{center}
			\subfloat[]{
				{\includegraphics[width=6.5cm]{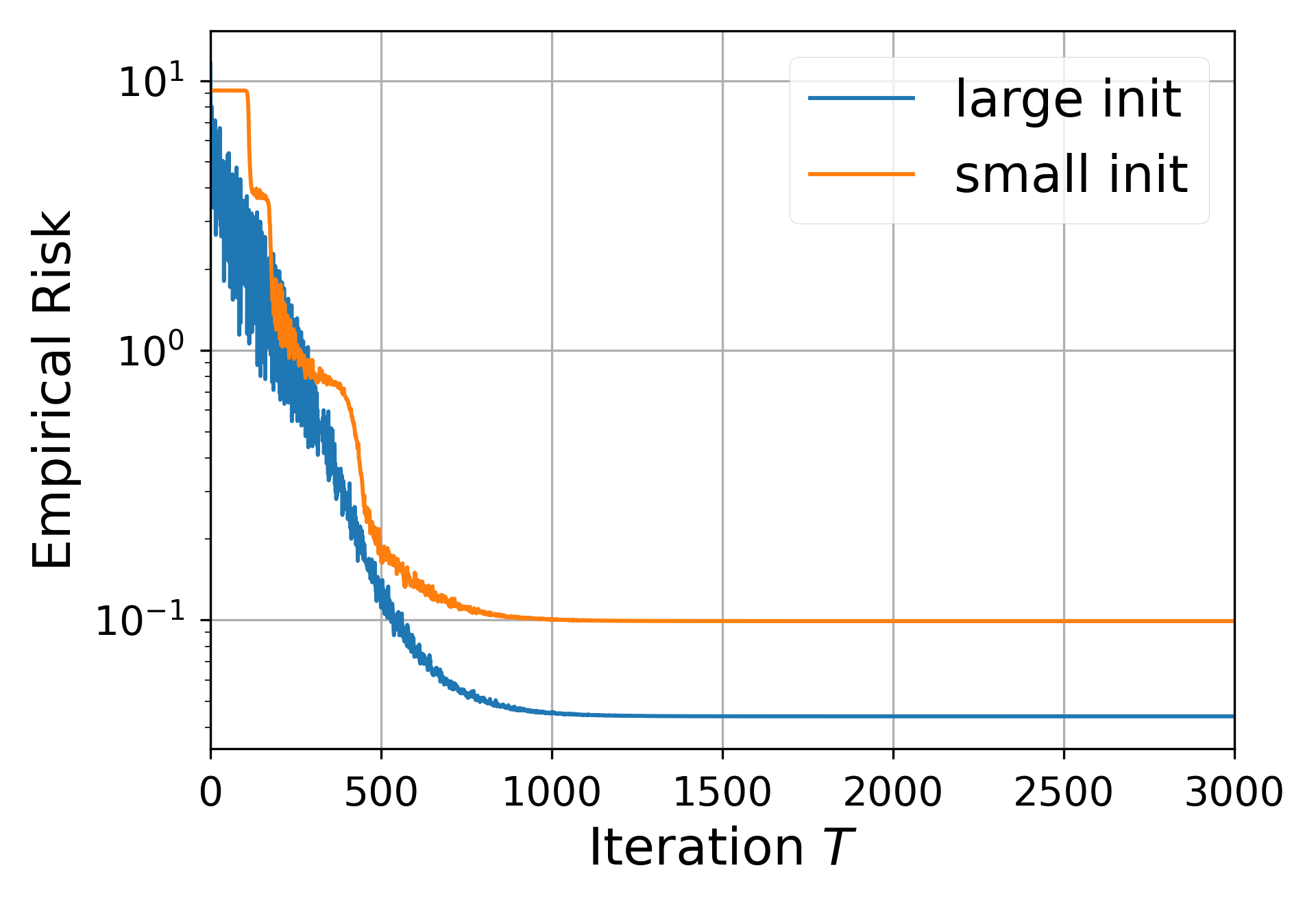}}\label{fig::empirical-risk}}
			\subfloat[]{
				{\includegraphics[width=6.5cm]{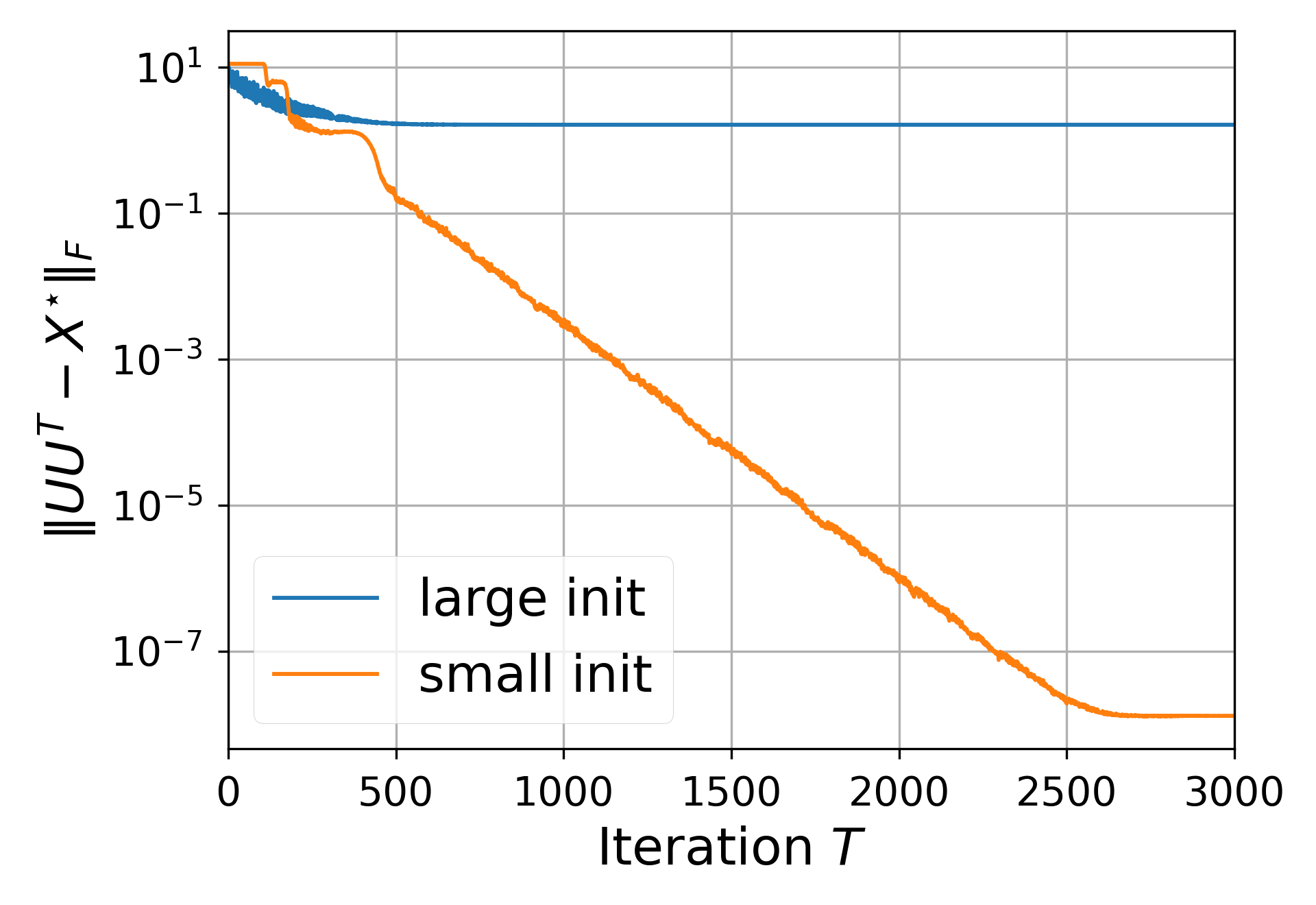}}\label{fig::generalization-error}}
		\end{center}
		\caption{\footnotesize (a) The objective value of the solutions obtained via SubGM with and without small initialization. (b) The error of the solutions obtained via SubGM with and without small initialization. In both instances, the initial point is chosen as $U_0 = \alpha B$, where $B$ is obtained from Algorithm~\ref{alg::spectral-initialization}. The initialization scale $\alpha$ is chosen as $\alpha = 10^{-15}$ and $\alpha = 1$, for SubGM with and without small initialization, respectively.}\label{fig:benign}
	\end{figure*}
	
	\noindent{\bf Property 3: Emergence of ``spurious'' global minima.} Inspired by these simulations, a natural approach to explain the desirable performance of SubGM is by showing that the robust matrix recovery problem enjoys a benign landscape. We refute this conjecture by showing that, not only does the robust matrix recovery with over-parameterized rank have sub-optimal solutions, but also its globally optimal solutions may be ``spurious'', i.e., they do not correspond to the ground truth $X^*$. Figure~\ref{fig:benign} shows the performance of SubGM with and without small initialization. It can be seen that SubGM converges to the ground truth, which is a local solution for the $\ell_1$-loss with sub-optimal objective value. On the other hand, SubGM without small initialization converges to a high-rank solution with strictly smaller objective value. In other words, the ground truth is not necessarily a globally optimal solution, and conversely, globally optimal solutions do not necessarily correspond to the ground truth. 
	
	From a statistical perspective, our simulations support the common empirical observation that first-order methods ``generalize well''. In particular, SubGM converges to a low-rank solution that is close to the ground truth---i.e., has a better \textit{generalization error}---rather than recovering a high-rank solution with a smaller objective value (or better \textit{training error}). The smaller objective values for higher rank solutions is precisely due to the overfitting phenomenon: it is entirely possible that the globally optimal solution to~\eqref{eq_l1} achieves a zero objective value by absorbing the noise into its redundant ranks.
	To circumvent the issue of overfitting, a common approach is to regularize the high-rank solutions in favor of the low-rank ones via different regularization techniques. Therefore, the desirable performance of SubGM with small initialization can be attributed to its \textit{implicit regularization} property. In particular, we show that small initialization of SubGM is akin to implicitly regularizing the redundant rank of the over-parameterized model, thereby avoiding overfitting; a recent work~\cite{stoger2021small} has shown a similar property for the gradient descent algorithm on the noiseless matrix recovery with $\ell_2$-loss.
	
	\subsection{Summary of Results}
	
	In this part, we present a summary of our results.
	Let $\sigma_1$ and $\sigma_r$ be the largest and smallest (nonzero) eigenvalues of $X^\star$, and define the condition number $\kappa$ as $\sigma_1/\sigma_r$.
	
	\begin{theorem}[Convergence of SubGM; Informal]
		Suppose that the measurements satisfy a direction-preserving property delineated in Section~\ref{subsec:sign-RIP}. Suppose that the initial point is chosen as $ U_0 = \alpha B$, for a special choice of $B$ and a initialization scale $\alpha$. Consider the iterations $\{U_t\}_{t=0}^T$ generated by SubGM applied to the robust matrix recovery with step-size $\eta_t=\eta\rho^t$, for an appropriate choice of $0<\rho<1$ and sufficiently small $\eta$. Then, for any arbitrary accuracy $\err>0$ and initialization scale $\alpha = \mathcal{O}((\err/d)^{1/\beta})$, we have
		\begin{equation}\label{eq_gen_error_informal}
			\norm{U_TU_T^{\top}-X^{\star}}_F\leq \err
		\end{equation}
		after $T = \cO\left(\frac{ \kappa\log^2\left( d/\err\right)}{{\beta\eta}}\right)$ iterations, where $0<\beta\leq 2$ is a constant depending on the parameters of the problem.
		\label{thm::informal-1}
	\end{theorem}
	
	The above result characterizes the performance of SubGM for the robust matrix recovery with $\ell_1$-loss. In particular, it shows that SubGM converges almost \textit{linearly} to the true low-rank solution $X^\star$, with a final error that is proportional to the initialization scale. Surprisingly, the required number of iterations is independent of the search rank $r'$ and depends only logarithmically on $d$.
	
	At the crux of our analysis lies a new restricted isometry property of the sub-differentials, which we call Sign-RIP.
	Under Sign-RIP, the sub-differentials of the $\ell_1$-loss are $\delta$-away from the sub-differentials of an ideal, expected loss function (see Section~\ref{subsec:sign-RIP} for precise definitions). We will show that the classical notions of $\ell_2$-RIP~\cite{recht2010guaranteed} and $\ell_1/\ell_2$-RIP~\cite{li2020nonconvex} face major breakdowns in the presence of noise. In contrast, Sign-RIP provides a much better robustness against noisy measurements, while being no more restrictive than its classical counterparts. We will show that, with Gaussian measurements, the Sign-RIP holds with an overwhelming probability under two popular noise models, namely \textit{outlier noise model} and \textit{Gaussian noise model}. 
	
	Our next theorem establishes the convergence of SubGM under outlier noise model. To streamline the presentation, we use $\tilde{\mathcal{O}}(\cdot)$ and $\tilde{\Omega}(\cdot)$ to hide the dependency on logarithmic factors.
	
	\begin{theorem}[Convergence of SubGM under Outlier Noise Model; Informal]\label{thm_outlier_informal}
		Suppose that the measurement matrices $\{A_i\}_{i=1}^m$ have i.i.d. standard Gaussian entries, and a fraction $p<1$ of the measurements are corrupted with arbitrarily large noise values. Suppose that the initial point is chosen as $ U_0 = \alpha B$, for a special choice of $B$ and a sufficiently small initialization scale $\alpha$. Consider the iterations $\{U_t\}_{t=0}^T$ generated by SubGM applied to the robust matrix recovery with an exponentially decaying step-size $\eta_t = \eta\rho^t$, for an appropriate choice of $0<\rho<1$ and sufficiently small $\eta$. Finally, suppose that the number of measurements satisfies $m=\tilde{\Omega}\left( {\kappa^4 dr^2}/{(1-p)^2}\right)$. Then, for any arbitrary accuracy $\err>0$ and initialization scale $\alpha = \err/d$, and with an overwhelming probability, we have
		\begin{equation}
			\norm{U_TU_T^{\top}-X^{\star}}_F\leq \err,
		\end{equation}
		after $T = {\mathcal O}\left(\frac{\kappa\log^2(d/\err)}{\eta}\right)$ iterations.
	\end{theorem}
	
	Theorem~\ref{thm_outlier_informal} shows that small initialization enables SubGM to converge almost linearly, which is exponentially faster than the sublinear rate $\tilde{\mathcal{O}}(1/\err)$ introduced by~\citet{ding2021rank}. Second,~\citet{ding2021rank} show that SubGM requires $\tilde{\Omega}(\kappa^{12}d{r'}^{3})$ samples to converge, which depends on the search rank $r'$. In the over-parameterized regime, where the true rank is small (i.e., $r = \mathcal{O}(1)$) and the search rank is large (i.e., $r' = \Omega(d)$), our result leads to \textit{three orders of magnitude} improvement in the required number of samples (modulo the dependency on $\kappa$). Moreover,~\citet{ding2021rank} crucially rely on the equivalence between globally optimal solutions and the ground truth, which only holds when $p\leq 1/\sqrt{r'}$. We relax this assumption and show that SubGM converges to the ground truth, even if $p$ is arbitrarily close to 1. 
	
	Next, we turn our attention to the Gaussian noise model, and show that SubGM converges even if the measurements are corrupted with a dense, Gaussian noise.
	
	\begin{theorem}[Convergence of SubGM under Gaussian Noise Model; Informal]\label{thm_dense_informal}
		Suppose that the measurement matrices $\{A_i\}_{i=1}^m$ have i.i.d. standard Gaussian entries, and each measurement is corrupted with a zero-mean Gaussian noise with a variance of at most $\nu^2$. Suppose that the initial point is chosen as $ U_0 = \alpha B$, for a special choice of $B$ and a sufficiently small initialization scale $\alpha$. Consider the iterations $\{U_t\}_{t=0}^T$ generated by SubGM applied to the robust matrix recovery with exponentially decaying step-sizes $\eta_t = \eta\rho^t$, for an appropriate choice of $0<\rho<1$ and sufficiently small $\eta$. Finally, suppose that the number of measurements satisfies $m\gtrsim \tilde{\Omega}( \nu^2\kappa^4 dr^2)$. Then, with an overwhelming probability, we have
		\begin{align}\label{eq_dense_informal}
			\norm{U_tU_t^\top-X^\star}_F = \tilde{\mathcal{O}}\left(\sqrt{\frac{\nu^2dr^2}{m}}\right),
		\end{align}
		after $T = {\mathcal O}\left(\frac{\kappa}{\eta}\log^2\left(\frac{m\kappa}{\nu r}\right)\right)$ iterations.
	\end{theorem}
	
	\begin{figure}
		% \vskip 0.2in
		\begin{center}
			\subfloat[]{
				{\includegraphics[width=6cm]{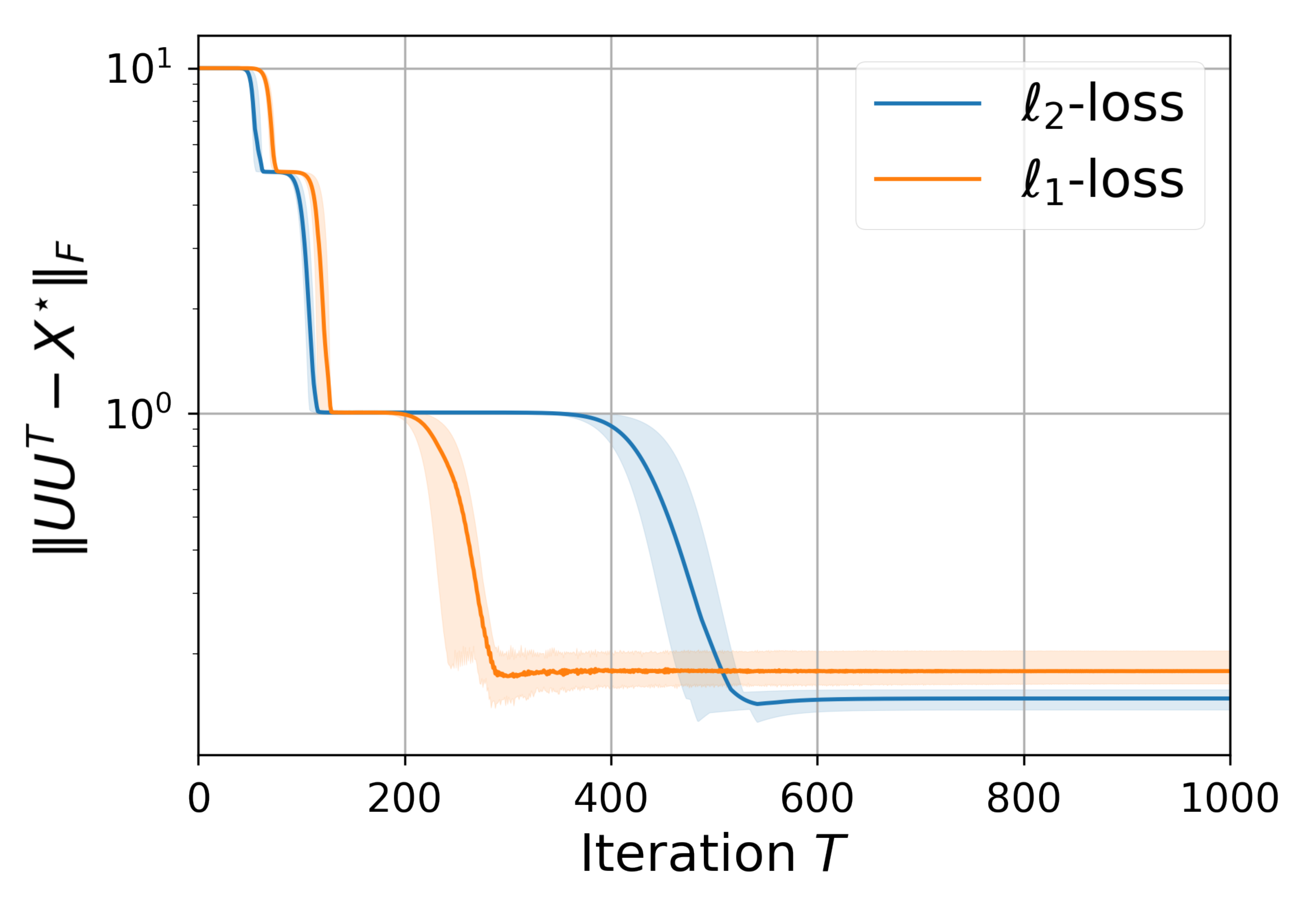}}\label{fig::comprision_l1_l2}}
			\subfloat[]{
				{\includegraphics[width=6.1cm]{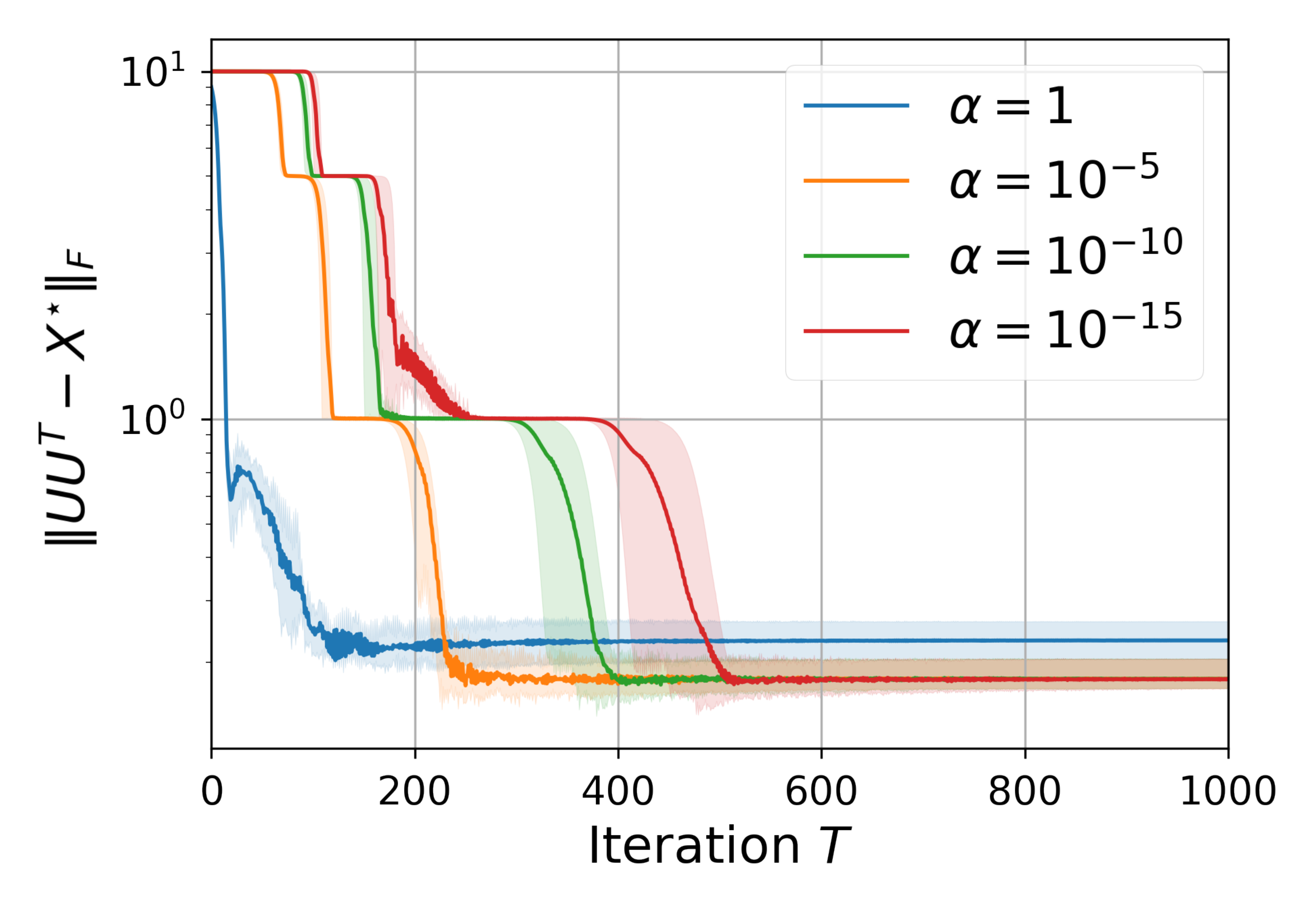}}\label{fig::init-fully-corrupted}}
		\end{center}\vspace{-5mm}
		\caption{\footnotesize (a) Performance of SubGM and GD under the Gaussian noise model with $\ell_1$- and $\ell_2$-loss functions, respectively. (b) The effect of small initialization on the performance of SubGM under the Gaussian noise model.}\label{fig_outlier-2}
	\end{figure}
	
	Traditionally, $\ell_2$-loss has been used for recovering the ground truth under Gaussian noise model, due to its correspondence to the so-called maximum likelihood estimation. Our paper extends the application of $\ell_1$-loss to this setting, proving that SubGM is robust against not only the outlier, but also Gaussian noise values. More precisely, Theorem~\ref{thm_dense_informal} shows that SubGM outputs a solution with an estimation error of $\tilde{\mathcal{O}}(\sqrt{{\nu^2dr^2}/{m}})$, which is again independent of the search rank $r'$. To the best of our knowledge, the sharpest known estimation error for gradient descent (GD)~\cite{zhuo2021computational} and its variants~\cite{zhang2021preconditioned} on $\ell_2$-loss is $\mathcal{O}(\sqrt{{\nu^2dr'}/{m}})$, which scales with the search rank $r'$; in the fully over-parameterized regime, our provided bound improves upon this error by a factor of $\mathcal{O}(\sqrt{d/r})$. Figure~\ref{fig_outlier-2} compares the performance of SubGM and GD on $\ell_1$- and $\ell_2$-losses, when the measurements are corrupted with Gaussian noise.~\citet{candes2011tight} showed that \textit{any} estimate $\widehat{U}$ suffers from a minimax error of $\|{\widehat{U}\widehat{U}^\top-X^\star}\|_F = \Omega(\sqrt{{\nu^2dr}/{m}})$. Compared to this information-theoretic lower bound, our provided final error is sub-optimal only by a factor of $\sqrt{r}$.
	
	\section{Related Work}
	
	\paragraph{Landscape v.s. Trajectory Analysis:}
	It has been recently shown that different variants of low-rank matrix recovery (e.g., matrix completion \cite{ge2016matrix}, matrix recovery \cite{ge2017no}, robust PCA \cite{fattahi2020exact}) enjoy benign landscape. In particular, it is shown that low-rank matrix recovery with $\ell_2$-loss and noiseless measurements has a benign landsacpe in both exact~\cite{ge2017no, ge2016matrix, zhang2019sharp} and over-parameterized~\cite{zhang2021sharp} settings.
	On the other hand, it is known that $\ell_1$-loss possesses better robustness against outlier noise. However, there are far fewer results characterizing the landscape of low-rank matrix recovery with $\ell_1$-loss. \citet{fattahi2020exact} and \citet{josz2018theory} prove that robust matrix recovery with $\ell_1$-loss has no spurious local solution, provided that  with $r'=r=1$ and the measurement matrices correspond to element-wise projection operators. However, it is unclear whether these results extend to higher ranks or more general measurement matrices.
	
	Despite its theoretical significance, benign landscape is too restrictive to hold in practice:~\citet{zhang2019sharp} and~\citet{zhang2021sharp} show that spurious local minima are ubiquitous in the low-rank matrix recovery, even under fairly mild conditions. On the other hand, our experiments in Subsection~\ref{subsec:subgm} reveals that local-search algorithms may be able to avoid spurious local/global solutions with proper initialization. An alternative approach to explain the desirable performance of local-search algorithms is via \textit{trajectory analysis}. It has been recently shown that the trajectories picked up by gradient-based algorithms benefit from implicit regularization~\cite{gunasekar2018implicit}, or behave non-monotonically over short timescales, yet consistently improve over long timescales~\cite{cohen2021gradient}. In the context of over-parameterized low-rank matrix recovery with $\ell_2$-loss, \citet{li2018algorithmic} and~\citet{stoger2021small} use trajectory analysis to show that GD with small initialization can recover the ground truth, provided that the measurements are noiseless. \citet{zhuo2021computational} extend this result to the noisy setting, showing that GD converges to a minimax optimal solution at a sublinear rate, and with a number of samples that scale with the search rank.
	
	\paragraph{Iteration and Sample Complexity:} Despite their guaranteed convergence, local-search algorithms may suffer from notoriously slow convergence rates: whereas 10 digits of accuracy can be expected in a just few hundred iterations of GD when $r' = r$, tens of thousands of iterations might produce just 1-2 accurate digits once $r'>r$~\cite{zhang2021preconditioned}. Table~\ref{table::computational-complexity} shows the iteration complexity of the existing algorithms with different loss functions, compared to our proposed method. Evidently, under the outliear noise model, GD does not perform well due to the sensitivity of the $\ell_2$-loss to ourliers. In contrast, SubGM converges linearly in the exact setting ($r'=r$), and at a significantly slower (sublinear) rate in the over-parameterized regime ($r'>r$). In contrast, our proposed SubGM algorithm with small initialization converges near-linearly in \textit{both} the exact and over-parameterized regimes. In the Gaussian noise model, it is known that GD converges linearly to a minimax optimal solution in the exact setting, but suffers from a drastic, exponential slow-down in the over-parameterized regime. In contrast, our proposed SubGM algorithm with small initialization is not affected by the over-parameterization, and maintains its desirable convergence rate in both settings. 
	
	\begin{table}[]
		\centering\small
		\begin{tabular}{|c|c|c|c|c|}
			\hline
			\multicolumn{1}{|c|}{\multirow{2}{*}{Algorithm}}   & \multicolumn{2}{c|}{Outlier noise model}                                                                      & \multicolumn{2}{c|}{Gaussian noise model}                                                                                                                                                                                                                      \\ \cline{2-5}  &\multicolumn{1}{c|}{$r'=r$} &\multicolumn{1}{c|}{$r'\geq r$}  &\multicolumn{1}{c|}{$r'=r$} &\multicolumn{1}{c|}{$r'\geq r$}\\ \hline
			\multicolumn{1}{|c|}{\texttt{GD}+$\ell_2$-loss}    & \multicolumn{1}{c|}{N/A}                                                                                      & \multicolumn{1}{c|}{N/A}                                                   & \multicolumn{1}{c|}{$\kappa\log\left(\frac{m}{\nu dr}\right)$ \citep{chen2015fast}} & \multicolumn{1}{c|}{$\frac{\sigma_r}{\nu}\sqrt{\frac{m}{d}}$ \citep{zhuo2021computational}} \\ \hline
			\multicolumn{1}{|c|}{\texttt{SubGM}+$\ell_1$-loss} & \multicolumn{1}{c|}{$\frac{r\kappa}{\sigma_r}\log\left(\frac{\sigma_r}{\err}\right)$ \citep{li2020nonconvex}} & \multicolumn{1}{c|}{$\frac{\sigma_1\kappa}{\err}$~\cite{ding2021rank}}     & \multicolumn{1}{c|}{N/A}                                                            & \multicolumn{1}{c|}{N/A}                                                                    \\ \hline
			\multicolumn{1}{|c|}{\bf Our results}              & \multicolumn{2}{c|}{$\kappa^2\log^3\left(\frac{d}{\err}\right)$ (see Corollary~\ref{cor:outlier})}                                              & \multicolumn{2}{c|}{$\kappa^2\log^3\left(\frac{\kappa m}{\nu r}\right)$ (see Corollary~\ref{cor::gaussian-noise})}                                                                                                                                                                                     \\ \hline
		\end{tabular}
		\caption {\footnotesize A comparison between the {\bf iteration complexity} of different techniques. We show that SubGM with small initialization and exponentially decaying step-size converges near-linearly to: $(i)$ an arbitrary accuracy in the outlier noise model, and $(ii)$ a nearly-minimax optimal error in the Gaussian noise model. Our derived iteration complexities are obtained from Theorems~\ref{thm_outlier_informal} and~\ref{thm_dense_informal} after choosing an appropriate value for the step-size; see Corollaries~\ref{cor:outlier} and~\ref{cor::gaussian-noise}  for the precise statements.}
		\label{table::computational-complexity}
	\end{table}
	
	\begin{table}[]
		\centering\small
		\begin{tabular}{|c|c|c|c|c|}
			\hline
			\multicolumn{1}{|c|}{\multirow{2}{*}{Algorithm}}   & \multicolumn{2}{c|}{Outlier noise model}                          & \multicolumn{2}{c|}{Gaussian noise model}                                                                                                                                                                 \\ \cline{2-5}  &\multicolumn{1}{c|}{$r'=r$} &\multicolumn{1}{c|}{$r'\geq r$}  &\multicolumn{1}{c|}{$r'=r$} &\multicolumn{1}{c|}{$r'\geq r$}\\ \hline
			\multicolumn{1}{|c|}{\texttt{GD}+$\ell_2$-loss}    & \multicolumn{1}{c|}{N/A}                                          & \multicolumn{1}{c|}{N/A}                                           & \multicolumn{1}{c|}{$\nu^2\kappa^2 dr$ \citep{chen2015fast}} & \multicolumn{1}{c|}{$\nu^2\kappa^2dr'$ \citep{zhuo2021computational}} \\ \hline
			\multicolumn{1}{|c|}{\texttt{SubGM}+$\ell_1$-loss} & \multicolumn{1}{c|}{${\kappa^2dr^2}$ \citep{li2020nonconvex}$^*$} & \multicolumn{1}{c|}{$\kappa^{12} d{r'}^3$~\cite{ding2021rank}$^*$} & \multicolumn{1}{c|}{N/A}                                     & \multicolumn{1}{c|}{N/A}                                              \\ \hline
			\multicolumn{1}{|c|}{\bf Our results}              & \multicolumn{2}{c|}{$\frac{\kappa^4dr^{2}}{(1-p)^2}$ (see Corollary~\ref{cor:outlier})}             & \multicolumn{2}{c|}{$\nu^2 \kappa^4dr^{2}$ (see Corollary~\ref{cor::gaussian-noise})}                                                                                                                                                               \\ \hline
		\end{tabular}
		\caption {\footnotesize A comparison between the {\bf sample complexity} of different techniques. Our results provide the best sample complexity bounds in the over-parameterized setting where $r'\gg r$, under both outlier and Gaussian noise models. For simplicity, we hide the dependency on the logarithmic factors.
			$^*$Under an oulier noise model, the results of~\cite{li2020nonconvex, ding2021rank} holds under the assumption $p\lesssim 1/\sqrt{r'}$. In contrast, our result relaxes this assumption to $p<1$.}
		\label{table::sample-complexity}
	\end{table}
	Another important aspect of local-search algorithms is their sample complexity. 
	Table~\ref{table::sample-complexity} provides a comparison between the sample complexity of the existing algorithms, and our proposed method. In the outlier noise model,~\citet{li2019nonconvex} show that SubGM with spectral initialization on $\ell_1$-loss requires $\mathcal{O}(dr^2)$ samples (modulo the condition number), provided that the true rank is known ($r'=r$), and the corruption probability is upper bounded as $p\lesssim 1/\sqrt{r'}$.~\citet{ding2021rank} extend this result to the over-parameterized regime, showing that SubGM with spectral initialization requires $\mathcal{O}(d{r'}^3)$ samples to converge, under the same assumption $p\lesssim 1/\sqrt{r'}$. In both works, the upper bound $p\lesssim 1/\sqrt{r'}$ is imposed to guarantee that the global minima of the $\ell_1$-loss correspond to the true solution. On the other hand, our result relaxes this upper bound on the corruption probability by showing that SubGM converges to the ground truth, even if the ground truth is not globally optimal. In the Gaussian noise model,~\citet{zhuo2021computational} shows that GD recovers the true solution with $\mathcal{O}(dr')$ samples. In the over-parameterized regime, our result reduces the sample complexity to $\mathcal{O}(dr^2)$, showing that the sample complexity of SubGM is independent of the search rank $r'$.

	\section*{Notations}
	For a rank-$r$ matrix $M\in \R^{m\times n}$, its singular values are denoted as $\norm{M}=\sigma_1(X)\geq \sigma_2(X)\geq \cdots \geq \sigma_r(X):=\sigma_{\min}(X)$. For a square matrix $X\in \R^{n\times n}$, its eigenvalues are defined as $\lambda_1(X)\geq \lambda_2(X)\geq \cdots \geq \lambda_n(X):=\lambda_{\min}(X)$.
	For two matrices $X$ and $Y$ of the same size, their inner product is defined as $\inner{X}{Y}=\mathrm{Tr}(X^{\top}Y)$, where $\mathrm{Tr}(\cdot)$ is the trace operator.
	For a matrix $X$, its operator and Frobenius norms are denoted as $\norm{X}$ and $\norm{X}_F$, respectively. The unit rank-$r$ sphere is defined as $\mathbb{S}_r=\{X\in \R^{d\times d}:\norm{X}_F=1, \rank(X) \leq r\}$. We define $\proj_V$ as the projection operator onto the row space of $V$. The notation $\mathbb{B}(X,\err)$ refers to a ball of radius $\err$, centered at $X$. 
	The $\ell_q$ norm of a vector $x$ is defined as $\norm{x}_{q}=(\sum |x_i|^q)^{1/q}$.  Given two sequences $f(n)$ and $g(n)$, the notation $f(n)\lesssim g(n)$ implies that there exists a constant $C<\infty$ satisfying $f(n) \leq Cg(n)$. Moreover, the notation $f(n)\asymp g(n)$ implies that $f(n)\lesssim g(n)$ and $g(n)\lesssim f(n)$. Throughout the paper, the symbols $C, c_1,c_2, \dots$ refer to universal constants whose precise value may change according to the context. The sign function $\sign(\cdot)$ is defined as $\sign(x)=x/|x|$ if $x\neq 0$, and $\sign(0)=[-1,1]$. For two sets $\mathcal{X}$ and $\mathcal{Y}$, the notation $\mathcal{X}+\mathcal{Y}$ refers to their Minkowski sum. Given two scalars $a$ and $b$, the symbols $a\wedge b$ and $a\vee b$ are used to denote their minimum and maximum, respectively.
	
	\section{Our Overarching Framework}
	In this section, we present our overarching framework for the analysis of SubGM. 
	To this goal, we first explain why the existing techniques for studying the smooth variants of the low-rank matrix recovery cannot be extended to their robust counterparts.
	
	\subsection{Failure of Existing Techniques}\label{subsec::failure}
	%\vspace{-1mm}
	The majority of existing methods study the behavior of the gradient descent on $\ell_2$-loss $f_{\ell_2}(U)=\frac{1}{m}\norm{\mathbf{y}-\cA(UU^{\top})}^2$ by analyzing its deviation from an ``ideal'', noiseless loss function $\bar{f}_{\ell_2}(U) = \|UU^\top-X^\star\|_F^2$. It is known that $\bar{f}_{\ell_2}(U) = \|UU^\top-X^\star\|_F^2$ is devoid of spurious local minima, and its saddle points are strict, and hence, escapable (see~\cite[Appendix A]{zhang2020symmetry} for a simple proof). Therefore, by controlling the deviation of ${f}_{\ell_2}(U)$ and its gradients from $\bar{f}_{\ell_2}(U)$, one can show that ${f}_{\ell_2}(U)$ inherits the desirable properties of $\bar{f}_{\ell_2}(U)$.
	More concretely, the gradient of $f_{\ell_2}(U)$ can be written as $\nabla f_{\ell_2}(U) = Q(UU^\top-X^\star)U$, where $Q(X) = ({2}/{m})\sum_{i=1}^m\left(\langle A_i, X\rangle-s_i\right) \left({A_i+A_i^\top}\right).$
	One sufficient condition for $\nabla f_{\ell_2}(U)\approx \nabla \bar f_{\ell_2}(U)$ is to ensure that $Q(M)$ remains uniformly close to $M$ for every rank-$(r+r')$ matrix $X$. In the noiseless setting, this condition can be guaranteed via $\ell_2$-RIP:\vspace{-1mm}
	\begin{definition}[$\ell_2$-RIP,~\citet{recht2010guaranteed}]
		\label{def::l2-RIP}
		The linear operator $\mathcal{A}(\cdot)$ satisfies $\ell_2$-RIP with parameters $(k,\delta)$ if, for every rank-$k$ matrix $M$, we have
		$(1-\delta)\|M\|_F^2\leq \frac{1}{m}\|\mathcal{A}(M)\|^2\leq (1+\delta)\|M\|_F^2.$
	\end{definition}
	\begin{wrapfigure}{r}{7.2cm}
		\vspace{-4mm}
		{\includegraphics[width=7cm]{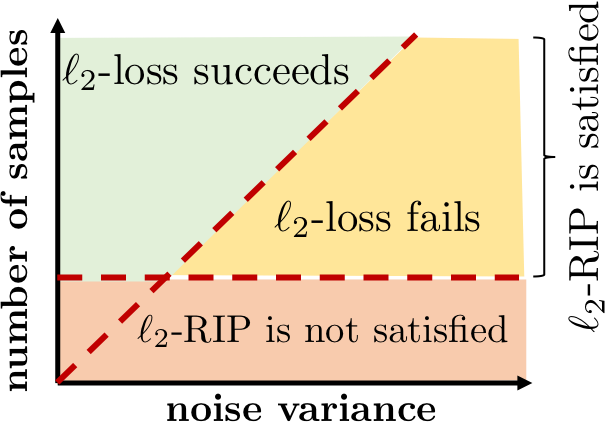}}\vspace{-2mm}
		\caption{\footnotesize The number of samples to satisfy $\ell_2$-RIP is independent of the noise variance. However, the performance of $\ell_2$ highly depends on the noise variance.}\label{fig:RIP}
		\vspace{-5mm}
	\end{wrapfigure}
	Roughly speaking, $\ell_2$-RIP entails that the linear operator $\mathcal{A}(\cdot)$ is nearly ``norm-preserving'' for every rank-$k$ matrix. In the noiseless setting, this implies that $Q(UU^\top-X^\star)\approx 4(UU^\top-X^\star)$, which in turn leads to $\nabla f_{\ell_2}(U)\approx \nabla \bar f_{\ell_2}(U)$. On the other hand, it is known that $\ell_2$-RIP is satisfied under mild conditions. For instance, $(k,\delta)$-$\ell_2$-RIP holds with $m =\Omega(dk/\delta^2)$ Gaussian measurements~\cite{recht2010guaranteed}.
	However, the next proposition shows that $\ell_2$-RIP is not enough to guarantee $Q(M)\approx 4M$ when the measurements are subject to noise.\\\\
	\begin{proposition}[\citet{ma2021implicit}]
		\label{uniform-convergence-noisy-l2}
		Suppose that $r'=d$ and the measurement matrices $\{A_i\}_{i=1}^m$ have i.i.d. standard Gaussian entries. Moreover, suppose that the noise vector $\mathbf{s}$ satisfies $s_i\stackrel{i.i.d.}{\sim} \mathcal{N}(0,\sigma^2)$ with probability $p$, and $s_i=0$ with probability $1-p$, for every $i=1,\dots,m$. Then we have
		\begin{equation}
			\mathbb{P}\left(\sup_{M\in\mathbb{S}}\left\|{Q(M)-4M}\right\|_F\gtrsim \sqrt{\frac{(1+p\sigma^2)d^2}{m}}\right)\geq \frac{1}{2}.\nonumber
		\end{equation}
	\end{proposition}

	Proposition~\ref{uniform-convergence-noisy-l2} sheds light on a fundamental shortcoming of $\ell_2$-RIP: in the presence of noise, it is possible for the measurements to satisfy $\ell_2$-RIP, yet $\nabla f_{\ell_2}(U)$ may be far from $\nabla\bar f_{\ell_2}(U)$. In particular, we show that, in order to have $\nabla f_{\ell_2}(U)\approx \nabla \bar f_{\ell_2}(U)$, the number of measurements must grow with the
	noise variance. On the other hand, for any fixed $\delta$, $\ell_2$-RIP is guaranteed to be satisfied with a number of measurements that is \textit{independent} of the noise variance. Figure~\ref{fig:RIP} shows that $\ell_2$-RIP cannot capture the behavior of $\ell_2$-loss in the high noise regime.
	Other notions of RIP, such as $\ell_1/\ell_2$-RIP~\cite{li2020nonconvex}, are also oblivious to the nature of the noise.
	
	\begin{algorithm}
		\caption{Subgradient Method}
		\label{algorithm}
		\begin{algorithmic}
			\STATE {\bfseries Input:} measurement matrices $\{A_i\}_{i=1}^m$, measurement vector $\mathbf{y}=[y_1,\cdots,y_m]^\top$, number of iterations $T$, the initial point $U_0$;
			\STATE {\bfseries Output:} Solution $\hat{X}_T=U_T U^{\top}_T$ to~\eqref{eq_l1};
			\FOR{$t\leq T$}
			\STATE Compute a sub-gradient $D_t\in \partial f_{\ell_1}(U_t)$;
			\STATE Select the step-size $\eta_t$ (see~\eqref{eq_stepsize});
			\STATE Set $U_{t+1}\leftarrow U_t - \eta_t D_t$;
			\ENDFOR
		\end{algorithmic}
	\end{algorithm}
	
	\subsection{Sign-RIP: A New Robust Restricted Isometry Property}\label{subsec:sign-RIP}
	%\vspace{-2mm}
	To address the aforementioned challenges, we argue that, while the measurements may not be norm-preserving in the presence of noise, they may still enjoy a ``direction-preserving'' property. At the heart of our analysis lies the following decomposition of the sub-differential of the $\ell_1$-loss:
	\begin{align}
		\partial f_{\ell_1}(U) &= \ \underbrace{\gamma\cdot\partial\bar f_{\ell_1}(U)}_{\text{expected sub-differential}} + \underbrace{\left({\partial f_{\ell_1}(U)-\gamma\cdot\partial\bar f_{\ell_1}(U)}\right)}_{\text{sub-differential deviation}}, \nonumber
	\end{align}
	where $\gamma$ is an strictly positive number. In the above decomposition, the function $\bar f_{\ell_1}(U)$ is called the \textit{expected loss}, and it is defined as $\norm{UU^\top-X^*}_F$.  As will be shown later, $\bar f_{\ell_1}(U)$ captures the expectation of the \textit{empirical loss} $f_{\ell_1}(U)$, when the measurement matrices have i.i.d. Gaussian entries. To analyze the behavior of SubGM on $f_{\ell_1}(U)$ (Algorithm~\ref{algorithm}), we first study the ideal scenario, where the loss deviation is zero, and hence, $f_{\ell_1}(U)$ coincides with its expectation. Under such ideal scenario, we establish the global convergence of SubGM with small initialization. We then extend our result to the general case by carefully controlling the effect of sub-differential deviation. More specifically, we show that the desirable performance of SubGM extends to the empirical loss $f_{\ell_1}(U)$, provided that the sub-differentials are ``direction-preserving'', that is, $D \approx \gamma \bar{D}$ for every $D\in\partial f_{\ell_1}(U)$ and $\bar{D}\in\partial \bar{f}_{\ell_1}(U)$, where  \begin{align}\label{eq_subdiff}
		\partial f_{\ell_1}(U) \!=\! \left\{\left(Q+Q^{\top}\right)U: Q\in \mathcal{Q}(UU^\top\!-\!X^\star)\right\},\ \text{with}\ \mathcal{Q}(X) = \frac{1}{m}\sum_{i=1}^m\sign(\langle A_i,X\rangle\!-\!s_i)A_i.
	\end{align}
	
	\begin{definition}[$\err$-approximate rank-$k$ matrix]
		We say matrix $X$ is $\err$-approximate rank-$k$ if there exists a matrix $X'$ with $\rank(X')\leq k$, such that $\norm{X-X'}_F\leq \err$.
	\end{definition}
	
	{\begin{definition}[Sign-RIP]\label{def_sign_RIP}
			The measurements are said to satisfy \textit{Sign-RIP} with parameters $(k,\delta,\err, \cS)$ and a uniformly positive and bounded scaling function $\varphi:\cS\to\mathbb{R}$  over the set $\cS$ if for every nonzero $\err$-approximate rank-$k$ $X,Y\in \cS$, and every $Q\in\mathcal{Q}(X)$, we have
			\begin{align}\label{eq_sign_RIP}
				\inner{Q-\varphi(X)\frac{X}{\|X\|_F}}{\frac{Y}{\norm{Y}_F}}\leq \varphi(X)\delta.
			\end{align}
	\end{definition}}

	According to our definition, the scaling function satisfies {$\underline{\varphi}\leq \varphi(X)\leq \bar{\varphi}, \forall X\in \cS$}, for some constants $0<\underline{\varphi}\leq \bar{\varphi}<\infty$. Without loss of generality, we assume that $\underline{\varphi}\leq 1\leq \bar{\varphi}$. Later, we will show that this assumption is satisfied for Gaussian measurements and different noise models. Whenever there is no ambiguity, we say the measurements satisfy $(k,\delta,\err, \cS)$-Sign-RIP if they satisfy Sign-RIP with parameters $(k,\delta,\err, \cS)$ and a (possibly unknown) uniformly positive and bounded scaling function $\varphi:\cS\to\mathbb{R}$.
	
	Next, we provide the intuition behind Sign-RIP. For any $U\in\mathbb{R}^{d\times r'}$, the rank of $UU^\top-X^*$ is at most $r+r'$. Now, suppose that the measurements satisfy {$(r'+r,\delta,\err,\cS)$}-{Sign-RIP} with small $\delta$ {and suitable choices of $\err,\cS$}. Then, upon defining $\gamma = \varphi(UU^\top-X^\star)\leq\bar{\varphi}$, we have $\norm{D-\gamma\bar{D}}\leq 2\bar{\varphi}\norm{U}\delta$ for every $D\in\partial f_{\ell_1}(U)$ and $\bar{D}\in\partial \bar f_{\ell_1}(U)$. In other words, for sufficiently small $\delta$, $\partial f_{\ell_1}(U)$ and $\partial \bar f_{\ell_1}(U)$ are almost aligned under $(r'+r,\delta,\err,\cS)$-Sign-RIP. A caveat of this analysis is that the required parameters of Sign-RIP depend on the search rank $r'$. One of the major contributions of this work is to relax this dependency by showing that every matrix in the sequence $\{U_tU_t^\top-X^*\}_{t=0}^T$ generated by SubGM is $\err$-approximate rank-$r$, for some small $\err>0$.
	
	At the first glance, one may speculate that Sign-RIP is extremely restrictive: roughly speaking, it requires the uniform concentration of the set-valued function $\mathcal{Q}(X)$ over $\err$-approximate rank-$k$ matrices. However, we show that, Sign-RIP is not statistically more restrictive than $\ell_2$-~\cite{recht2010guaranteed} and $\ell_1/\ell_2$-RIP~\cite{li2020nonconvex}, and---unlike its classical counterparts---holds under different noise models.
	\begin{definition}[Outlier Noise Model]\label{assumption::outlier}
		With probability $p$, each entry of the noise vector ${\bf s}$ is independently drawn from a zero mean distribution $ \mathbb{P}$; otherwise, it is set to zero. 
	\end{definition}
	
	Notice that our proposed noise model does not impose any assumption on the magnitude of the
	nonzero elements of $\mathbf{s}$, or the specific form of their distribution, which makes it particularly suitable
	for modeling outliers with arbitrary magnitudes.
	
	\begin{definition}[Gaussian Noise Model]
		\label{assumption::gaussian}
		Each element of the noise vector ${\bf s}$ is independently drawn from a Gaussian distribution with zero mean and variance $\nu_g^2<\infty$.
	\end{definition}
	
	Our next two theorems characterize the sample complexity of Sign-RIP under the outlier and Gaussian noise models.
	
	\begin{theorem}[Sign-RIP under Outlier Noise Model]\label{thm::sign-RIP-partially-corrupted}
		\begin{sloppypar}
			Assume that the measurement matrices $\{A_i\}_{i=1}^\measurementnumber$ defining the linear operator $\mathcal{A}(\cdot)$ have i.i.d. standard Gaussian entries, and that the noise vector $\mathbf{s}$ follows the outlier noise model with $0\leq p<1$ (Definition \ref{assumption::outlier}). Then, with probability of at least $1-C_1e^{-C_2\measurementnumber(1-p)^2\delta^2}$, $(k,\delta,\err,\cS)$-Sign-RIP holds with parameters $k\leq d$, $\delta\leq 1$, $\cS=\{X:\lowerbound\leq\norm{X}_F\leq R\}$ with arbitrary $R\geq\zeta>0$, $\err\lesssim\zeta\sqrt{k/m}$, and a scaling function $\varphi(X) = \sqrt{\frac{2}{\pi}}\left(1-p+p\mathbb{E}\left[e^{-s^2/(2\|X\|_F^2)}\right]\right)$, provided that the number of samples satisfies $\measurementnumber\gtrsim \frac{dk\log^2(m)\log(R/\zeta)}{(1-p)^2\delta^2}$.
		\end{sloppypar}
	\end{theorem}
	
	The proof of the above theorem is provided in Appendix~\ref{sec::proof-sign-RIP-partially-corrupted}. Theorem~\ref{thm::sign-RIP-partially-corrupted} shows that, for any fixed $R$, $\lowerbound$, $p$, and $\delta$, Sign-RIP is satisfied with $\tilde{\mathcal{O}}(dk)$ number of Gaussian measurements, which has the same order as $\ell_2$-~\cite{recht2010guaranteed} and $\ell_1/\ell_2$-RIP~\cite{li2020nonconvex} (modulo logarithmic factors). However, unlike $\ell_2$- and $\ell_1/\ell_2$-RIP, Sign-RIP is \textit{not} oblivious to noise. In particular, our theorem shows that Sign-RIP holds with a number of samples that scales with $(1-p)^{-2}$, ultimately alleviating the issue raised in Subsection~\ref{subsec::failure}. Moreover, our result does not impose any restriction on $p$, which improves upon the assumption $p<1/\sqrt{r'}$ made by \citet{li2020nonconvex} and \citet{ding2021rank}. 
	
	\begin{theorem}[Sign-RIP for Gaussian noise model]\label{thm::sign-RIP-gaussian-noise}
		\begin{sloppypar}
			Assume that the measurement matrices $\{A_i\}_{i=1}^\measurementnumber$ defining the linear operator $\mathcal{A}(\cdot)$ have i.i.d. standard Gaussian entries, and that the noise vector $\mathbf{s}$ follows the Gaussian noise model (Definition \ref{assumption::gaussian}). 
			Then, with probability of at least $1-C_1e^{-C_2m\zeta^2\delta^2/\nu_g^2}$ $(k,\delta,\err,\cS)$-Sign-RIP holds with parameters $k\leq m$, $\delta\leq 1$, $\cS=\{X:\lowerbound\leq \norm{X}_F\leq R \}$ for arbitrary $R\geq \zeta>0$, $\err\lesssim\zeta\sqrt{k/m}$, and a scaling function $\varphi(X)=\sqrt[]{\frac{2}{\pi}}\frac{\norm{X}_F}{\sqrt{\norm{X}_F^2+\nu_g^2}}$, provided that the number of samples satisfies $m\gtrsim \frac{\nu_g^2 dk\log^2(m)\log\left(R/\zeta\right)}{\lowerbound^2\delta^2}$.
		\end{sloppypar}
	\end{theorem}
	The proof of the above theorem is provided in Appendix~\ref{sec::proof-sign-RIP-gaussian-noise}. Theorem~\ref{thm::sign-RIP-gaussian-noise} extends Sign-RIP beyond outlier noise model, showing that it holds even when all measurements are corrupted with Gaussian noise. However, unlike the outlier noise model, the sample complexity of Sign-RIP scales with the noise variance.
	
	\subsection{Choice of Step-size}\label{subsec::stepsize}
	Next, we discuss our choice of the step-size, and its effect on the performance of SubGM. For simplicity, let $\Delta_t = U_tU_t^\top-X^\star$ and $\varphi_t = \varphi(U_tU_t-X^\star)$. Under Sign-RIP, we have $D_t\approx (\varphi_t/\|\Delta_t\|_F)\cdot{\Delta_t U_t}$ for every $D_t\in\partial f_{\ell_1}(U_t)$. Therefore, the iterations of SubGM can be approximated as $U_{t+1} \approx U_t - (\eta_t\varphi_t/\|\Delta_t\|_F)\cdot{\Delta_t U_t}$.
	Consequently, with the choice of $\eta_t = \eta\varphi_t^{-1}\|\Delta_t\|_F$, the iterations of SubGM reduce to
	\begin{equation}\label{eq_gd}
		U_{t+1} = U_t - \eta\cdot{\Delta_t U_t} + \mathsf{deviation}.
	\end{equation}
	Ignoring the deviation term, the above update coincides with the iterations of GD with a constant step-size $\eta$, applied to the expected loss function $\bar{f}_{\ell_2}(U) = \|UU^\top-X^\star\|_F^2$. By controlling the effect of the deviation term, we show that SubGM on $\bar{f}_{\ell_2}(U)$ behaves similar to GD with a constant step-size.  A caveat of this analysis is that the proposed step-size $\eta_t = \eta\varphi_t^{-1}\|\Delta_t\|_F$ is not known \textit{a priori}. In the noiseless scenario, Sign-RIP can be invoked to show that $\varphi_t^{-1}\|\Delta_t\|_F$ can be accurately estimated by $f_{\ell_1}(U_t)$, as shown in the following lemma.
	\begin{lemma}\label{lem_lf1}
		\label{lem::appendix-sign-RIP-imply-l1/l2}
		Suppose that the measurements are noiseless, and satisfy $(k,\delta,\err, \cS)$-Sign-RIP for some $\delta\leq 1$, $k\leq d$, $\err\geq 0$, $\cS\not=\emptyset$, and uniformly positive and bounded scaling function $\varphi(\cdot)$. Moreover, suppose that {$\Delta_t\in \cS$} is $\err$-approximate rank-$k$. Then, we have
		\begin{equation}
			(1-\delta)\varphi_t\norm{\Delta_t}_F\leq f_{\ell_1}(U_t)\leq (1+\delta)\varphi_t\norm{\Delta_t}_F.
		\end{equation}
	\end{lemma}
	The above lemma is the byproduct of a more general result presented in Appendix~\ref{app_lem_lf1}. It implies that, for small $\delta$, the step-size $\eta_t = \eta f_{\ell_1}(U_t)$ satisfies $\eta_t\approx \eta \varphi_t\norm{\Delta_t}_F$, and hence, $U_{t+1} \approx U_t - \eta\varphi_t^2{\Delta_t U_t}$, which again reduces to the iterations of GD on $\bar{f}_{\ell_2}(U)$ with the “effective” step-size $\eta\varphi_t^2$, {which is uniformly bounded since $\underline{\varphi}\leq\varphi_t\leq\bar \varphi$}.
	
	However, in the noisy setting, the value of $\varphi_t^{-1}\|\Delta_t\|_F$ \textit{cannot} be estimated merely based on $f_{\ell_1}(U_t)$, since $f_{\ell_1}(U_t)$ is highly sensitive to the magnitude of the noise. To alleviate this issue, we propose an exponentially decaying step-size that circumvents the need for an accurate estimate of $\|\Delta_t\|_F$.
	In particular, consider the following choice of step-size
	\begin{align}\label{eq_stepsize}
		\eta_t = \frac{\eta}{\norm{Q_t}}\cdot{\rho^t}, \quad\text{where}\quad Q_t\in\cQ(\Delta_t),
	\end{align}
	for appropriate values of $\eta$ and $0<\rho<1$. We note that the set $\cQ(\Delta_t)$ can be explicitly characterized without any prior knowledge on $\Delta_t$:
	\begin{align*}
		\mathcal{Q}(\Delta_t) = \frac{1}{m}\sum_{i=1}^m\sign\left(\langle A_i,\Delta_t\rangle\!-\!s_i\right)A_i=\frac{1}{m}\sum_{i=1}^m\sign\left(\langle A_i,U_tU_t^\top\rangle\!-\!y_i\right)A_i.
	\end{align*}
	Our next lemma shows that the above choice of step-size is well-defined (i.e., $Q_t\not=0$), so long as $\Delta_t$ is not too small and the measurements satisfy $(k,\delta, \err, \cS)$-Sign-RIP.
	\begin{lemma}\label{lem_stepsize}
		Suppose that the measurements satisfy $(k,\delta, \err, \cS)$-Sign-RIP with $\delta<2/(1+5\sqrt{k})$, $k\leq d$,  $\err>0$, $\cS\not=\emptyset$, and a uniformly positive and bounded scaling function $\varphi(\cdot)$. Moreover, suppose that $\Delta_t$ is $\err$-approximate rank-$k$ and $\norm{\Delta_t}\geq 4\err$. Then, we have
		\begin{equation}
			\left(1-\left(\frac{1+5\sqrt{k}}{2}\right)\delta\right) \frac{\eta\rho^t}{\varphi_t}\frac{\norm{\Delta_t}_F}{\norm{\Delta_t}}\leq \eta_t\leq \left(1+\left(\frac{1+5\sqrt{k}}{2}\right)\delta\right) \frac{\eta\rho^t}{\varphi_t}\frac{\norm{\Delta_t}_F}{\norm{\Delta_t}}.
		\end{equation}
	\end{lemma}	
	The proof of the above lemma can be found in Appendix~\ref{app_lem_stepsize}. Lemma~\ref{lem_stepsize} implies that the chosen step-size remains close to ${(\eta\rho^t/\varphi(\Delta_t))(\norm{\Delta_t}_F}/\norm{\Delta_t})$, as long as the error is not close to zero. Due to Lemma~\ref{lem_stepsize}, the iterations of SubGM with exponentially-decaying step-size can be approximated as 
	\begin{align}\label{eq_effective_noisy}
		U_{t+1} = U_t - \left(\frac{\eta\rho^t}{\norm{\Delta_t}}\right)\Delta_t U_t + \mathsf{deviation}.
	\end{align}
	In other words, SubGM selects an approximately correct direction of descent, while the exponentially decaying step-size ensures convergence to the ground truth. 
	
	\subsection{Effect of Over-parameterization}
	At every iteration of SubGM, the rank of the error matrix $\Delta_t = U_tU_t^\top-X^\star$ can be as large as $r+r'$. Therefore, in order to guarantee the direction-preserving property of the sub-differentials, a sufficient condition is to satisfy Sign-RIP for every rank-$(r+r')$ matrix. Such crude analysis implies that the performance of SubGM may depend on the search rank $r'$. In particular, with Gaussian measurements, this would increase the required number of samples to $\tilde{\mathcal{O}}\left(\frac{dr'}{(1-p)^2\delta^2}\right)$, which scales linearly with the over-parameterized rank. To address this issue, we provide a finer analysis of the iterations. Consider the eigen-decomposition of $X^\star$, given as
	$$X^\star = \begin{bmatrix} V & V_\perp\end{bmatrix}\begin{bmatrix} \Sigma & 0\\ 0 & 0\end{bmatrix}\begin{bmatrix} V^\top \\ V^\top_\perp\end{bmatrix}^\top = V \Sigma V^\top,$$
	where $V\in\mathbb{R}^{d\times r}$ and $V_\perp\in\mathbb{R}^{d\times (d-r)}$ are (column) orthonormal matrices satisfying $V^\top V_\perp = 0$, and $\Sigma\in\mathbb{R}^{r\times r}$ is a diagonal matrix collecting the nonzero eigenvalues of $X^*$. We assume that the diagonal entries of $\Sigma$ are in decreasing order, i.e., $\sigma_1\geq \sigma_2\geq \dots\geq \sigma_r>0$. Moreover, without loss of generality, we assume that $\sigma_1\geq 1\geq \sigma_r$. Based on this eigen-decomposition, we introduce the \textit{signal-residual decomposition} of $U_t$ as follows:
	\begin{align}\label{eq_decomp}
		U_t = VS_t+V_\perp \underbrace{(F_t+G_t)}_{E_t},\quad \text{where}\quad S_t = V^\top U_t\text{,}\ E_t = V_\perp^\top U_t\text{,}\ F_t = E_t \proj_{S_{t}} \text{,}\ G_t = E_t \proj_{S_{t}}^\perp.
	\end{align}
	In the above expression, $\proj_{S_{t}}$ is the orthogonal projection onto the row space of $S_t$, and $\proj_{S_{t}}^\perp$ is its orthogonal complement. It is easy to see that $U_tU_t^\top = X^\star$ if and only if $S_tS_t^\top = \Sigma$ and $E_tE_t^\top = 0$. Therefore, our goal is to show that $S_tS_t^\top$ and $E_tE_t^\top$ converge to $\Sigma$ and $0$, respectively. Based on the above signal-residual decomposition, one can write
	\begin{align}
		\Delta_t=U_tU_t^{\top}-X^{\star}
		& =\underbrace{V\left(S_tS_t^{\top}-\Sigma\right)V^{\top}+VS_tE_t^{\top}V_{\perp}^{\top}+V_{\perp}E_tS_t^{\top}V^{\top}+V_{\perp}F_tF_t^\top V_{\perp}^{\top}}_{\text{rank-$4r$}}+\underbrace{V_{\perp}G_tG_t^\top V_{\perp}^{\top}}_{\text{small norm}},\nonumber
	\end{align}
	An important implication of the above equation is that $\Delta_t$ can be treated as an $\err$-approximate rank-$4r$ matrix, where $\err = \norm{V_{\perp}G_tG_t^\top V_{\perp}^{\top}}_F$. We show that $\norm{V_{\perp}G_tG_t^\top V_{\perp}^{\top}}_F=\mathcal{O}(\sqrt{d}\alpha)$, and hence, $\Delta_t$ is approximately rank-$4r$, provided that the initialization scale is small enough. To this goal, we first characterize the generalization error $\norm{\Delta_t}$ in terms of the \textit{signal term} $\norm{S_tS_t^\top-X^\star}$, \textit{cross term} $\norm{S_tE_t^\top}$, and the \textit{residual term} $\norm{E_tE_t^\top}$.
	
	\begin{lemma}
		\label{lem::decomposition}
		We have
		\begin{align} 
			\norm{\Delta_t} \leq \norm{\Sigma - S_tS_t^\top}+2\norm{S_tE_t^\top}+\norm{E_tE_t^\top}.
		\end{align}
	\end{lemma}
	The proof of the above lemma follows directly from the signal-residual decomposition~\eqref{eq_decomp}, and omitted for brevity. Motivated by the above lemma, we will study the dynamics of the signal, cross, and residual terms under different settings.	
	
	\section{Expected Loss}
	
	In this section, we consider a special scenario, where the measurement matrices $\{A_i\}_{i=1}^m$ have i.i.d. standard Gaussian entries, and the number of measurements $m$ approaches infinity. Evidently, these assumptions do not hold in practice. Nonetheless, our analysis for this ideal scenario will be the building block for our subsequent analysis. 
	Since the number of measurements approaches infinity, the uniform law of large numbers implies that $f_{\ell_1}(U)$ converges to its expectation $\E[f_{\ell_1}(U)]$ almost surely, over any compact set of $U$~\cite{wainwright2019high}. The next lemma provides the explicit form of $\E[f_{\ell_1}(U)]$.
	
	\begin{lemma}
		Suppose that the measurements are noiseless and the measurement matrices $\{A_i\}_{i=1}^m$ have i.i.d. standard Gaussian entries. Then, we have
		\begin{align}\label{eq_expected}
			\E[f_{\ell_1}(U)] = \sqrt{\frac{2}{\pi}}\norm{UU^\top-X^*}_F.
		\end{align}
	\end{lemma}
	\begin{proof}
		Due to the i.i.d. nature of $\{A_i\}_{i=1}^m$ and the absence of noise, one can write $\E[f_{\ell_1}(U)] = \E\left[|\langle A, UU^\top-X^*\rangle|\right]$,
		where $A$ is random matrix with i.i.d. standard Gaussian entries. It is easy to see that $\langle A, UU^\top-X^*\rangle$ is a Gaussian random variable with variance $\norm{UU^\top-X^*}_F^2$. The proof is completed by noting that, for a zero-mean Gaussian random variable $X$ with variance $\sigma^2$, we have $\E[|X|] = \sqrt{\frac{2}{\pi}}\sigma$.
	\end{proof}
	
	Next, we study the performance of SubGM with small initialization for $\bar f_{\ell_1}(U) = \sqrt{\pi/2}\E[f_{\ell_1}(U)]$. First, it is easy to see that $\partial \bar f_{\ell_1}(U)=\frac{2\left(UU^{\top}-X^{\star}\right)U}{\norm{UU^{\top}-X^{\star}}_F}$ for $UU^\top\not=X^*$. Moreover, $\partial \bar f_{\ell_1}(U)$ is nonempty and bounded for every $U$ that satisfies $UU^\top=X^*$. Therefore, upon choosing the step-size as $\eta_t = {(\eta/2)}\norm{UU^{\top}-X^{\star}}_F$, the update rule for SubGM reduces to 
	$U_{t+1} = U_t - \eta_tD_t =  \eta\left(U_tU_t^{\top}-X^{\star}\right)U_t$,
	for any $D_t\in\partial \bar f_{\ell_1}(U_t)$. In other words, the iterations of SubGM with step-size $\eta_t = (\eta/2)\norm{UU^{\top}-X^{\star}}_F$ on $\bar f_{\ell_1}(U)$ are equivalent to the iterations of GD with constant step-size ${\eta}$ on the expected $\ell_2$-loss function $\bar f_{\ell_2}(U) = (1/4)\norm{UU^{\top}-X^{\star}}_F^2$.
	
	Due to this equivalence, we instead study the behavior of GD on $\bar f_{\ell_2}(U)$. Based on the decomposition of the generalization error in Lemma~\ref{lem::decomposition}, we show that the iterations of SubGM on the expected loss undergo three phases:

	\begin{figure}
		\centering
		\includegraphics[totalheight=7cm]{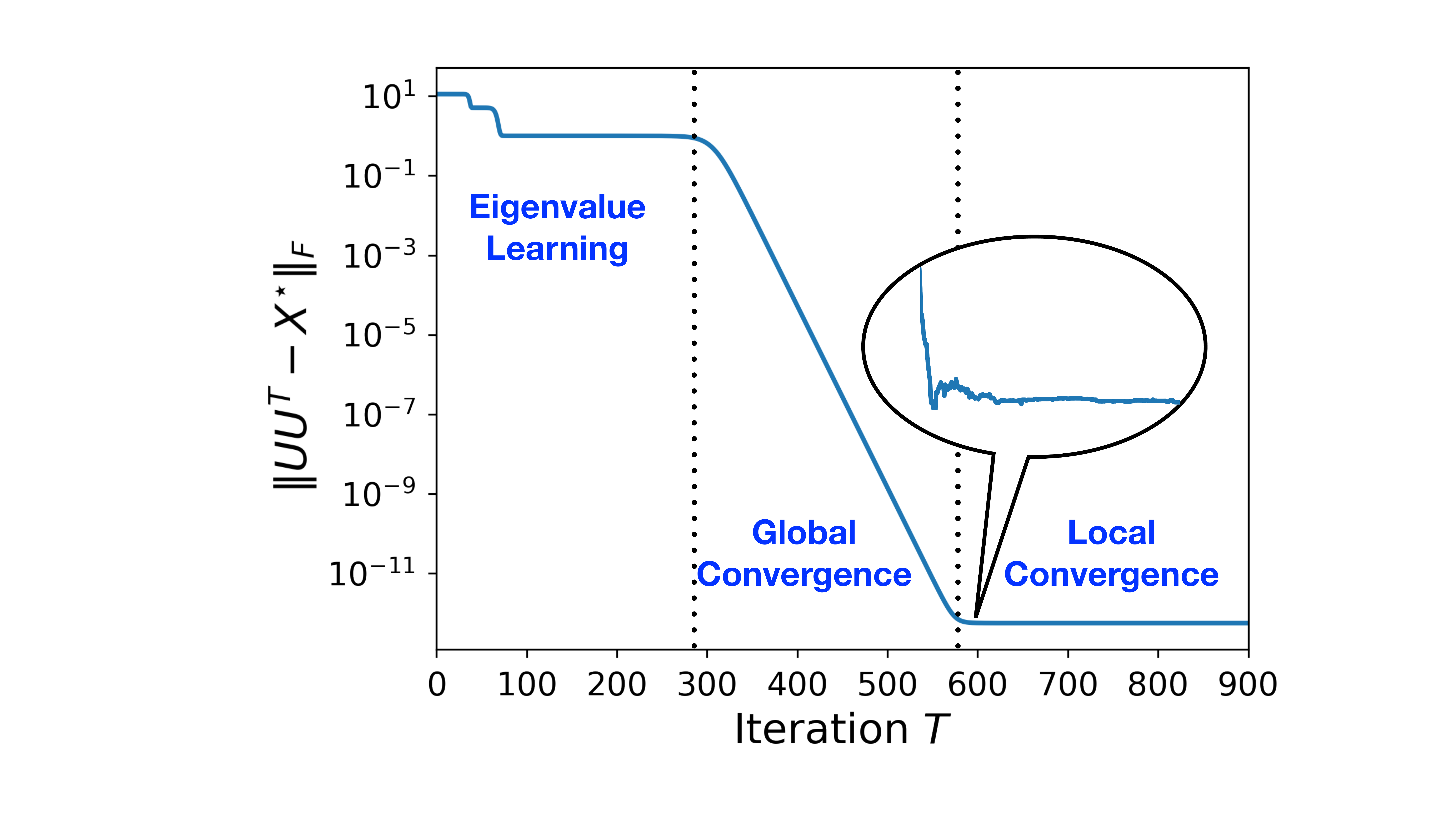}
		\caption{\footnotesize The iterations of SubGM for the expected loss~\eqref{eq_expected} undergo three phases: eigenvalue learning phase, where the eigenvalues of $S_tS_t^\top$ converge to those of $X^*$; global convergence phase, where the generalization error decays linearly; and local convergence phase, where the generalization error decays sub-linearly.}
		\label{fig::population-three-phase}
	\end{figure}

	\begin{itemize}
		\item[-] {\it Eigenvalue learning:} Due to small initialization, the signal, residual, and cross terms are small at the intial point. Therefore, the generalization error is dominated by the signal term $\norm{S_tS_t^{\top}-\Sigma}\approx \norm{\Sigma}$. We show that, in the first phase, SubGM improves the generalization error by \textit{learning the eigenvalues of $X^\star$}, i.e., by reducing $\norm{S_tS_t^{\top}-\Sigma}$.
		During this phase, the residual term $\norm{E_tE_t^\top}$ will decrease at a sublinear rate. 
		\item[-] {\it Global convergence:} Once the eigenvalues are learned to certain accuracy, both signal and cross terms $\norm{S_tS_t^{\top}-\Sigma}$ and $\norm{S_tE_t^\top}$ start to decay at a linear rate, while the residual term maintains its sublinear decay rate.
		\item[-] {\it Local convergence:} The discrepancy between the decay rates of the signal and cross terms, and that of the residual term implies that, at some point, the residual term becomes the dominant term, and hence, the generalization error starts to decay at a sublinear rate.
	\end{itemize}

	Figure~\eqref{fig::population-three-phase} illustrates the three phases of SubGM on the expected loss $\bar f_{\ell_1}(U)$ with a rank-3 ground truth $X^*$. Here, we assume that the problem is fully over-parameterized, i.e., $r' = d = 20$. A closer look at the first phase of the algorithm reveals that SubGM learns the eigenvalues of $X^*$ at different rates: the larger eigenvalues are learned faster than the small ones (Figure~\ref{fig::population-eigenvalue}). A similar observation has been made for gradient flow applied to low-rank matrix factorization~\cite{li2020towards}, and is referred to as \textit{incremental learning}~\cite{gidel2019implicit}. Finally, Figure~\ref{fig::population-decomposition} illustrates the dynamics of the signal, cross, and residual terms.	
	
	\begin{figure*}
		% \vskip 0.2in
		\begin{center}
			\subfloat[]{
				{\includegraphics[width=6.5cm]{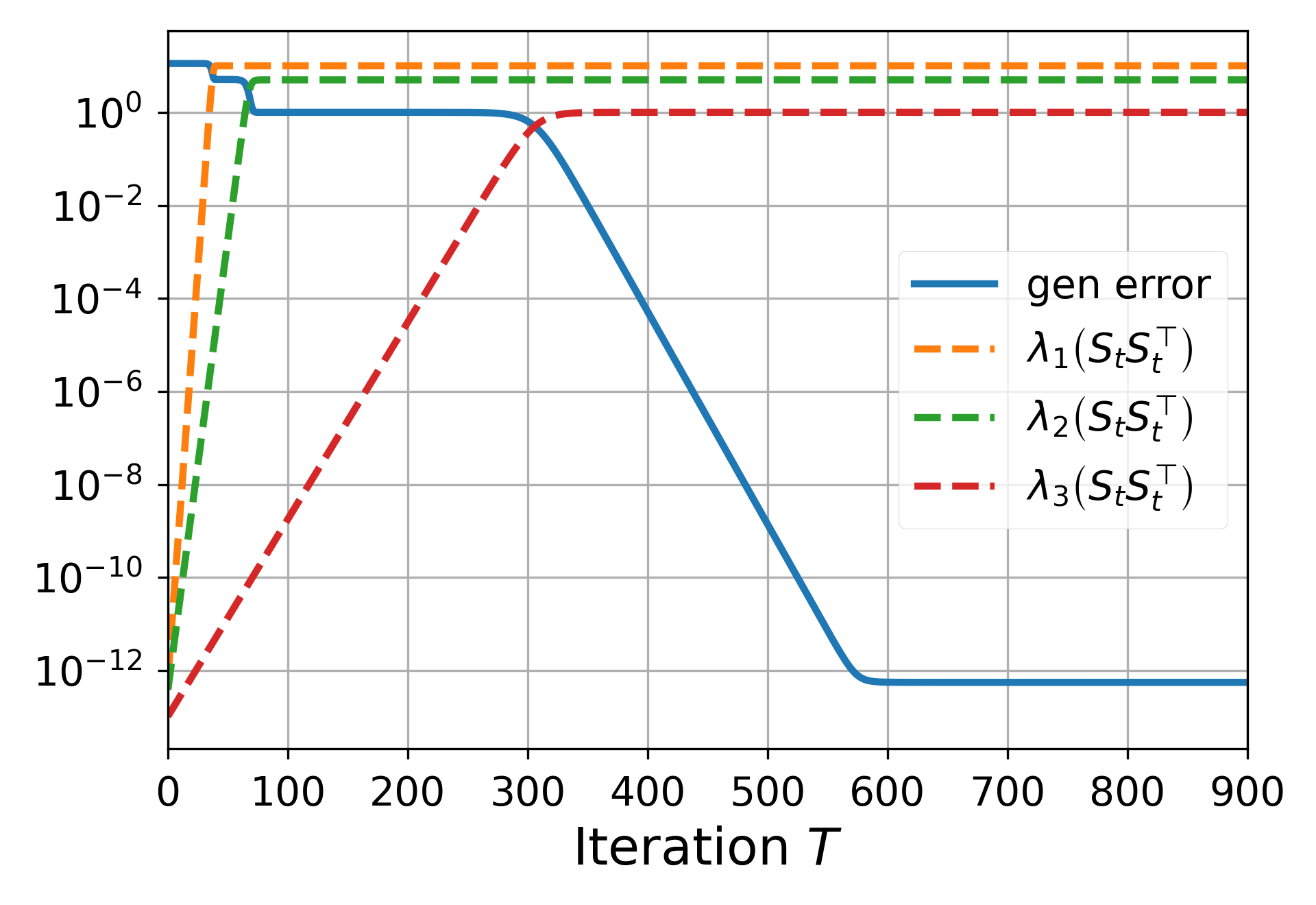}}\label{fig::population-eigenvalue}}
			\subfloat[]{
				{\includegraphics[width=6.5cm]{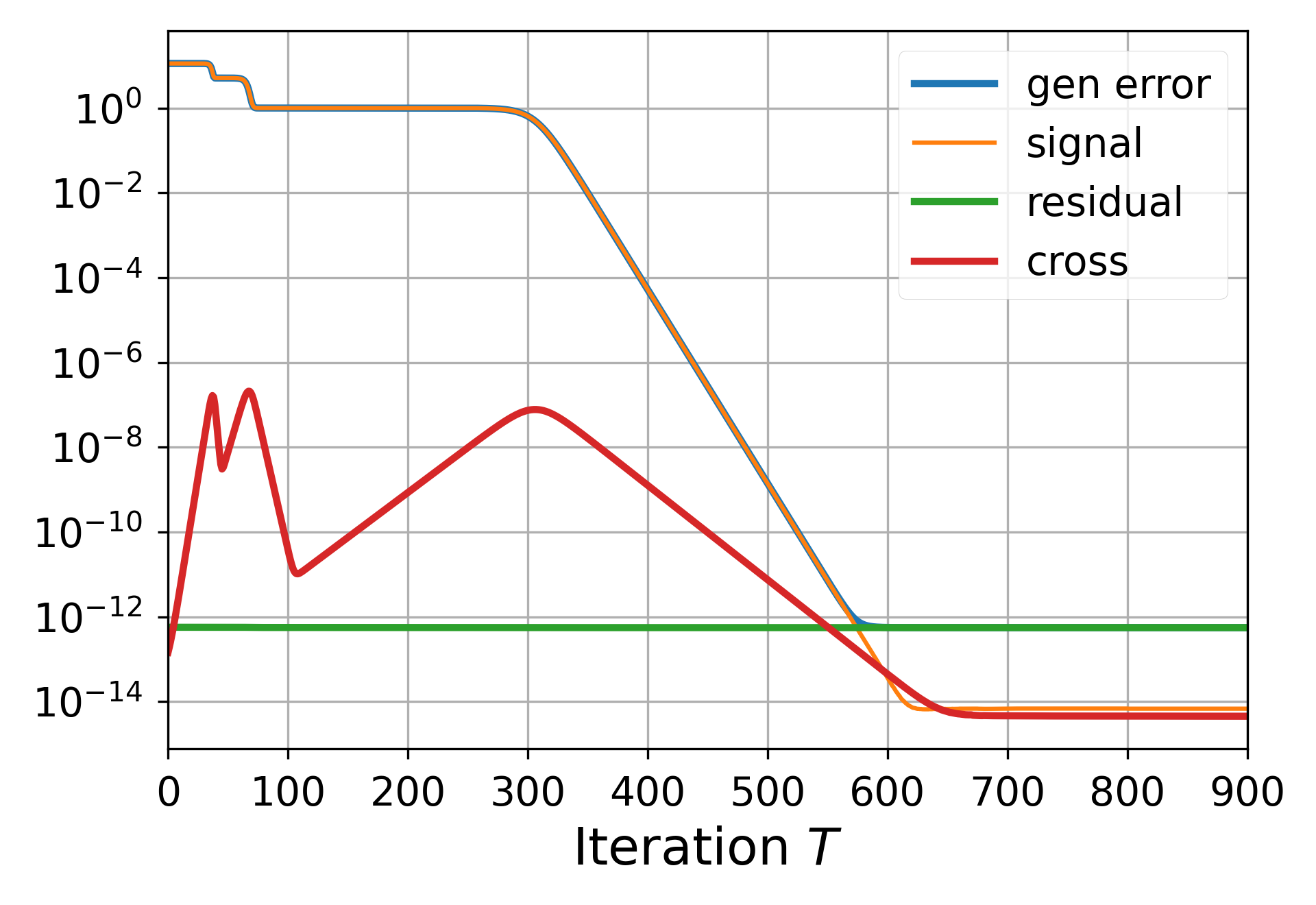}}\label{fig::population-decomposition}}
		\end{center}
		\caption{\footnotesize (a) The eigenvalues of $X^\star$ are learned at different rates. (b) During the eigenvalue learning phase, the generalization error is dominated by the signal term. In the global convergence phase, both signal and cross terms decay linearly. Finally, the residual term becomes the dominant term in the local convergence phase, and it governs the generalization error.}
	\end{figure*}
	
	\begin{proposition}[Minimum eigenvalue dynamic]\label{lem:min_eig_population}
		Consider the iterations of SubGM for the expected loss $\bar{f}_{\ell_1}(U)$, and with the step-size $\eta_t = (\eta/2)\bar f_{\ell_1}(U_t)$. Suppose that $\eta \lesssim 1/\sigma_1$, $S_tS_t^\top\succ 0$, $\norm{E_tE_t}\leq \sigma_1$, and $\norm{S_tS_t}\leq 2\sigma_1$. Then, we have
		\begin{align}
			\lambda_{\min}\left(S_{t+1}S_{t+1}^{\top}\right)\geq \left(\left(1+\eta\sigma_r\right)^2-2\eta\norm{E_tE_t^{\top}}\right)\lambda_{\min}\left(S_tS_t^{\top}\right)-2\eta\left(1+\eta\sigma_r\right)\lambda_{\min}\left(S_tS_t^{\top}\right)^2.\nonumber
		\end{align}
	\end{proposition}	
	
	The proof of Proposition~\ref{lem:min_eig_population} can be found in Appendix~\ref{app:min_eig_population}.
	The above proposition shows that the minimum eigenvalue of $S_tS_t^\top$ grows exponentially fast at a rate of $1+\Theta(\eta\sigma_r)$, provided that $\eta$ and $\norm{E_tE_t^{\top}}$ are small. This implies that the minimum eigenvalue satisfies $\lambda_{\min}\left(S_{t}S_{t}^{\top}\right) \gtrsim \sigma_r$ after $ \mathcal{O}\left({\log\left(1/\lambda_{\min}(S_0S_0^\top)\right)}/({\eta\sigma_r})\right)$ iterations.

	\begin{proposition}[Signal, cross, and residual dynamics]\label{lem:dynamics_pop}
		Consider the iterations of SubGM for the expected loss $\bar{f}_{\ell_1}(U)$, and with the step-size $\eta_t = (\eta/2)\bar f_{\ell_1}(U_t)$. Suppose that $\eta \lesssim 1/\sigma_1$, $\norm{S_tS_t^\top}\leq 1.01\sigma_1$ and $\norm{E_tE_t^\top}\leq\sigma_1$. Then, we have 
		\begin{align}
			\norm{\Sigma-S_{t+1}S_{t+1}^{\top}} &\leq \left(1-\eta\lambda_{\min}\left(S_tS_t^{\top}\right)\right)\norm{\Sigma-S_tS_t^{\top}}+5\eta \norm{S_tE_t^\top }^2,\label{eq_signal_pop}\\
			\norm{S_{t+1}E_{t+1}^{\top}} &\leq \left(1-\eta \lambda_{\min}\left(S_tS_t^{\top}\right)+2\eta \norm{\Sigma-S_tS_t^{\top}}+2\eta\norm{E_tE_t}\right)\norm{S_{t}E_{t}^{\top}},\label{eq_cross_pop}\\
			\norm{E_{t+1}E_{t+1}^{\top}} &\leq \norm{E_{t}E_{t}^{\top}} - \eta\norm{E_{t}E_{t}^{\top}}^2,\label{eq_error_pop}\\
			\norm{S_{t+1}S_{t+1}^\top} &\leq 1.01\sigma_1.\label{eq_up_pop}
		\end{align}
	\end{proposition}
	The proof of Proposition~\ref{lem:dynamics_pop} can be found in Appendix~\ref{app:dynamics_pop}. The above proposition shows that, once the minimum eigenvalue of $S_tS_t^\top$ approaches $\sigma_r$, the iterations enter the second phase, in which the signal and cross terms start to decay exponentially fast at the rate of $1-\Theta(\eta\sigma_r)$. Moreover, it shows that the residual term is independent of $\lambda_{\min}(S_tS_t^\top)$, and decreases sublinearly throughout the entire solution path. Given these dynamics, we present our main result.
	
	\begin{theorem}[Global convergence of SubGM for expected loss]\label{thm_population}
		Consider the iterations of SubGM for the expected loss $\bar{f}_{\ell_1}(U)$, and with the step-size $\eta_t = (\eta/2)\bar f_{\ell_1}(U_t)$. Suppose that $\eta \lesssim 1/\sigma_1$, and the initial point is selected such that $\norm{U_0U_0^\top-2\alpha^2 X^{\star}}\leq \alpha^2\sigma_r$, for some $\alpha\lesssim\sqrt{\sigma_r}$. Then, the following statements hold:
		\begin{itemize}
			\item[-] {\textbf{\textit{Linear convergence:}}} After $ \bar{T} \lesssim \frac{\log\left( \sigma_1/{\alpha}\right)}{\eta\sigma_r}$ iterations, we have
			\begin{align}
				\norm{U_{\bar{T}}U_{\bar{T}}^\top-X^*}\lesssim \alpha^2.\nonumber
			\end{align}
			\item[-] {\textbf{\textit{Sub-linear convergence:}}} For every $t\geq \bar{T}$, we have
			\begin{equation}
				\norm{U_tU_t^{\top}-X^{\star}}\lesssim \frac{\alpha^2}{\eta \alpha^2 t+1}.\nonumber
			\end{equation}
		\end{itemize}
	\end{theorem}
	
	The detailed proof of Theorem~\ref{thm_population} is presented in Section~\ref{proof_thm_pop}. According to the above theorem, for any accuracy $\err>0$, one can guarantee {$\norm{U_{\bar{T}}U_{\bar{T}}^\top-X^*}\leq \err$ after $\mathcal{O}\left({\log(\sigma_1/\err)}/{(\eta\sigma_r)}\right)$ iterations, provided that $\alpha \lesssim \sqrt{\err}\wedge\sqrt{\sigma_r}$}. In Section~\ref{sec:emp_loss}, we will show how to obtain an initial point that satisfies the conditions of Theorem~\ref{thm_population}. 
	
	%============================

	\section{Empirical Loss with Noiseless Measurements}\label{sec:emp_loss}
	A key difference between the behavior of SubGM for the empirical loss $f_{\ell_1}(U)$ and its expected counterpart $\bar{f}_{\ell_1}(U)$ is the fact that the residual term $\norm{E_tE_t^\top}$ no longer enjoys a monotonically decreasing behavior. In particular, Figure~\ref{fig::finite-decomposition} shows that, even with an infinitesimal initialization scale $\alpha$, the residual term grows to a non-negligible value, before decaying linearly to a small level.  
	In order to analyze this behavior, we further decompose $E_t$ as 
	$$
	E_t = F_t+G_t,\qquad \text{where}\qquad F_t=E_t\proj_{S_t}, \ \ \text{and}\ \  G_t=E_t\proj_{S_t}^{\perp}.
	$$
	Based on the above decomposition and Lemma~\ref{lem::decomposition}, the generalization error can be written as:
	\begin{equation}\label{eq_signalresidual2}
		\norm{U_tU_t^{\top}-X^{\star}}\leq\norm{S_tS_t^{\top}-\Sigma}+2\norm{S_tE_t^{\top}}+\norm{F_tF_t^{\top}}+\norm{G_tG_t^{\top}}.
	\end{equation}
	This decomposition plays a key role in characterizing the behavior of the residual term: 
	we show that the increasing nature of $\norm{E_tE_t^\top}$ in the initial stage of the algorithm can be attributed to the dynamic of $\norm{F_tF_t^\top}$. During this phase, the term $\norm{G_tG_t^\top}$ also increases, but at a much slower rate. In particular, we show that $\norm{G_tG_t^\top}$ remains in the order of $\alpha^\gamma$ for some $0<\gamma\leq 2$ throughout the entire solution path. In the second phase, $\norm{G_tG_t^\top}$ remains roughly in the same order, while $\norm{F_tF_t^\top}$ decays linearly until it is dominated by $\norm{G_tG_t^\top}$. At the end of this phase, the overall error will be in the order of $\mathcal{O}(\alpha^\gamma)$. Figure~\ref{fig::finite-error-decomposition} illustrates the behavior of $\norm{F_tF_t^\top}$ and $\norm{G_tG_t^\top}$, together with $\norm{E_tE_t^\top}$.
	
	Similar to our analysis for the expected loss, our first step towards analyzing the behavior of SubGM is to characterize the dynamic of the minimum eigenvalue of $S_tS_t^\top$. For simplicity of notation, we define $\bareta = \eta\varphi(\Delta_t)^2$ in the sequel. Recall that, due to our assumption on $\varphi(\Delta_t)^2$, we have $\underline{\varphi}^2\eta\leq \bareta\leq \bar{\varphi}^2\eta$, provided that $\Delta_t\in\cS$.
	
	\begin{figure*}
		% \vskip 0.2in
		\begin{center}
			\subfloat[]{
				{\includegraphics[width=6.5cm]{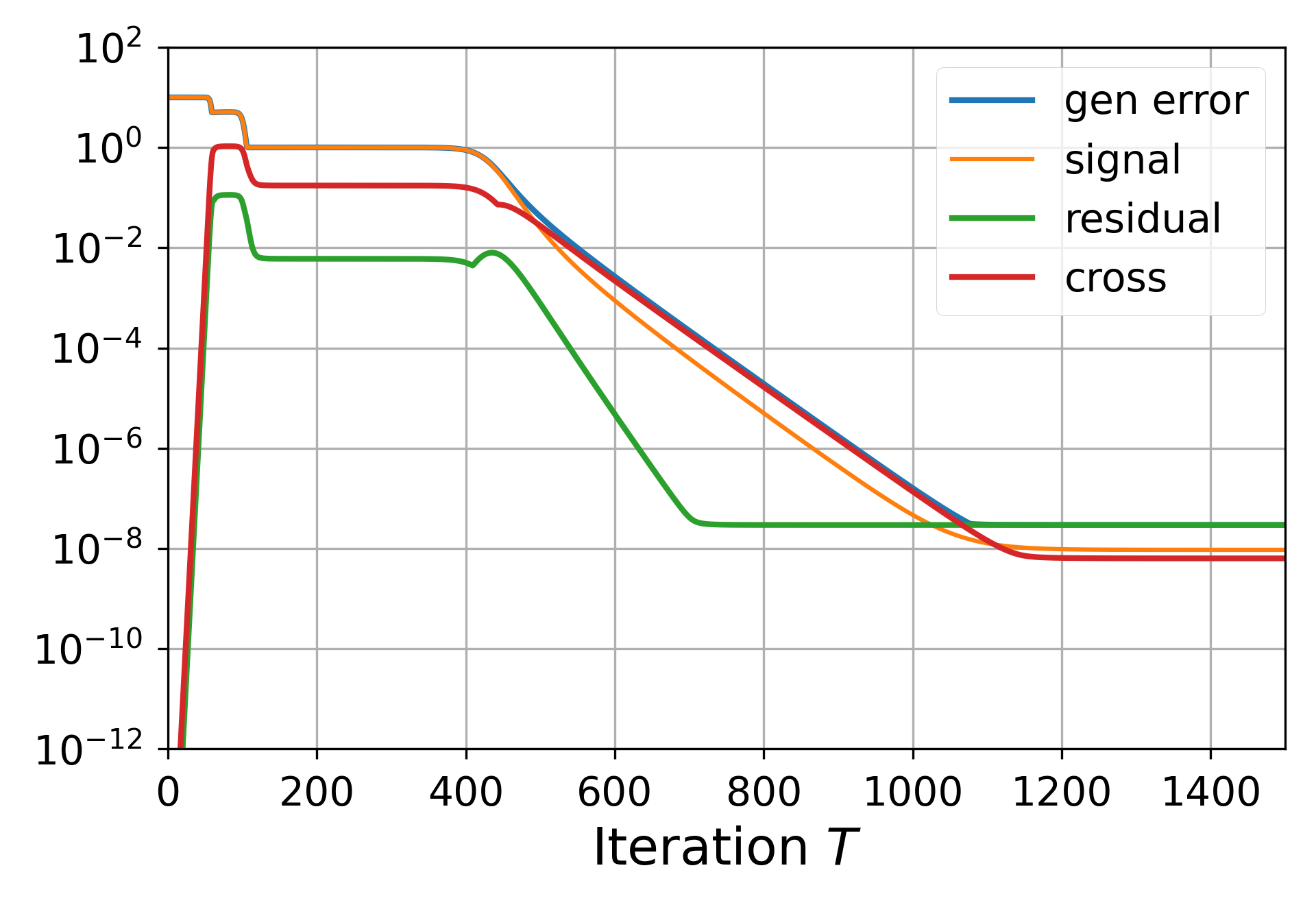}}\label{fig::finite-decomposition}}
			\subfloat[]{
				{\includegraphics[width=6.5cm]{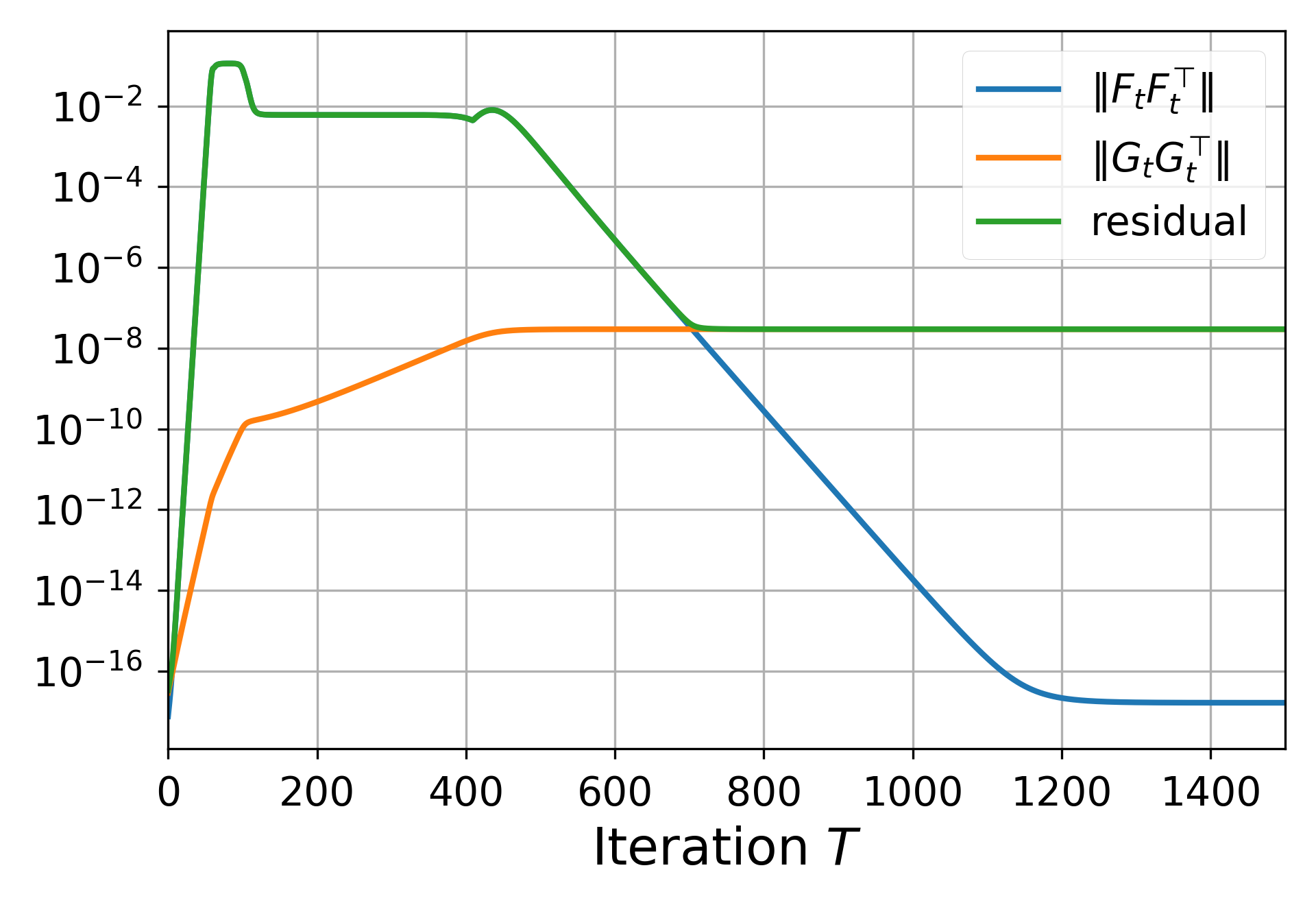}}\label{fig::finite-error-decomposition}}
		\end{center}
		\caption{\footnotesize (a) The dynamics of the signal, cross, and residual terms for the empirical loss. Unlike the expected loss, the residual term for the empirical loss has a non-monotonic behavior. (b) The dynamics of $\norm{F_tF_t^\top}$ and $\norm{G_tG_t^\top}$. The non-monotonic behavior of the residual term can be attributed to the dynamic of $\norm{F_tF_t^\top}$.
		}
	\end{figure*}
	
	\begin{proposition}[Minimum eigenvalue dynamic]\label{prop:min_eig_empirical}
		Consider the iterations of SubGM for the empirical loss ${f}_{\ell_1}(U)$ with the step-size $\eta_t = (\eta/2)f_{\ell_1}(U_t)$. Suppose that the measurements are noiseless and satisfy {$\left(4r,\delta,\err,\cS\right)$-Sign-RIP with $\delta\lesssim{1}/{\sqrt{r}}$, $\err=\sqrt{d}\norm{G_t}^2$, and $\cS=\left\{X: \lowerbound\leq\norm{X}_F\leq R\right\}$ for $\lowerbound=\err \left({1}/{\delta}\vee\sqrt{d}\right)$ and $R = 5\sqrt{r}\sigma_1$.} Moreover, suppose that $\eta \lesssim 1/(\bar{\varphi}^2\sigma_1)$, $S_tS_t^\top\succ 0$, $\norm{E_tE_t^{\top}}\leq \sigma_1$, $\norm{S_tS_t^{\top}}\leq 2\sigma_1$, $\Delta_t\in \cS$ is $\err$-approximate rank-$4r$, and $\norm{E_tS_t^{\top}(S_tS_t^{\top})^{-1}}\leq 1/3$. Then, we have
		{\begin{align}
				\lambda_{\min}\left(S_{t+1}S_{t+1}^{\top}\right)\geq& \left(\left(1+\bareta\sigma_r\right)^2-2\bareta\norm{E_tE_t^{\top}}-72\bareta\delta\norm{\Delta_t}_F\right)\lambda_{\min}\left(S_tS_t^{\top}\right)\nonumber\\
				&-2\bareta\left(1+\bareta\sigma_r\right)\lambda_{\min}\left(S_tS_t^{\top}\right)^2.\nonumber
		\end{align}}
	\end{proposition}
	The proof of Proposition~\ref{prop:min_eig_empirical} can be found in Appendix~\ref{app_prop:min_eig_empirical}. Later, we will show that the conditions of Proposition~\ref{prop:min_eig_empirical} are satisfied with a sufficiently small initial point. The above proposition shows that, in the first phase of the algorithm, $\lambda_{\min}\left(S_{t}S_{t}^{\top}\right)$ grows exponentially with a rate of least {$1+\Omega(\eta\underline\varphi^2\sigma_r)$}. Comparing this result with Proposition~\ref{lem:min_eig_population} reveals that $\lambda_{\min}\left(S_{t+1}S_{t+1}^{\top}\right)$ for the empirical loss behaves almost the same as its expected counterpart. This will play an important role in establishing the linear convergence of SubGM for the empirical loss. Finally, note that Sign-RIP must be satisfied for every $\err$-approximate rank-$4r$ matrix, where $\err = \sqrt{d}\norm{G_t}^2$. Later, we will show that, with small initialization, the value of $\sqrt{d}\norm{G_t}^2$ scales with $\alpha$, and hence, can be kept small throughout the iterations.
	Our next proposition characterizes the behavior of the signal and cross terms for the empirical loss. 
	\begin{proposition}[Signal and cross dynamics]\label{lem:dynamics_empirical}
		Consider the iterations of SubGM for the empirical loss ${f}_{\ell_1}(U)$ with the step-size $\eta_t = (\eta/2)f_{\ell_1}(U_t)$. Suppose that the measurements are noiseless and satisfy {$\left(4r,\delta,\err,\cS\right)$-Sign-RIP with $\delta\lesssim{1}/{\sqrt{r}}$, $\err=\sqrt{d}\norm{G_t}^2$, and $\cS=\left\{X:\lowerbound\leq \norm{X}_F\leq R\right\}$ for $\lowerbound=\err \left({1}/{\delta}\vee\sqrt{d}\right)$ and $R = 5\sqrt{r}\sigma_1$}. Moreover, suppose that $\eta\lesssim {1}/({\bar{\varphi}^2\sigma_1})$, $\norm{S_tS_t^\top}\leq 1.01\sigma_1$, $\norm{E_t E_t^\top}\leq\sigma_1$, $\norm{E_tE_t^\top}_F\leq \sqrt{r}\sigma_1$, and $\Delta_t\in \cS$ is $\err$-approximate rank-$4r$. Then, we have 
		{\begin{align}
				\norm{\Sigma -S_{t+1}S_{t+1}^{\top}}&\leq \left(1-\bareta\lambda_{\min}\left(S_tS_t^{\top}\right)\right)\norm{\Sigma-S_tS_t^{\top}}+5\bareta \norm{S_tE_t^\top }^2+37\bareta\delta\sigma_1\norm{\Delta_t}_F,\label{eq_signal_empirical}\\
				\norm{S_{t+1}E_{t+1}^{\top}}&\leq \left(1-\bareta \lambda_{\min}\left(S_tS_t^{\top}\right)\!+\!2\bareta\norm{\Sigma-S_tS_t^{\top}}\!+\!2\bareta\norm{E_tE_t^{\top}}\right)\norm{S_{t}E_{t}^{\top}} +22\bareta\delta\sigma_1\norm{\Delta_t}_F,\label{eq_cross_empirical}\\ \norm{S_{t+1}S_{t+1}^\top}&\leq 1.01\sigma_1.\label{eq_up_empirical}
		\end{align}}
	\end{proposition}

	\begin{algorithm}[tb]
		\caption{Initialization Scheme}
		\label{alg::spectral-initialization}
		\begin{algorithmic}
			\STATE {\bfseries Input:} initialization scale $\alpha$, measurement matrices $\{A_i\}_{i=1}^m$, measurement vector $\mathbf{y}=[y_1,\cdots,y_m]^\top$,  and the search rank $r'$;
			\STATE {\bfseries Output:} An initialization matrix $U_0\in \R^{d\times r^{\prime}}$;
			\STATE Obtain $C\in\frac{1}{m}\sum_{i=1}^{m}\sign(y_i)\frac{A_i+A_i^{\top}}{2}$;
			\STATE Compute the eigenvalue decomposition $C=\Lambda D \Lambda^{\top}$;
			\STATE Define $D_{+}^{r^\prime}$ as the top $r^{\prime}\times r^{\prime}$ sub-matrix of $D$ corresponding to the $r'$ largest eigenvalues of $C$, whose negative values are replaced by $0$;
			\STATE Set $U_0 = \alpha \Lambda \left(D_{+}^{r^\prime}\right)^{1/2}$.
		\end{algorithmic}
	\end{algorithm}
	
	The proof of this proposition is presented in Appendix~\ref{app_lem:dynamics_empirical}. Proposition~\ref{lem:dynamics_empirical} shows that, under Sign-RIP, the one-step dynamics of the signal and cross terms behave almost the same as their expected counterparts, provided that $\delta$ is sufficiently small.
	
	Finally, we provide the one-step dynamic of the residual term. To this goal, we will separately analyze $F_t$ and $G_t$, i.e., the projection of $E_t$ onto the row space of $S_t$ and its orthogonal complement. This together with ${E_t} = {F_t}+{G_t}$ characterizes the dynamic of the residual term. 
	
	\begin{proposition}[Residual dynamic]
		\label{prop::F-G}
		Consider the iterations of SubGM for the empirical loss ${f}_{\ell_1}(U)$ with the step-size $\eta_t = (\eta/2)f_{\ell_1}(U_t)$. Suppose that the measurements are noiseless and satisfy {$\left(4r,\delta,\err,\cS\right)$-Sign-RIP with $\delta\lesssim{1}/{\sqrt{r}}$, $\err=\sqrt{d}\norm{G_t}^2$, and $\cS=\left\{X:\lowerbound\leq \norm{X}_F\leq R\right\}$ for $\lowerbound=\err \left({1}/{\delta}\vee\sqrt{d}\right)$ and $R = 5\sqrt{r}\sigma_1$}. Moreover, suppose that $\norm{S_tS_t^\top}\leq 1.01\sigma_1$, $\norm{E_t E_t^\top}\leq\sigma_1$, $\Delta_t\in \cS$ is $\err$-approximate rank-$4r$, and $\norm{E_tS_t(S_tS_t)^{-1}}\leq 1/3$. Then, the following statements hold:
		\begin{itemize}
			\item If $\bareta\lesssim 1/\norm{\Delta_t}$, we have
			\begin{equation}\nonumber
				\hspace{-1cm}\norm{G_{t+1}}\leq \left(1+\bareta^2\left(2\norm{E_tS_t^{\top}}^2+\norm{E_t}^4+2 \norm{\Delta_t}\norm{E_tS_t^{\top}}\right)+7\bareta\delta \norm{\Delta_t}_F\right)\norm{G_t}.
			\end{equation}
			\item If $\eta\lesssim 1/(\bar{\varphi}^2\sigma_1)$, we have
			\begin{align}\nonumber
				\hspace{-0.8cm}\norm{F_{t+1}}\leq & \left(1-\bareta\lambda_{\min}\left(S_tS_t^{\top}\right)+3\bareta\delta\norm{\Delta_t}_F\right)\norm{F_t}+3\bareta\delta \norm{\Delta_t}_F\norm{S_t}+6\bareta\norm{\Delta_t}\norm{G_t}.
			\end{align}
		\end{itemize}
	\end{proposition}
	The proof of the above proposition can be found in Appendix~\ref{app_prop::F-G}. Note that the condition $\eta\lesssim 1/(\bar{\varphi}^2\sigma_1)$ for the dynamic of $\norm{F_t}$ readily implies $\bareta\lesssim 1/\norm{\Delta_t}$. Therefore, the one-step dynamic of $\norm{G_t}$ holds under a milder condition on the step-size. Moreover, unlike $\norm{F_t}$, the dynamic of $\norm{G_{t}}$ is independent of $\lambda_{\min}(S_tS_t^\top)$. At the early stages of the algorithm, the term $\bareta^2\left(2\norm{E_tS_t^{\top}}^2+\norm{E_t}^4+2 \norm{\Delta_t}\norm{E_tS_t^{\top}}\right) = \mathcal{O}(\alpha^2)$ is dominated by $\bareta\delta \norm{\Delta_t}_F\approx \bareta\delta\norm{X^\star}_F$. Therefore, $\norm{G_t}$ grows at a slow rate of $1+\mathcal{O}(1)\eta\delta\bar{\varphi}^2 \norm{X^{\star}}_F$. As the algorithm makes progress, $\norm{\Delta_t}_F$ decreases, leading to an even slower growth rate for $\norm{G_t}$. This is in line with the behavior of $\norm{G_t}$ in Figure~\ref{fig::finite-error-decomposition}: the growth rate of $\norm{G_t}$ decreases as SubGM makes progress towards the ground truth, and it eventually ``flattens out'' at a level proportional to the initialization scale.
	However, unlike $\norm{G_t}$, the term $\norm{F_t}$ does not have a monotonic behavior. In particular, according to Proposition~\ref{prop::F-G}, $\norm{F_t}$ may increase at the early stages of the algorithm, where $\lambda_{\min}\left(S_tS_t^{\top}\right)$ is negligible compared to $\norm{\Delta_t}_F$. However, $\norm{F_t}$ will start decaying as soon as $\lambda_{\min}\left(S_tS_t^{\top}\right)\gtrsim \delta\norm{\Delta_t}_F$, which, according to Proposition~\ref{prop:min_eig_empirical}, is guaranteed to happen after certain number of iterations. The non-monotonic behavior of $\norm{F_t}$ is also observed in practice (see Figure~\ref{fig::finite-error-decomposition}).
	
	Before presenting the main result, we provide our proposed initialization scheme in Algorithm~\ref{alg::spectral-initialization}. The presented initialization method is analogous to the classical spectral initialization in the noiseless matrix recovery problems~\cite{ma2021subgm2}, with a key difference that we scale down the norm of the initial point by a factor of $\alpha^2$. As will be shown later, the scaling of the initial point is crucial for establishing the linear convergence of SubGM; without such scaling, both GD and SubGM suffer from sublinear convergence rates, as evidenced by the recent works~\cite{zhuo2021computational,ding2021rank}.

	\begin{theorem}[Global Convergence of SubGM with Noiseless Measurements]
		\label{thm::finite-noiseless}
		Consider the iterations of SubGM for the empirical loss ${f}_{\ell_1}(U)$ with the step-size $\eta_t = (\eta/2)f_{\ell_1}(U_t)$.
		Suppose that the initial point $U_0$ is obtained from Algorithm~\ref{alg::spectral-initialization} with an initialization scale that satisfies {$\alpha\lesssim {1}/({\bar\varphi \sqrt{d}})\wedge {1}/{\kappa}$}. Suppose that the measurements are noiseless and satisfy {$\left(4r,\delta,\err,\cS\right)$-Sign-RIP with $\deltaassumptionfinite$, $\err\asymp \sqrt{d}\alpha^{2-\cO\left(\sqrt{r}\kappa^2\delta\right)}\delta$, and $\cS=\left\{X:\lowerbound\leq\norm{X}_F\leq  R\right\}$ for $\lowerbound=\err \left({1}/{\delta}\vee\sqrt{d}\right)$ and $R = 5\sqrt{r}\sigma_1$}. Finally, suppose that $\eta \lesssim 1/(\bar{\varphi}^2\sigma_1)$.
		Then, after $t = T_{end}\lesssim{\log\left( 1/{\alpha}\right)}/{(\eta\sigma_r\underline{\varphi}^2)}$ iterations, we have
		{\begin{equation}\nonumber
				\norm{U_tU_t^{\top}-X^{\star}}_F\lesssim d\alpha^{2-\cO\left(\sqrt{r}\kappa^2\delta\right)}.
		\end{equation}}
	\end{theorem}
	The proof of Theorem~\ref{thm::finite-noiseless} is presented in Subsection~\ref{app_thm::finite-noiseless}, and follows the same structure as the proof of Theorem~\ref{thm_population}. However, unlike the expected loss, the final error will be in the order of $\alpha^{\beta(\delta)}$, for some $0<\beta(\delta)\leq 2$ that is a decreasing function of $\delta$. Indeed, smaller $\delta$ will improve the dependency of the final generalization error on $\alpha$. Moreover, for an arbitrarily small $\err>0$, one can guarantee $\norm{U_TU_T^{\top}-X^{\star}}_F\leq \err$ within $T\lesssim\log(d/\err)/(\beta(\delta)\eta\sigma_r\underline{\varphi}^2)$ iterations, provided that the initialization scale satisfies $\alpha\lesssim (\err/d)^{1/\beta(\delta)}$. 
	
	Finally, we characterize the sample complexity of SubGM with noiseless, Gaussian measurements.
	
	\begin{corollary}[Gaussian Measurements]\label{cor:gaussian_noiseless}
		Suppose that the measurement matrices $\left\{A_i\right\}_{i=1}^m$ have i.i.d. standard Gaussian entries. Consider the iterations of SubGM for the empirical loss ${f}_{\ell_1}(U)$, with the step-size $\eta_t = (\eta/2)f_{\ell_1}(U_t)$ and {$\eta \lesssim 1/\sigma_1$}. Suppose that the initial point $U_0$ is obtained from Algorithm~\ref{alg::spectral-initialization} with an initialization scale that satisfies {$\alpha\lesssim 1/\sqrt{d}\wedge {1}/{\kappa}$}. Finally, suppose that the number of measurements satisfies $m\gtrsim \kappa^4dr^2\log^5(1/\alpha^2)\log^2(m)$. Then, after {$t=T_{end}\lesssim \log\left(1/{\alpha}\right)/{(\eta\sigma_r)}$} iterations, and with an overwhelming probability, we have
		{\begin{equation}\nonumber
				\norm{U_TU_T^{\top}-X^{\star}}_F\lesssim d\alpha^{2-\cO\left(\sqrt{\frac{\kappa^4 dr^2\log(1/\alpha)\log^2(m)}{m}}\right)}.
		\end{equation}}
	\end{corollary}
	The above corollary is a direct consequence of Theorem~\ref{thm::sign-RIP-partially-corrupted} after setting the corruption probability $p$ to zero. To the best of our knowledge,
	Corollary~\ref{cor:gaussian_noiseless} is the first result showing that, with Gaussian measurements, the sample complexity of SubGM is \textit{independent} of the search rank, provided that the initial point is sufficiently close to the origin.
	
	%==================================
	\section{Empirical Loss with Noisy Measurements}
	In this section, we establish the convergence of SubGM with small initialization and noisy measurements. A key difference compared to our previous analysis is in the choice of the step-size: in the presence of noise, the value of $\eta_t = (\eta/2)f_{\ell_1}(U_t)$ can be arbitrarily far from the error $\eta\varphi_t^{-1}\norm{\Delta_t}$. To circumvent this issue, we instead propose to use the following geometric step-size:
	\begin{align}\label{eq_stepsize_outlier}
		\eta_t = \eta\cdot\frac{\rho^t}{\norm{Q_t}}, \quad\text{where}\quad Q_t\in\cQ(\Delta_t).
	\end{align}
	Our first goal is to show that, under a similar Sign-RIP condition, our previous guarantees on SubGM extend to geometric step-size. Then, we show how our general result can be readily tailored to specific noise models. 
	Our next result characterizes the dynamic of $\lambda_{\min}(S_tS_t^\top)$ with the above choice of step-size.
	
	\begin{proposition}[Minimum eigenvalue dynamic]\label{prop_min_eig_outlier}
		Consider the iterations of SubGM on $f_{\ell_1}(U)$ with the step-size defined as~\eqref{eq_stepsize_outlier}. Suppose that the measurements satisfy $(4r,\err,\delta)$-Sign-RIP with $\delta\lesssim 1/\sqrt{r}$, $\err=\sqrt{d}\norm{G_t}^2$, and $\cS=\left\{X:\lowerbound\leq\norm{X}_F\leq R\right\}$ for $\lowerbound=\err\left(1/\delta\vee\sqrt{d}\right)$ and $R = 5\sqrt{r}\sigma_1$. Moreover, suppose that $S_tS_t^{\top} \succ 0$, $\norm{E_tE_t^{\top}}\leq \sigma_1$, $\norm{S_tS_t^{\top}}\leq 2\sigma_1$, $\norm{E_tS_t^{\top}\left(S_tS_t^{\top}\right)^{-1}} \leq 1/3$, $\frac{\eta\rho^t}{\norm{\Delta_t}}\lesssim \frac{1}{\sigma_1}$, and $\Delta_t\in\cS$ is $\err$-approximate rank-$4r$.
		Then, we have
		{\begin{align}
				\lambda_{\min}\left(S_{t+1}S_{t+1}^{\top}\right)\geq& \left(\left(1+\frac{\eta\rho^t}{\norm{\Delta_t}}\sigma_r\right)^2-\frac{2\eta\rho^t}{\norm{\Delta_t}}\norm{E_tE_t^{\top}}-384\sqrt{r}\eta\rho^t\delta\right)\lambda_{\min}\left(S_tS_t^{\top}\right)\nonumber\\
				&-2\frac{\eta\rho^t}{\norm{\Delta_t}}\left(1+\frac{\eta\rho^t}{\norm{\Delta_t}}\sigma_r\right)\lambda_{\min}\left(S_tS_t^{\top}\right)^2.\nonumber
		\end{align}}
	\end{proposition}
	The proof of the above proposition is presented in Appendix~\ref{app_prop_min_eig_outlier}. Recalling our discussion in Section~\ref{subsec::stepsize}, SubGM with geometric step-size moves towards a direction close to $\Delta_t/\norm{\Delta_t}$ with an ``effective'' step-size ${\eta\rho^t}/{\norm{\Delta_t}}$.
	In light of this, the above proposition is analogous to Proposition~\ref{prop:min_eig_empirical}, with an additional assumption that the effective step-size is upper bounded by $1/\sigma_1$.  Proposition~\ref{prop_min_eig_outlier} can be used to show the exponential growth of $\lambda_{\min}(S_{t+1}S_{t+1}^\top)$ in the first phase of the algorithm. To see this, note that, due to small initialization, we have $\norm{\Delta_t}\approx \norm{X^\star} = \sigma_1$, $\lambda_{\min}\left(S_tS_t^{\top}\right)^2\ll \lambda_{\min}\left(S_tS_t^{\top}\right)$, and $\norm{E_tE_t^\top}\approx 0$ at the early stages of the algorithm. This implies that the minimum eigenvalue dynamic can be accurately approximated as $\lambda_{\min}\left(S_{t+1}S_{t+1}^{\top}\right)\geq \left(1+\Omega(\eta/\kappa)\right)\lambda_{\min}\left(S_tS_t^{\top}\right)$, which grows exponentially fast.
	We next characterize the dynamics of the signal and cross terms.
	
	\begin{proposition}[Signal and cross dynamics]
		\label{prop::finite-partially-corrupted}
		Consider the iterations of SubGM on $f_{\ell_1}(U)$ with the step-size defined as~\eqref{eq_stepsize_outlier}. Suppose that the measurements satisfy $(4r,\err,\delta,\cS)$-Sign-RIP with $\delta\lesssim 1/\sqrt{r}$, $\err=\sqrt{d}\norm{G_t}^2$, and $\cS=\left\{X:\lowerbound\leq\norm{X}_F\leq R\right\}$ for $\lowerbound=\err\left(1/\delta\vee\sqrt{d}\right)$ and $R = 5\sqrt{r}\sigma_1$. Moreover, suppose that $S_tS_t^{\top} \succ 0$, $\norm{E_tE_t^{\top}}\leq \sigma_1$, $\norm{S_tS_t^{\top}}\leq 1.01\sigma_1$, $\norm{E_tS_t^{\top}\left(S_tS_t^{\top}\right)^{-1}} \leq 1/3$, $\frac{\eta\rho^t}{\norm{\Delta_t}}\lesssim \frac{1}{\sigma_1}$, and $\Delta_t\in\cS$ is $\err$-approximate rank-$4r$. Then, we have
		\begin{align}
			\norm{\Sigma \!-\!S_{t+1}S_{t+1}^{\top}} & \!\leq\! \left(1\!-\!\frac{\eta\rho^t}{\norm{\Deltatsecond}}\lambda_{\min}\left(S_tS_t^{\top}\right)\right)\norm{\Sigma-S_{t}S_{t}^{\top}} \!+\! 5\frac{\eta\rho^t}{\norm{\Deltatsecond}} \norm{S_tE_t^{\top}}^2 \!+\! 193\sqrt{r}\eta\rho^t\delta\sigma_1,
			\\
			\norm{S_{t+1}E_{t+1}^{\top}}         & \!\leq\! \left(1\!-\!\frac{\eta\rho^t}{\norm{\Deltatsecond}} \left(\lambda_{\min}\left(S_tS_t^{\top}\right)\!-\!2 \norm{\Sigma\!-\!S_tS_t^{\top}}\!-\!2\norm{E_tE_t}\right)\right)\norm{S_{t}E_{t}^{\top}}\!+\! 113\sqrt{r}\eta\rho^t\delta\sigma_1,\label{eq_cross_partial}\\
			\norm{S_{t+1}S_{t+1}^{\top}}         & \leq 1.01\sigma_1.\label{eq_ub_St}
		\end{align}
	\end{proposition}
	
	The proof of Proposition~\ref{prop::finite-partially-corrupted} is analogous to Proposition~\ref{lem:dynamics_empirical}, and can be found in Appendix~\ref{app_prop::finite-partially-corrupted}. Assuming $\norm{\Delta_t}\asymp \rho^t$, the above proposition shows that both signal and cross terms behave similar to their expected counterparts in Proposition~\ref{lem:dynamics_pop}, and their deviation diminishes exponentially fast. 
	
	\begin{proposition}
		\label{prop::F-G_outlier}
		Consider the iterations of SubGM on $f_{\ell_1}(U)$ with the step-size defined as~\eqref{eq_stepsize_outlier}. Suppose that the measurements satisfy $(4r,\err,\delta,\cS)$-Sign-RIP with $\delta\lesssim 1/\sqrt{r}$, $\err=\sqrt{d}\norm{G_t}^2$, and $\cS=\left\{X:\lowerbound\leq \norm{X}_F\leq R\right\}$ for $\lowerbound=\err\left(1/\delta\vee\sqrt{d}\right)$ and $R = 5\sqrt{r}\sigma_1$. Moreover, suppose that $S_tS_t^{\top} \succ 0$, $\norm{E_tE_t^{\top}}\leq \sigma_1$, $\norm{S_tS_t^{\top}}\leq 1.1\sigma_1$, $\norm{E_tS_t^{\top}\left(S_tS_t^{\top}\right)^{-1}} \leq 1/3$, and $\Delta_t\in\cS$ is $\err$-approximate rank-$4r$. Then, the following statements hold:
		\begin{itemize}
			\item If $\eta\lesssim \frac{1}{\norm{\Delta_t}}$, we have
			\begin{equation}\label{eq_Gt_noisy}
				\begin{aligned}
					\norm{G_{t+1}} & \!\leq\! \left(1\!+\!\frac{\eta^2\rho^{2t}}{\norm{\Delta_t}^2}\left(2\norm{E_tS_t^{\top}}^2\!+\!\norm{E_t}^4\!+\!2 \norm{\Delta_t}\norm{E_tS_t^{\top}}\right)\!+\!49\sqrt{r}\eta_0\rho^t\delta \right)\norm{G_t},
				\end{aligned}
			\end{equation}
			which can be further simplified as
			\begin{align}\label{eq_Gt_simple}
				\norm{G_{t+1}}\leq \left(1+5\eta^2\rho^{2t}+49\sqrt{r}\eta_0\rho^t\delta\right)\norm{G_t}.
			\end{align}
			\item If $\frac{\eta\rho^t}{\norm{\Delta_t}}\lesssim \frac{1}{\sigma_1}$, we have
			\begin{align}
				\norm{F_{t+1}}\leq & \left(1-\frac{\eta\rho^t}{\norm{\Delta_t}}\lambda_{\min}\left(S_tS_t^{\top}\right)+16\sqrt{r}\eta\rho^t\delta\right)\norm{F_t}+16\sqrt{r}\eta\rho^t\delta \norm{S_t}+6\eta\rho^t\norm{G_t}.
			\end{align}
		\end{itemize}
	\end{proposition}
	The proof of Proposition~\ref{prop::F-G_outlier} follows that of Proposition~\ref{prop::F-G}, and can be found in Appendix~\ref{app_prop::F-G_outlier}.
	Inequality~\eqref{eq_Gt_simple} implies that the growth rate of $\norm{G_t}$ diminishes with $t$. We will use this property to show that $\norm{G_t}$ remains proportional to the initialization scale $\alpha$ throughout the solution trajectory, which will be used to control the final generalization error.
	Moreover, unlike the dynamic of $\norm{F_t}$,~\eqref{eq_Gt_simple} holds even when $\norm{\Delta_t}$ decays faster than $\eta\sigma_1\rho^t$; this will play a key role in the proof of our next theorem.
	
	\begin{theorem}[Global Convergence of SubGM with Noisy Measurements]
		\label{thm::finite-noisy}
		Consider the iterations of SubGM on $f_{\ell_1}(U)$ with the step-size defined as~\eqref{eq_stepsize_outlier}, and parameters $\eta\lesssim{1}/\left({\kappa\log(1/\alpha)}\right)$ and $\rho = 1-\Theta\left({\eta}/({\kappa\log(1/\alpha)})\right)$.
		Suppose that the initial point $U_0$ is obtained from Algorithm~\ref{alg::spectral-initialization} with an initialization scale that satisfies {$\alpha\lesssim {1}/({\sqrt{d}})\wedge {1}/{\kappa}\wedge \underline\varphi$}. Suppose that the measurements satisfy {$\left(4r,\delta,\err,\cS\right)$-Sign-RIP with $\deltaassumptionfinite$, $\err\asymp \sqrt{d}\alpha^{2-\cO\left(\sqrt{r}\kappa\delta\right)}\delta$, and $\cS=\left\{X:\lowerbound\leq \norm{X}_F\leq R\right\}$ for $\lowerbound \geq\err\left(1/\delta\vee\sqrt{d}\right)$ and $R = 5\sqrt{r}\sigma_1$}.
		Then, after $T_{end}\lesssim{(\kappa/\eta)\log^2\left(1/{\alpha}\right)}$ iterations, we have
		{\begin{equation}\nonumber
				\norm{U_tU_t^{\top}-X^{\star}}_F\lesssim d\alpha^{2-\cO\left(\sqrt{r}\kappa\delta\right)}\vee \zeta.
		\end{equation}}
	\end{theorem}
	
	The proof of the above theorem can be found in Section~\ref{subsec:thm_partial}. Upon defining $\beta(\delta) = 2-\mathcal{O}(\sqrt{r}\kappa\delta)$, the above result implies that, for any arbitrary accuracy {$\err\geq \zeta$}, SubGM converges to a solution that satisfies $\norm{U_tU_t^{\top}-X^{\star}}_F\leq\err$ within $\mathcal{O}(\kappa\log^2(d/\err)/(\beta(\delta)\eta))$ iterations, provided that $\alpha\lesssim (\err/d)^{1/\beta(\delta)}$. Compared to the noiseless setting, the final error in Theorem~\eqref{thm::finite-noisy} has an additional term $\zeta$. This is due to the fact that we only require a lower bound on the choice of $\zeta$; as will be explained later, this additional freedom will be used to show the convergence of SubGM under the Gaussian noise model. Moreover, compared to the noiseless setting, the iteration complexity of SubGM in the noisy regime is higher by a factor of $\log(d/\err)${, and its step-size must be chosen more conservatively.} The higher iteration complexity is due to the lack of a prior estimate of $\norm{\Delta_t}_F$; to alleviate this issue, we proposed a geometric step-size, which inevitably lead to a slightly higher iteration complexity.
	
	Equipped with the above theorem and Theorems~\ref{thm::sign-RIP-partially-corrupted} and~\ref{thm::sign-RIP-gaussian-noise}, we next characterize the behavior of SubGM under both outlier and Gaussian noise regimes.
	
	\begin{corollary}[Outlier Noise Model]\label{cor:outlier}
		Suppose that the measurement matrices $\left\{A_i\right\}_{i=1}^m$ have i.i.d. standard Gaussian entries, and the noise vector $\mathbf{s}$ follows an outlier noise model with a corruption probability $0\leq p< 1$ (Definition~\ref{assumption::outlier}). Consider the iterations of SubGM on ${f}_{\ell_1}(U)$ with the step-size defined as~\eqref{eq_stepsize}, and parameters $\eta\lesssim{1}/\left({\kappa\log(1/\alpha)}\right)$ and $\rho = 1-\Theta\left({\eta}/({\kappa\log(1/\alpha)})\right)$. Suppose that the initial point $U_0$ is obtained from Algorithm~\ref{alg::spectral-initialization} with an initialization scale that satisfies {$\alpha\lesssim 1/\sqrt{d}\wedge {1}/{\kappa}\wedge (1-p)$}. Suppose that the number of measurements satisfies $m\gtrsim \kappa^4dr^2\log^5(1/\alpha^2)\log^2(m)/(1-p)^2$. Then, after {$T_{end}\lesssim (\kappa/\eta)\log^2\left(1/{\alpha}\right)$} iterations, and with an overwhelming probability, we have
		{\begin{equation}
				\norm{U_tU_t^{\top}-X^{\star}}_F\lesssim d\alpha^{2-\cO\left(\sqrt{\frac{\kappa^2 dr^2\log(1/\alpha)\log^2(m)}{(1-p)^2m}}\right)}.
		\end{equation}}
	\end{corollary}
	
	\begin{corollary}[Gaussian Noise Model]
		\label{cor::gaussian-noise}
		Suppose that the measurement matrices $\left\{A_i\right\}_{i=1}^m$ have i.i.d. standard Gaussian entries, and the noise vector $\mathbf{s}$ follows a Gaussian noise model with a variance $\nu_g<\infty$ (Definition~\ref{assumption::gaussian}). Consider the iterations of SubGM on ${f}_{\ell_1}(U)$ with the step-size defined as~\eqref{eq_stepsize}, and parameters $\eta\lesssim{1}/\left({\kappa\log(1/\alpha)}\right)$ and $\rho = 1-\Theta\left({\eta}/({\kappa\log(1/\alpha)})\right)$. Suppose that the initial point $U_0$ is obtained from Algorithm~\ref{alg::spectral-initialization} with an initialization scale that satisfies {$\alpha\lesssim 1/\sqrt{d}\wedge \sqrt{dr/m}\wedge 1/\kappa$}. Then, after $T_{end}\lesssim ({\kappa}/{\eta})\log^2\left({1}/{\alpha}\right)$ iterations, and with an overwhelming probability, we have
		\begin{equation}
			\norm{U_tU_t^{\top}-X^{\star}}_F=\cO\left(\sqrt{\frac{\nu_g^2\kappa^4dr^{2}\log^5(1/\alpha)\log^2(m)}{m}}\right).
		\end{equation}
	\end{corollary}
	The proof of Corollary~\ref{cor::gaussian-noise} follows directly from Theorems~\ref{thm::sign-RIP-gaussian-noise} and~\ref{thm::finite-noisy} after choosing $$\zeta = C\sqrt{\nu_g^2\kappa^4dr^2\log^5(1/\alpha)\log^2(m)/m},$$ for sufficiently large constant $C$. The details are omitted for brevity.
	
	\begin{remark}
		Our result can be readily extended to settings where the measurements are corrupted with both outlier and Gaussian noise values. Consider measurements of the form $y_i=\inner{A_i}{X^{\star}}+s^{(1)}_i+s^{(2)}_i$, where $s^{(1)}_i$ and $s^{(2)}_i$ follow the outlier and Gaussian noise models delineated in Definitions~\ref{assumption::outlier} and~\ref{assumption::gaussian}. In this setting, Corollaries~\ref{cor:outlier} and~\ref{cor::gaussian-noise} can be combined to show that, with $m = \tilde{\Omega}\left(\nu_g^2\kappa^4 dr^2/(1-p)^2\right)$ samples, SubGM with small initialization and geometric step-size achieves the error $\norm{U_tU_t^{\top}-X^{\star}}_F^2 = \tilde\cO\left({{\nu_g^2\kappa^4 dr^2}/({(1-p)^2m})}\right)$ (modulo logarithmic factors).
	\end{remark}

	\section{Concluding Remarks}
	In this work, we study the performance of sub-gradient method (SubGM) on a nonconvex and nonsmooth formulation of the robust matrix recovery with noisy measurements, where the rank of the true solution $r$ is unknown, and over-estimated instead with $r'\geq r$. We prove that the over-estimation of the rank has no effect on the performance of SubGM, provided that the initial point is sufficiently close to the origin. Moreover, we prove that SubGM is robust against outlier and Gaussian noise values. In particular, we show that SubGM provably converges to the ground truth, even if the globally optimal solutions of the problem are ``spurious'', i.e., they do not correspond to the ground truth. At the heart of our method lies a new notion of restricted isometry property, called Sign-RIP, which guarantees a direction-preserving property for the sub-differentials of the $\ell_1$-loss. We show that, while the classical notions of restricted isometry property face major breakdowns in the face of noise, Sign-RIP can handle a wide range of noisy measurements, and hence, is better-suited for analyzing the robust variants of low-rank matrix recovery. A few remarks are in order next:\vspace{2mm}

	\noindent{\bf Spectral vs. random initialization:} In our work, we assume that the initial point is obtained via a special form of the spectral method, followed by a norm reduction. A natural question thus arises as to whether the spectral method can be replaced by small random initialization. Based on our simulations, we observed that SubGM with small random initialization behaves almost the same as SubGM with spectral initialization. Therefore, we conjecture that small random initialization followed by a few iterations of SubGM is in fact equivalent to spectral initialization; a similar result has been recently proven by~\citet{stoger2021small} for gradient descent on $\ell_2$-loss. We consider a rigorous verification of this conjecture as an enticing challenge for future research.
	\vspace{2mm}
	
	\noindent{\bf Beyond Sign-RIP:} Another natural question pertains to the performance of SubGM on problems that \textit{do not} satisfy Sign-RIP. An important and relevant example is \textit{over-parameterized matrix completion}, where the linear measurement operator is an element-wise projector that reveals partial and potentially noisy observations of a low-rank matrix. Indeed, the performance SubGM on problems of this type requires a more refined analysis, which is left as future work.
	
	\section*{Acknowledgments}
	This research is supported by grants from the Office of Naval Research (ONR), Michigan Institute for Data Science (MIDAS), and Michigan Institute for Computational Discovery and Engineering (MICDE). The authors would like to thank Richard Y. Zhang and C\'edric Josz for fruitful discussions on earlier versions of this manuscript.
	
	\bibliography{ref.bib}
	\newpage

	%=======================================		
	
	\appendix
	\resumetocwriting
	\tableofcontents	
	
	\section{Proofs of the Main Theorems}
	\subsection{Proof of Theorem~\ref{thm_population}}\label{proof_thm_pop}
	Before delving into the details, we first present the general overview of our proof technique for Theorem~\ref{thm_population}. First, we prove that the conditions of Propositions~\ref{lem:min_eig_population} and~\ref{lem:dynamics_pop} hold for every $0\leq t<\infty$. Then, we use the minimum eigenvalue dynamic in Proposition~\ref{lem:dynamics_pop} to show that $\lambda_{\min}(S_tS_t^\top)\geq 0.98\sigma_r$ after $\mathcal{O}\left({\log(1/\alpha^2)}/{(\eta\sigma_r)}\right)$ iterations. In the second phase, we leverage the lower bound $\lambda_{\min}(S_tS_t^\top)\geq 0.98\sigma_r$ to further simplify the one-step dynamics in Proposition~\ref{lem:min_eig_population}, and show that both signal and cross term decay linearly, while the residual term remains in the order of $\alpha^2$. This phase lasts for $\mathcal{O}\left({\log(\sigma_1/\alpha^2)}/{(\eta\sigma_r)}\right)$ iterations, and the generalization error can be upper bounded by $\alpha^2$ at the end of this phase. Finally, in the third phase, we show that the residual term will dominate the signal and cross terms, and the generalization error will decay at a sublinear rate.
	
	\begin{lemma}
		\label{lem:conditions}
		The conditions of Propositions~\ref{lem:min_eig_population} and~\ref{lem:dynamics_pop} are satisfied for every $0\leq t<\infty$. In particular, for any $0\leq t<\infty$, we have
		\begin{align}
			\norm{E_{t}E_t^{\top}} &\leq \frac{\alpha^2}{\eta\alpha^2 t +1},\label{eq_error}\\
			\norm{S_tS_t^\top}&\leq 1.01\sigma_1\label{eq_up_signal},\\
			S_tS_t^\top &\succ 0.\label{eq_psd}
		\end{align}
	\end{lemma}
	
	The proof of the above lemma can be found in Appendix~\ref{app_lem:conditions}. Given Lemma~\ref{lem:conditions}, we proceed to prove Theorem~\ref{thm_population}.
	
	\noindent{\bf Phase 1: Eigenvalue Learning.} Due to Proposition~\ref{lem:min_eig_population} and Lemma~\ref{lem:conditions}, we have
	\begin{equation}\label{eq_lb_lambda}
		\begin{aligned}
			\lambda_{\min}\left(S_{t+1}S_{t+1}^{\top}\right)&\geq \left(1+2\eta\sigma_r-2\eta\norm{E_tE_t^{\top}}\right)\lambda_{\min}\left(S_tS_t^{\top}\right)-2.01\eta\lambda_{\min}^2\left(S_tS_t^{\top}\right)\\
			&\geq \left(1+1.99\eta\sigma_r\right)\lambda_{\min}\left(S_tS_t^{\top}\right)-2.01\eta\lambda_{\min}^2\left(S_tS_t^{\top}\right),
		\end{aligned}
	\end{equation}
	where we used the assumption $2\norm{E_tE_t^{\top}}\leq 2\alpha^2\leq 0.01\sigma_r$ due to our choice of $\alpha$. Now, we consider two cases:
	
	\begin{itemize}
		\item[-] Suppose that $T_1$ is the largest iteration such that  $\lambda_{\min}\left(S_tS_t^{\top}\right)\leq \sigma_r/2.01$ for every $t\leq T_1$. According to~\eqref{eq_lb_lambda}, we have
		$$\lambda_{\min}\left(S_{t}S_{t}^{\top}\right)\geq (1+0.99\eta\sigma_r)^t\lambda_{\min}\left(S_{0}S_{0}^{\top}\right)\geq (1+0.99\eta\sigma_r)^t\alpha^2\sigma_r.$$
		This implies that, after $\mathcal{O}\left({\log(1/\alpha^2)}/{(\eta\sigma_r)}\right)$ iterations, we have $\lambda_{\min}\left(S_tS_t^{\top}\right)> \sigma_r/2.01$, and hence, $T_1 = \mathcal{O}\left({\log(1/\alpha^2)}/{(\eta\sigma_r)}\right)$.
		\item[-] For $t> T_1$, let $x_t = \sigma_r - \lambda_{\min}\left(S_tS_t^{\top}\right)$. Then, according to~\eqref{eq_lb_lambda}, we have
		\begin{equation}
			\begin{aligned}
				x_{t+1}&\leq \left(1-2.03\eta\sigma_r\right)x_t+2.01\eta x_t^2+0.02\eta\sigma_r^2\\
				&\leq \left(1-1.02\eta\sigma_r\right)x_t+0.02\eta\sigma_r^2,
			\end{aligned}\label{eq_xt}
		\end{equation}
		where in the second inequality, we used the fact that $x_t=\sigma_r-\lambda_{\min}\left(S_tS_t^{\top}\right)\leq {1.01\sigma_r}/{2.01}$. The above inequality implies
		\begin{align}
			x_{t+1}-0.0196\sigma_r&\leq \left(1-1.02\eta\sigma_r\right)\left(x_{t}-0.0196\sigma_r\right)\nonumber\\
			\implies x_{t+1}-0.0196\sigma_r &\leq \left(1-1.02\eta\sigma_r\right)^{t-T_1+1}\left(x_{T_1}-0.0196\sigma_r\right).\nonumber
		\end{align}
		Hence, we have $x_t\leq 0.02\sigma_r$ after $ T_3 = T_1+T_2$ iterations, where  $T_2 = \cO\left( 1/\eta\sigma_r\right)$, which in turn shows that $\lambda_{\min}\left(S_tS_t^{\top}\right)\geq 0.98\sigma_r$.
	\end{itemize}
	The above analysis shows that $\lambda_{\min}(S_tS_t^\top)\geq 0.98\sigma_r$ for every $t\geq T_3 = T_1 + T_2 = \cO\left( {\log(1/\alpha^2)}/{(\eta\sigma_r)}\right)$. \vspace{2mm}
	
	\noindent{\bf Phase 2: Global Convergence.} We have $0.98\sigma_r\leq \lambda_{\min}(S_tS_t^\top)\leq 1.01\sigma_1$ for every $t\geq T_3$. This combined with the one-step signal dynamics~\eqref{eq_signal_pop} implies that
	\begin{equation}
		\norm{\Sigma -S_{t+1}S_{t+1}^{\top}}\leq \left(1-0.98\eta \sigma_r\right)\norm{\Sigma-S_tS_t^{\top}}+5\eta\norm{S_tE_t^\top }^2.\nonumber
	\end{equation}
	On the other hand, due to Lemma~\ref{lem:conditions}, we have
	$$
	\norm{S_tE_t^\top }^2\leq \norm{S_tS_t^{\top}}\norm{E_tE_t^{\top}}\leq (1.01\sigma_1)\alpha^2.
	$$
	This implies that
	\begin{align}
		&\norm{\Sigma -S_{t+1}S_{t+1}^{\top}}\leq \left(1-0.98\eta \sigma_r\right)\norm{\Sigma-S_tS_t^{\top}}+6
		\eta\sigma_1\alpha^2\label{eq_prelim}\\
		\implies& \norm{\Sigma -S_{t+1}S_{t+1}^{\top}}-\frac{6
			\sigma_1\alpha^2}{0.98 \sigma_r}\leq \left(1-0.98\eta\sigma_r\right)\left(\norm{\Sigma -S_{t}S_{t}^{\top}}-\frac{6
			\sigma_1\alpha^2}{0.98 \sigma_r}\right)\nonumber\\
		\implies& \norm{\Sigma -S_{t+1}S_{t+1}^{\top}}-\frac{6
			\sigma_1\alpha^2}{0.98 \sigma_r}\leq \left(1-0.98\eta\sigma_r\right)^{t-T_3+1}\left(\norm{\Sigma -S_{T_3}S_{T_3}^{\top}}-\frac{6
			\sigma_1\alpha^2}{0.98 \sigma_r}\right).\nonumber
	\end{align}
	Therefore,
	
	$$\norm{\Sigma -S_{t}S_{t}^{\top}}\leq{7\kappa
		\alpha^2}\quad \text{for}\quad t\geq T_5 = T_3+T_4,\quad \text{where}\quad T_4 = \cO\left( \frac{\log\left(\frac{\sigma_1}{\kappa\alpha^2}\right)}{\eta\sigma_r}\right).$$
	Here, we use the inequality $\norm{\Sigma -S_{T_3}S_{T_3}^{\top}}\leq \norm{\Sigma}+\norm{S_{T_3}S_{T_3}^\top}\leq 2.01\sigma_1$. On the other hand, the one-step dynamics for the cross term~\eqref{eq_cross_pop} implies that
	
	\begin{align}
		\norm{S_{t+1}E_{t+1}^{\top}}&\leq \left(1-\eta \lambda_{\min}\left(S_tS_t^{\top}\right)+2\eta \norm{\Sigma-S_tS_t^{\top}}+2\eta\norm{E_tE_t^\top}\right)\norm{S_{t}E_{t}^{\top}}\nonumber\\
		&\leq \left(1-0.5\eta\sigma_r\right)\norm{S_tE_t^\top}\nonumber\\
		\implies \norm{S_{t+1}E_{t+1}^{\top}}&\leq \left(1-0.5\eta\sigma_r\right)^{t-T_5+1} \norm{S_{T_5}E_{T_5}^\top}\nonumber\\
		&\leq \left(1-0.5\eta\sigma_r\right)^{t-T_5+1}\left(1.01\alpha\sqrt{\sigma_1}\right),\label{eq_SE_pop}
	\end{align}
	where the second inequality follows from the proven upper bound $\norm{\Sigma -S_{t}S_{t}^{\top}}\leq{7\kappa
		\alpha^2}$ and Lemma~\ref{lem:conditions}. Moreover, the last inequality is due to the fact that
	
	$$\norm{S_{T_5}E_{T_5}^{\top}}\leq \norm{S_{T_5}}\norm{E_{T_5}}\leq 1.01\alpha\sqrt{\sigma_1}.$$
	The inequality~\eqref{eq_SE_pop} results in 
	
	$$\norm{S_tE_t^\top}\leq \alpha^2\quad \text{for}\quad t\geq T_7 = T_6+T_5 \quad \text{where}\quad T_6=\cO\left( \frac{\log\left(\frac{\sqrt{\sigma_1}}{\alpha}\right)}{\eta\sigma_r}\right).$$
	This upper bound can in turn be used in~\eqref{eq_prelim} to further strengthen the upper bound on the signal term as follows
	$$\norm{\Sigma -S_{t}S_{t}^{\top}}\leq{
		\alpha^2}\qquad \text{for}\qquad t\geq T_8 = T_7+T_6 = \cO\left( \frac{\log\left(\frac{ \sigma_1}{\alpha^2}\right)}{\eta\sigma_r}\right).$$
	Finally, invoking the signal-residual decomposition in Lemma~\ref{lem::decomposition}, we have
	$$\norm{U_tU_t^\top-X^*}\leq  \norm{S_tS_t^{\top}-\Sigma}+2\norm{S_tE_t^{\top}}+\norm{E_tE_t^{\top}}\lesssim \alpha^2\ \ \text{for}\ t\geq T_8 = \cO\left( \frac{\log\left(\frac{ \sigma_1}{\alpha^2}\right)}{\eta\sigma_r}\right).$$
	
	\noindent{\bf Phase 3: Sublinear convergence.} Once both signal and cross terms are in the order of $\alpha^2$, the residual term becomes the dominant term, while both signal and cross terms maintain their linear decay rates. Therefore, we have $$\norm{U_tU_t^\top-X^*}\lesssim\norm{E_tE_t^{\top}}\lesssim \frac{\alpha^2}{\eta\alpha^2t+1}.$$
	This completes the proof.$\hfill\square$
	
	\subsection{Proof of Theorem~\ref{thm::finite-noiseless}}\label{app_thm::finite-noiseless}
	The proof of Theorem~\ref{thm::finite-noiseless} follows the same structure as the proof of Theorem~\ref{thm_population}: first, we use Proposition~\ref{prop:min_eig_empirical} to show that $\lambda_{\min}(S_tS_t^\top)$ reaches $0.98\sigma_r$ after {$T_1\lesssim {\log(1/\alpha)}/{(\eta\sigma_r\underline\varphi^2)}$} iterations. Given this inequality and equipped with the one-step dynamics of the signal, cross, and residual terms (Propositions~\ref{lem:dynamics_empirical} and~\ref{prop::F-G}), we then establish the linear convergence of SubGM to the ground truth. As a first step, we show an important property of the proposed initialization scheme.
	\begin{lemma}
		\label{prop::spec-init}
		Suppose that the measurements are noiseless and satisfy {$\left(4r,\delta, \err, \cS\right)$}-Sign-RIP where $\err=0$, and $X^{\star}\in \cS$. Then, the initial point $U_0 = VS_0+V_{\perp}E_0$ generated from Algorithm~\ref{alg::spectral-initialization} satisfies
		\begin{equation}
			\norm{U_0U_0^{\top}-\alpha^2\varphi(X^{\star})\frac{X^{\star}}{\norm{X^{\star}}_F}}\leq 2\alpha^2 \varphi(X^{\star})\delta,
		\end{equation}
	\end{lemma}
	The proof can be found in Appendix~\ref{app_prop::spec-init}. An immediate consequence of the above lemma is the following inequality:
	\begin{equation}
		\norm{S_0S_0^{\top}-\alpha^2\varphi(X^{\star})\frac{\Sigma}{\norm{X^{\star}}_F}}\vee \norm{S_0E_0^{\top}}\vee \norm{E_0E_0^{\top}}\leq 2\alpha^2 \varphi(X^{\star})\delta.
	\end{equation}
	Given this property of the proposed initialization scheme,
	we next show that the conditions of Propositions~\ref{prop:min_eig_empirical},~\ref{lem:dynamics_empirical}, and~\ref{prop::F-G} are satisfied throughout solution path.
	\begin{lemma}
		\label{lem:conditions_empirical}
		We either have $\norm{\Delta_t}_F\lesssim d\alpha^{2-\mathcal{O}(\sqrt{r}\kappa^2\delta)}$ for some $0\leq t\leq T_{end}$, or the conditions of Propositions~\ref{prop:min_eig_empirical},~\ref{lem:dynamics_empirical}, and~\ref{prop::F-G} are satisfied for every $0\leq t\leq T_{end}$. In particular, for any $0\leq t\leq T_{end}$, we have
		{\begin{align}
				\norm{F_t} &\leq \eta\bar\varphi^2\left(100\sqrt{r}\sigma_1^{1.5}\delta +30\sigma_1 \sqrt{\alpha\bar\varphi\delta}\right)(t+1),\label{eq_F_emp}\\
				\norm{G_t} &\leq \alpha^{1-\cO\left(\sqrt{r}\kappa^2\delta\right)}\sqrt{\bar{\varphi} \delta},\label{eq_G_emp}\\
				\norm{\Delta_t} &\leq 5\sigma_1,\label{eq_gen_error_emp}\\
				\norm{S_tS_t^\top}&\leq 1.01\sigma_1\label{eq_up_signal_emp},\\
				S_tS_t^\top &\succ 0,\label{eq_psd_emp}\\
				\norm{E_tS_t^\top\left(S_tS_t^\top\right)^{-1}} &\leq 1/3.\label{eq_Ht_emp}
		\end{align}}
	\end{lemma}
	
	The proof of Lemma~\ref{lem:conditions_empirical} is provided in Appendix~\ref{app_lem:conditions_empirical}. Note that the inequality $\norm{\Delta_t}_F\lesssim d\alpha^{2-\mathcal{O}(\sqrt{r}\kappa^2\delta)}$ for some $0\leq t\leq T_{end}$ readily implies the final result. On the the other hand, if $\norm{\Delta_t}_F\gtrsim d\alpha^{2-\mathcal{O}(\sqrt{r}\kappa^2\delta)}$, Lemma~\ref{lem:conditions_empirical} implies that Propositions~\ref{prop:min_eig_empirical},~\ref{lem:dynamics_empirical}, and~\ref{prop::F-G} hold for every $0\leq t\leq T_{end}$.
	
	\vspace{2mm}
	
	\noindent{\bf Phase 1: Eigenvalue Learning.} Based on Lemma~\ref{lem:conditions_empirical}, the conditions of Proposition~\ref{prop:min_eig_empirical} are satisfied for {$T_1\lesssim {\log(1/\alpha)}/{(\eta\sigma_r\underline\varphi^2)}$}, and we have
	{\begin{align}\label{eq_eig_min2}
			\lambda_{\min}\left(S_{t+1}S_{t+1}^{\top}\right)\geq& \left(\left(1+2\bareta\sigma_r\right)-2\bareta\norm{E_tE_t^{\top}}-72\bareta\delta\norm{\Delta_t}_F\right)\lambda_{\min}\left(S_tS_t^{\top}\right)\nonumber\\
			&-2\bareta\left(1+\bareta\sigma_r\right)\lambda_{\min}\left(S_tS_t^{\top}\right)^2.
	\end{align}}
	{\begin{sloppypar}
			\noindent Due to~\eqref{eq_F_emp} and~\eqref{eq_G_emp}, we have $\norm{E_tE_t^\top} \leq \norm{F_t}^2+\norm{G_t}^2\leq 0.005\sigma_r$ for every $t\lesssim {\log\left({1}/{\alpha}\right)}/{(\eta\sigma_r\underline{\varphi}^2)}$, where we used Lemma~\ref{lem:conditions_empirical} and the assumed upper bounds on $\delta$ and $\alpha$. Moreover, $72\delta\norm{\Delta_t}_F\leq 150\sigma_1\delta\leq 0.005{\sigma_r}$, where again we used the assumed upper bound on $\delta$ and Lemma~\ref{lem:conditions_empirical}. These two inequalities together with~\eqref{eq_eig_min2} lead to
		\end{sloppypar}
		\begin{align*}
			\lambda_{\min}\left(S_{t+1}S_{t+1}^{\top}\right)\geq\left(1+1.99\bareta\sigma_r\right)\lambda_{\min}\left(S_tS_t^{\top}\right)-2.01\bareta\lambda_{\min}^2\left(S_tS_t^{\top}\right).
		\end{align*}
		The above inequality is identical to~\eqref{eq_lb_lambda}, after noticing that $\eta\underline{\varphi}^2\leq \bareta\leq \eta\bar\varphi^2$. On the other hand, Lemma~\ref{prop::spec-init} shows that $\lambda_{\min}\left(S_0S_0^{\top}\right)\geq \alpha^2\varphi(X^{\star})\left(\frac{\sigma_r}{\norm{X^{\star}}_F}-2\delta\right)\geq \alpha^2\underline{\varphi}\frac{1}{2\sqrt{r}\kappa}$. Therefore, using an argument analogous to the proof of Theorem~\ref{thm_population}, we have $\lambda_{\min}(S_tS_t^\top)\geq 0.98\sigma_r$ for $t\geq T_1 = \mathcal{O} (\log(1/\alpha)/(\eta\sigma_r\underline\varphi^2))$. The details are omitted for brevity. 
	}
	\vspace{2mm}

	\noindent{\bf Phase 2: Global convergence.} Recall the signal-residual decomposition
	\begin{align}
		\norm{U_tU_t^\top-X^\star}\leq \norm{\Sigma - S_tS_t^\top}+2\norm{S_tE_t^\top}+\norm{F_tF_t^\top}+\norm{G_tG_t^\top}.\nonumber
	\end{align}
	In what follows, we show that once $\lambda_{\min}(S_tS_t^\top)\geq 0.98\sigma_r$, all terms in the above inequality decay at a linear rate, except for $\norm{G_tG_t^\top}$. To this goal, first note that Propositions~\ref{lem:dynamics_empirical} and~\ref{prop::F-G} together with $0.98\sigma_r\leq \lambda_{\min}(S_tS_t^\top)\leq 1.01\sigma_1$ lead to the following one-step dynamics:
	{\begin{align}
			\norm{\Sigma-S_{t+1}S_{t+1}^{\top}} \leq&  \left(1-0.98\bareta\sigma_r\right)\norm{\Sigma-S_tS_t^{\top}}+5\bareta \norm{S_tE_t^\top }^2+37\bareta\delta\bar{\varphi}^2\sigma_1\norm{\Delta_t}_F,\label{eq_signal3}\\
			\norm{S_{t+1}E_{t+1}^{\top}}\leq& \left(1-0.98\bareta \sigma_r+2\bareta \norm{\Sigma-S_tS_t^{\top}}+2\bareta\norm{E_tE_t}\right)\norm{S_{t}E_{t}^{\top}}+22\bareta \delta\sigma_1\norm{\Delta_t}_F,\\
			\norm{F_{t+1}}\leq & \left(1-0.98\bareta\sigma_r\right)\norm{F_t}+5\sqrt{\sigma_1}\bareta\delta \norm{\Delta_t}_F+6\bareta\norm{\Delta_t}\norm{G_t}.\label{eq_F3}
	\end{align}}
	Note that, unlike $\norm{\Sigma-S_{t}S_{t}^{\top}}$ and $\norm{F_t}$, the cross term $\norm{S_{t}E_{t}^{\top}}$ enjoys linear decay only under the condition $\norm{\Sigma-S_{t}S_{t}^{\top}}<0.98\sigma_r$, which is not necessarily satisfied in the eigenvalue learning phase. Our next lemma shows that this condition is satisfied, shortly after the eigenvalue learning phase.
	
	\begin{sloppypar}
		\begin{lemma}\label{lem_sigma_r}
			We have $\norm{\Sigma-S_{t}S_{t}^{\top}}<0.03\sigma_r$, for every $T_{end}\geq t\geq T_1+ T_2$, where $T_2 = \cO\left({\log(\kappa)}/{(\eta\sigma_r\underline{\varphi}^2}\right)$.
		\end{lemma}
	\end{sloppypar}
	\begin{proof}
		It is easy to see that
		{\begin{align}\label{eq_gamma_t}
				5\bareta\norm{S_tE_t^\top }^2\leq 5.05\sigma_1\bareta\norm{E_t}^2\leq 5.05\bareta\sigma_1\left(\norm{F_t}+\norm{G_t}\right)^2\leq 0.01\bareta\sigma_r^2,
		\end{align}}
		where the last inequality is due to our choice of $\delta$ and Lemma~\ref{lem:conditions_empirical}. Similarly, we can show that {$37\bareta\delta\sigma_1\norm{\Delta_t}_F\leq 0.01\bareta\sigma_r^2$}. These two inequalities combined with~\eqref{eq_signal3} lead to
		{\begin{align*}
				\norm{\Sigma-S_{t+1}S_{t+1}^{\top}} \leq&  \left(1-0.98\bareta\sigma_r\right)\norm{\Sigma-S_tS_t^{\top}}+0.02\bareta\sigma_r^2.
		\end{align*}}
		This implies that $\norm{\Sigma-S_{t}S_{t}^{\top}}\leq 0.03\sigma_r$ after $\cO\left({\log(\kappa)}/{(\eta\sigma_r\underline{\varphi}^2)}\right)$ iterations.
	\end{proof}
	The above lemma shows that the one-step dynamic of the cross term can be simplified as
	{\begin{align}
			\norm{S_{t+1}E_{t+1}^{\top}} \leq& \left(1-0.49\bareta\sigma_r\right)\norm{S_{t}E_{t}^{\top}}+22\bareta\delta^2\sigma_1\norm{\Delta_t}_F.\label{eq_cross3}
	\end{align}}
	Moreover, recall that 
	{\begin{align}
			\norm{G_{t+1}}&\leq \left(1+\bareta^2\left(2\norm{E_tS_t^{\top}}^2+\norm{E_t}^4+2 \norm{\Delta_t}\norm{E_tS_t^{\top}}\right)+7\bareta\delta \norm{\Delta_t}_F\right)\norm{G_t}\nonumber\\
			&\leq \left(1+5\bareta^2\norm{\Delta_t}^2+7\bareta\delta\norm{\Delta_t}_F\right)\norm{G_t}.\label{eq_G3}
	\end{align}}
	{Here we use the fact that $\norm{E_tS_t^{\top}}\vee \norm{E_t}^2\leq \norm{\Delta_t}$.}
	Now, let us define $\gamma_t:=\norm{\Sigma -S_{t}S_{t}^{\top}}+2\norm{S_tE_t^{\top}}+ \norm{F_t}^2+\norm{G_t}^2$ as an upper bound for the generalization error $\norm{U_tU_t^\top-X^\star}$. Combining~\eqref{eq_signal3},~\eqref{eq_F3},~\eqref{eq_cross3}, and~\eqref{eq_G3}, we have
	{\begin{align*}
			\gamma_{t+1}& =\norm{\Sigma -S_{t+1}S_{t+1}^{\top}}+2\norm{S_{t+1}E_{t+1}^{\top}}+ \norm{F_{t+1}}^2+\norm{G_{t+1}}^2\\
			& \leq \left(1-C_1\bareta\sigma_r\right)\gamma_t+C_2\bareta\delta\sigma_1\norm{\Delta_t}_F+C_3\bareta\sigma_r\norm{G_t}^2\\
			&\leq\left(1-C_1\bareta\sigma_r+C_2\bareta\sqrt{r}\sigma_1\delta\right)\gamma_t+C_4\bareta\sigma_r\sqrt{d}\norm{G_t}^2\\
			&\leq\left(1-C_5\bareta\sigma_r\right)\gamma_t+C_4\bareta\sigma_r\sqrt{d}\norm{G_t}^2,
	\end{align*}}
	for some universal constants $C_1,C_2,C_3,C_4, C_5>0$. This implies that
	{\begin{align*}
			&\gamma_{t+1}-C_6\sqrt{d}\norm{G_t}^2\leq \left(1-C_5\bareta\sigma_r\right)\left(\gamma_t-C_6\sqrt{d}\norm{G_t}^2\right)\\
			\implies &\gamma_{t+1}-C_6\sqrt{d}\norm{G_t}^2\leq \left(1-C_5\bareta\sigma_r\right)^{t-(T_1+T_2)}(\gamma_{T_1+T_2}-C_6\sqrt{d}\norm{G_t}^2).
	\end{align*}}
	\begin{sloppypar}
		Note that $\gamma_{T_1+T_2} \leq 0.1\sigma_r$, according to Lemma~\ref{lem_sigma_r}, inequality~\eqref{eq_gamma_t}, and the upper bounds on $\norm{F_t}^2$ and $\norm{G_t}^2$. On the other hand, $\norm{G_t}\leq \alpha^{1-\cO\left(\sqrt{r}\kappa^2\delta\right)}\sqrt{\bar{\varphi} \delta}$, according to Lemma~\ref{lem:conditions_empirical}. Therefore, after $T_1+T_2+T_3$ iterations with $T_3=\cO\left({\log(1/\alpha)}/{(\eta\sigma_r\underline\varphi^2)}\right)$, we have $\norm{\Delta_t}_F\leq \sqrt{d}\gamma_t\lesssim {d}\alpha^{2-\cO\left(\sqrt{r}\kappa^2\delta\right)}{\bar{\varphi} \delta}$. This completes the proof.$\hfill\square$
	\end{sloppypar}

	%=================================================
	\subsection{Proof of Theorem~\ref{thm::finite-noisy}}\label{subsec:thm_partial}
	Without loss of generality, we assume that $\norm{\Delta_t}_F\geq \zeta$ for every $0\leq t\leq T_{end}$; otherwise, the final bound for the generalization error holds and the proof is complete. Before delving into the details, we first provide a general overview of our approach. The proof of Theorem~\ref{thm::finite-noisy} is similar to that of Theorem~\ref{thm::finite-noiseless} with a key difference that we divide our analysis into two parts depending on the value of $\norm{\Delta_t}$: if $\eta\rho^t/\norm{\Delta_t}\lesssim 1/\sigma_1$, then $\norm{\Delta_t}$ decays slower than $\sigma_1\eta\rho^t$. Under this assumption, we will use the one-step dynamics of signal, cross, and residual terms in Propositions~\ref{prop_min_eig_outlier},~\ref{prop::finite-partially-corrupted}, and~\ref{prop::F-G_outlier} to prove that $\norm{\Delta_t}$ decays exponentially fast. Alternatively, $\eta\rho^t/\norm{\Delta_t}\gtrsim 1/\sigma_1$ implies that $\norm{\Delta_t}\lesssim \sigma_1\eta\rho^t$, which readily establishes the exponential decay of the generalization error. Indeed, the above two cases may occur alternatively, which requires a more delicate analysis.
	
	Similar to the proof of Theorem~\ref{thm::finite-noiseless}, we show that SubGM undergoes two phases: (1) eigenvalue learning phase, and (2) global convergence phase. In the first phase, we show that $\lambda_{\min}(S_tS_t^\top)$ and $\norm{\Delta_t}$ converge to $0.98\sigma_r$ and $0.02\sigma_r$, respectively. The main difference between the proofs of Theorems~\ref{thm::finite-noiseless} and~\ref{thm::finite-noisy} is the fact that the one-step dynamics of signal, cross, and residual terms may not hold during the entire solution trajectory. However, our next lemma shows that they indeed hold in the first phase of our analysis.
	
	\begin{lemma}
		\label{lem:conditions_noisy}
		Suppose that $\norm{\Delta_t}\geq 0.02\sigma_r$ for every $0\leq t\leq \bar{T}\leq T_{end}$. Then, the assumptions of Propositions~\ref{prop_min_eig_outlier},~\ref{prop::finite-partially-corrupted}, and~\ref{prop::F-G_outlier} are satisfied for every $0\leq t\leq  \bar{T}$. In particular, for any $0\leq t\leq \bar{T}$, we have
		{\begin{align}
				\norm{F_t} &\leq \left(15\sqrt{r}\eta\rho^t\sqrt{\sigma_1}\delta+12\sqrt{2}\eta\sqrt{\bar{\varphi}\alpha\delta}\right)(t+1),\label{eq_F_emp_noisy}\\
				\norm{G_t} &\leq 2\sqrt{2}\alpha^{1-\cO\left(\sqrt{r}\kappa\delta\right)}\sqrt{\bar{\varphi} \delta},\label{eq_G_emp_noisy}\\
				\norm{\Delta_t} &\leq 5\sigma_1,\label{eq_gen_error_emp_noisy}\\
				\norm{S_tS_t^\top}&\leq 1.1\sigma_1\label{eq_up_signal_emp_noisy},\\
				S_tS_t^\top &\succ 0,\label{eq_psd_emp_noisy}\\
				\norm{E_tS_t^\top\left(S_tS_t^\top\right)^{-1}} &\leq 1/3.\label{eq_Ht_emp_noisy}
		\end{align}}
	\end{lemma}
	The proof of Lemma~\ref{lem:conditions_noisy} is analogous to that of Lemma~\ref{lem:conditions_empirical}, and can be found in Appendix~\ref{app_lem:conditions_noisy}. Equipped with this lemma, we next provide the proof for the eigenvalue learning phase.
	
	\paragraph{Phase 1: Eigenvalue Learning.}
	We will show that $\lambda_{\min}(S_tS_t^\top)\geq 0.98\sigma_r$ within $t\leq T_1=\mathcal{O}((\kappa/\eta)\log(1/\alpha))$ iterations. After that, we prove that we only need $T_2 = \mathcal{O}((\kappa/\eta)\log(\kappa))$ additional iterations to ensure that $\norm{\Delta_t}\leq 0.02\sigma_r$; this marks the end of Phase 1. Without loss of generality, we may assume that $\norm{\Delta_t}\geq 0.02\sigma_r$ for every $t\leq T_1$. To see this, suppose that $\norm{\Delta_t}\leq 0.02\sigma_r$ for some $t\leq T_1$. This implies that $\norm{\Sigma-S_tS_t^\top}\leq\norm{\Delta_t}\leq 0.02\sigma_r$, which in turn leads to $\lambda_{\min}(S_tS_t^\top)\geq 0.98\sigma_r$. On the other hand, $\norm{\Delta_t}\geq 0.02\sigma_r$ together with Lemma~\ref{lem:conditions_noisy} implies that the one-step dynamic in Proposition~\ref{prop_min_eig_outlier} holds and we have
	\begin{align}
		\lambda_{\min}\left(S_{t+1}S_{t+1}^{\top}\right)\geq& \left(\left(1+\frac{\eta\rho^t}{\norm{\Delta_t}}\sigma_r\right)^2-\frac{2\eta\rho^t}{\norm{\Delta_t}}\norm{E_tE_t^{\top}}-384\sqrt{r}\eta\rho^t\delta\right)\lambda_{\min}\left(S_tS_t^{\top}\right)\nonumber\\
		&-2\frac{\eta\rho^t}{\norm{\Delta_t}}\left(1+\frac{\eta\rho^t}{\norm{\Delta_t}}\sigma_r\right)\lambda_{\min}\left(S_tS_t^{\top}\right)^2.\nonumber\\
		\geq& \left(1+\frac{2\eta\rho^t}{\norm{\Delta_t}}\left(\sigma_r-\mathcal{O}(\sqrt{r}\sigma_1\delta)\right)\right)\lambda_{\min}\left(S_tS_t^{\top}\right)-\frac{2.01\eta\rho^t}{\norm{\Delta_t}}\lambda_{\min}\left(S_tS_t^{\top}\right)^2\nonumber\\
		\geq& \left(1+\frac{1.99\eta\rho^t\sigma_r}{\norm{\Delta_t}}\right)\lambda_{\min}\left(S_tS_t^{\top}\right)-\frac{2.01\eta\rho^t}{\norm{\Delta_t}}\lambda_{\min}\left(S_tS_t^{\top}\right)^2\label{eq_eig_min_noisy2}
	\end{align}
	where the second inequality follows from $\norm{\Delta_t}\geq 0.02\sigma_r$ and $\eta\lesssim 1$, and the last inequality follows from $\delta\lesssim 1/(\sqrt{r}\sigma_1)$. The rest of the proof for the eigenvalue learning phase is similar to the arguments made after~\eqref{eq_lb_lambda} in the proof of Theorem~\ref{thm_population}. Suppose that $T'_1$ is the largest iteration such that $\lambda_{\min}\left(S_tS_t^{\top}\right)\leq \sigma_r/2.01$ for every $t\leq T_1'$. We show that $T_1' = \mathcal{O}((\kappa/\eta)\log(1/\alpha))$. To see this, note that for every $t\leq T_1'$, the above inequality can be simplified as
	\begin{align}
		\lambda_{\min}\left(S_tS_t^{\top}\right)&\geq \left(1+\frac{0.99\eta\rho^t\sigma_r}{\norm{\Delta_t}}\right)\lambda_{\min}\left(S_tS_t^{\top}\right)\nonumber\\
		& \geq \lambda_{\min}\left(S_0S_0^{\top}\right)\prod_{s=0}^t\left(1+\frac{\eta\rho^s}{6\kappa}\right),\label{eq_eig_lb}
	\end{align}
	where we used the assumption $\delta\lesssim 1/(\sqrt{r}\kappa)$ and $\norm{\Delta_t}\leq 5\sigma_1$. To proceed with the proof, we need the following technical lemma.
	
	\begin{lemma}
		\label{lem::appendix-prod-exp}
		For any $\alpha>0, 0<\rho<1$ and $T\in \bN_{+}$, we have
		\begin{equation}
			\exp\left(\frac{\alpha T\rho^T}{1+\alpha}\right)\leq \prod_{t=0}^{T}\left(1+\alpha\rho^t\right)\leq \exp\left(\frac{\alpha}{1-\rho}\right).
		\end{equation}
	\end{lemma}
	The proof of Lemma~\ref{lem::appendix-prod-exp} can be found in Appendix~\ref{app_lem::appendix-prod-exp}. Lemma~\ref{lem::appendix-prod-exp} together with Lemma~\ref{prop::spec-init} and~\eqref{eq_eig_lb} leads to
	\begin{align*}
		\lambda_{\min}\left(S_tS_t^{\top}\right)&\geq \exp\left(\frac{\eta t\rho^t}{6\kappa+\eta}\right)\alpha^2\varphi(X^{\star})\left(\frac{\sigma_r}{\norm{X^{\star}}_F}-\delta\right)\\
		&\geq \exp\left(\frac{\eta t\rho^t}{6\kappa+\eta}\right)\alpha^2\varphi(X^{\star})\left(\frac{\sigma_r}{\sqrt{r}\sigma_1}-\delta\right)\\
		&\geq \exp\left(\frac{\eta t\rho^t}{6\kappa+\eta}\right)\frac{\alpha^2\varphi(X^{\star})}{2\sqrt{r}\kappa}.
	\end{align*}
	Due to our assumption $\rho = 1-\Theta(\eta/(\kappa\log(1/\alpha)))$, we have $\rho^t\geq 1-\mathcal{O}(\eta t/(\kappa\log(1/\alpha)))$. This together with the assumption $t\leq \mathcal{O}((\kappa/\eta)\log(1/\alpha))$ leads to $\rho^t\geq 1-\mathcal{O}(1)\geq \Omega(1)$. Therefore, one can write
	\begin{align*}
		\lambda_{\min}\left(S_tS_t^{\top}\right)\geq \exp\left(\Omega(1)\frac{\eta t}{\kappa}\right)\frac{\alpha^2\varphi(X^{\star})}{2\sqrt{r}\kappa}.
	\end{align*}
	Given the above equation, it is easy to verify that after $\mathcal{O}((\kappa/\eta)\log(1/\alpha))$ iterations, we have $\lambda_{\min}\left(S_tS_t^{\top}\right)\geq \sigma_r/2.01$. This implies that $T_1'=\mathcal{O}((\kappa/\eta)\log(1/\alpha))$. 
	
	For $t>T_1'$, define $x_t = \sigma_r - \lambda_{\min}\left(S_tS_t^{\top}\right)$. Then, arguments analogous to the proof of Theorem~\ref{thm_population} can be used to write
	\begin{align*}
		x_{t+1}-0.0196\sigma_r&\leq \left(1-1.02\frac{\eta\rho^t\sigma_r}{\norm{\Delta_t}}\right)(x_{t}-0.0196\sigma_r)\\
		&\leq \left(1-\frac{\eta\rho^t}{5\kappa}\right)(x_{t}-0.0196\sigma_r)\\
		&\leq (x_{T_1'}-0.0196\sigma_r)\prod_{s=T_1'}^t\left(1-\frac{\eta\rho^t}{5\kappa}\right).
	\end{align*}
	Recall that $\rho^t = \Omega(1)$ for every $t\leq \mathcal{O}((\kappa/\eta)\log(1/\alpha))$. Therefore, for every $T_1'<t\leq \mathcal{O}((\kappa/\eta)\log(1/\alpha))$, we have
	\begin{align*}
		x_{t+1}-0.0196\sigma_r&\leq  (x_{T_1'}-0.0196\sigma_r)\left(1-\Omega(1)\frac{\eta\sigma_r}{\kappa}\right)^{t-T_1'+1}.
	\end{align*}
	This in turn implies that after additional $T_1'' = \mathcal{O}(\kappa/\eta)$ iterations, we have $x_{t+1}\leq 0.02\sigma_r$. Therefore, we have $\lambda_{\min}\left(S_tS_t^{\top}\right)\geq 0.98\sigma_r$ after $T_1 = T_1'+T_1'' = \mathcal{O}((\kappa/\eta)\log(1/\alpha))$ iterations. Now, it suffices to show that, after additional $T_2 = \mathcal{O}((\kappa/\eta)\log(\kappa))$ iterations, we have $\norm{\Delta_t}\leq 0.02\sigma_r$. To this goal, suppose that $\norm{\Delta_{t}}\geq 0.02\sigma_r$ for every $T_1\leq t\leq T_1+T_2$. Recalling the signal-residual decomposition~\eqref{eq_signalresidual2}, one can write
	\begin{equation}\nonumber
		\begin{aligned}
			\norm{\Delta_t} & \leq \norm{\Sigma -S_{t}S_{t}^{\top}}+2\norm{S_tE_t^{\top}}+\norm{F_t}^2+\norm{G_t}^2 \\
			& \leq \norm{\Sigma -S_{t}S_{t}^{\top}} + 0.01\sigma_r,
		\end{aligned}
	\end{equation}
	where the second inequality follows from Lemma~\ref{lem:conditions_noisy}. Given the above inequality, $\norm{\Delta_{t}}\geq 0.02\sigma_r$ implies that $\norm{\Sigma -S_{t}S_{t}^{\top}}\geq 0.01\sigma_r$ for every $T_1\leq t\leq T_1+T_2$. Combined with the one-step dynamic for the signal term, we have
	
	\begin{equation}
		\begin{aligned}
			\norm{\Sigma -S_{t+1}S_{t+1}^{\top}}      & \leq \left(1-\frac{\eta\rho^t}{\norm{\Deltatsecond}}\lambda_{\min}\left(S_tS_t^{\top}\right)\right)\norm{\Sigma-S_{t}S_{t}^{\top}} + 5\frac{\eta\rho^t}{\norm{\Deltatsecond}} \norm{S_tE_t^{\top}}^2+193\sqrt{r}\eta\rho^t\delta\sigma_1 \\
			& \stackrel{(a)}{\leq} \left(1-\frac{0.98\eta\rho^t}{\norm{\Delta_t}}\sigma_r\right)\norm{\Sigma-S_{t}S_{t}^{\top}}+\cO\left(\sqrt{r}\eta\rho^t\delta\sigma_1\right)\\
			&\leq\norm{\Sigma-S_{t}S_{t}^{\top}}-0.98\eta\rho^t\sigma_r\frac{\norm{\Sigma-S_{t}S_{t}^{\top}}}{\norm{\Sigma-S_{t}S_{t}^{\top}}+0.01\sigma_r}+\cO\left(\sqrt{r}\eta\rho^t\delta\sigma_1\right)\\
			&\leq\norm{\Sigma-S_{t}S_{t}^{\top}}-0.49\eta\rho^t\sigma_r+\cO\left(\sqrt{r}\eta\rho^t\delta\sigma_1\right)\\
			&\leq\norm{\Sigma-S_{t}S_{t}^{\top}}-\Omega\left(\eta\rho^t\sigma_r\right),
		\end{aligned}
		\label{eq::69}
	\end{equation}
	where in (a) we used the fact that $\norm{S_tE_t^{\top}}\lesssim \sqrt{\sigma_1}\norm{F_t}\lesssim \sigma_1 \sqrt{r}\eta_0\rho^t\delta t$, $\deltaassumptionfinite$, and $t\lesssim ({\kappa}/{\eta})\log\left({1}/{\alpha}\right)$. The above inequality leads to 
	\begin{align*}
		\norm{\Sigma -S_{t}S_{t}^{\top}} \leq \norm{\Sigma-S_{T_1}S_{T_1}^{\top}}-\Omega\left(\sum_{s=T_1}^{t-1}\eta\rho^t\sigma_r\right).
	\end{align*}
	Therefore, after $T_3 = T_1+\mathcal{O}((\kappa/\eta)\log\kappa)$ iterations, we have $\norm{\Sigma -S_{t}S_{t}^{\top}}\leq 0.01\sigma_r$. This completes the proof of the first phase.

	\paragraph{Phase 2: Global Convergence.} In the second phase, we show that, once $\norm{\Delta_t}\leq 0.02\sigma_r$, $\norm{\Delta_t}$ starts to decay linearly until it is dominated by $\sqrt{d}\norm{G_t}^2$. Similar to before, we define $\gamma_t = \norm{\Sigma-S_tS_t^\top}+2\norm{S_tE_t^\top}+\norm{F_t}^2+\norm{G_t}^2$. Our next lemma plays a central role in our subsequent arguments:
	\begin{lemma}\label{lem_Gt2}
		Suppose that $T_3\leq t\leq T_{end}$ is chosen such that
		$\gamma_s\leq 0.1\sigma_r$ and $\norm{\Delta_s}\geq \sqrt{d}\norm{G_s}^2$, for every $T_3\leq s\leq t$. Then, we have 
		\begin{itemize}
			\item[-] $\gamma_{t+1}\leq 0.1\sigma_r$.
		\end{itemize}
		Moreover, at least one of the following statements are satisfied:
		\begin{itemize}
			\item[-] $\norm{\Delta_{t+1}}\geq \sqrt{d}\norm{G_{t+1}}^2$,
			\item[-] $\norm{\Delta_{t+1}}\lesssim \sqrt{d}\alpha^{2-\mathcal{O}(\sqrt{r}\kappa\delta)}$.
		\end{itemize}
	\end{lemma}
	
	To streamline the presentation, we defer the proof of the above lemma to Appendix~\ref{app_lem_Gt2}. According to the above lemma, we may assume that $\gamma_t\leq 0.1\sigma_r$ and $\norm{\Delta_{t}}\geq \sqrt{d}\norm{G_{t}}^2$ for all iterations $T_3\leq t\leq T_{end}$; otherwise, we have $\norm{\Delta_{t}}\lesssim \sqrt{d}\alpha^{2-\mathcal{O}(\sqrt{r}\kappa\delta)}$ for some $T_3\leq t\leq T_{end}$, which readily completes the proof. On the other hand, the assumptions $\gamma_t\leq 0.1\sigma_r$ and $\norm{\Delta_{t}}\geq \sqrt{d}\norm{G_{t}}^2$ lead to 
	\begin{equation}\nonumber
		\lambda_{\min}\left(S_tS_t^{\top}\right)\geq 0.9\sigma_r, \quad \norm{S_tS_t^{\top}}\leq 1.1\sigma_1, \quad \norm{E_tE_t^{\top}}\leq 0.1\sigma_r, \quad \norm{E_tS_t^{\top}\left(S_tS_t^{\top}\right)^{-1}}\leq 0.2.
	\end{equation}
	Together with our analysis in Phase 1, this implies that the one-step dynamic of $G_t$ holds for every $0\leq t\leq T_{end}$, and we have
	$\norm{G_t} \leq 2\sqrt{2}\alpha^{1-\cO\left(\sqrt{r}\kappa\delta\right)}\sqrt{\bar{\varphi} \delta}$, for every $0\leq t\leq T_{end}$.
	
	Under the assumption $\norm{\Delta_{t}}\geq \sqrt{d}\norm{G_{t}}^2$, our next lemma shows that, if $\norm{\Delta_{t+T_3}}\geq 0.02\sigma_r\rho^t$ for some $t\geq 0$, then there exists $t'$ satisfying $t'-t = \mathcal{O}(1/\eta)$ such that $\norm{\Delta_{t'+T_3}}\leq 0.02\sigma_r\rho^{t'}$. This in turn ensures that the generalization error decays by a constant factor every $\mathcal{O}(1/\eta)$ iterations until it reaches the same order as $\sqrt{d}\norm{G_{t}}^2$. In particular, we have the following lemma:
	
	\begin{lemma}
		\label{lem::appendix-linear-convergence-geometric}
		Suppose that $\norm{\Delta_t}\geq \sqrt{d}\norm{G_t}^2$ for every $T_3\leq t\leq T_{end}$. Suppose that $t_0\geq 1$ satisfies $\norm{\Delta_{t_0+T_3-1}}\leq 0.02 \sigma_r\rho^{t_0-1}$ and $\norm{\Delta_{t_0+T_3}}>0.02 \sigma_r\rho^{t_0}$. Then, after at most $\Delta t=\cO\left(1/\eta\right)$ iterations, we have $\norm{\Delta_{t_0+\Delta t+T_3}}\leq 0.02\sigma_r\rho^{t_0+\Delta t}$.
	\end{lemma}
	The proof of the above lemma is presented in Appendix~\ref{app_lem::appendix-linear-convergence-geometric}. We show how Lemma~\ref{lem::appendix-linear-convergence-geometric} can be used to finish the proof of Theorem~\ref{thm::finite-noisy}. Recall that $\rho = 1- \Theta\left(\eta/(\kappa\log(1/\alpha))\right)$. Let us pick $T_4 = \mathcal{O}((\kappa/\eta)\log^2(1/\alpha))$. Simple calculation reveals that $0.02\sigma_r\rho^{T_4}\lesssim \alpha^2$. According to Lemma~\ref{lem::appendix-linear-convergence-geometric}, if $\norm{\Delta_{T_4+T_3}}>0.02 \sigma_r\rho^{T_4}$, then there exists $\Delta t=\cO\left(1/\eta\right)$ such that $\norm{\Delta_{T_4+T_3+\Delta t}}\leq 0.02 \sigma_r\rho^{T_4+\Delta t}\lesssim \alpha^2$. Combined with Lemma~\ref{lem_Gt2}, we have $\norm{\Delta_t}_F\leq \sqrt{d}\norm{\Delta_t}\lesssim d\norm{G_t}^2\vee \zeta\lesssim d\alpha^{2-\cO\left(\sqrt{r}\kappa\delta\right)}\vee\zeta$ after at most $T_{end} = T_3+T_4+\Delta t = \mathcal{O}((\kappa/\eta)\log^2(1/\alpha))$ iterations. This completes the proof of Theorem~\ref{thm::finite-noisy}.$\hfill\square$
	
	\section{Proofs of Sign-RIP}
	\label{sec::sign-RIP}
	\subsection{Preliminary}\label{subsec::prelim}
	We first provide the preliminary probability tools for proving Theorems~\ref{thm::sign-RIP-partially-corrupted} and~\ref{thm::sign-RIP-gaussian-noise}.
	
	\begin{definition}[Sub-Gaussian random variable]
		\label{def-sub-gaussian}
		We say a random variable $X\in \R$ with expectation $\E[X]=\mu$ is $\sigma^2$-sub-Gaussian if for all $\lambda\in \R$, we have $\E\left[e^{\lambda (X-\mu)}\right]\leq e^{\frac{\lambda^2\sigma^2}{2}}$. 
		Moreover, the sub-Gaussian norm of $X$ is defined as $\norm{X}_{\psi_2}:=\sup_{p\in\bN_+} \left\{p^{-1/2} (\bE[|X|^p])^{1/p}\right\}$.
	\end{definition}
	According to~\cite{wainwright2019high}, the following statements are equivalent:
	\begin{itemize}
		\item $X$ is $\sigma^2$-sub-Gaussian.
		\item (Tail bound) For any $t>0$, we have $\mathbb{P}(|X-\mu|\geq t)\leq 2e^{-\frac{t^2}{2\sigma^2}}$.
		\item (Moment bound) We have $\norm{X}_{\psi_2}\lesssim \sigma$.
	\end{itemize}
	Next, we provide the definitions of the sub-Gaussian process, $\xi$-net, and covering number.
	\begin{definition}[Sub-Gaussian process]
		\label{def::sub-Gaussian-process}
		A zero mean stochastic process $\{\cX_{\theta}, \theta \in \bT\}$ is a $\sigma^2$-sub-Gaussian process with respect to a metric $d$ on a set $\bT$, if for every $\theta,\theta'\in \bT$, the random variable $\cX_{\theta}-\cX_{\theta'}$ is $\left(\sigma d(\theta,\theta')\right)^2$-sub-Gaussian.
	\end{definition}
	\begin{definition}[$\xi$-net and covering number]
		A set $\cN$ is called an $\xi$-net on $(\bT,d)$ if for every $t\in \bT$, there exists $\pi(t) \in \cN$ such that $d(t,\pi(t)) \leq \xi$. The covering number $N(\bT, d, \xi)$ is defined as the smallest cardinality of an $\xi$-net for $(\bT,d)$:
		\begin{equation}\nonumber
			N(\bT, d, \xi) := \inf\{|\cN| : \cN \text{ is an } \xi\text{-net for } (\bT, d)\}.
		\end{equation}
	\end{definition}
	
	Next, we introduce some additional notations which will be used throughout our arguments. Define the rank-$k$ and $\err$-approximate rank-$k$ unit balls as:
	\begin{align}
		\mathbb{S}_k&=\{X\in\R^{d\times d}:\rank(X)\leq k, \norm{X}_F=1\},\nonumber\\
		\mathbb{S}_{k,\err}&=\{X\in\R^{d\times d}:\norm{X}_F=1 \text{ and } \exists X' \text{ such that }\rank(X')\leq k, \norm{X-X'}_F\leq\err\}.\nonumber
	\end{align}
	For simplicity of notation, we use $N_{k,\xi}$ to denote $N(\mathbb{S}_k, \norm{\cdot}_F, \xi)$. Moreover, we define $\cS_{k, \err}$ as the restriction of the set $\cS$ to the set of $\err$-approximate rank-$k$ matrices, i.e., $\cS_{k,\err}=\{X:X\in \cS, X \text{ is }\err\text{-approximate rank-}k\}$.
	The following lemma characterizes the covering number of the set of low-rank matrices with unit norm.
	
	\begin{lemma}[\citet{li2020nonconvex}]
		\label{lem::covering}
		We have $N_{k,\xi} = N(\mathbb{S}_k, \norm{\cdot}_F, \xi)\leq \left(\frac{9}{\xi}\right)^{(2d+1)k}$.
	\end{lemma}
	
	The following well-known result characterizes a concentration bound on the supremum of a sub-Gaussian process.
	\begin{theorem}[Corollary 5.25 and Theorem 5.29 in \cite{van2014probability}] 
		\label{thm::5.25}
		Let $\{X_t\}_{t\in \bT}$ be a separable sub-Gaussian process on $(\bT,d)$. Then, the following statements hold:
		\begin{itemize}
			\item We have
			\begin{equation}
				\mathbb{E}\left[\sup _{t \in \bT} X_{t}\right] \leq 12 \int_{0}^{\infty} \sqrt{\log N(\bT, d, \xi)} d \xi.\nonumber
			\end{equation}
			\item For all $t_0 \in \bT$ and $x \geq 0$, we have
			\begin{equation}
				\mathbb{P}\left(\sup _{t \in \bT}\left\{X_{t}-X_{t_{0}}\right\} \geq C \int_{0}^{\infty} \sqrt{\log N(\bT, d, \xi)} d \xi+x\right) \leq C e^{-x^{2} / C \operatorname{diam}(\bT)^{2}},\nonumber
			\end{equation}
			where $C < \infty$ is a universal constant, and $\operatorname{diam}(\bT)=\sup_{t,t'\in \bT}\{d(t,t')\}$ is the diameter of $\bT$.
		\end{itemize}
	\end{theorem}
	
	Equipped with these preliminary results, we proceed with the proof of Theorem~\ref{thm::sign-RIP-partially-corrupted}.

	%=========================================
	
	\subsection{Proof of Theorem~\ref{thm::sign-RIP-partially-corrupted}}
	\label{sec::proof-sign-RIP-partially-corrupted}
	
	To provide the proof of Theorem~\ref{thm::sign-RIP-partially-corrupted}, we first define the following stochastic process: 
	$$\cH_{X,Y} = \sup\left\{ \frac{1}{m}\sum_{i=1}^m\sign(\inner{A_i}{X}-s_i)\inner{A_i}{Y}-\varphi(X)\inner{\frac{X}{\norm{X}_F}}{Y}\right\},$$
	where $\varphi(X) := \sqrt{\frac{2}{\pi}}\left(1-p+p\E\left[e^{-s^2/(2\norm{X}_F^2)}\right]\right)$, and the supremum is taken over the set-valued function $\sign(\cdot)$. Moreover, to streamline the presentation and whenever there is no ambiguity, we drop the supremum when it is taken with respect to the set-valued function $\sign(\cdot)$. Our next lemma provides a sufficient condition for Sign-RIP.
	\begin{lemma}
		\label{lem::sufficient-condition-sign-RIP}
		Sign-RIP holds with parameters delineated in Theorem~\ref{thm::sign-RIP-partially-corrupted}, if
		\begin{equation}
			\sup_{X\in \cS_{k,\err},Y\in \bS_{k,\err/\zeta}}\cH_{X,Y}\lesssim \inf_{X\in \cS_{k,\err}} \varphi(X)\delta.
		\end{equation}
	\end{lemma}
	\begin{proof}
		According to the definition, Sign-RIP is satisfied if, for every $X,Y\in\cS_{k,\err}$ and $Q\in\mathcal{Q}(X)$, we have
		\begin{align}
			\inner{Q-\varphi(X)\frac{X}{\|X\|_F}}{\frac{Y}{\norm{Y}_F}}\leq \varphi(X)\delta.
		\end{align}
		Recall that $\cS_{k,\err} = \{X:\lowerbound\leq\norm{X}_F\leq R,  X \text{ is }\err\text{-approximate rank-}k\}$. This implies that $Y\in\cS_{k,\err}$ if $Y/\norm{Y}_F\in \bS_{k,\err/\zeta}$. Hence, it suffices to restrict $Y\in \bS_{k,\err/\zeta}$. Therefore, Sign-RIP is satisfied if 
		\begin{equation}\nonumber
			\cH_{X,Y}\leq \varphi(X)\delta, \quad \forall X\in \cS_{k,\err},Y\in \bS_{k,\err/\zeta}.
		\end{equation}
		Hence, to guarantee Sign-RIP, it suffices to have
		\begin{equation}\nonumber
			\sup_{X\in \cS_{k,\err},Y\in \bS_{k,\err/\zeta}}\cH_{X,Y}\leq \inf_{X\in \cS_{k,\err}} \varphi(X)\delta.
		\end{equation}
	\end{proof}
	
	\begin{sloppypar}
		Relying on the above lemma, we instead focus on analyzing the stochastic process $\{\cH_{X,Y}\}_{X\in \cS_{k,\err}, Y\in \bS_{k,\err/\zeta}}$. As a first step towards this goal, we show that the scaling function can be used to characterize $\E[\inner{Q}{Y}]$, for every $Q\in\cQ(X)$.
	\end{sloppypar}
	
	\begin{lemma}\label{l_Hxyl}
		Suppose that the matrix $A$ has i.i.d. standard Gaussian entries and the noise satisfies Assumption~\ref{assumption::outlier}. Then, for every $Q\in\cQ(X)$, we have
		$$\bE\left[\inner{Q}{Y}\right]=\sqrt{\frac{2}{\pi}}\left(1-p+ p \mathbb{E}\left[e^{-s^{2} /\left(2\left\|X\right\|_{F}\right)}\right]\right)\inner{\frac{X}{\norm{X}_F}}{Y}.$$
	\end{lemma}
	\begin{proof}
		Without loss of generality, we assume $\norm{X}_F=\norm{Y}_F=1$. Let us denote $u:=\inner{A}{X},v:=\inner{A}{Y},\rho:=\Cov(u,v)=\inner{X}{Y}$. Then, we have
		\begin{equation}\nonumber
			\begin{aligned}
				\E\left[\sign\left(\inner{A}{X}-s\right)\inner{A}{Y}|s\not=0\right] & =\E\left[\sign\left(u-s\right)v|s\not=0\right]                                                                                                \\
				& \stackrel{(a)}{=}\rho \E\left[\sign(u-s)u|s\not=0\right]                                                                                      \\
				& =\rho\E_{s\sim\mathbb{P}}\left[ \int_{s}^{\infty}u \frac{1}{\sqrt{2\pi}}e^{-u^2/2}du-\int_{-\infty}^{s}u \frac{1}{\sqrt{2\pi}}e^{-u^2/2}du\right] \\
				& =\rho\E_{s\sim\mathbb{P}}\left[ \int_{s}^{\infty}u \frac{1}{\sqrt{2\pi}}e^{-u^2/2}du+\int_{-s}^{\infty}u \frac{1}{\sqrt{2\pi}}e^{-u^2/2}du\right]   \\
				& =2\rho \E_{s\sim\mathbb{P}}\left[\int_{|s|}^{\infty} u \frac{1}{\sqrt{2\pi}}e^{-u^2/2}du\right]                                                   \\
				& =\sqrt{\frac{2}{\pi}}\E_{s\sim\mathbb{P}}\left[\int_{|s|}^{\infty}d\left(-e^{-u^2/2}\right)\right] \inner{X}{Y}                                       \\
				& =\sqrt{\frac{2}{\pi}}\E_{s\sim\mathbb{P}}\left[e^{-s^2/2}\right]\inner{X}{Y},
			\end{aligned}
		\end{equation}
		\begin{sloppypar}
			\noindent where, in (a), we used the fact that $v|u, s\sim \mathcal{N}(\rho u, 1-\rho^2)$ since $s$ is independent of $u, v$. Similarly, one can show that
			$\E\left[\sign\left(\inner{A}{X}\right)\inner{A}{Y}\right] = \sqrt{\frac{2}{\pi}}\inner{X}{Y}$. The proof is completed by noting that
			\begin{align}\nonumber
				\E\left[\sign\left(s+\inner{A}{X}\right)\inner{A}{Y}\right] =& p \E\left[\sign\left(\inner{A}{X}-s\right)\inner{A}{Y}|s\not=0\right]+(1-p)\E\left[\sign\left(\inner{A}{X}\right)\inner{A}{Y}\right].\nonumber
			\end{align}
		\end{sloppypar}
		This completes the proof.
	\end{proof}
	
	Now, we provide an overview of our proof technique for Theorem~\ref{thm::sign-RIP-partially-corrupted}. Let $\cG_Y = \sup_{X\in\cS_{k,\err}}\cH_{X,Y}$ and $\bar \cG_Y = \cG_Y - \E[\cG_Y]$ be stochastic processes indexed by $Y$. According to Lemma~\ref{lem::sufficient-condition-sign-RIP}, it suffices to control $\sup_{Y\in\bS_{k,\err/\zeta}}\cG_Y$. We consider the following decomposition:
	\begin{align}\label{eq_Gy_decomp}
		\sup_{Y\in\bS_{k,\err/\zeta}}\cG_Y &\leq \sup_{Y\in\bS_{k,\err/\zeta}}\bar \cG_Y+ \sup_{Y\in\bS_{k,\err/\zeta}}\bE[{\cG_Y}]\nonumber\\
		&\leq \sup_{Y\in\bS_{k}}\bar \cG_Y+\frac{\err}{\zeta}\sup_{Y\in\bS}\bar \cG_Y+\sup_{Y\in\bS_{k}}\bE[ \cG_Y]+\frac{\err}{\zeta}\sup_{Y\in\bS}\bE[\cG_Y],
	\end{align}
	where the second inequality follows from $\bS_{k,\err/\lowerbound}\subset \bS_{k}+\frac{\err}{\lowerbound}\bS$, and the fact that both $\bar\cG_Y$ and $\bE[\cG_Y]$ are linear functions of $Y$. We control the individual terms in the above decomposition separately.
	To provide an upper bound for $\sup_{Y\in\bS_{k}}\bar \cG_Y$ and $\sup_{Y\in\bS}\bar \cG_Y$, we rely on the following key lemma.
	
	\begin{lemma}\label{l_subG}
		The stochastic processes $\{\bar\cG_Y\}_{Y\in \bS_{k}}$ and $\{\bar\cG_Y\}_{Y\in \bS}$ are $\cO\left({1}/{\measurementnumber}\right)$-sub-Gaussian processes.
	\end{lemma}
	\begin{proof}
		Since $\bS_k\subset \bS$, it suffices to show that $\{\bar\cG_Y\}_{Y\in \bS}$ is $\cO\left({1}/{\measurementnumber}\right)$-sub-Gaussian. According to Definition~\ref{def::sub-Gaussian-process}, the stochastic process $\left\{\bar\cG_Y\right\}_{Y\in \bS}$ is sub-Gaussian if for any arbitrary $Y, Y'\in \bS$, $\bar\cG_Y-\bar\cG_{Y'}$ is $\cO\left(\norm{Y-Y'}_F^2/m\right)$-sub-Gaussian. Note that $\bar\cG_Y-\bar\cG_{Y'}$ is sub-Gaussian if and only if $\cG_Y-\cG_{Y'}$ is sub-Gaussian with the same parameter.
		The latter will be proven by checking the moment bound condition in Definition~\ref{def-sub-gaussian}. For arbitrary $Y,Y'\in \bS$, denote $\Delta Y=Y-Y'$. Then, for any $p\in \bN_+$, we have
		\begin{align*}
			\bE\left[|\cG_Y-\cG_{Y'}|^{2p}\right] & \leq \bE\left[\left|\sup_{X\in \cS_{k,\err}}\frac{1}{\measurementnumber}\sum_{i=1}^{\measurementnumber}\sign\left(\inner{A_i}{X}-s_i\right)\inner{A_i}{\Delta Y}-\varphi(X)\inner{\frac{X}{\norm{X}_F}}{\Delta Y}\right|^{2p}\right]  \\
			&\leq\! \bE\left[\left(\sup_{X\in \cS_{k,\err}}\!\left|\frac{1}{\measurementnumber}\sum_{i=1}^{\measurementnumber}\sign\left(\inner{A_i}{X}\!-\!s_i\right)\inner{A_i}{\Delta Y}\right|\!+\!\!\!\sup_{X\in \cS_{k,\err}}\!\left|\varphi(X)\inner{\frac{X}{\norm{X}_F}}{\Delta Y}\right|\right)^{2p}\right]\\
			& \leq \bE\left[\left(\frac{1}{\measurementnumber}\sum_{i=1}^{\measurementnumber}|\inner{A_i}{\Delta Y}|+\sqrt[]{\frac{2}{\pi}}\norm{\Delta Y}_F\right)^{2p}\right] \\
			\\
			& \leq 2^{2p}\left(\bE\left[\left(\frac{1}{\measurementnumber}\sum_{i=1}^{\measurementnumber}|\inner{A_i}{\Delta Y}|\right)^{2p}\right]+\left(\frac{2}{\pi}\right)^p \norm{\Delta Y}_F^{2p}\right),
		\end{align*}
		where in the third inequality, we used $\varphi(X)\leq\sqrt{2/\pi}$, which holds for every $X\in \R^{d\times d}$. According to~\cite[Appendix A.2]{li2020nonconvex}, the random variable $({1}/{\measurementnumber})\sum_{i=1}^{\measurementnumber}|\inner{A_i}{\Delta Y}|$ is $\cO\left({\norm{\Delta Y}_F^2}/{\measurementnumber}\right)$-sub-Gaussian with mean $\sqrt[]{{2}/{\pi}}\norm{\Delta Y}_F$. Therefore, we have
		\begin{equation}\nonumber
			\begin{aligned}
				\bE\left[\left(\frac{1}{\measurementnumber} \sum_{i=1}^{\measurementnumber}\left|\left\langle A_{i}, \Delta Y\right\rangle\right|\right)^{2 p}\right] & =\bE\left[\left(\frac{1}{\measurementnumber} \sum_{i=1}^{\measurementnumber}\left|\left\langle A_{i}, \Delta Y\right\rangle\right|-\sqrt[]{\frac{2}{\pi}}\|\Delta Y\|_{F}+\sqrt[]{\frac{2}{\pi}}\|\Delta Y\|_{F}\right)^{2 p}\right] \\ & \leq 2^{2 p}\left(\bE\left[\left(\frac{1}{\measurementnumber} \sum_{i=1}^{\measurementnumber}\left|\left\langle A_{i}, \Delta Y\right\rangle\right|-\sqrt[]{\frac{2}{\pi}}\|\Delta Y\|_{F}\right)^{2 p}\right]+\left(\frac{2}{\pi}\right)^{p}\|\Delta Y\|_{F}^{2 p}\right) \\ & \leq 2^{2 p}\left((2 p-1) ! ! \frac{1}{m^{p}}\|\Delta Y\|_{F}^{2 p}+\left(\frac{2}{\pi}\right)^{p}\|\Delta Y\|_{F}^{2 p}\right).
			\end{aligned}
		\end{equation}
		Combining the above inequalities leads to
		\begin{equation}\nonumber
			\bE\left[|\cG_Y-\cG_{Y'}|^{2p}\right]^{1 / 2 p} \leq 4\left((2 p-1) ! ! \frac{1}{m^{p}}\|\Delta Y\|_{F}^{2 p}+\left(\frac{2}{\pi}\right)^{p}\|\Delta Y\|_{F}^{2 p}\right)^{1 / 2 p} \lesssim \sqrt[]{\frac{p}{\measurementnumber}}\norm{\Delta Y}_F,
		\end{equation}
		given that $p>m$. Therefore, $\{\cG_Y\}_{Y \in \bS}$ is a $\cO\left({1}/{\measurementnumber}\right)$-sub-Gaussian process, which implies that $\{\cG_Y\}_{Y \in \bS_k}$ is also a $\cO\left({1}/{\measurementnumber}\right)$-sub-Gaussian process.
	\end{proof}
	
	Given that both $\{\bar\cG_Y\}_{Y\in \bS_{k}}$ and $\{\bar\cG_Y\}_{Y\in \bS}$ are sub-Gaussian processes, we can readily obtain sharp concentration bounds on their suprema.
	\begin{lemma}
		The following statements hold:
		\begin{align}
			& \bE\left[\sup _{Y\in \bS_{k}}\bar\cG_Y\right]\lesssim \sqrt[]{\frac{dk}{\measurementnumber}}, && \mathbb{P}\left(\sup _{Y\in \bS_{k}} \bar\cG_Y \geq \bE\left[\sup _{Y\in \bS_{k}}\bar\cG_Y\right]+\gamma\right) \lesssim  e^{-c\measurementnumber \gamma^{2}};\label{eq_gbar_Sk}\\
			&\bE\left[\sup _{Y\in \bS}\bar\cG_Y\right]\lesssim \sqrt[]{\frac{d^2}{\measurementnumber}}, &&\mathbb{P}\left(\sup _{Y\in \bS} \bar\cG_Y \geq \bE\left[\sup _{Y\in \bS}\bar\cG_Y\right]+\gamma\right) \lesssim  e^{-c\measurementnumber \gamma^{2}}.\label{eq_gbar_S}
		\end{align}
	\end{lemma}
	\begin{proof}
		The proof follows directly from Theorem~\ref{thm::5.25}. The details are omitted for brevity.
	\end{proof}

	Equipped with the above lemma, we provide a concentration bound on $\sup_{Y\in\bS_{k,\err/\zeta}}\bar\cG_Y$.
	
	\begin{lemma}\label{l_sup_Gbar}
		Assume that $m\gtrsim\sqrt{dk/\gamma^2}$ and $\err/\zeta\lesssim\sqrt{k/d}$. Then, the following inequality holds with probability of at least $1-Ce^{-cm\gamma^2}$:
		\begin{align}\nonumber
			\sup _{Y\in \bS_{k,\err/\zeta}} \bar\cG_Y\leq \left(3+\sqrt{\frac{k}{d}}\right)\gamma.
		\end{align}
	\end{lemma}
	\begin{proof}
		Based on~\eqref{eq_gbar_Sk}, we have, with probability of at least $1-Ce^{-cm\gamma^2}$
		\begin{align*}
			\sup _{Y\in \bS_{k}} \bar\cG_Y\leq \bE\left[\sup _{Y\in \bS_{k}}\bar\cG_Y\right]+\gamma\lesssim \sqrt{\frac{dk}{m}}+\gamma\leq 2\gamma,
		\end{align*}
		where the last inequality follows from the assumption $m\gtrsim\sqrt{dk/\gamma^2}$. Similarly, based on~\eqref{eq_gbar_S}, the following inequalities hold with a probability of at least $1-Ce^{-cm\gamma^2}$:
		\begin{align*}
			\frac{\err}{\zeta} \sup _{Y\in \bS} \bar\cG_Y\leq \frac{\err}{\zeta}\left(\bE\left[\sup _{Y\in \bS}\bar\cG_Y\right]+\gamma\right)\lesssim \sqrt{\frac{\err^2d^2}{\zeta^2m}}+\frac{\err \gamma}{\zeta}\leq \left(1+\sqrt{\frac{k}{d}}\right)\gamma,
		\end{align*}
		where the last inequality follows from $m\gtrsim\sqrt{dk/\gamma^2}$ and $\err/\zeta\lesssim\sqrt{k/d}$. Finally, a simple union bound implies that
		\begin{align*}
			\sup _{Y\in \bS_{k,\err/\zeta}} \bar\cG_Y\leq \sup _{Y\in \bS_{k}} \bar\cG_Y+\frac{\err}{\zeta}\sup _{Y\in \bS} \bar\cG_Y\leq \left(3+\sqrt{\frac{k}{d}}\right)\gamma,
		\end{align*}
		with probability of at least $1-Ce^{-cm\gamma^2}$, thereby completing the proof.
	\end{proof}
	Recalling~\eqref{eq_Gy_decomp}, it remains to control the terms $\sup_{Y\in\bS_{k}}\bE[\cG_Y]$ and $\sup_{Y\in\bS}\bE[\cG_Y]$.

	\begin{lemma}\label{l_EG}
		The following inequality holds:
		\begin{align}
			\sup_{Y\in\bS_{k}}\bE[\cG_Y]\leq \sup_{Y\in\bS}\bE[\cG_Y]\lesssim \sqrt{\frac{dk}{m}\log^2(m)\log\left(\frac{R}{\zeta}\right)}.
		\end{align}
	\end{lemma}
	Due to its length, we defer the proof of Lemma~\ref{l_EG} to Appendix~\ref{app_l_EG}. Equipped with Lemmas~\ref{l_sup_Gbar} and~\ref{l_EG}, we are ready to present the proof of Theorem~\ref{thm::sign-RIP-partially-corrupted}.\vspace{2mm}
	
	\noindent {\it Proof of Theorem~\ref{thm::sign-RIP-partially-corrupted}.} The inequality~\eqref{eq_Gy_decomp} combined with Lemmas~\ref{l_sup_Gbar} and~\ref{l_EG} implies that the following inequality holds with probability of at least $1-Ce^{-cm\gamma^2}$:
	\begin{align}\nonumber
		\sup_{Y\in\bS_{k,\err/\zeta}}\cG_Y\lesssim \sqrt{\frac{dk}{m}\log^2(m)\log\left(\frac{R}{\zeta}\right)}+\gamma.
	\end{align}
	On the other hand, one can write
	\begin{equation}
		\inf_{X\in \cS_{k,\err}} \varphi(X)=\inf_{X\in \cS_{k,\err}} \left\{\sqrt{\frac{2}{\pi}}\left(1-p+p\E\left[e^{-s^2/(2\norm{X}_F^2)}\right]\right)\right\}\geq \sqrt{\frac{2}{\pi}}(1-p).
	\end{equation}
	Therefore, upon choosing 
	\begin{align*}
		m\gtrsim\frac{dk\log^2(m)\log\left(R/\zeta\right)}{(1-p)^2\delta^2},\quad\text{and}\quad \gamma\lesssim (1-p)\delta, 
	\end{align*}
	we have 
	\begin{align*}
		\sup_{Y\in\bS_{k,\err/\zeta}}\cG_Y = \sup_{\substack{X\in\cS_{k,\err}\\ Y\in\bS_{k,\err/\zeta}}} \cH_{X,Y} \leq \inf_{X\in \cS_{k,\err}} \varphi(X)\delta 
	\end{align*}
	with probability of at least $1-Ce^{-cm(1-p)^2\delta^2}$. This completes the proof of Theorem~\ref{thm::sign-RIP-partially-corrupted}. $\hfill\square$
	
	%================================
	
	\subsection{Proof of Theorem~\ref{thm::sign-RIP-gaussian-noise}}
	\label{sec::proof-sign-RIP-gaussian-noise}
	\begin{sloppypar}
		Recall that, with outlier noise model, the scaling function takes the form $\varphi(X) = \sqrt{\frac{2}{\pi}}\left(1-p+p\mathbb{E}\left[e^{-s^2/(2\|X\|_F^2)}\right]\right)$. Setting $p=1$, and $s_i\sim \cN(0,\nu_g^2)$ immediately implies $\varphi(X)=\bE_{s\sim \cN(0,\nu_g^2)}\left[e^{-s^2/(2\|X\|_F^2)}\right]=\sqrt{\frac{2}{\pi}}\frac{\norm{X}_F}{\sqrt{\norm{X}_F^2+\nu_g^2}}$. On the other hand, using the same method in the proof of Theorem~\ref{thm::sign-RIP-partially-corrupted}, we can show that the following inequality holds with probability at least $1-Ce^{-cm\gamma^2}$:
	\end{sloppypar}
	\begin{align}\nonumber
		\sup_{Y\in\bS_{k,\err/\zeta}}\cG_Y\lesssim \sqrt{\frac{dk}{m}\log^2(m)\log\left(\frac{R}{\zeta}\right)}+\gamma,
	\end{align}
	where $\cG_Y$ is defined in the proof of Theorem~\ref{thm::sign-RIP-partially-corrupted}. It remains to bound $\inf_{X\in \cS_{k,\err}} \varphi(X)$.
	To this goal, note that $\norm{X}_F\geq \lowerbound, \forall X\in \cS_{k,\err}$. Hence, we have the following lower bound for the scaling function
	\begin{equation}
		\inf_{X\in \cS_{k,\err}} \varphi(X)=\inf_{X\in \cS_{k,\err}}\sqrt{\frac{2}{\pi}}\frac{\norm{X}_F}{\sqrt{\norm{X}_F^2+\nu_g^2}}\geq \sqrt{\frac{2}{\pi}}\frac{\lowerbound}{\sqrt{\lowerbound^2+\nu_g^2}}\gtrsim \frac{\zeta}{\nu_g},
	\end{equation}
	provided that $\zeta\lesssim \nu_g$.
	Therefore, upon choosing $m\gtrsim \frac{\nu_g^2 dk\log^2(m)\log\left(R/\lowerbound\right)}{\lowerbound^2\delta^2}$ and $\gamma\lesssim \frac{\zeta}{\nu_g}\delta$, we have 
	\begin{equation}
		\sup_{Y\in\bS_{k,\err/\zeta}}\cG_Y = \sup_{\substack{X\in\cS_{k,\err}\\ Y\in\bS_{k,\err/\zeta}}} \cH_{X,Y} \leq \inf_{X\in \cS_{k,\err}}\varphi(X)\delta,
	\end{equation}
	with probability at least $1-C_1e^{-C_2m\zeta^2\delta^2/\nu_g^2}$, which completes the proof.
	$\hfill\square$
	
	%================================
	
	\subsection{Proof of Lemma~\ref{lem_lf1}}\label{app_lem_lf1}
	Lemma~\ref{lem_l1l2} is a direct consequence of the following more general result:
	\begin{lemma}\label{lem_l1l2}
		Suppose that the measurements are noiseless, and satisfy $(k,\delta,\err, \cS)$-Sign-RIP for any $\err\geq 0$, $\cS$, and scaling function $\varphi(\cdot)$. Then, for any $\err$-approximate rank-$k$ matrix $X\in \cS$, we have
		\begin{equation}\nonumber
			\left|\frac{1}{m}\sum_{i=1}^{m}\left|\inner{A_i}{X}\right|-\varphi(X)\norm{X}_F\right|\leq \delta\varphi(X)\norm{X}_F.
		\end{equation}
	\end{lemma}
	\begin{proof}
		According to the definition of Sign-RIP, for every $\err$-approximate rank-$k$ matrices $X, Y\in\cS$, and every $Q\in\mathcal{Q}(X)$, we have $\left|\inner{Q-\varphi(X)\frac{X}{\norm{X}_F}}{\frac{Y}{\|Y\|_F}}\right|\leq \varphi(X)\delta$. In particular, upon choosing $Y=X$, we have
		\begin{equation}\nonumber
			\left|\inner{Q-\varphi(X)\frac{X}{\norm{X}_F}}{X}\right|=\left|\frac{1}{m}\sum_{i=1}^{m}\left|\inner{A_i}{X}\right|-\varphi(X)\norm{X}_F\right|\leq \varphi(X)\delta\norm{X}_F.
		\end{equation}
		This completes the proof.
	\end{proof}
	Evidently, the result of Lemma~\ref{lem_l1l2} can be readily recovered from the above lemma after substituting $X$ with $\Delta_t$.
	Moreover, it is worth noting that the above lemma recovers the so-called $\ell_1/\ell_2$-RIP introduced in~\cite{li2020nonconvex} with $\err=0$ and $\varphi(X) = \sqrt{2/\pi}$.

	\subsection{Proof of Lemma~\ref{lem_stepsize}}\label{app_lem_stepsize}
	Due to Sign-RIP, we have
	\begin{align}
		& \norm{Q_t}\leq \varphi(\Delta_t)\left(\frac{\norm{\Delta_t}}{\norm{\Delta_t}_F}+\delta\right)\nonumber\\
		\implies & \frac{1}{\norm{Q_t}}\geq \frac{1}{\varphi(\Delta_t)}\frac{1}{\delta+\frac{\norm{\Delta_t}}{\norm{\Delta_t}_F}}\nonumber\\
		\implies & \frac{1}{\norm{Q_t}}\geq\frac{1}{\varphi(\Delta_t)}\frac{\norm{\Delta_t}_F}{\norm{\Delta_t}}\left(1-\delta\frac{\norm{\Delta_t}_F}{\norm{\Delta_t}}\right)\label{eq_Qinv}
	\end{align}
	By our assumption, $\Delta_t$ is $\err$-approximate rank-$k$. This implies that, there exists a rank-$k$ matrix $X'$ such that $\norm{\Delta_t-X'}_F\leq \err$. This in turn implies that
	\begin{align*}
		\norm{\Delta_t}_F&\leq \norm{\Delta_t-X'}_F+\norm{X'}_F\leq \err+\norm{X'}_F\leq \err+\sqrt{k}\norm{X'}\leq (1+\sqrt{k})\err+\sqrt{k}\norm{\Delta_t}\\
		\norm{\Delta_t}&\geq \norm{X'}-\norm{\Delta_t-X'}\geq \norm{X'}-\err\geq \norm{\Delta_t}-2\err\geq \norm{\Delta_t}/2
	\end{align*}
	where in the last inequality we used the assumption $\norm{\Delta_t}\geq 4\err$. Combining the above inequalities implies that 
	\begin{align*}
		\frac{\norm{\Delta_t}_F}{\norm{\Delta_t}}\leq \frac{2(1+\sqrt{k})\err}{\norm{\Delta_t}}+2\sqrt{k}\leq \frac{1+5\sqrt{k}}{2}.
	\end{align*}
	Combining the above inequality with~\eqref{eq_Qinv} leads to
	\begin{align*}
		\eta_t = \frac{\eta\rho^t}{\norm{Q_t}}\geq\left(1-\left(\frac{1+5\sqrt{k}}{2}\right)\delta\right) \frac{\eta\rho^t}{\varphi(\Delta_t)}\frac{\norm{\Delta_t}_F}{\norm{\Delta_t}}.
	\end{align*}
	The upper bound can be derived in a similar fashion.$\hfill\square$
	
	%=============================
	\section{Proofs of the Expected Loss}\label{app:expected}
	\subsection{Proof of Proposition~\ref{lem:min_eig_population}}\label{app:min_eig_population}
	According to the update rule $U_{t+1}=U_t-\eta \left(U_tU_t^{\top}-X^{\star}\right)U_t$, we have
	\begin{equation}\nonumber
		\begin{aligned}
			S_{t+1}&= V^\top U_{t+1} = S_t-\eta \left(\left(S_tS_t^{\top}-\Sigma\right)S_t+S_tE_t^{\top}E_t\right).
		\end{aligned}
	\end{equation}
	Define an auxiliary matrix $M$ as
	\begin{equation}\label{eq_M}
		M:=\left(I-\Xi\right)\left(S_t-\eta\left(S_tS_t^{\top}-\Sigma\right)S_t\right)\left(S_t^{\top}-\eta S_t^{\top}\left(S_tS_t^{\top}-\Sigma\right)\right)\left(I-\Xi^{\top}\right),
	\end{equation}
	where $\Xi=\eta S_tE_t^{\top} E_tS_t^{\top}\left(S_tS_t^{\top}\right)^{-1}\left(I-\eta\left(S_tS_t^{\top}-\Sigma\right)\right)^{-1}$. Note that, due to our assumptions and the choice of $\eta$, we have $S_tS_t^\top\succ 0$ and $\eta\norm{S_tS_t^\top-\Sigma}<1$. Therefore, both matrices $S_tS_t^\top$ and $I-\eta\left(S_tS_t^{\top}-\Sigma\right)$ are invertible and the matrix $\Xi$ is well-defined. Our goal is to first show that $\lambda_{\min}(S_{t+1}S_{t+1}^\top)\geq \lambda_{\min}(M)$, and then derive a lower bound for $\lambda_{\min}(M)$. Based on the definition of $\Xi$, it is easy to verify that
	$\Xi \left(S_t-\eta\left(S_tS_t^{\top}-\Sigma\right)S_t\right) = S_tE_t^\top E_t\proj_{S_t}.$
	Combining this with~\eqref{eq_M} reveals that
	\begin{equation}\nonumber
		\begin{aligned}
			S_{t+1}S_{t+1}^{\top}-M&=\eta^2 S_tE_t^{\top}E_t\proj_{S_t}^{\perp}E_t^{\top}E_tS_t^{\top}\succeq 0.
		\end{aligned}
	\end{equation}
	Therefore, instead of obtaining a lower bound for $\lambda_{\min}(S_{t+1}S_{t+1}^\top)$, it suffices to obtain a lower bound for $\lambda_{\min}(M)$.
	One can write
	\begin{equation}\label{eq_Xi}
		\begin{aligned}
			\norm{\Xi}&\leq \eta\norm{S_tE_t^{\top} E_tS_t^{\top}\left(S_tS_t^{\top}\right)^{-1}}\norm{\left(I-\eta\left(S_tS_t^{\top}-\Sigma\right)\right)^{-1}}\\
			&\lesssim
			\eta\norm{S_tE_t^{\top} E_tS_t^{\top}\left(S_tS_t^{\top}\right)^{-1}}\\
			&=\eta \norm{\left(S_tS_t^{\top}\right)^{-1/2}S_tE_t^{\top} E_tS_t^{\top}\left(S_tS_t^{\top}\right)^{-1/2}}\\
			&\leq\eta\norm{E_tE_t^{\top}}\norm{\left(S_tS_t^{\top}\right)^{-1/2}S_tS_t^{\top}\left(S_tS_t^{\top}\right)^{-1/2}}\\
			&=\eta \norm{E_tE_t^{\top}}.\\
			&<1,
		\end{aligned}
	\end{equation}
	where the last inequality is due to the assumption $\norm{E_tE_t^\top}\leq \sigma_1$ and the choice of $\eta$. Therefore, we have 
	\begin{equation}\label{eq_lb2}
		\begin{aligned}
			\lambda_{\min}\left(S_{t+1}S_{t+1}^{\top}\right)&\geq \lambda_{\min}\left(M\right)\\
			&\geq \left(1-\norm{\Xi}\right)^2\lambda_{\min}\left(\left(S_t-\eta\left(S_tS_t^{\top}-\Sigma\right)S_t\right)\left(S_t^{\top}-\eta S_t^{\top}\left(S_tS_t^{\top}-\Sigma\right)\right)\right).
		\end{aligned}
	\end{equation}
	Now it suffices to bound $\lambda_{\min}\left(\left(S_t-\eta\left(S_tS_t^{\top}-\Sigma\right)S_t\right)\left(S_t^{\top}-\eta S_t^{\top}\left(S_tS_t^{\top}-\Sigma\right)\right)\right)$. First note that
	\begin{equation}\nonumber
		S_t-\eta\left(S_tS_t^{\top}-\Sigma\right)S_t=\left(I+\eta\Sigma \left(I-\eta S_tS_t^{\top}\right)^{-1}\right)\left(S_t-\eta S_tS_t^{\top}S_t\right).
	\end{equation}
	Based on the above equality, one can write
	\begin{equation}\label{eq_lb}
		\begin{aligned}
			&\lambda_{\min}\left(\left(S_t-\eta\left(S_tS_t^{\top}-\Sigma\right)S_t\right)\left(S_t^{\top}-\eta S_t^{\top}\left(S_tS_t^{\top}-\Sigma\right)\right)\right)\\
			&=\lambda_{\min}\left(\left(I+\eta\Sigma \left(I-\eta S_tS_t^{\top}\right)^{-1}\right)\left(S_t-\eta S_tS_t^{\top}S_t\right)\left(S_t^{\top}-\eta S_t^{\top}S_tS_t^{\top}\right)\left(I+\eta\Sigma \left(I-\eta S_tS_t^{\top}\right)^{-1}\right)\right)\\
			& \geq \lambda_{\min}\left(\left(I+\eta\Sigma \left(I-\eta S_tS_t^{\top}\right)^{-1}\right)\right)^2\lambda_{\min}\left(\left(S_t-\eta S_tS_t^{\top}S_t\right)\left(S_t^{\top}-\eta S_t^{\top}S_tS_t^{\top}\right)\right)\\
			&\stackrel{(a)}{\geq} \left(1+\eta\sigma_r\left(1-\eta\lambda_{\min}\left(S_tS_t^{\top}\right)\right)^{-1}\right)^2 \lambda_{\min}\left(S_tS_t^{\top}-2\eta\left(S_tS_t^{\top}\right)^2+\eta^2\left(S_tS_t^{\top}\right)^3\right)\\
			&\stackrel{(b)}{=}\left(1+\eta\sigma_r\left(1-\eta\lambda_{\min}\left(S_tS_t^{\top}\right)\right)^{-1}\right)^2 \left(\lambda_{\min}\left(S_tS_t^{\top}\right)-2\eta\lambda_{\min}\left(S_tS_t^{\top}\right)^2+\eta^2\lambda_{\min}\left(S_tS_t^{\top}\right)^3\right)\\
			&\geq \left(1+\eta\sigma_r\left(1-\eta\lambda_{\min}\left(S_tS_t^{\top}\right)\right)^{-1}\right)^2 \left(\lambda_{\min}\left(S_tS_t^{\top}\right)-2\eta\lambda_{\min}\left(S_tS_t^{\top}\right)^2\right)\\
			&=\left(1+\eta\sigma_r\right)^2\lambda_{\min}\left(S_tS_t^{\top}\right)-2\eta\left(1+\eta\sigma_r\right)\lambda_{\min}\left(S_tS_t^{\top}\right)^2,
		\end{aligned}
	\end{equation}
	where in (a), we used the fact that
	\begin{equation}
		\begin{aligned}
			\lambda_{\min}\left(I+\eta\Sigma\left (I-\eta S_tS_t^{\top}\right)^{-1}\right)&=1+\lambda_{\min}\left(\eta\Sigma\left (I-\eta S_tS_t^{\top}\right)^{-1}\right)\\
			&=1+\eta\lambda_{\min}\left(\Sigma^{1/2}\left (I-\eta S_tS_t^{\top}\right)^{-1}\Sigma^{1/2}\right)\\
			&\geq 1+\eta\sigma_r\lambda_{\min}\left(\left (I-\eta S_tS_t^{\top}\right)^{-1}\right)\\
			&=1+\eta\sigma_r\left(1-\eta\lambda_{\min}\left(S_tS_t^{\top}\right)\right)^{-1},
		\end{aligned}\nonumber
	\end{equation}
	and in (b) we used the fact that the matrices $S_tS_t^\top$, $(S_tS_t^\top)^2$, and $(S_tS_t^\top)^3$ share the same eigenvectors. Combining~\eqref{eq_lb} with~\eqref{eq_Xi} and~\eqref{eq_lb2} completes the proof.
	$\hfill \square$
	
	\subsection{Proof of Proposition~\ref{lem:dynamics_pop}}\label{app:dynamics_pop}
	Before delving into the details, we provide the update rule for $S_t$ and $E_t$, which will be used frequently throughout our proof. Applying the signal-residual decomposition to $U_{t+1}=U_t-\eta \left(U_tU_t^{\top}-X^{\star}\right)U_t$ leads to
	\begin{equation}
		S_{t+1}= V^\top U_{t+1}= S_t-\eta \left(\left(S_tS_t^{\top}-\Sigma\right)S_t+S_tE_t^{\top}E_t\right),
		\label{eq::appendix-S}
	\end{equation}
	\begin{equation}
		E_{t+1}=V_\perp^\top U_{t+1}= E_t-\eta E_t\left(S_t^{\top}S_t+E_t^{\top}E_t\right).
		\label{eq::appendix-E}
	\end{equation}
	
	\paragraph{Bounding $\norm{\Sigma-S_{t+1}S_{t+1}^\top}$:}
	The update rule for $S_{t+1}$ leads to
	
	\begin{equation}\label{eq:signal_dynamics}
		\begin{aligned}
			\Sigma \!-\! S_{t+1}S_{t+1}^{\top}&\!=\!\underbrace{\Sigma\!-\!S_tS_t^{\top}\!+\!\eta S_tS_t^{\top}\left(S_tS_t^{\top}\!-\!\Sigma\right)\!+\!\eta \left(S_tS_t^{\top}\!-\!\Sigma\right)S_tS_t^{\top}\!-\!\eta^2\left(S_tS_t^{\top}\!-\!\Sigma\right)S_tS_t^{\top}\left(S_tS_t^{\top}\!-\!\Sigma\right)}_{(A)}\\
			&\underbrace{+2\eta S_tE_t^{\top}E_tS_t^{\top}-\eta^2\left(S_tS_t^{\top}-\Sigma\right)S_tE_t^{\top}E_tS_t^{\top}-{\eta^2S_tE_t^{\top}E_tS^{\top}_t\left(S_tS_t^{\top}-\Sigma\right)}}_{(B)}\\
			&-\underbrace{\eta^2S_tE_t^{\top}E_tE_t^{\top}E_tS_t^{\top}}_{(C)}. 
		\end{aligned}
	\end{equation}
	First, we provide an upper bound for (B). Recall that $\eta=\cO\left({1}/{\sigma_1}\right)$. Moreover, we have $\norm{S_tS_t^{\top}-\Sigma}\leq 2.01\sigma_1$ due to the assumption $\norm{S_tS_t^{\top}}\leq 1.01\sigma_1$. Therefore, one can write
	\begin{equation}
		\norm{(B)}\leq 2\eta\norm{S_tE_t^\top E_tS_t^\top}+2\eta^2\norm{S_tS_t^{\top}-\Sigma}\norm{S_tE_t^\top E_tS_t^\top}\leq 4\eta \norm{S_tE_t^\top }^2.\nonumber
	\end{equation}
	Similarly, due to our assumption on $\norm{E_tE_t^\top}$ and $\eta$, we have $\norm{E_tE_t^{\top}}\leq 1/\eta$, which in turn implies
	\begin{equation}
		\norm{(C)}\leq \eta^2\norm{E_tE_t^{\top}}\norm{S_tE_t^{\top}}^2\leq \eta \norm{S_tE_t^\top }^2.\nonumber
	\end{equation} 
	Finally, we provide an upper bound for (A). First, one can verify that
	\begin{equation}\nonumber
		\begin{aligned}
			(A)=\underbrace{\left(\Sigma-S_tS_t^{\top}\right)\left(0.5I-\eta S_tS_t^{\top}\right)}_{(A_1)}+\underbrace{\left(0.5I-\eta S_tS_t^{\top}+\eta^2\left(S_tS_t^{\top}-\Sigma\right)S_tS_t^{\top}\right)\left(\Sigma-S_tS_t^{\top}\right)}_{(A_2)}.
		\end{aligned}
	\end{equation}
	For the first term, we have
	\begin{equation}\nonumber
		\begin{aligned}
			\norm{(A_1)}&\leq \norm{0.5I-\eta S_tS_t^{\top}}\norm{\Sigma-S_tS_t^{\top}}\leq \left(0.5-\eta\lambda_{\min}\left(S_tS_t^{\top}\right)\right)\norm{\Sigma-S_tS_t^{\top}}.
		\end{aligned}
	\end{equation}
	To provide a bound for $(A_2)$, observe that
	\begin{equation}\nonumber
		0.5I-\eta S_tS_t^{\top}+\eta^2\left(S_tS_t^{\top}-\Sigma\right)S_tS_t^{\top}=0.5I-\eta \left(I+\eta\left(\Sigma-S_tS_t^{\top}\right)\right)S_tS_t^{\top}.
	\end{equation}
	The next step in our proof is to show that, with the choice of $\eta\lesssim 1/\sigma_1$, the eigenvalues of $\left(I+\eta\left(\Sigma-S_tS_t^{\top}\right)\right)S_tS_t^{\top}$ are nonnegative. We prove this by showing that the eigenvalues of $\left(I+\eta\left(\Sigma-S_tS_t^{\top}\right)\right)S_tS_t^{\top}$ are close to those of $S_tS_t^\top$. First, note that $I+\eta\left(\Sigma-S_tS_t^{\top}\right)$ is positive definite due to our choice of step-size. Therefore, the matrix $D=\left(I+\eta\left(\Sigma-S_tS_t^{\top}\right)\right)^{1/2}$ is well-defined and invertable. Hence, for any $1\leq i\leq r$, we have
	\begin{equation}
		\lambda_i\left(\left(I+\eta\left(\Sigma-S_tS_t^{\top}\right)\right)S_tS_t^{\top}\right)=\lambda_i\left(D^{-1}\left(I+\eta\left(\Sigma-S_tS_t^{\top}\right)\right)S_tS_t^{\top}D\right)=\lambda_i\left(DS_tS_t^{\top}D\right).\nonumber
	\end{equation} 
	To proceed with our argument, we first present a relative perturbation bound for symmetric matrices.
	\begin{lemma}[Relative Perturbation Bound, \citet{eisenstat1995relative}]
		Let $X\in\mathbb{R}^{d\times d}$ be a symmetric matrix with eigenvalues $\lambda_1\geq\dots\geq \lambda_d$. Moreover, suppose $R$ is a non-singular matrix. Let $\hat\lambda_1\geq\dots\geq \hat\lambda_d$ be the eigenvalues of $Y=R^\top XR$. Then, we have
		\begin{equation}
			\left|\lambda_i-\hat{\lambda}_i\right|\leq |\lambda_i|\norm{I-R^{\top}R}, \quad \text{for all } i.\nonumber
		\end{equation}
		\label{lem::appendix-relative-perturbation}
	\end{lemma}
	Invoking Lemma~\ref{lem::appendix-relative-perturbation} with $X = S_tS_t$ and $Y = DS_tS_tD$ results in 
	\begin{align}
		{\left|\lambda_{i}\left(\left(I+\eta\left(\Sigma-S_tS_t^{\top}\right)\right)S_tS_t^{\top}\right)-\lambda_{i}\left(S_tS_t^{\top}\right)\right|}&\leq \left|\lambda_{i}(S_tS_t^\top)\right|\norm{I-D^2}\nonumber\\
		&=\left|\lambda_{i}(S_tS_t^\top)\right|\norm{I-\left(I+\eta\left(\Sigma-S_tS_t^{\top}\right)\right)}\nonumber\\
		&=\eta\left|\lambda_{i}(S_tS_t^\top)\right|\norm{\Sigma-S_tS_t^{\top}}\nonumber\\
		&\leq 0.5\left|\lambda_{i}(S_tS_t^\top)\right|.\nonumber
	\end{align}
	for every $i=1,2,\dots, r$. The above inequality implies that $\lambda_{\min}\left(\left(I+\eta\left(\Sigma-S_tS_t^{\top}\right)\right)S_tS_t^{\top}\right) \asymp \lambda_{\min}\left(S_tS_t^{\top}\right)$, and therefore, $\lambda_{\min}\left(\left(I+\eta\left(\Sigma-S_tS_t^{\top}\right)\right)S_tS_t^{\top}\right)\geq 0$. This leads to
	\begin{equation}
		\begin{aligned}
			\norm{0.5I-\eta S_tS_t^{\top}+\eta^2\left(S_tS_t^{\top}-\Sigma\right)S_tS_t^{\top}}&=\norm{0.5I-\eta S_tS_t^{\top}+\eta^2\left(S_tS_t^{\top}\right)^{1/2}\left(S_tS_t^{\top}-\Sigma\right)\left(S_tS_t^{\top}\right)^{1/2}}\\
			&\leq 0.5-\lambda_{\min}\left(\eta S_tS_t^{\top}-\eta^2\left(S_tS_t^{\top}\right)^{1/2}\left(S_tS_t^{\top}-\Sigma\right)\left(S_tS_t^{\top}\right)^{1/2}\right)\\
			&=0.5-\lambda_{\min}\left(\eta S_tS_t^{\top}-\eta^2\left(S_tS_t^{\top}-\Sigma\right)\left(S_tS_t^{\top}\right)\right)\\
			&=0.5-\eta \lambda_{\min}\left(\left(I+\eta\left(\Sigma-S_tS_t^{\top}\right)\right)S_tS_t^{\top}\right)\leq 0.5,\nonumber
		\end{aligned}
	\end{equation}
	where in the last inequality, we used the fact that $ \lambda_{\min}\left(\left(I+\eta\left(\Sigma-S_tS_t^{\top}\right)\right)S_tS_t^{\top}\right)\geq 0$. Combined with the definition of $(A_2)$, we have
	\begin{equation}
		\begin{aligned}
			&\norm{(A_2)}\leq \norm{0.5I-\eta S_tS_t^{\top}+\eta^2\left(S_tS_t^{\top}-\Sigma\right)S_tS_t^{\top}}\norm{\Sigma-S_tS_t^{\top}}\leq 0.5 \norm{\Sigma-S_tS_t^{\top}},
		\end{aligned}\nonumber
	\end{equation}
	which in turn implies 
	\begin{align}\nonumber
		\norm{(A)}\leq \left(1-\eta\lambda_{\min}(S_tS_t^\top)\right)\norm{\Sigma-S_tS_t^\top}.
	\end{align}
	Combining the derived upper bounds for $(A)$, $(B)$, and $(C)$ completes the proof for the signal dynamics.$\hfill\square$
	
	\paragraph{Bounding $\norm{S_{t+1}E_{t+1}^\top}$:}
	Recalling the update rules \eqref{eq::appendix-S} and~\eqref{eq::appendix-E}, we have
	\begin{equation}
		\begin{aligned}
			S_{t+1}E_{t+1}^{\top}&= \underbrace{S_tE_t^{\top}+\eta(\Sigma-S_tS_t^{\top})S_tE_t^{\top}-\eta S_t(S_t^{\top}S_t+E_t^{\top}E_t)E_t^{\top}}_{(A)}\\
			&+\underbrace{\eta^2\left(S_tS_t^{\top}-\Sigma\right)S_t(S_t^{\top}S_t+E_t^{\top}E_t)E_t^{\top}}_{(B)}-\underbrace{\eta S_tE_t^{\top}E_tE^{\top}_t}_{(C)}+\underbrace{\eta^2S_tE^{\top}_tE_t\left(S_t^{\top}S_t+E_t^{\top}E_t\right)E_t^{\top}}_{(D)}.
		\end{aligned}\nonumber
	\end{equation}
	We provide separate bounds for the individual terms in the above equality. First, observe that $\norm{(C)}\leq \eta\norm{E_tE_t^{\top}}\norm{S_tE_t^{\top}}$. Moreover, one can write
	\begin{equation}
		\begin{aligned}
			\norm{(B)}&\leq \eta^2 \norm{\Sigma-S_tS_t^{\top}}\norm{S_t(S_t^{\top}S_t+E_t^{\top}E_t)E_t^{\top}}\\
			&\leq \eta^2 \norm{\Sigma-S_tS_t^{\top}}\left(\norm{S_tS_t^{\top}}+\norm{E_tE_t^{\top}}\right)\norm{S_tE_t^{\top}}\\
			&\leq \eta \norm{\Sigma-S_tS_t^{\top}}\norm{S_tE_t^{\top}},
		\end{aligned}\nonumber
	\end{equation}
	where, in the last inequality, we used the assumption that $\norm{E_tE_t^{\top}}\lesssim \sigma_1$ and $\norm{S_tS_t^{\top}}\lesssim \sigma_1$. Similarly, we have
	\begin{equation}
		\begin{aligned}
			\norm{(D)}&\stackrel{(a)}{\leq} \eta^2\norm{S_tE_t^{\top}}\left(\norm{S_tS_t^{\top}}+\norm{E_tE_t^{\top}}\right)\norm{E_tE_t^{\top}}\\
			&\leq \eta \norm{E_tE_t^{\top}}\norm{S_tE_t^{\top}},
		\end{aligned}\nonumber
	\end{equation}
	where, in (a), we used $\norm{S_tS_t^{\top}}E_tE_t^{\top}\succeq E_tS_t^{\top}S_tE_t^{\top}\succeq 0$. Finally, we provide an upper bound for $(A)$. 
	\begin{equation}
		\begin{aligned}
			\norm{(A)}&\leq\left(\eta\norm{\Sigma-S_tS_t^{\top}}+\norm{I-\eta S_tS_t^{\top}}\right)\norm{S_tE_t^{\top}}\\
			&\leq \left(\eta\norm{\Sigma-S_tS_t^{\top}}+1-\eta\lambda_{\min}\left(S_tS_t^{\top}\right)\right)\norm{S_tE_t^{\top}}.
		\end{aligned}\nonumber
	\end{equation}
	Combining the derived bounds for (A), (B), (C), and (D) concludes the proof.$\hfill\square$
	
	\paragraph{Bounding $\norm{E_{t+1}E_{t+1}^\top}$:} Due to the update rule for $E_{t+1}$, one can write
	\begin{equation}
		\begin{aligned}
			E_{t+1}E_{t+1}^{\top}&=E_tE_t^{\top}-2\eta E_t\left(E^{\top}_tE_t+S^{\top}_tS_t\right)E_t^{\top}+\eta^2E_t\left(E^{\top}_tE_t+S^{\top}_tS_t\right)^2E_t^{\top}\\
			&=E_t\left(I-2\eta \left(E^{\top}_tE_t+S^{\top}_tS_t\right)+\eta^2\left(E^{\top}_tE_t+S^{\top}_tS_t\right)^2\right)E_t^{\top}.
		\end{aligned}
		\label{eq::appendix-46}
	\end{equation}
	On the other hand, $\eta\lesssim 1/\sigma_1$, $\norm{E_tE_t}\lesssim \sigma_1$, and $\norm{S_tS_t}\lesssim \sigma_1$ imply that
	\begin{equation}\label{eq:psd}
		I-\eta E^{\top}_tE_t\succ I-2\eta \left(E^{\top}_tE_t+S^{\top}_tS_t\right)+\eta^2\left(E^{\top}_tE_t+S^{\top}_tS_t\right)^2\succ 0.
	\end{equation}
	Hence, we have
	\begin{align}
		\norm{E_tE_t^\top} &= \norm{E_t\left(I-2\eta \left(E^{\top}_tE_t+S^{\top}_tS_t\right)+\eta^2\left(E^{\top}_tE_t+S^{\top}_tS_t\right)^2\right)E_t^{\top}}\nonumber\\
		& \leq \norm{E_t\left(I-\eta E_t^{\top}E_t\right)E_t^{\top}}\nonumber\\
		& = \norm{E_tE_t^{\top}-\eta E_tE^{\top}_tE_tE_t^{\top}}\nonumber\\
		& = \left(1-\eta \norm{E_tE_t^{\top}}\right)\norm{E_tE_t^{\top}},\nonumber
	\end{align}
	where in the last inequality, we used the fact that $\norm{E_tE_t^{\top}}\leq \sigma_1$.
	
	\paragraph{Bounding $\norm{S_{t+1}S_{t+1}^\top}$:} First, recall that
	\begin{equation}\nonumber
		\begin{aligned}
			S_{t+1}S_{t+1}^{\top}=&{S_tS_t^{\top}-2\eta S_tS_t^{\top}S_tS_t^{\top}+\eta\Sigma S_tS_t^{\top}+\eta S_tS_t^{\top}\Sigma+\eta^2 \left(S_tS_t^{\top}-\Sigma\right)S_tS_t^{\top}\left(S_tS_t^{\top}-\Sigma\right)}\nonumber\\
			&{-2\eta S_tE_tE_t^{\top}S_t^{\top}+\eta^2 S_tE_t^{\top}E_tE_t^{\top}E_tS_t^{\top}}\nonumber\\
			&{+\eta^2\left(S_tS_t^{\top}-\Sigma\right)S_tE_tE_t^{\top}S_t^{\top}+\eta^2 S_tE_tE_t^{\top}S_t^{\top}\left(S_tS_t^{\top}-\Sigma\right)}\nonumber\\
			\preceq& \underbrace{S_tS_t^{\top}-2\eta S_tS_t^{\top}S_tS_t^{\top}+\eta\Sigma S_tS_t^{\top}+\eta S_tS_t^{\top}\Sigma+\eta^2 \left(S_tS_t^{\top}-\Sigma\right)S_tS_t^{\top}\left(S_tS_t^{\top}-\Sigma\right)}_{(A)}\nonumber\\
			&+\underbrace{\eta^2 S_tE_t^{\top}E_tE_t^{\top}E_tS_t^{\top}+\eta^2\left(S_tS_t^{\top}-\Sigma\right)S_tE_tE_t^{\top}S_t^{\top}+\eta^2 S_tE_tE_t^{\top}S_t^{\top}\left(S_tS_t^{\top}-\Sigma\right)}_{(B)},\nonumber\\
		\end{aligned}
	\end{equation}
	where the last inequality follows by noting that ${-2\eta S_tE_tE_t^{\top}S_t^{\top}}\preceq 0$. Now, it is easy to see that
	\begin{align}\nonumber
		\norm{(B)}&\leq \eta^2\norm{E_tE_t^\top}^2\norm{S_tS_t^\top}+2\eta^2\norm{E_tE_t^{\top}}\norm{S_tS_t^{\top}}\left(\sigma
		_1 + \norm{S_tS_t^{\top}}\right)\nonumber\\
		&\leq 3\eta^2\norm{E_tE_t^{\top}}\norm{S_tS_t^{\top}}\left(\sigma_1 + \norm{S_tS_t^{\top}}\right).\nonumber
	\end{align}
	It remains to provide an upper bound for $(A)$. One can write
	\begin{equation}
		\begin{aligned}
			(A)
			=&\underbrace{S_tS_t^{\top}-2\eta \left(S_tS_t^{\top}\right)^2+\eta^2\left(\left(S_tS_t^{\top}\right)^3+\Sigma S_tS_t^{\top}\Sigma\right)}_{({A_1})}\\
			&+\underbrace{\eta \Sigma S_tS_t^{\top}\left(I-\eta S_tS_t^{\top}\right)\eta  S_tS_t^{\top}\left(I-\eta S_tS_t^{\top}\right)\Sigma}_{(A_2)}.
		\end{aligned}\nonumber
	\end{equation}
	For ($A_1$), we have
	\begin{equation}
		\begin{aligned}
			\norm{(A_1)}&\leq \norm{S_tS_t^{\top}-2\eta \left(S_tS_t^{\top}\right)^2+\eta^2\left(S_tS_t^{\top}\right)^3}+\eta^2 \norm{\Sigma S_tS_t^{\top}\Sigma}\\
			&\leq \norm{S_tS_t^{\top}-2\eta \left(S_tS_t^{\top}\right)^2+\eta^2\left(S_tS_t^{\top}\right)^3}+\eta^2 \sigma_1^2 \norm{S_tS_t^{\top}}\\
			&\stackrel{(a)}{\leq}\norm{S_tS_t^{\top}}-2\eta\norm{S_tS_t^{\top}}^2+\eta^2 \norm{S_tS_t^{\top}}^3+\eta^2 \sigma_1^2 \norm{S_tS_t^{\top}},
		\end{aligned}\nonumber
	\end{equation}
	where (a) follows from the fact that $S_tS_t^{\top}$, $\left(S_tS_t^{\top}\right)^2$, $\left(S_tS_t^{\top}\right)^3$ share the same eigenvectors, and the assumption $\eta\lesssim 1/\sigma_1$. On the other hand, one can easily verify that
	\begin{equation}
		\norm{(A_2)}\leq 2\eta \sigma_1 \norm{S_tS_t^{\top}\left(I-\eta S_tS_t^{\top}\right)}\leq 2\eta \sigma_1 \left(\norm{S_tS_t^{\top}}-\eta\norm{S_tS_t^{\top}}^2\right). \nonumber
	\end{equation}
	Combining the upper bounds for (A) and (B) leads to 
	\begin{equation}
		\begin{aligned}
			\norm{S_{t+1}S_{t+1}^{\top}}&\leq \left((1+\eta\sigma_1)^2+3\eta^2\sigma_1\norm{E_tE_t^{\top}}\right)\norm{S_tS_t^{\top}} -2\eta\left(1+\eta\sigma_1-1.5\eta\norm{E_tE_t^{\top}}\right) \norm{S_tS_t^{\top}}^2\\
			&\ \ +\eta^2 \norm{S_tS_t^{\top}}^3\\
			&\stackrel{(a)}{\leq} \left((1+2.001\sigma_1\right)\norm{S_tS_t^{\top}} -2\eta \norm{S_tS_t^{\top}}^2+\eta^2 \norm{S_tS_t^{\top}}^3=f\left(\norm{S_tS_t^{\top}}\right),\nonumber
		\end{aligned}
	\end{equation}
	where, in (a), we used the assumption $\norm{E_tE_t^{\top}}\lesssim \sigma_1$ and $\eta\lesssim 1/\sigma_1$. Now, let us define the function $$f(x):=\eta ^2x^3-2\eta x^2+\left(1+2.001\sigma_1\right)x.$$ It is easy to see that $f(x)$ is increasing within the interval $x\leq {1}/{4\eta}$. On the other hand, we have $1.01\sigma_1\leq {1}/{4\eta}$ due to our choice of $\eta$. Therefore, we have
	$$\norm{S_{t+1}S_{t+1}^{\top}}\leq f\left(\norm{S_tS_t^{\top}}\right)\leq f\left(1.01\sigma_1\right).$$
	On the other hand, simple calculation reveals that
	\begin{align}\nonumber
		f(1.01\sigma_1) = \eta ^2(1.01\sigma_1)^3-2\eta (1.01\sigma_1)^2+\left(1+2.001\sigma_1\right)(1.01\sigma_1)\leq 1.01\sigma_1.
	\end{align}
	for $\eta\leq {c}/{\sigma_1}$ with sufficiently small constant $c$. This completes the proof.$\hfill\square$
	
	\subsection{Proof of Lemma~\ref{lem:conditions}}\label{app_lem:conditions}
	We prove this lemma by induction on $t$. First, due to our choice of the initial point in Theorem~\ref{thm_population}, we have $\norm{E_0E_0^\top}\leq \alpha^2$, $\norm{S_0S_0^\top}\leq 1.01\sigma_1$ and $S_0S_0^\top\succ 0$.  Now, suppose that~\eqref{eq_error},~\eqref{eq_up_signal}, and~\eqref{eq_psd} hold for $t$. Then, it is easy to see that
	\begin{align}
		\norm{E_tE_t^\top}\leq \norm{E_0E_0^\top}\leq \alpha^2,
	\end{align}
	due to the decreasing nature of $\norm{E_tE_t^\top}$. Therefore, Proposition~\ref{lem:dynamics_pop} can be invoked to show that $\norm{S_{t+1}S_{t+1}^\top}\leq 1.01\sigma_1$. Moreover, we have 
	\begin{equation}
		\begin{aligned}
			\norm{E_{t+1}E_{t+1}^{\top}}&\leq \norm{E_{t}E_{t}^{\top}} - \eta\norm{E_{t}E_{t}^{\top}}^2\\
			&\leq \frac{1}{\eta t +\frac{1}{\alpha^2}}- \frac{\eta}{\left(\eta t +\frac{1}{\alpha^2}\right)^2}\\
			&\leq \frac{1}{\eta (t+1) +\frac{1}{\alpha^2}}.
		\end{aligned}
	\end{equation}
	On the other hand, Proposition~\ref{lem:min_eig_population} can be used to show that
	\begin{align}
		\lambda_{\min}\left(S_{t+1}S_{t+1}^{\top}\right)&\geq \left(1+\eta^2\sigma_r^2\right)\lambda_{\min}\left(S_tS_t^{\top}\right)-2\eta\left(1+\eta\sigma_r\right)\lambda_{\min}\left(S_tS_t^{\top}\right)^2\nonumber\\
		&\geq \left(1-2\eta\sigma_1\left(1+\eta\sigma_r\right)+\eta^2\sigma_r^2\right)\lambda_{\min}\left(S_tS_t^{\top}\right)\nonumber\\
		&>0,\nonumber
	\end{align}
	where in the first inequality, we used $\norm{E_tE_t^{\top}}\lesssim \sigma_r$, which follows from $\norm{E_tE_t^\top}\leq \alpha^2$ and our choice of $\alpha$. Moreover, the last inequality follows from our choice of $\eta$. $\hfill\square$
	
	%===============================
	
	\section{Proofs of Empirical Loss with Noiseless Measurements}\label{app:empirical_noiseless}
	\subsection{Preliminaries}
	For simplicity, we define $\Delta_t = U_tU_t^\top-X^\star$.
	Before presenting the proofs for the empirical loss, we first introduce a key decomposition.
	Recall the update rule
	\begin{equation}\label{eq_decomp_emp}
		U_{t+1}=U_t-2\eta_tQ_tU_t:=\tilde{U}_{t+1}+\del U_{t},
	\end{equation}
	where $\eta_t=\frac{\eta}{2}\frac{1}{\measurementnumber}\sum_{i=1}^{\measurementnumber}|\inner{A_i}{\Delta_t}|$ and $Q_t\in \mathcal{Q}(\Delta_t)$. Moreover, denote $\tilde{U}_{t+1}=U_t-\bareta\Delta_tU_t$, where $\bareta = \eta\varphi(\Delta_t)^2$ is the update rule obtained from the expected loss, and $\del U_{t}$ with $\del=\bareta\Delta_t-2\eta_t Q_t$ is the residual caused by the deviation of the empirical loss from its expectation. Finally, we define $\tilde{S}_t=V^{\top}\tilde{U}_{t}$, and $\tilde{E}_t=V_{\perp}^{\top}\tilde{U}_{t}$. 
	
	\begin{lemma}
		\label{lem::del}
		Suppose that the measurements satisfy Sign-RIP with parameters $\left(4r,\delta, \sqrt{d}\norm{G_t}^2\right)$. Then, we have $\norm{\del}\leq 3\bareta\delta\norm{\Delta_t}_F$.
	\end{lemma}
	
	\begin{proof}
		One can write
		\begin{equation}
			\Delta_t=\underbrace{V\left(S_tS_t^{\top}-\Sigma\right)V^{\top}+VS_tE_t^{\top}V_{\perp}^{\top}+V_{\perp}E_tS_t^{\top}V^{\top}+V_{\perp}G_tG_t^{\top}V_{\perp}^{\top}}_{\text{rank-$4r$}}+\underbrace{V_{\perp}G_tG_t^{\top}V_{\perp}^{\top}}_{\text{small perturbation}}.\nonumber
		\end{equation}
		Note that $\norm{V_{\perp}G_tG_t^{\top}V_{\perp}^{\top}}_F\leq \sqrt{d}\norm{G_t}^2$. Therefore, $\Delta_t$ is an $\sqrt{d}\norm{G_t}^2$-approximate rank-$4r$ matrix. One can write
		\begin{equation}
			\begin{aligned}
				R_t =& \eta \varphi(\Delta_t)^2\Delta_t-2\eta_t Q_t&\\
				=&\left(\frac{\eta}{2}\varphi(\Delta)\norm{\Delta_t}_F-\eta_t\right)Q_t+\left(\frac{\eta}{2}\varphi(\Delta)\norm{\Delta_t}_F-\eta_t\right)Q_t^\top\\
				&+\frac{\eta}{2}\varphi(\Delta_t)\norm{\Delta_t}_F\left(\varphi(\Delta_t)\frac{\Delta_t}{\norm{\Delta_t}_F}-Q_t\right)+\frac{\eta}{2}\varphi(\Delta_t)\norm{\Delta_t}_F\left(\varphi(\Delta_t)\frac{\Delta_t}{\norm{\Delta_t}_F}-Q_t^\top\right)
			\end{aligned}\nonumber
		\end{equation}
		Due to the above decomposition, one can write
		\begin{equation}\label{eq_Rt}
			\begin{aligned}
				\norm{R_t}\leq {2\left|{\frac{\eta}{2}\varphi(\Delta)\norm{\Delta_t}_F-\eta_t}\right|\norm{Q_t}}+{\eta\varphi(\Delta_t)\norm{\Delta_t}_F\norm{\varphi(X)\frac{\Delta_t}{\norm{\Delta_t}_F}-Q_t}}.
			\end{aligned}
		\end{equation}
		First, note that $\left|{(\eta/2)\varphi(\Delta)\norm{\Delta_t}_F-\eta_t}\right|\leq \delta(\eta/2)\varphi(\Delta_t)\norm{\Delta_t}_F$ due to Lemma~\ref{lem_l1l2}. Moreover, due to Sign-RIP, we have $\norm{Q_t}\leq (1+\delta)\varphi(\Delta_t)$ and $\norm{\varphi(X)\frac{\Delta_t}{\norm{\Delta_t}_F}-Q_t}\leq \delta\varphi(\Delta_t)$. 
		% Finally, note that $\varphi(\Delta_t)\leq \bar{\varphi}$. 
		Combining these upper bounds with~\eqref{eq_Rt} completes the proof.
	\end{proof}
	
	\subsection{Proof of Proposition~\ref{prop:min_eig_empirical}}\label{app_prop:min_eig_empirical}

	Due to the proposed decomposition~\eqref{eq_decomp_emp}, one can write
	\begin{equation}\label{eq_S}
		\begin{aligned}
			S_{t+1}&=\tilde{S}_{t+1}+V^\top R_tU_t.
		\end{aligned}
	\end{equation}
	Our main goal is to show that the minimum eigenvalue of $S_tS_t^\top$ follows that of $\tilde{S}_t\tilde{S}_t^\top$---which has been characterized in Lemma~\ref{lem:min_eig_population}---plus an additional deviation caused by the term $V^\top R_tU_t$. 
	Similar to the proof of Proposition~\ref{lem:min_eig_population}, we characterize the growth rate of $\lambda_{\min}(S_{t+1}S_{t+1}^\top)$ by first resorting to a more tractable lower bound. Adopting the notation introduced in Appendix~\ref{app:empirical_noiseless}, consider the following auxiliary matrix
	\begin{align}
		M &:= \left(I+\Xi\right)\tilde{S}_{t+1}\tilde{S}_{t+1}^\top\left(I+\Xi\right)^\top=\left(I+\Xi\right)V^{\top}\left(I-\bareta\Delta_t\right)U_tU_t^\top\left(I-\bareta\Delta_t\right)\left(I+\Xi\right)^\top,\nonumber
	\end{align}
	where $\Xi:=V^{\top} \del U_{t}\tilde{S}_{t+1}^\top\left(\tilde{S}_{t+1}\tilde{S}_{t+1}^\top\right)^{-1}$. Note that, according to Lemma~\ref{lem:conditions}, we have $\tilde{S}_{t+1}\tilde{S}_{t+1}^\top\succ 0$ due to our assumption $S_tS_t^{\top}\succ 0$. Therefore, $\tilde{S}_{t+1}\tilde{S}_{t+1}^\top$ is invertible, and hence, the matrix $\Xi$ is well-defined. On the other hand, it is easy to see that $\Xi$ satisfies
	\begin{equation}
		V^{\top} \del U_{t}\tilde{S}_{t+1}^\top=\Xi \tilde{S}_{t+1}\tilde{S}_{t+1}^\top.
	\end{equation}
	Based on the above equality, one can write
	\begin{equation}\label{eq_lb_SM}
		\begin{aligned}
			S_{t+1}S_{t+1}^{\top}-M&=\left(\tilde{S}_{t+1} \tilde{S}_{t+1}^{\top}+V^{\top} R_{t} U_{t} U_{t}^{\top} R_{t}^{\top} V+\tilde{S}_{t+1} U_{t}^{\top} R_{t}^{\top} V+V^{\top} R_{t} U_{t} \tilde{S}_{t+1}^\top\right)\\
			&\ \ \ -\left(\tilde{S}_{t+1} \tilde{S}_{t+1}^{\top}-\Xi \tilde{S}_{t+1} \tilde{S}_{t+1}^{\top} \Xi^{\top}+\Xi \tilde{S}_{t+1} \tilde{S}_{t+1}^{\top}+\tilde{S}_{t+1} \tilde{S}_{t+1}^{\top} \Xi^{\top}\right)\\
			&= V^{\top}\del U_{t}U_{t}^{\top} \del^{\top} V- V^{\top}\del U_{t}\tilde{S}_{t+1}^{\top} \left(\tilde{S}_{t+1}\tilde{S}_{t+1}^{\top}\right)^{-1}\tilde{S}_{t+1}U_{t}^{\top}\del^{\top}V\\
			&=V^{\top}\del U_{t}\left(I-\tilde{S}_{t+1}^{\top}\left(\tilde{S}_{t+1}\tilde{S}_{t+1}^{\top}\right)^{-1}\tilde{S}_{t+1}\right)U_{t}^{\top}\del^{\top} V\\
			&= V^{\top}\del U_{t}\proj_{\tilde{S}+1}^\perp U_t^\top \del^\top V\\
			&\succeq 0.
		\end{aligned}
	\end{equation}
	Therefore, we have $\lambda_{\min}(S_{t+1}S_{t+1}^{\top})\geq \lambda_{\min}(M)$. Our next goal is to provide a lower bound for $\lambda_{\min}(M)$. To this goal, we will show that the minimum eigenvalue of $M$ is close to that of $\tilde{S}_{t+1}\tilde{S}_{t+1}^\top$. Combined with the minimum eigenvalue dynamics of $\tilde{S}_{t+1}\tilde{S}_{t+1}^\top$ presented in Proposition~\ref{lem:min_eig_population}, this completes the proof. 
	
	To show that $\lambda_{\min}(M)\approx \lambda_{\min}(\tilde{S}_{t+1}\tilde{S}_{t+1}^\top)$, we will use the relative perturbation bound presented in Lemma~\ref{lem::appendix-relative-perturbation}. To this goal, first we need to provide an upper bound for $\norm{\Xi}$.
	
	\paragraph{Bounding $\norm{\Xi}$.} One can write
	\begin{equation}\label{eq_Xi2}
		\begin{aligned}
			\norm{\Xi}&\leq \norm{\del} \norm{U_{t}\tilde{S}_{t+1}^{\top}\left(\tilde{S}_{t+1}\tilde{S}_{t+1}^{\top}\right)^{-1}}\\
			&= \norm{\del} \norm{\left(I-\bareta\left(U_tU_t^{\top}-X^{\star}\right)\right)^{-1}\tilde{U}_{t+1}\tilde{S}_{t+1}^{\top}\left(\tilde{S}_{t+1}\tilde{S}_{t+1}^{\top}\right)^{-1}}\\
			&\leq \norm{\del} \norm{\left(I-\bareta\left(U_tU_t^{\top}-X^{\star}\right)\right)^{-1}}\norm{\tilde{U}_{t+1}\tilde{S}_{t+1}^{\top}\left(\tilde{S}_{t+1}\tilde{S}_{t+1}^{\top}\right)^{-1}}\\
			&\stackrel{(a)}{\leq} 2\norm{\del}\norm{\tilde{U}_{t+1}\tilde{S}_{t+1}^{\top}\left(\tilde{S}_{t+1}\tilde{S}_{t+1}^{\top}\right)^{-1}}\\
			&\stackrel{(b)}{\leq} 6\bareta\delta\norm{\Delta_t}_F\norm{\tilde{U}_{t+1}\tilde{S}_{t+1}^{\top}\left(\tilde{S}_{t+1}\tilde{S}_{t+1}^{\top}\right)^{-1}},
		\end{aligned}
	\end{equation}
	where in (a), we used the fact that $\norm{U_tU_t^\top-X^\star}\lesssim \sigma_1$ due to our assumptions on $\norm{S_tS_t^\top}$ and $\norm{E_tE_t^\top}$, and our choice of $\eta$. Moreover, (b) follows from Lemma~\ref{lem::del}. Now, note that
	\begin{equation}
		\begin{aligned}
			\tilde{U}_{t+1}\tilde{S}_{t+1}^{\top}\left(\tilde{S}_{t+1}\tilde{S}_{t+1}^{\top}\right)^{-1}&=V\tilde{S}_{t+1}\tilde{S}_{t+1}^{\top}\left(\tilde{S}_{t+1}\tilde{S}_{t+1}^{\top}\right)^{-1}+V_{\perp}\tilde{E}_{t+1}\tilde{S}_{t+1}^{\top}\left(\tilde{S}_{t+1}\tilde{S}_{t+1}^{\top}\right)^{-1}\\
			&=V+V_{\perp}\tilde{E}_{t+1}\tilde{S}_{t+1}^{\top}\left(\tilde{S}_{t+1}\tilde{S}_{t+1}^{\top}\right)^{-1}.
		\end{aligned}
	\end{equation}
	Therefore, we have
	\begin{align*}
		\norm{\tilde{U}_{t+1}\tilde{S}_{t+1}^{\top}\left(\tilde{S}_{t+1}\tilde{S}_{t+1}^{\top}\right)^{-1}}\leq 1+\norm{\tilde{E}_{t+1}\tilde{S}_{t+1}^{\top}\left(\tilde{S}_{t+1}\tilde{S}_{t+1}^{\top}\right)^{-1}}.
	\end{align*}	
	In order to provide an upper bound for $\norm{\Xi}$, it suffices to bound $\norm{\tilde{E}_{t+1}\tilde{S}_{t+1}^{\top}\left(\tilde{S}_{t+1}\tilde{S}_{t+1}^{\top}\right)^{-1}}$. To this goal, first we present the following technical lemma, the proof of which can be found in Appendix~\ref{app_lem:appendix-4}.
	\begin{lemma}
		The following statements hold:
		\begin{itemize}
			\item[-] $\frac{1}{2}\leq \norm{S_tS_t^{\top}\left(\tilde{S}_{t+1}\tilde{S}_{t+1}^{\top}\right)^{-1}}\leq 2$ and $\frac{1}{3}\leq \norm{S_tS_t^{\top}\left({S}_{t+1}{S}_{t+1}^{\top}\right)^{-1}}\leq 3$;
			\item[-]  $\norm{\tilde{E}_{t+1}\tilde{S}_{t+1}^{\top}\left(\tilde{S}_{t+1}\tilde{S}_{t+1}^{\top}\right)^{-1}}\leq 3\norm{E_tS_t^{\top}\left(S_tS_t^{\top}\right)^{-1}}$.
		\end{itemize}    
		\label{lem:appendix-4}
	\end{lemma}
	
	\begin{sloppypar}
		The second statement of the above lemma connects $\norm{\tilde{E}_{t+1}\tilde{S}_{t+1}^{\top}\left(\tilde{S}_{t+1}\tilde{S}_{t+1}^{\top}\right)^{-1}}$ to $\norm{E_tS_t^{\top}\left(S_tS_t^{\top}\right)^{-1}}$ which, according to our assumption, is upper bounded by $1/3$. This upper bound, together with~\eqref{eq_Xi2} implies that 
		$$
		\norm{\Xi}\leq 12\bareta\delta\norm{\Delta_t}_F.
		$$
	\end{sloppypar}

	\noindent On the other hand, applying the relative perturbation bound to $M$ and $\tilde{S}_{t+1}\tilde{S}_{t+1}^\top$ implies that
	\begin{align*}
		\left|\lambda_{\min}(M)-\lambda_{\min}(\tilde{S}_{t+1}\tilde{S}_{t+1}^\top)\right|\leq & \lambda_{\min}(\tilde{S}_{t+1}\tilde{S}_{t+1}^\top)\norm{I-(I+\Xi)(I+\Xi)^\top}\nonumber\\
		\leq & 3\norm{\Xi}\lambda_{\min}\left(\tilde{S}_{t+1}\tilde{S}_{t+1}^\top\right)\nonumber\\
		\leq & 36\bareta\delta\norm{\Delta_t}_F\lambda_{\min}\left(\tilde{S}_{t+1}\tilde{S}_{t+1}^\top\right).
	\end{align*}
	This in turn implies that
	\begin{align*}
		\lambda_{\min}({S}_{t+1}{S}_{t+1}^\top)&\geq \lambda_{\min}(M)\\
		&\geq (1-36\bareta\delta\norm{\Delta_t}_F)\lambda_{\min}\left(\tilde{S}_{t+1}\tilde{S}_{t+1}^\top\right)\\
		&\geq \left(\left(1+\bareta\sigma_r\right)^2-2\bareta\norm{E_tE_t^{\top}}-72\bareta\delta\norm{U_tU_t^\top-X^\star}_F\right)\lambda_{\min}\left(S_tS_t^{\top}\right)\nonumber\\
		&-2\bareta\left(1+\bareta\sigma_r\right)\lambda_{\min}\left(S_tS_t^{\top}\right)^2,
	\end{align*}
	which completes the proof.$\hfill\square$
	
	\subsection{Proof of Proposition~\ref{lem:dynamics_empirical}}\label{app_lem:dynamics_empirical}
	
	\paragraph{Bounding $\norm{\Sigma-S_{t+1}S_{t+1}^\top}$:}
	
	Recall that $S_{t+1}=\tilde{S}_{t+1}+V^{\top}\del U_t$. Therefore, we have
	\begin{equation}
		\begin{aligned}
			\Sigma-S_{t+1}S_{t+1}^{\top}=\Sigma-\tilde{S}_{t+1}\tilde{S}_{t+1}^{\top}\underbrace{-V^{\top}\del U_{t}\tilde{S}_{t+1}^{\top}-\tilde{S}_{t+1} U_{t}^{\top}\del^{\top}V-V^{\top}\del U_tU_t^{\top}\del^{\top}V}_{\text{deviation}}.
		\end{aligned}
		\label{eq::appendix-112}
	\end{equation}
	The main idea behind our proof is to first control the norm of the deviation term, and then invoke Proposition~\ref{lem:dynamics_pop} for $\Sigma-\tilde{S}_{t+1}\tilde{S}_{t+1}^{\top}$.
	One can write 
	\begin{equation}
		\begin{aligned}
			\norm{V^{\top}\del U_{t}\tilde{S}_{t+1}^{\top}}
			\leq \norm{\del}\norm{\tilde{S}_{t+1}U_t^{\top}}\leq 3\bareta\delta \norm{\Delta_t}_F\norm{\tilde{S}_{t+1}U_t^{\top}},
		\end{aligned}\nonumber
	\end{equation}
	where in the last inequality, we used Lemma~\ref{lem::del}. Moreover, $\tilde{S}_{t+1}U_t^{\top}$ can be rewritten as
	\begin{equation}\nonumber
		\begin{aligned}
			\tilde{S}_{t+1}U_t^{\top}&=\left(\left(I-\bareta \left(S_tS_t^{\top}-\Sigma\right)S_t+S_tE_t^{\top}E_t\right)\right)\left(S^{\top}V^{\top}+E^{\top}V_{\perp}^{\top}\right)\\
			&=\left(I-\bareta \left(S_tS_t^{\top}-\Sigma\right)\right)S_t S_t^{\top}V^{\top}+\left(I-\bareta \left(S_tS_t^{\top}-\Sigma\right)\right)S_t E_t^{\top}V_{\perp}^{\top}\\
			&+S_tE_t^{\top}E_t S_t^{\top}V^{\top}+S_tE_t^{\top}E_t E_t^{\top}V_{\perp}^{\top}.
		\end{aligned}
	\end{equation}
	Note that $\norm{I-\bareta\left(S_tS_t^{\top}-\Sigma\right)}\leq 2$ due to our choice of $\eta$ and our assumption on $S_tS_t^{\top}$. Therefore, we have
	\begin{equation}
		\begin{aligned}
			\norm{\tilde{S}_{t+1}U_t^{\top}}&\leq 2\norm{S_tS_t^{\top}}+2\norm{S_tE_t^{\top}}+\norm{S_tE_t^{\top}}^2+\norm{S_tE_t^{\top}}\norm{E_t}^2\leq 6\sigma_1, 
		\end{aligned}
	\end{equation}
	where we used our assumptions on $\norm{S_tS_t^{\top}}$ and $\norm{E_tE_t^{\top}}$. Combining the above two inequalities leads to
	\begin{equation}
		\norm{\tilde{S}_{t+1}U_{t}^\top \del^\top V} = \norm{V^{\top}\del U_{t}\tilde{S}_{t+1}^{\top}}\leq 18\bareta\delta\sigma_1\norm{\Delta_t}_F.
	\end{equation}
	Using a similar technique, we have
	\begin{equation}\label{eq_RV}
		\norm{V^{\top}\del U_tU_t^{\top}\del^{\top}V}\leq 9\bareta^2\delta^2 \sigma_1\norm{\Delta_t}_F^2\leq \bareta\delta\sigma_1\norm{\Delta_t}_F,
	\end{equation}
	where we used the assumed upper bounds on $\eta$ and $\delta$, and the fact that $\norm{E_tE_t^\top}\leq \sigma_1$, $\norm{S_tS_t^{\top}}\leq 1.01\sigma_1$, and $\norm{E_tE_t^\top}_F\leq \sqrt{r}\sigma_1$, which in turn implies $\norm{\Delta_t}_F\lesssim \sqrt{r}\sigma_1$ due to Lemma~\ref{lem::decomposition}. Moreover, we have already shown in Proposition~\ref{lem:dynamics_pop} that 
	\begin{equation}\nonumber
		\norm{\Sigma-\tilde S_{t+1}\tilde S_{t+1}^{\top}}\leq \left(1-\bareta\lambda_{\min}\left(S_tS_t^{\top}\right)\right)\norm{\Sigma-S_tS_t^{\top}}+5\bareta \norm{S_tE_t^\top }^2.
	\end{equation}
	Hence, combining the above inequalities leads to
	{\begin{equation}\nonumber
			\begin{aligned}
				\norm{\Sigma -S_{t+1}S_{t+1}^{\top}}&\leq \left(1-\bareta\lambda_{\min}\left(S_tS_t^{\top}\right)\right)\norm{\Sigma-S_tS_t^{\top}}+5\bareta \norm{S_tE_t^\top }^2+37\bareta\delta\sigma_1\norm{\Delta_t}_F.
			\end{aligned}
	\end{equation}}
	
	\paragraph{Bounding $\norm{S_{t+1}E_{t+1}^{\top}}$:} First, note that
	\begin{equation}\nonumber
		\begin{aligned}
			S_{t+1}E_{t+1}^{\top}&= \tilde{S}_{t+1}\tilde{E}_{t+1}^{\top}+\underbrace{V^{\top}\del U_t\tilde{E}_{t+1}^{\top}+\tilde{S}_{t+1}U_t^{\top}\del^{\top}V_{\perp}+V^{\top}\del U_tU_t^{\top}\del^{\top}V_{\perp}}_{\text{deviation}}.
		\end{aligned}
	\end{equation}
	Similar to the signal term, the main idea behind our proof is to first control the norm of the deviation term, and then invoke Proposition~\ref{lem:dynamics_pop} for $\tilde{S}_{t+1}\tilde{E}_{t+1}^{\top}$.
	We first provide an upper bound for $\norm{V^{\top}\del U_t\tilde{E}_{t+1}^{\top}}$ 
	\begin{equation}\nonumber
		\norm{V^{\top}\del U_t\tilde{E}_{t+1}^{\top}}\leq \norm{\del}\norm{U_t\tilde{E}_{t+1}^{\top}}\leq 3\bareta\delta\norm{\Delta_t}_F\norm{U_t\tilde{E}_{t+1}^{\top}}.
	\end{equation}
	To bound $\norm{U_t\tilde{E}_{t+1}^{\top}}$, one can write
	\begin{equation}\nonumber
		\begin{aligned}
			\norm{U_t\tilde{E}_{t+1}^{\top}}&\leq \norm{U_t}\norm{E_t-\bareta S_t^{\top}S_tE_t^{\top}-\bareta E_tE_t^{\top}E_t}\\
			&\leq 2\left(\norm{S_t}+\norm{E_t}\right)\norm{E_t}\\
			&\leq 5{\sigma_1},
		\end{aligned}
	\end{equation}
	where we used the assumption $\norm{S_t}\leq 1.01\sqrt{\sigma_1}$ and $\norm{E_t}\leq \sqrt{\sigma_1}$. Similar to~\eqref{eq_RV}, we have
	\begin{align*}
		\norm{\tilde{S}_{t+1}U_t^{\top}\del^{\top}V_{\perp}} &\leq \norm{\tilde{S}_{t+1}U_t^{\top}}\norm{\del}\\
		&\leq \norm{S_t-\bareta\left(\left(S_tS_t^\top-\Sigma\right)S_t+S_tE_t^\top E_t\right)}\norm{\del}\\
		&\leq 6\bareta\delta\sqrt{\sigma_1}\norm{\Delta_t}_F,
	\end{align*}
	and
	\begin{equation}\nonumber
		\norm{V^{\top}\del U_tU_t^{\top}\del^{\top}V_{\perp}}\leq \norm{U_tU_t^{\top}}\norm{\del}^2\leq 10\sigma_1 \bareta^2\delta^2\norm{\Delta_t}_F^2\leq \bareta\delta\sigma_1\norm{\Delta_t}_F.
	\end{equation}
	Moreover, we have already shown in Proposition~\ref{lem:dynamics_pop} that 
	\begin{equation}
		\norm{\tilde{S}_{t+1}\tilde{E}_{t+1}^{\top}}\leq \left(1-\bareta \lambda_{\min}\left(S_tS_t^{\top}\right)+2\bareta \norm{\Sigma-S_tS_t^{\top}}+2\bareta\norm{E_tE_t}\right)\norm{S_{t}E_{t}^{\top}}.
		\label{eq::41}
	\end{equation}
	Combining the above inequalities leads to
	{\begin{equation}\nonumber
			\begin{aligned}
				\norm{S_{t+1}E_{t+1}^{\top}}&\leq \left(1-\bareta \lambda_{\min}\left(S_tS_t^{\top}\right)+2\bareta \norm{\Sigma-S_tS_t^{\top}}+2\bareta\norm{E_tE_t}\right)\norm{S_{t}E_{t}^{\top}}+22\bareta\delta\sigma_1\norm{\Delta_t}_F.
			\end{aligned}
		\end{equation}
	}
	\paragraph{Bounding $\norm{S_{t+1}S_{t+1}^{\top}}$:} First, note that 
	\begin{equation}\nonumber
		S_{t+1}S_{t+1}^{\top}=\tilde{S}_{t+1}\tilde{S}_{t+1}^{\top}+\underbrace{V^{\top}\del U_t S_{t+1}^{\top}+S_{t+1}U_t^{\top}\del^{\top}V+V^{\top}\del U_tU_t^{\top}\del^{\top}V}_{\text{deviation}}.
	\end{equation}
	Similar to our previous arguments, we will provide an upper bound on the deviation term, and then resort to Proposition~\ref{lem:dynamics_pop} to provide an upper bound for $\tilde{S}_{t+1}\tilde{S}_{t+1}^{\top}$. First, note that
	\begin{align}
		\norm{V^{\top}\del U_t S_{t+1}^{\top}}&\leq\norm{V^{\top}\del V S_tS_{t+1}^{\top}}+\norm{V^{\top}\del V_\perp E_tS_{t+1}^{\top}}\nonumber\\
		&\leq 3\bareta\delta\norm{\Delta_t}_F\left(\norm{S_tS_{t+1}^{\top}}+\norm{E_tS_{t+1}^{\top}}\right), \label{eq_VR2}
	\end{align}
	where we used Lemma~\ref{lem::del} in the last inequality. On the other hand
	\begin{equation}\nonumber
		\begin{aligned}
			\norm{S_tS_{t+1}^{\top}}&=\norm{S_t\left(S_t^{\top}-\bareta\left(S_t^{\top}\left(S_tS_t^{\top}-\Sigma\right)+E_t^{\top}E_tS_t^{\top}\right)+U_t^{\top}\del^{\top}V\right)}\\
			&\leq \norm{S_tS_t^{\top}}\left(1+\bareta \norm{S_tS_t^{\top}-\Sigma}+\bareta\norm{E_tE_t^{\top}}\right)+\norm{S_tU_t^{\top}\del^{\top}V}\\
			&\leq 2\norm{S_tS_t^{\top}} + \norm{S_tS_t^{\top}V^{\top}\del^{\top}V}+\norm{S_tE_t^{\top}V_{\perp}^{\top}\del^{\top}V}\\
			&\leq 3\norm{S_tS_t^{\top}}+3\bareta\delta\norm{\Delta_t}_F\norm{S_tE_t^{\top}}.
		\end{aligned}
	\end{equation}
	Similarly, we have
	\begin{equation}\nonumber
		\norm{E_tS_{t+1}^{\top}}\leq 3\norm{E_tS_t^{\top}}+3\bareta\delta\norm{\Delta_t}_F\norm{E_tE_t^{\top}}.
	\end{equation}
	The above two inequalities combined with~\eqref{eq_VR2} results in
	\begin{align*}
		2\norm{V^{\top}\del U_t S_{t+1}^{\top}}\leq 216\bareta\sigma_1^2\sqrt{r}\delta.
	\end{align*}
	which follows from our assumption on $\norm{S_t}$, $\norm{E_t}$ and $\delta$. Similarly, one can show that
	\begin{align*}
		\norm{V^{\top}\del U_tU_t^{\top}\del^{\top}V}\leq \bareta\sigma_1^2\sqrt{r}\delta.
	\end{align*}
	Therefore, the norm of the deviation term is upper bounded by $217\bareta\sigma_1^2\sqrt{r}\delta$. Moreover, Proposition~\ref{lem:dynamics_pop} implies
	\begin{align*}
		& \norm{\tilde S_{t+1}\tilde S_{t+1}^\top}\leq \left(1+2\sigma_1\bareta+4\sigma_1^2\bareta^2\right)\norm{S_tS_t^{\top}} -2\bareta \norm{S_tS_t^{\top}}^2+\bareta^2 \norm{S_tS_t^{\top}}^3,\nonumber
	\end{align*}
	which in turn leads to
	\begin{align*}
		\norm{ S_{t+1} S_{t+1}^\top}\leq \left(1+2\sigma_1\bareta+4\sigma_1^2\bareta^2\right)\norm{S_tS_t^{\top}} -2\bareta \norm{S_tS_t^{\top}}^2+\bareta^2 \norm{S_tS_t^{\top}}^3+217\bareta\sigma_1^2\sqrt{r}\delta.
	\end{align*}
	The rest of the proof is analogous to that of Proposition~\ref{lem:dynamics_pop}, and hence, omitted for brevity.$\hfill\square$

	\subsection{Proof of Proposition~\ref{prop::F-G}}\label{app_prop::F-G}
	{
		\paragraph{Bounding $\norm{G_{t+1}}$:}
		To provide an upper bound for $\norm{G_{t+1}} = \norm{E_{t+1}\proj_{S_{t+1}}^\perp}$ in terms of $\norm{G_{t}} = \norm{E_{t}\proj_{S_{t}}^\perp}$, it is crucial to characterize the relationship between the projection operators $\proj_{S_{t+1}}^\perp$ and $\proj_{S_{t}}^\perp$. To this goal, we decompose $\proj_{S_{t+1}}^\perp$ as follows
		\begin{align*}
			\proj_{S_{t+1}}^\perp = \proj_{S_{t}}\proj_{S_{t+1}}^\perp+\proj_{S_{t}}^\perp\proj_{S_{t+1}}^\perp.
		\end{align*}
		Based on the above decomposition, $G_{t+1}$ can be written as
		\begin{align*}
			G_{t+1} = \underbrace{E_{t+1}\proj_{S_{t}}\proj_{S_{t+1}}^\perp}_{(A)} + \underbrace{E_{t+1}\proj_{S_{t}}^\perp\proj_{S_{t+1}}^\perp}_{(B)}.
		\end{align*}
		We first study $(A)$. Let $M_tD_tN_t^\top$ be the singular value decomposition of $S_t$, where $M_t\in\mathbb{R}^{r\times r}$ and $N_t\in\mathbb{R}^{r'\times r}$ are (row) orthonormal matrices, and $D_t\in\mathbb{R}^{r\times r}$ is a diagonal matrix collecting the singular values of $S_t$. Based on this definition, we have $\proj_{S_t}=N_tN_t^{\top}$. 
		On the other hand, we have $S_{t+1}\proj_{S_t}\proj_{S_{t+1}}^{\perp}=-S_{t+1}\proj_{S_t}^{\perp}\proj_{S_{t+1}}^{\perp}$, which is equivalent to
		\begin{equation}
			S_{t+1}N_tN_t^{\top}\proj_{S_{t+1}}^{\perp}=-S_{t+1}\proj_{S_{t}}^{\perp}\proj_{S_{t+1}}^{\perp}.
			\label{eq::appendix-156}
		\end{equation}
		Our next technical lemma shows that $S_{t+1}N_t$ is invertible. 
		\begin{lemma}
			The matrix $S_{t+1}N_t$ is invertible.
			\label{claim::appendix-1}
		\end{lemma}
		The proof of this lemma can be found in Appendix~\ref{app_claim::appendix-1}. Lemma~\eqref{claim::appendix-1} combined with~\eqref{eq::appendix-156} leads to
		$$
		N_t^{\top}\proj_{S_{t+1}}^{\perp} = -\left(S_{t+1}N_t\right)^{-1}S_{t+1}\proj_{S_{t}}^{\perp}\proj_{S_{t+1}}^{\perp}.
		$$
		Therefore, we have
		\begin{align*}
			(A) =& E_{t+1}\proj_{S_{t}}\proj_{S_{t+1}}^\perp\\
			=& E_{t+1}N_tN_t^\top \proj_{S_{t+1}}^\perp\\
			=& -E_{t+1}N_t\left(S_{t+1}N_t\right)^{-1}S_{t+1}\proj_{S_{t}}^{\perp}\proj_{S_{t+1}}^{\perp}\\
			=& \underbrace{-E_{t+1}N_t\left(S_{t+1}N_t\right)^{-1}\tilde S_{t+1}\proj_{S_{t}}^{\perp}\proj_{S_{t+1}}^{\perp}}_{(A_1)}-\underbrace{E_{t+1}N_t\left(S_{t+1}N_t\right)^{-1}\left(S_{t+1} - \tilde S_{t+1}\right)\proj_{S_{t}}^{\perp}\proj_{S_{t+1}}^{\perp}}_{(A_2)}.
		\end{align*}
		We first control $(A_1)$. Observe that $$\tilde{S}_{t+1}\proj^{\perp}_{S_t}=\left(S_t-\bareta \left(S_tS_t^{\top}-\Sigma\right)S_t-\bareta S_tE_t^{\top}E_t\right)\proj^{\perp}_{S_t}=-\bareta S_tE_t^{\top}G_t.$$ 
		Therefore, we have
		\begin{align}\label{eq_A1}
			(A_1)&=\bareta E_{t+1}N_t\left(S_{t+1}N_t\right)^{-1}S_tE_t^{\top}G_t\proj_{S_{t+1}}^{\perp}. 
		\end{align}
		On the other hand, note that $E_{t+1} = E_t-\bareta E_t\left(S_t^{\top}S_t+E_t^{\top}E_t\right)+V_\perp^\top R_tU_t$.
		Hence, we have
		\begin{equation}\label{eq_BB}
			(B)=\left(I-\bareta E_tE_t^{\top}\right)G_t\proj_{S_{t+1}}^{\perp}+V_{\perp}^{\top}\del V_{\perp}G_t\proj_{S_{t+1}}^{\perp}.
		\end{equation}
		Combining equations~\eqref{eq_A1} and~\eqref{eq_BB}, we obtain
		\begin{equation}\nonumber
			G_{t+1}=\left(I-\bareta E_tE_t^{\top}+\bareta E_{t+1}N_t\left(S_{t+1}N_t\right)^{-1}S_tE_t^{\top}\right)G_t\proj_{S_{t+1}}^{\perp}+V_{\perp}^{\top}\del V_{\perp}G_t\proj_{S_{t+1}}^{\perp}+(A_2),
		\end{equation}
		which results in
		\begin{equation}\nonumber
			\begin{aligned}
				\norm{G_{t+1}}&\leq\underbrace{ \norm{I-\bareta E_tE_t^{\top}+\bareta E_{t+1}N_t\left(S_{t+1}N_t\right)^{-1}S_tE_t^{\top}}}_{(C)}\norm{G_t}+\norm{\del}\norm{G_t}+\norm{(A_2)}.
			\end{aligned}
		\end{equation}
		Therefore, it remains to control the terms $\norm{(C)}$ and $\norm{(A_2)}$.
		First, we provide an upper bound on $\norm{(C)}$. Define $S_{t+1}N_t=\left(I+\Xi\right)S_{t}N_t$, where $\Xi=S_{t+1}N_t\left(S_{t}N_t\right)^{-1}-I$ (note that $S_{t}N_t = M_tD_t$ which implies that $S_{t}N_t$ is invertible). Hence, we have
		\begin{equation}
			\begin{aligned}
				\norm{(C)}&\leq \underbrace{\norm{I-\bareta E_tE_t^{\top}+\bareta E_{t+1}N_t\left(S_{t}N_t\right)^{-1}S_tE_t^{\top}}}_{(C_1)}+\underbrace{\bareta\norm{E_{t+1}N_t\left(S_tN_t\right)^{-1}}\norm{(I+\Xi)^{-1}-I}\norm{S_tE_t^{\top}}}_{(C_2)}.
			\end{aligned}
		\end{equation}
		It is shown in the proof of Lemma~\ref{claim::appendix-1} that $\norm{\Xi}\leq 3\bareta \norm{\Delta_t}\leq 1/2$. Therefore, one can write $\norm{(I+\Xi)^{-1}-I}\leq \norm{\Xi}\norm{\left(I+\Xi\right)^{-1}}\leq 6\bareta \norm{\Delta_t}$. To provide an upper bound for $\norm{E_{t+1}N_t\left(S_tN_t\right)^{-1}}$, one can write
		\begin{align*}
			E_{t+1}N_t\left(S_{t}N_t\right)^{-1}&=E_{t+1}S_{t}^{\top}\left(S_tS_t^{\top}\right)^{-1}\\
			&=\left(E_t-\bareta E_t\left(S_t^{\top}S_t+E_t^{\top}E_t\right)+V_{\perp}^{\top}\del U_t\right)S_{t}^{\top}\left(S_tS_t^{\top}\right)^{-1}\\
			&=\left(I-\bareta E_tE_t^{\top}\right)H_t -\bareta E_tS_t^{\top}+V_{\perp}^{\top}\del U_tS_{t}^{\top}\left(S_tS_t^{\top}\right)^{-1}\\
			&=\left(I-\bareta E_tE_t^{\top}\right)H_t -\bareta E_tS_t^{\top}+V_{\perp}^{\top}\del V+V_{\perp}^{\top}\del V_{\perp} H_t,
		\end{align*}
		where $H_t = E_tS_t^\top\left(S_tS_t^\top\right)^{-1}$. Therefore, we have
		\begin{align*}
			\norm{E_{t+1}N_t\left(S_{t}N_t\right)^{-1}}&\leq \norm{H_t}+\bareta\norm{E_tS_t^\top}+\norm{R_t}(1+\norm{H_t})\\
			&\leq \norm{H_t}+\bareta\norm{E_tS_t^\top}+4\bareta\delta\norm{\Delta_t}_F.
		\end{align*} 
		Hence, we obtain
		\begin{align}\nonumber
			\norm{C_2}\leq 6\bareta^2\norm{\Delta_t}\left(\norm{H_t}+\bareta\norm{E_tS_t^{\top}}+\bareta\delta\norm{\Delta_t}_F\right)\norm{E_tS_t^{\top}}.
		\end{align}
		To control $\norm{(C_1)}$, we use triangle inequality to arrive at the following decomposition
		\begin{equation}\nonumber
			\begin{aligned}
				(C_1)&\leq \underbrace{\norm{I-\bareta E_tE_t^{\top}+\bareta E_{t}N_t\left(S_{t}N_t\right)^{-1}S_tE_t^{\top}}}_{(C_{11})}+\underbrace{\bareta^2\norm{S_tS_t^{\top}E_tS_t^{\top}\left(S_tS_t^{\top}\right)^{-1}S_tE_t^{\top}}}_{(C_{12})}\\
				&+\underbrace{\bareta^2\norm{E_tE_t^{\top}E_tS_t^{\top}\left(S_tS_t^{\top}\right)^{-1}S_tE_t^{\top}}}_{(C_{13})}+\underbrace{\bareta\norm{\del U_t N_t\left(S_{t}N_t\right)^{-1}S_tE_t^{\top}}}_{(C_{14})}.
			\end{aligned}
		\end{equation}
		It is easy to see that 
		$\norm{(C_{11})}=\norm{I-\bareta G_tG_t^{\top}}\leq 1$, due to the assumed upper bound on $\bareta$ and $\norm{E_tE_t^\top}$. Moreover, one can verify that $\norm{(C_{12})}\leq \bareta^2 \norm{E_tS_t^{\top}}^2$, $\norm{(C_{13})}\leq \bareta^2 \norm{E_t}^4$, and $\norm{(C_{14})}\leq 3\bareta^2 \delta \norm{\Delta_t}_F\left(1+\norm{H_t}\right)\norm{S_tE_t^{\top}}\leq 6\bareta^2 \delta \norm{\Delta_t}_F\norm{S_tE_t^{\top}}$. Combining the derived upper bounds for $\norm{C_1}$ and $\norm{C_2}$, we have
		\begin{equation}\nonumber
			(C)\leq 1+\bareta^2 \left(2\norm{E_tS_t^{\top}}^2+\norm{E_t}^4+7\delta \norm{\Delta_t}_F\norm{E_tS_t^{\top}}+6\norm{H_t}\norm{\Delta_t}\norm{E_tS_t^{\top}}\right).
		\end{equation}
		To complete the proof, it remains to provide an upper bound for $\norm{(A_2)}$. First, note that $\left(S_{t+1}-\tilde{S}_{t+1}\right)\proj_{S_t}^{\perp}=V^\top \del V_{\perp}G_t$. Given this equality, one can write
		\begin{align*}
			\norm{A_2}\leq& \norm{E_{t+1}N_t\left(S_{t+1}N_t\right)^{-1}V^\top \del V_{\perp}G_t\proj_{S_{t+1}}^{\perp}}\\
			\leq& \norm{E_{t+1}N_t\left(S_{t+1}N_t\right)^{-1}}\norm{\del}\norm{G_t}\\
			\leq& 3\bareta\delta\norm{E_{t+1}N_t\left(S_{t+1}N_t\right)^{-1}}\norm{\Delta_t}_F\norm{G_t}.
		\end{align*}
		Therefore, it suffices to provide an upper bound for $\norm{E_{t+1}N_t\left(S_{t+1}N_t\right)^{-1}}$:
		\begin{align*}
			\norm{E_{t+1}N_t\left(S_{t+1}N_t\right)^{-1}} &\leq \norm{E_{t+1}N_t\left(S_{t}N_t\right)^{-1}} \norm{S_{t}N_t\left(S_{t+1}N_t\right)^{-1}}\\
			&\leq \norm{E_{t+1}N_t\left(S_{t}N_t\right)^{-1}} \norm{(I+\Xi)^{-1}}\\
			&\leq 2\norm{E_{t+1}N_t\left(S_{t}N_t\right)^{-1}}\\
			&\leq 2.
		\end{align*}
		Combining the above inequalities leads to 
		\begin{align*}
			\norm{(A_2)}&\leq 6\bareta\delta\norm{\Delta_t}_F\norm{G_t}.
		\end{align*}
		Finally, combining the derived upper bounds for $(C)$ and $(A_2)$ gives rise to the following inequalities
		\begin{align*}
			\norm{G_{t+1}}&\leq \left(1+2\bareta^2\norm{E_tS_t^{\top}}^2+\bareta^2\norm{E_t}^4+6\bareta^2\norm{H_t} \norm{\Delta_t}\norm{E_tS_t^{\top}}+7\bareta\delta \norm{\Delta_t}_F\right)\norm{G_t}\\
			&\leq \left(1+2\bareta^2\norm{E_tS_t^{\top}}^2+\bareta^2\norm{E_t}^4+2\bareta^2 \norm{\Delta_t}\norm{E_tS_t^{\top}}+7\bareta\delta \norm{\Delta_t}_F\right)\norm{G_t}.
		\end{align*}
	}		
	
	\paragraph{Bounding $\norm{F_{t+1}}:$} Similar to the previous part, we use the decomposition $\proj_{S_{t+1}} = \proj_{S_{t}}\proj_{S_{t+1}}+\proj_{S_{t}}^\perp\proj_{S_{t+1}}$ to write
	\begin{align*}
		F_{t+1} = E_{t+1}\proj_{S_{t+1}} = \underbrace{E_{t+1}\proj_{S_{t}}\proj_{S_{t+1}}}_{(A)}+\underbrace{E_{t+1}\proj_{S_{t}}^\perp\proj_{S_{t+1}}}_{(B)}.
	\end{align*}
	First, we provide an upper bound for $\norm{(B)}$:
	\begin{align}
		\norm{{({B})}}&=\norm{E_{t+1}\proj_{S_{t}}^{\perp}\left(\proj_{S_{t+1}}-\proj_{S_{t}}\right)}\nonumber\\
		&\leq \norm{E_{t+1}\proj_{S_{t}}^{\perp}}\norm{\proj_{S_{t+1}}-\proj_{S_{t}}}\nonumber\\
		&=\norm{\left(I-\bareta E_tE_t^{\top}\right)G_t+V^{\top}_{\perp}\del U_t\proj_{S_{t}}^{\perp}}\norm{\proj_{S_{t+1}}-\proj_{S_{t}}}\nonumber\\
		&\leq \left(\norm{G_t}+\norm{V^{\top}_{\perp}\del V_{\perp}G_t}\right)\norm{\proj_{S_{t+1}}-\proj_{S_{t}}}\nonumber\\
		&\leq \left(\norm{G_t}+3\bareta\delta\norm{\Delta_t}_F\norm{G_t}\right)\norm{\proj_{S_{t+1}}-\proj_{S_{t}}}\nonumber\\
		&\leq 2\norm{G_t}\norm{\proj_{S_{t+1}}-\proj_{S_{t}}}.\label{eq_B}
	\end{align}
	To control $\norm{\proj_{S_{t+1}}-\proj_{S_{t}}}$, we use the following technical lemma.
	\begin{lemma}[Theorem 2.4, \citet{chen2016perturbation}]
		\label{lem::appendix-projection}
		Let $A\in \R^{m\times n}$, and $B=A+E\in \R^{m\times n}$ have the same rank. Then, we have
		\begin{equation}\nonumber
			\norm{\proj_A-\proj_B}\leq \norm{EA^\dagger}\vee \norm{EB^\dagger}.
		\end{equation}
	\end{lemma}
	Due to Proposition~\ref{prop:min_eig_empirical} and our assumptions, we have $\lambda_{\min}(S_{t+1}S_{t+1}^\top)\geq \lambda_{\min}(S_{t}S_{t}^\top)>0$, and hence, both $S_{t+1}S_{t+1}^\top$ and $S_{t}S_{t}^\top$ are rank-$r$. Invoking Lemma~\ref{lem::appendix-projection}, we have
	\begin{equation}\label{eq_Ps_diff}
		\begin{aligned}
			\norm{\proj_{S_{t+1}}-\proj_{S_{t}}}&\leq \norm{\left(S_{t+1}-S_t\right)S_t^{\top}\left(S_tS_t^{\top}\right)^{-1}}\\
			&=\norm{\left(-\bareta\left(\left(S_tS_t^{\top}-\Sigma\right)S_t+S_tE^{\top}_tE_t\right)+V^{\top}\del U_t\right)S_t^{\top}\left(S_tS_t^{\top}\right)^{-1}}\\
			&\leq \bareta \norm{S_tS_t^{\top}-\Sigma}+\bareta \norm{S_tE_t^{\top}}\norm{H_t}+\norm{\del}\left(1+\norm{H_t}\right)\\
			&\leq \bareta \norm{S_tS_t^{\top}-\Sigma}+\bareta \norm{S_tE_t^{\top}}+6\bareta\delta\norm{\Delta_t}_F\\
			&\leq 3\bareta\norm{\Delta_t},
		\end{aligned}
	\end{equation}
	where in the last inequality, we used the following auxiliary lemma.
	\begin{lemma}
		\label{lem::9}
		If $\norm{\Delta_t}\geq \sqrt{d}\norm{G_t}^2$, then we have $\norm{\Delta_t}_F\leq 5\sqrt{r}\norm{\Delta_t}$.
	\end{lemma}
	\begin{proof}
		Recall the signal-residual decomposition
		\begin{equation}\nonumber
			\Delta_t=V\left(S_tS_t^{\top}-\Sigma\right)V^{\top}+VS_tE_t^{\top}V_{\perp}^{\top}+V_{\perp}E_tS_t^{\top}V^{\top}+V_{\perp}F_tF_t^{\top}V_{\perp}^{\top}+V_{\perp}G_tG_t^{\top}V_{\perp}^{\top}.
		\end{equation}
		One can write
		\begin{equation}\nonumber
			\begin{aligned}
				\norm{\Delta_t}_F&\leq \sqrt{r}\left(\norm{S_tS_t^{\top}-\Sigma}+2\norm{S_tE_t^{\top}}+\norm{F_tF_t^{\top}}\right)+\sqrt{d}\norm{G_t}^2\\
				&=4\sqrt{r}\norm{\Delta_t}+\norm{\Delta_t}\\
				&\leq 5\sqrt{r}\norm{\Delta_t},
			\end{aligned}
		\end{equation}
		which completes the proof.
	\end{proof}
	Combining~\eqref{eq_Ps_diff} with~\eqref{eq_B} leads to
	\begin{align*}
		\norm{(B)}\leq 6\bareta\norm{\Delta_t}\norm{G_t}.
	\end{align*}
	Next, we will provide an upper bound for $(A)$. One can write
	\begin{align}
		\norm{\left({A}\right)}&\leq \norm{E_{t+1}\proj_{S_t}}\nonumber\\
		&\leq\norm{\left(E_t-\bareta E_t\left(E_t^{\top}E_t+S^{\top}_tS_t\right)+V^{\top}_{\perp}\del U_t\right)\proj_{S_t}}\nonumber\\
		&\leq\norm{F_t-\bareta E_tE_t^{\top}F_t-\bareta F_t S^{\top}_tS_t}+\norm{V^{\top}_{\perp}\del U_t\proj_{S_t}}.\label{eq_A2}
	\end{align}
	The first term in the above inequality can be bounded as follows
	\begin{align*}
		\norm{F_t-\bareta E_tE_t^{\top}F_t-\bareta F_t S^{\top}_tS_t}&\leq \norm{\left(0.5I-\bareta E_tE_t^{\top}\right)F_t}+\norm{F_t\left(0.5I-\bareta S^{\top}_tS_t\right)}\\
		&\leq \left(\norm{0.5I-\bareta E_tE_t^{\top}}+\norm{0.5I-\bareta S^{\top}_tS_t}\right)\norm{F_t}\\
		&\leq \left(1-\bareta \lambda_{\min}\left(S_tS_t^{\top}\right)\right)\norm{F_t}.
	\end{align*}
	Moreover, we have
	\begin{align*}
		\norm{V^{\top}_{\perp}\del U_t\proj_{S_t}}\leq \norm{\del}\left(\norm{S_t}+\norm{F_t}\right)\leq 3\bareta\delta\norm{\Delta_t}_F\left(\norm{S_t}+\norm{F_t}\right).
	\end{align*}
	Therefore, we have
	\begin{align*}
		\norm{\left({A}\right)}\leq \left(1-\bareta\lambda_{\min}\left(S_tS_t^{\top}\right)+3\bareta\delta\norm{\Delta_t}_F\right)\norm{F_t}+3\bareta\delta \norm{\Delta_t}_F\norm{S_t}.
	\end{align*}
	Finally, combining the derived upper bounds for $(A)$ and $(B)$ leads to
	{
		\begin{align*}
			\norm{F_{t+1}}\leq & \left(1-\bareta\lambda_{\min}\left(S_tS_t^{\top}\right)+3\bareta\delta\norm{\Delta_t}_F\right)\norm{F_t}+3\bareta\delta \norm{\Delta_t}_F\norm{S_t}+6\bareta\norm{\Delta_t}\norm{G_t},
		\end{align*}
	}
	{which completes the proof}.$\hfill\square$
	
	\subsection{Proof of Lemma~\ref{prop::spec-init}}
	\label{app_prop::spec-init}
	\begin{sloppypar}
		Recall that $U_0=\alpha B$, where $BB^{\top}$ is the best rank-$r'$ approximation of $C\in \frac{1}{2\measurementnumber}\sum_{i=1}^{\measurementnumber}\sign\left(y_i\right)\left({A_i+A_i^{\top}}\right)$. Since $\rank(X^{\star})=r$, Sign-RIP implies
	\end{sloppypar}			
	\begin{equation}\nonumber
		\norm{C-\varphi(X^{\star})\frac{X^{\star}}{\norm{X^{\star}}_F}}\leq \varphi(X^{\star})\delta.
	\end{equation}
	Note that $BB^{\top}$ is the best rank-$r'$ approximation of $C$. We have
	\begin{equation}\nonumber
		\begin{aligned}
			\norm{BB^{\top}-\varphi(X^{\star})\frac{X^{\star}}{\norm{X^{\star}}_F}}&\leq \norm{BB^{\top}-C}+\norm{C-\varphi(X^{\star})\frac{X^{\star}}{\norm{X^{\star}}_F}}\\
			&\leq \left|\lambda_{r'+1}(C)\right|+\varphi(X^{\star})\delta\\
			&\leq \norm{C-\varphi(X^{\star})\frac{X^{\star}}{\norm{X^{\star}}_F}}+ \varphi(X^{\star})\delta\\
			&\leq 2\varphi(X^{\star})\delta.
		\end{aligned}
	\end{equation}
	Therefore, based on the definition of $U_0$, we have
	\begin{equation}\nonumber
		\norm{U_0U_0^{\top}-\alpha^2\varphi(X^{\star})\frac{X^{\star}}{\norm{X^{\star}}_F}}\leq 2\alpha^2\varphi(X^{\star})\delta.
	\end{equation}
	Given this bound, one can write
	\begin{equation}\nonumber
		\begin{aligned}
			\norm{S_0S_0^{\top}-\alpha^2\varphi(X^{\star})\frac{\Sigma}{\norm{X^{\star}}_F}}&=\norm{V^{\top}\left(U_0U_0^{\top}-\alpha^2\varphi(X^{\star})\frac{X^{\star}}{\norm{X^{\star}}_F}\right)V}\\
			&\leq \norm{U_0U_0^{\top}-\alpha^2\varphi(X^{\star})\frac{X^{\star}}{\norm{X^{\star}}_F}}\\
			&\leq 2\alpha^2 \varphi(X^{\star})\delta.
		\end{aligned}
	\end{equation}
	Similarly, we have
	\begin{equation}\nonumber
		\begin{aligned}
			\norm{S_0E_0^{\top}}&=\norm{V^{\top}\left(U_0U_0^{\top}-\alpha^2\varphi(X^{\star})\frac{X^{\star}}{\norm{X^{\star}}_F}\right)V_{\perp}}\leq 2\alpha^2 \varphi(X^{\star})\delta,\\ \norm{E_0E_0^{\top}}&=\norm{V_{\perp}^{\top}\left(U_0U_0^{\top}-\alpha^2\varphi(X^{\star})\frac{X^{\star}}{\norm{X^{\star}}_F}\right)V_{\perp}}\leq 2\alpha^2 \varphi(X^{\star})\delta.
		\end{aligned}
	\end{equation}
	This completes the proof.$\hfill\square$

	\subsection{Proof of Lemma~\ref{lem:conditions_empirical}}\label{app_lem:conditions_empirical}
	We prove this lemma by induction on $t$. First, due to Lemma~\ref{prop::spec-init}, it is easy to verify that~\eqref{eq_gen_error_emp}-\eqref{eq_Ht_emp} hold for $t=0$. Now, suppose that~\eqref{eq_gen_error_emp}-\eqref{eq_psd_emp} are satisfied for $t< T_{end}$. Moreover, without loss of generality, we assume that $\norm{\Delta_t}_F\gtrsim d\alpha^{2-\mathcal{O}(\sqrt{r}\kappa^2\delta)}$ for every $0\leq t\leq T_{end}$; otherwise, the statement of the lemma holds. Together with the induction hypothesis on $\norm{G_t}$, this implies that $\norm{\Delta_t}_F\geq \zeta$, for $\zeta>0$ defined in Propositions~\ref{prop:min_eig_empirical},~\ref{lem:dynamics_empirical}, and~\ref{prop::F-G}.
	
	\paragraph{Bounding $\norm{F_{t+1}}$:}  
	In order to apply Proposition~\ref{prop::F-G}, first we verify its assumptions. One can write $\norm{E_tE_t^\top} = \norm{F_t}^2+\norm{G_t}^2$. Therefore, we have $\norm{E_tE_t^\top}\leq \sigma_1$, due to the induction hypothesis on $\norm{F_t}$ and $\norm{G_t}$. On the other hand, $\norm{S_tS_t^\top}\leq 1.01\sigma_1$ and $\norm{E_tS_t^\top(S_tS_t^\top)^{-1}}\leq 1/3$ due to our induction hypothesis. It remains to show that $(4r,\delta, \err, \cS)$-Sign-RIP with $\err\geq \sqrt{d}\bar{\varphi}\alpha$ implies $(4r,\delta, \sqrt{d}\norm{G_t}^2, \cS)$-Sign-RIP. To show this, first note that $(4r,\delta, \err_1, \cS)$-Sign-RIP implies $(4r,\delta, \err_2, \cS)$-Sign-RIP, for any $\err_2\leq \err_1$. Using this fact, it suffices to show that $\sqrt{d}\norm{G_t}^2\leq \sqrt{d}\bar{\varphi}\alpha\lesssim \err$, which is immediate due to our induction hypothesis on $\norm{G_t}$, and our choice of $\delta$. Therefore, Proposition~\ref{prop::F-G} holds and we have
	{
		\begin{align}\label{eq_F_app}
			\norm{F_{t+1}}\leq & \left(1-\bareta\lambda_{\min}\left(S_tS_t^{\top}\right)+3\bareta\delta\norm{\Delta_t}_F\right)\norm{F_t}+3\bareta\delta \norm{\Delta_t}_F\norm{S_t}+6\bareta\norm{\Delta_t}\norm{G_t}.
		\end{align}
	}
	\noindent Due to our induction hypothesis,~\eqref{eq_F_app} can be simplified as
	\begin{align*}
		\norm{F_{t+1}}\leq & \left(1+75\sqrt{r}\sigma_1\bareta\delta\right)\norm{F_t}+80\sqrt{r}\sigma_1^{1.5}\bareta\delta+30\sigma_1\bareta\norm{G_t}\\
		\leq & \norm{F_t}+100\sqrt{r}\sigma_1^{1.5}\bareta\delta+30\sigma_1\bareta\norm{G_t}\\
		\leq & \norm{F_t}+100\sqrt{r}\sigma_1^{1.5}\bareta\delta+30\sigma_1\bareta \sqrt{\alpha\bar\varphi\delta}\\
		\leq& \eta\bar\varphi^2\left(100\sqrt{r}\sigma_1^{1.5}\delta +30\sigma_1 \sqrt{\alpha\bar\varphi\delta}\right)(t+2),
	\end{align*}
	where in the first inequality, we use Lemma~\ref{lem::9} and the induction hypothesis on $\norm{\Delta_t}$. Moreover, in the last inequality, we used the induction hypothesis on $\norm{G_t}$.

	\paragraph{Bounding $\norm{S_{t+1}S_{t+1}^\top}$:} Similar to our previous argument, it is easy to verify that the assumptions of Proposition~\ref{lem:dynamics_empirical} are satisfied at iteration $t$. Therefore, $\norm{S_{t+1}S_{t+1}^\top}\leq 1.01\sigma_1$ readily follows from Proposition~\ref{lem:dynamics_empirical}.
	
	\paragraph{Bounding $\norm{U_{t+1}U_{t+1}^\top-X^\star}$:} Note that 
	\begin{align*}
		\norm{U_{t+1}U_{t+1}^\top-X^\star}\leq \norm{\Sigma -S_{t+1}S_{t+1}^{\top}}+2\norm{S_{t+1}E_{t+1}^{\top}}+\norm{E_{t+1}E_{t+1}^{\top}}.
	\end{align*}
	On the other hand, we have
	\begin{align}\nonumber
		\norm{E_{t+1}E_{t+1}^{\top}}\leq \norm{F_{t+1}}^2+\norm{G_{t+1}}^2\leq 0.5\sigma_1.
	\end{align}
	where the last inequality is due to our choice of $\delta$ and $\alpha$. Similarly, we can show that $\norm{S_{t+1}E_{t+1}^{\top}}\leq \sigma_1$. This together with $\norm{\Sigma -S_{t+1}S_{t+1}^{\top}}\leq \norm{S_{t+1}S_{t+1}^{\top}}+\sigma_1\leq 2.01\sigma_1$ leads to
	\begin{align*}
		\norm{U_{t+1}U_{t+1}^\top-X^\star}\leq 5\sigma_1.
	\end{align*}
	
	\paragraph{Establishing $S_{t+1}S_{t+1}^\top\succ 0$:} The proof follows directly from the application of Proposition~\ref{prop:min_eig_empirical}. The details are omitted due to its similarity to the proof of~\eqref{eq_psd} in Lemma~\ref{lem:conditions}.
	
	\paragraph{Bounding $\norm{E_{t+1}S_{t+1}^\top\left(S_{t+1}S_{t+1}^\top\right)^{-1}}$:} To streamline the proof, let $H_t = E_tS_t^{\top}\left(S_tS_t^{\top}\right)^{-1}$. Our goal is to prove $\norm{H_t}\leq 1/3$ by showing the following recursive relationship.
	
	\begin{lemma}
		\label{lem::H_t}
		For every $0\leq s\leq t$, we have
		\begin{align}\label{eq_H}
			\norm{H_{s+1}}\leq (1-c\bareta\sigma_r)\norm{H_s}+c'\sqrt{r}\sigma_1\bareta\delta,
		\end{align}
		where $c,c'>0$ are some universal constants.
	\end{lemma}
	The proof of Lemma~\ref{lem::H_t} can be found in Appendix~\ref{app_Ht}. Equipped with this lemma, we are ready to derive the desired result.
	First, due to Lemma~\ref{prop::spec-init}, we have 
	\begin{align*}
		\norm{H_0}\leq \frac{\norm{E_0S_0^\top}}{\lambda_{\min}(S_0S_0^\top)}\leq \frac{2\alpha^2\delta{\varphi}(X^\star)}{\alpha^2\varphi(X^{\star})\left(\frac{1}{\sqrt{r}\kappa}-2\delta\right)}\leq 4\sqrt{r}\kappa\delta,
	\end{align*}
	provided that $\delta\leq\frac{1}{4\sqrt{r}\kappa}$. On the other hand,~\eqref{eq_H} implies that 
	{\begin{align*}
			\norm{H_{t+1}}-\frac{c'}{c}\sqrt{r}\kappa\delta\leq (1-c\bareta\sigma_r)^{t+1}\left(\norm{H_0}-\frac{c'}{c}\sqrt{r}\kappa\delta\right)\leq \norm{H_0}+\frac{c'}{c}\sqrt{r}\kappa\delta.
		\end{align*}
		Therefore, due to our choice of $\delta$, we have
		$$
		\norm{H_{t+1}}\leq \left(4\sqrt{r}\kappa+\frac{2c'}{c}\sqrt{r}\kappa\right)\delta\lesssim \sqrt{r}\kappa\delta\leq 1/3.
		$$}
	
	{\paragraph{Bounding $\norm{G_{t+1}}$:} Due to Proposition~\ref{prop::F-G}, we have
		\begin{align*}
			\norm{G_{t+1}}\leq \left(1+\bareta^2\left(2\norm{E_tS_t^{\top}}^2+\norm{E_t}^4+6\norm{E_tS_t(S_tS_t)^{-1}} \norm{\Delta_t}\norm{E_tS_t^{\top}}\right)+7\bareta\delta \norm{\Delta_t}_F\right)\norm{G_t}.
		\end{align*}
		Moreover, one can write
		\begin{align*}
			\bareta^2\norm{E_tS_t^{\top}}^2\lesssim \bareta \norm{E_t}^2\lesssim r\sigma_1^3\bareta^2\delta^2t^2+\sigma_1^2\bareta^2 \alpha\bar\varphi\delta t^2\lesssim \sqrt{r}\kappa\delta\log\left(\frac{1}{\alpha}\right).
		\end{align*}
		Similarly, it can be shown that
		\begin{equation}\nonumber
			\bareta^2\norm{E_t}^4\vee\bareta^2\norm{H_t}\norm{\Delta_t}\norm{E_tS_t^{\top}}\vee \bareta\delta\norm{\Delta_t}_F\lesssim \sqrt{r}\kappa\delta\log\left(\frac{1}{\alpha}\right).
		\end{equation}
		Therefore, for some universal constant $C>0$, we have
		\begin{align*}
			\norm{G_{t+1}} &\leq \left(1+ C \sqrt{r}\kappa\delta\log\left(\frac{1}{\alpha}\right)\right)\norm{G_t}.
		\end{align*}
		Hence, we have
		\begin{align*}
			\norm{G_{t+1}} &\leq \left(1+ C \sqrt{r}\kappa\delta\log\left(\frac{1}{\alpha}\right)\right)^{t+1}\norm{G_0}\\
			&\leq\left(1+ C \sqrt{r}\kappa\delta\log\left(\frac{1}{\alpha}\right)\right)^{C'\frac{\log({1}/{\alpha})}{\bareta\sigma_r}}\norm{G_0}\\
			&\leq \exp\left(C^{''} \sqrt{r}\kappa^2\delta\log\left(\frac{1}{\alpha}\right)\right)\norm{G_0}\\
			&\leq \alpha^{-\cO\left(\sqrt{r}\kappa^2\delta\right)}\norm{G_0}\\
			&\leq \alpha^{1-\cO\left(\sqrt{r}\kappa^2\delta\right)}\sqrt{\bar{\varphi} \delta},
		\end{align*}
		where in the last inequality, we used the upper bound $\norm{G_0}\leq \norm{E_0}$ and Lemma~\ref{prop::spec-init}.
	}
	
	%=================================
	
	\section{Proofs of Outlier Noise Model}
	\subsection{Preliminaries} 
	Given the update rule $U_{t+1} = U_t-2\eta_tQ_tU_t$, we consider the following decomposition
	\begin{align}\label{eq_decomp_gd}
		U_{t+1} = \tilde{U}_{t+1}+R_tU_t,\quad \text{where}\quad \tilde{U}_{t+1} = U_t - \frac{2\eta\rho^t}{\norm{\Delta_t}}\Delta_tU_t,\ R_t = \frac{2\eta\rho^t}{\norm{\Delta_t}}\Delta_t - 2\eta_tQ_t,
	\end{align}
	In the above decomposition, $\tilde{U}_{t+1}$ resembles one iteration of GD on $\bar f_{\ell_2} (U)$ with the ``effective'' step-size $\frac{\eta\rho^t}{2\norm{\Delta_t}}$. Moreover, the term $R_tU_t$ captures the deviation of SubGM and GD. Similar to the noiseless setting, the main idea behind our proof technique is to show that $R_t$ remains small throughout the iterations of SubGM, and consequently, SubGM behaves similar to GD. 
	To this goal, we first provide an upper bound on $\norm{\Delta_t}_F$ in terms of $\norm{\Delta_t}$.
	\begin{lemma}\label{lem_Delta}
		Suppose that $\sqrt{d}\norm{G_t}^2\leq\norm{\Delta_t}$. Then, we have $\norm{\Delta_t}_F\leq 2(1+\sqrt{r})\norm{\Delta}$.
	\end{lemma}
	\begin{proof}
		Due to our proposed signal-residual decomposition, one can write
		\begin{align}
			\norm{\Delta_t}_F
			& \leq\norm{\underbrace{V\left(S_tS_t^{\top}-\Sigma\right)V^{\top}+VS_tE_t^{\top}V_{\perp}^{\top}+V_{\perp}E_tS_t^{\top}V^{\top}+V_{\perp}F_tF_t^\top V_{\perp}^{\top}}_{\text{rank-$4r$}}}_F+\norm{\underbrace{V_{\perp}G_tG_t^\top V_{\perp}^{\top}}_{\text{small norm}}}_F\nonumber\\
			&\leq\sqrt{4r}\norm{{V\left(S_tS_t^{\top}-\Sigma\right)V^{\top}+VS_tE_t^{\top}V_{\perp}^{\top}+V_{\perp}E_tS_t^{\top}V^{\top}+V_{\perp}F_tF_t^\top V_{\perp}^{\top}}}+\sqrt{d}\norm{G_t}^2\nonumber\\
			&\leq \sqrt{4r}\norm{\Delta_t}+\sqrt{4r}\norm{V_{\perp}G_tG_t^\top V_{\perp}^{\top}}+\sqrt{d}\norm{G_t}^2\nonumber\\
			&\leq 2(1+\sqrt{r})\norm{\Delta_t},\nonumber
		\end{align}
		where the last inequality follows from the assumption $4r\leq d$ and $\sqrt{d}\norm{G_t}^2\leq\norm{\Delta_t}$. This completes the proof.
	\end{proof}
	Equipped with this technical lemma, we next provide an upper bound on $\norm{R_t}$.
	\begin{lemma}
		\label{lem::del-geometric}
		Suppose that the measurements satisfy $\left(4r,\delta,\err,\cS\right)$-Sign-RIP with $\delta<\frac{1}{4(1+\sqrt{r})}$, $\err=\sqrt{d}\norm{G_t}^2$, $\cS=\left\{X: \norm{X}_F\geq \zeta\right\}$ for $\zeta = \sqrt{d}\norm{G_t}^2\left(\frac{1}{\delta}\vee\sqrt{d}\right)$, and $\sqrt{d}\norm{G_t}^2\leq \norm{\Deltatsecond}$. Then, we have $\norm{\del}\leq 8(1+\sqrt{r})\eta\rho^t\delta$.
	\end{lemma}
	\begin{proof}
		One can write
		\begin{equation}\nonumber
			\begin{aligned}
				\del & = \frac{2\eta\rho^t}{\norm{\Delta_t}}\Delta_t-\frac{2\eta\rho^t}{\norm{Q_t}}Q_t\\
				&= -\frac{2\eta\rho^t}{\norm{Q_t}}\left(Q_t-\varphi(\Delta_t)\frac{\Delta_t}{\norm{\Delta_t}_F}\right)-\frac{2\eta\rho^t\varphi(\Delta_t)}{\norm{Q_t}\norm{\Delta_t}_F}\Delta_t+\frac{2\eta\rho^t}{\norm{\Delta_t}}\Delta_t\\ &=-\frac{2\eta\rho^t}{\norm{Q_t}}\left(Q_t-\varphi(\Delta_t)\frac{\Delta_t}{\norm{\Delta_t}_F}\right)+\frac{\norm{Q_t} - \scale(\Deltatsecond)\frac{\norm{\Deltatsecond}}{\norm{\Deltatsecond}_F}}{\norm{Q_t}}\frac{2\eta\rho^t}{\norm{\Deltatsecond}}\Deltatsecond.
			\end{aligned}
		\end{equation}
		The above equality implies that
		\begin{equation}
			\label{eq::Rt}
			\begin{aligned}
				\norm{\del} & \leq \frac{2\eta\rho^t}{\norm{Q_t}}\norm{Q_t-{\scale(\Deltatsecond)}\frac{\Deltatsecond}{\norm{\Deltatsecond}_F}}+2\eta\rho^t\cdot\frac{\left|\norm{Q_t} - \scale(\Deltatsecond)\frac{\norm{\Deltatsecond}}{\norm{\Deltatsecond}_F}\right|}{\norm{Q_t}}\\
				&\leq \left(\frac{1}{\norm{Q_t}}\right)2\eta\rho^t\varphi(\Delta_t)\delta,
			\end{aligned}
		\end{equation}
		where in the last inequality, we used Sign-RIP. Next, we provide an upper bound for $1/\norm{Q_t}$. 
		Due to Sign-RIP, we have $\norm{Q_t}\geq \left(\frac{\norm{\Delta_t}}{\norm{\Delta_t}_F}-\delta\right)\varphi(\Delta_t)$. On the other hand, due to Lemma~\ref{lem_Delta}, we have $\frac{\norm{\Delta_t}}{\norm{\Delta_t}_F}\leq \frac{1}{2(1+\sqrt{r})}$. Combining these inequalities with~\eqref{eq::Rt}, we have
		\begin{align*}
			\norm{\del}\leq \frac{2}{\frac{1}{2(1+\sqrt{r})}-\delta}\cdot\eta\rho^t\delta\leq 8(1+\sqrt{r})\eta\rho^t\delta,
		\end{align*}
		where the last inequality is due to the assumption $\delta\leq \frac{1}{4(1+\sqrt{r})}$.
	\end{proof}
	
	\subsection{Proof of Proposition~\ref{prop_min_eig_outlier}}\label{app_prop_min_eig_outlier}
	The proof is almost a line-by-line reconstruction of the proof of Proposition~\ref{prop:min_eig_empirical} in Appendix~\ref{app_prop:min_eig_empirical}. For brevity, we only provide a sketch of the proof. Similar to~\eqref{eq_S}, one can write
	\begin{equation}\label{eq_S2}
		\begin{aligned}	S_{t+1}&=\tilde{S}_{t+1}+V^\top R_tU_t.
		\end{aligned}
	\end{equation}
	Given this decomposition, we characterize the growth rate of $\lambda_{\min}(S_{t+1}S_{t+1}^\top)$ by first resorting to a more tractable lower bound. In particular, we define $M := \left(I+\Xi\right)\tilde{S}_{t+1}\tilde{S}_{t+1}^\top\left(I+\Xi\right)^\top$, where $\Xi:=V^{\top} \del U_{t}\tilde{S}_{t+1}^\top\left(\tilde{S}_{t+1}\tilde{S}_{t+1}^\top\right)^{-1}$. Based on the definition of $M$, a series of inequalities analogous to~\eqref{eq_lb_SM} can be used to show that $\lambda_{\min}(S_{t+1}S_{t+1}^{\top})\geq \lambda_{\min}(M)$. Therefore, it suffices to provide a lower bound for $\lambda_{\min}(M)$. Similar to the proof of Proposition~\ref{prop:min_eig_empirical}, we first show that $\lambda_{\min}(M)\approx \lambda_{\min}(\tilde S_{t+1}\tilde S_{t+1}^{\top})$:
	\begin{align}
		\left|\lambda_{\min}(M)-\lambda_{\min}\left(\tilde S_{t+1}\tilde S_{t+1}^{\top}\right)\right|&\leq 3\norm{\Xi}\lambda_{\min}\left(\tilde S_{t+1}\tilde S_{t+1}^{\top}\right)\leq 192\sqrt{r}\eta\rho^t\lambda_{\min}\left(\tilde S_{t+1}\tilde S_{t+1}^{\top}\right).\nonumber
	\end{align}
	Combining the above inequality with the one-step dynamic of $\lambda_{\min}\left(\tilde S_{t+1}\tilde S_{t+1}^{\top}\right)$ from Proposition~\ref{lem:min_eig_population} completes the proof.$\hfill\square$
	
	\subsection{Proof of Proposition~\ref{prop::finite-partially-corrupted}}\label{app_prop::finite-partially-corrupted}
	
	Similar to the minimum eigenvalue dynamics, the proof is identical to the proof of Proposition~\ref{lem:dynamics_empirical}. Hence, we only provide a sketch. Similar to~\eqref{eq::appendix-112}, one can write 
	\begin{equation}\nonumber
		\begin{aligned}	\Sigma-S_{t+1}S_{t+1}^{\top}=\Sigma-\tilde{S}_{t+1}\tilde{S}_{t+1}^{\top}{-V^{\top}\del U_{t}\tilde{S}_{t+1}^{\top}-\tilde{S}_{t+1} U_{t}^{\top}\del^{\top}V-V^{\top}\del U_tU_t^{\top}\del^{\top}V}.
		\end{aligned}
	\end{equation}
	Lemma~\ref{lem::del-geometric} combined with an argument similar to Appendix~\ref{app_lem:dynamics_empirical} leads to
	\begin{align*}
		\norm{V^{\top}\del U_{t}\tilde{S}_{t+1}^{\top}+\tilde{S}_{t+1} U_{t}^{\top}\del^{\top}V+V^{\top}\del U_tU_t^{\top}\del^{\top}V}\leq 193\sqrt{r}\eta\rho^t\sigma_1\delta.
	\end{align*}
	The above bound combined with the one-step dynamics of $\Sigma-\tilde{S}_{t+1}\tilde{S}_{t+1}^{\top}$ in Proposition~\ref{lem:dynamics_pop} completes the proof for the one-step dynamics of $\Sigma-S_{t+1}S_{t+1}^{\top}$. The dynamics of the cross term~\eqref{eq_cross_partial} and the upper bound on $\norm{S_{t+1}S_{t+1}^\top}$~\eqref{eq_ub_St} can be deduced in a similar fashion. The details are omitted for brevity.$\hfill\square$

	\subsection{Proof of Proposition~\ref{prop::F-G_outlier}}\label{app_prop::F-G_outlier}
	The proof of Proposition~\ref{prop::F-G_outlier} is identical to that of Proposition~\ref{prop::F-G}, with a key difference that $\norm{R_t}\leq 16\sqrt{r}\eta\rho^t$. The details are omitted for brevity.$\hfill\square$

	\subsection{Proof of Lemma~\ref{lem:conditions_noisy}}\label{app_lem:conditions_noisy}
	The proof is based on an inductive argument similar to the proof of Lemma~\ref{lem:conditions_empirical} in Appendix~\ref{app_lem:conditions_empirical}. Due to our special initialization, it is easy to verify that the statements of the lemma are satisfied for $t=0$. Now suppose that~\eqref{eq_F_emp_noisy}-\eqref{eq_Ht_emp_noisy} are satisfied for $t$. Due to Proposition~\ref{prop::F-G_outlier}, one can write
	\begin{equation}\nonumber
		\begin{aligned}
			\norm{G_{t+1}}
			&\leq \left(1+5\eta^2\rho^{2t}+49\sqrt{r}\eta_0\rho^t\delta\right)\norm{G_t},\\
			&\leq \exp\left({5\eta^2\rho^{2t}+49\sqrt{r}\eta\rho^t\delta}\right)\norm{G_t},\\
			&\leq \norm{G_{0}}\prod_{s=0}^{t} \exp\left({5\eta^2\rho^{2s}+49\sqrt{r}\eta_0\rho^s\delta}\right) \\
			&\leq \norm{G_{0}} \exp\left({\sum_{s=0}^t 5\eta^2\rho^{2s}+49\sqrt{r}\eta_0\rho^s\delta}\right)\\
			&\leq \norm{G_0} \exp\left({\frac{5\eta^2}{1-\rho^2}+\frac{49\sqrt{r}\eta\delta}{1-\rho}}\right).
		\end{aligned}
	\end{equation}
	Due to $\rho = 1-\Theta(\eta/(\kappa\log(1/\alpha)))$ and $\eta\lesssim 1/(\kappa\log(1/\alpha))$, we have
	$$
	\frac{5\eta^2}{1-\rho^2}\lesssim \eta\kappa\log(1/\alpha)\leq 1/2, \quad \text{and}\quad \frac{49\sqrt{r}\eta\delta}{1-\rho}\lesssim\sqrt{r}\kappa\delta\log(1/\alpha).
	$$
	Combining the above inequalities leads
	$$
	\norm{G_{t+1}}\leq 2\norm{G_0}\alpha^{-\mathcal{O}(\sqrt{r}\kappa\delta)}\leq 2\sqrt{2}\alpha^{1-\mathcal{O}(\sqrt{r}\kappa\delta)}\sqrt{\bar{\varphi}\delta},
	$$
	where the last inequality is due to our special initialization technique and Lemma~\ref{prop::spec-init}. The remaining bounds in Lemma~\ref{lem:conditions_noisy} can be established similar to Lemma~\ref{lem:conditions_empirical}. The details are omitted for brevity. $\hfill\square$

	%=================================
	
	\section{Omitted Proofs}
	
	\subsection{Proof of Lemma~\ref{l_EG}}\label{app_l_EG}
	To prove this lemma, we use one-step discretization technique. First note that, for any $X\in \cS_{k,\err}$, there exists a matrix in $X'\in \cS_{k}$ such that $\norm{X-X'}_F\leq {\err}$. Suppose that $\cN_{k,\xi}$ is a $\xi$-net of $\cS_k$ where $\xi\geq \err$. Based on the definition of $\xi$-net, there exists $X''\in\cN_{k,\xi}$ such that $\norm{X'-X''}_F\leq \xi$. This implies that $\norm{X-X''}_F\leq \norm{X-X'}_F+\norm{X'-X''}_F\leq 2\xi$, and hence, $\cN_{k,\xi}$ is a $2\xi$-net of $\cS_{k,\err}$. Given this fact, one can write the following chain of inequalities for every $Y\in\bS$:
	\begin{equation}\label{eq_GyE}
		\begin{aligned}
			\bE\left[\cG_Y\right] & \leq \underbrace{\bE\left[\sup_{X' \in \cN_{k,\xi}}\frac{1}{\measurementnumber}\sum_{i=1}^{\measurementnumber}\sign\left(\inner{A_i}{X'}-s_i\right)\inner{A_i}{Y}-\varphi(X')\inner{\frac{X'}{\norm{X'}_F}}{Y}\right]}_{{ (A)}}\\
			& +\underbrace{\bE\left[\sup_{\norm{X-X'}_F\leq 2\xi}\frac{1}{\measurementnumber}\sum_{i=1}^{\measurementnumber}\left(\sign\left(\inner{A_i}{X}-s_i\right)-\sign\left(\inner{A_i}{X'}- s_i\right)\right)\inner{A_i}{Y}\right]}_{ (B)} \\
			& +\underbrace{\sup_{\norm{X-X'}_F\leq 2\xi}\inner{\varphi(X)\frac{X}{\norm{X}_F}-\varphi(X')\frac{X'}{\norm{X'}_F}}{Y}}_{ (C)}.
		\end{aligned}
	\end{equation}
	We control each term in the above inequality separately.
	\paragraph{Bounding ${(A)}$.} To control ${(A)}$, note that $\frac{1}{\measurementnumber}\sum_{i=1}^{\measurementnumber}\sign\left(\inner{A_i}{X'}-s_i\right)\inner{A_i}{Y}-\varphi(X')\inner{\frac{X'}{\norm{X'}_F}}{Y}$ is $\cO\left(1/m\right)$-sub-Gaussian and ${(A)}$ is the supremum of the sub-Gaussian random variable over a the finite set $\cN_{k,\xi}$. Hence, the Maximum Inequality implies that
	\begin{equation}\label{eq_Elambda}
		{ (A)}\lesssim\sqrt[]{\frac{dk}{m}\log\left(\frac{R}{\xi}\right)}.
	\end{equation}

	\paragraph{Bounding ${(B)}$.} Invoking Hölder's inequality, one can write
	\begin{equation}
		\begin{aligned}
			{ (B)} & \leq \bE\left[\sup_{\norm{X-X'}_F\leq 2\xi}\left(\frac{1}{m}\sum_{i=1}^{m}\left|\sign\left(\inner{A_i}{X}-s_i\right)-\sign\left(\inner{A_i}{X'}-s_i\right)\right|\right)\max_{1\leq i\leq m}|\inner{A_i}{Y}|\right].
		\end{aligned}\nonumber
	\end{equation}
	Note that if $|\inner{A_i}{X-X'}|\leq |\inner{A_i}{X'-\lambda s_i}|$, then $\sign(\inner{A_i}{X'}-\lambda s_i)=\sign(\inner{A_i}{X}-\lambda s_i)$. Therefore, the above term can be further bounded by 
	\begin{equation}
		\begin{aligned}
			&\bE\left[\sup_{\norm{X-X'}_F\leq 2\xi}\left(\frac{1}{m}\sum_{i=1}^{m}\left|\sign\left(\inner{A_i}{X}-s_i\right)-\sign\left(\inner{A_i}{X'}-s_i\right)\right|\right)\max_{1\leq i\leq m}|\inner{A_i}{Y}|\right]\\
			\leq &\bE\left[\sup_{\norm{X-X'}_F\leq 2\xi}\left(\frac{1}{m}\sum_{i=1}^{m}\mathbbm{1}\left(|\inner{A_i}{X-X'}|\geq |\inner{A_i}{X'}-s_i|\right)\right)\max_{1\leq i\leq m}|\inner{A_i}{Y}|\right]\\
			\leq & \bE\left[\sup_{\norm{X-X'}_F\leq 2\xi}\left(\frac{1}{m}\sum_{i=1}^{m}\mathbbm{1}\left(|\inner{A_i}{X-X'}|\geq t\right)+\mathbbm{1}\left(|\inner{A_i}{X'}-s_i|\leq t\right)\right)\max_{1\leq i\leq m}|\inner{A_i}{Y}|\right]\\
			\leq & \underbrace{\bE\left[\sup_{\norm{X-X'}_F\leq 2\xi}\left(\frac{1}{m}\sum_{i=1}^{m}\mathbbm{1}\left(|\inner{A_i}{X-X'}|\geq t\right)\right)\max_{1\leq i\leq m}|\inner{A_i}{Y}|\right]}_{(B_1)} \\
			&+\underbrace{\bE\left[\sup_{X'\in \cN_{k,\xi}}\left(\frac{1}{m}\sum_{i=1}^{m}\mathbbm{1}\left(|\inner{A_i}{X'}-s_i|\leq t\right)\right)\max_{1\leq i\leq m}|\inner{A_i}{Y}|\right]}_{(B_2)}.
		\end{aligned}
	\end{equation} 
	For $(B_1)$, we have 
	\begin{equation}\label{eq_B1}
		\begin{aligned}
			(B_1)&\leq \bE\left[\left(\frac{1}{m}\sum_{i=1}^{m}\mathbbm{1}\left(\norm{A_i}_F\geq\frac{t}{2\xi}\right)\right)\max_{1\leq i\leq m}|\inner{A_i}{Y}|\right]\\
			&\leq \bE\left[\mathbbm{1}\left(\norm{A_i}_F\geq\frac{t}{2\xi}\right)\right]\bE\left[\max_{j\neq i}|\inner{A_j}{Y}|\right]+\bE\left[\mathbbm{1}\left(\norm{A_i}_F\geq\frac{t}{2\xi}\right)|\inner{A_i}{Y}|\right]\\
			&\leq \cO\left(e^{-C\frac{t^2}{\xi^2}}\sqrt{\log\left(m\right)}\right)+\bE\left[\mathbbm{1}\left(\norm{A_i}_F\geq\frac{t}{2\xi}\right)|\inner{A_i}{Y}|\right],
		\end{aligned}
	\end{equation} 
	where $C>0$ is a universal constant and ${t}/{\xi}\geq \sqrt{d}$. Furthermore, applying Cauchy-Schwarz inequality, we have
	\begin{equation}
		\bE\left[\mathbbm{1}\left(\norm{A_i}_F\geq\frac{t}{2\xi}\right)|\inner{A_i}{Y}|\right]\leq \sqrt{\bP\left(2\xi\norm{A_i}_F\geq t\right)}\sqrt{\bE\left[\inner{A_i}{Y}^2\right]}\lesssim e^{-C\frac{t^2}{\xi^2}}.
	\end{equation}
	Hence, we conclude that $(B_1)\lesssim e^{-C\frac{t^2}{\xi^2}}\sqrt{\log\left(m\right)}$.
	
	Now we turn to bound ${(B_2)}$. Note that $\left(\frac{1}{m}\sum_{i=1}^{m}\mathbbm{1}\left(|\inner{A_i}{X'}-s_i|\leq t\right)\right)\max_{1\leq i\leq m}|\inner{A_i}{Y}|$ is $\cO\left(\log(m)/m\right)$-sub-Gaussian, and $\cN_{k,\xi}$ is a finite set. Hence, the Maximum Inequality yields
	\begin{equation}\label{eq_B2}
		\begin{aligned}
			{ (B_2)} & \leq \sup_{X'\in \cN_{k,\xi}}\bE\left[\left(\frac{1}{m}\sum_{i=1}^{m}\mathbbm{1}\left(|\inner{A_i}{X'}-s_i|\leq t\right)\right)\max_{1\leq i\leq m}|\inner{A_i}{Y}|\right]+\cO\left(\sqrt{\frac{dk\log(m)\log\left(R/\xi\right)}{m}}\right).
		\end{aligned}
	\end{equation}
	For the first part, one can write
	\begin{equation}
		\begin{aligned}
			&\bE\left[\left(\frac{1}{m}\sum_{i=1}^{m}\mathbbm{1}\left(|\inner{A_i}{X'}-s_i|\leq t\right)\right)\max_{1\leq i\leq m}|\inner{A_i}{Y}|\right]\\
			\leq & \bE\left[\mathbbm{1}\left(|\inner{A_i}{X'}-s_i|\leq t\right)\right]\bE\left[\max_{j\neq i}|\inner{A_i}{Y}|\right]+\bE\left[\mathbbm{1}\left(|\inner{A_i}{X'}-s_i|\leq t\right)|\inner{A_i}{Y}|\right]\\
			&\leq \sqrt{\log(m)}\frac{t}{\zeta}.
		\end{aligned}
	\end{equation}
	Therefore, we conclude that $(B)\lesssim \sqrt{\log(m)}\frac{t}{\zeta}+e^{-C\frac{t^2}{\xi^2}}\sqrt{\log\left(m\right)}+\sqrt{\frac{dk\log(m)\log\left(R/\xi\right)}{m}}$.
	Finally, it remains to bound $(C)$ in~\eqref{eq_GyE}.
	
	\paragraph{Bounding $(C)$.} We have
	\begin{equation}
		\begin{aligned}
			(C) &= \sup_{\norm{X-X'}_F\leq 2\xi}\left\{\left(\varphi(X)-\varphi(X')\right)\inner{\frac{X}{\norm{X}_F}}{Y}+\varphi(X')\inner{\frac{X}{\norm{X}_F}-\frac{X'}{\norm{X'}_F}}{Y}\right\}\\
			&\leq \sup_{\norm{X-X'}_F\leq 2\xi} \left\{\varphi(X)-\varphi(X')\right\}+\sup_{\norm{X-X'}_F\leq 2\xi}\norm{\frac{X}{\norm{X}_F}-\frac{X'}{\norm{X'}_F}}_F.
		\end{aligned}
	\end{equation}
	For the first part, we use Mean Value Theorem to write
	\begin{equation}
		|\varphi(X')-\varphi(X)|\leq \norm{\nabla \varphi(Z)}_F\norm{X'-X}_F\leq 2\norm{\nabla \varphi(Z)}_F\xi,
	\end{equation}
	where $Z=\lambda X + (1-\lambda)X'$, $\lambda\in [0,1]$. Note that $\nabla \varphi(Z)=\sqrt{\frac{2}{\pi}}p\bE\left[\frac{s^2Z}{\norm{Z}_F^4}e^{-\frac{s^2}{2\norm{Z}_F^2}}\right]$. Hence, we have
	\begin{equation}
		\begin{aligned}
			\sup_{\norm{Z}_F\geq \zeta}\norm{\nabla \varphi(Z)}_F&\lesssim \sup_{\norm{Z}_F\geq \zeta}\bE\left[\frac{s^2}{\norm{Z}_F^3}e^{-\frac{s^2}{2\norm{Z}_F^2}}\right]\leq \frac{1}{\zeta}\sup_{\norm{Z}_F\geq \zeta}\bE\left[\frac{s^2}{\norm{Z}_F^2}e^{-\frac{s^2}{2\norm{Z}_F^2}}\right]\lesssim \frac{1}{\zeta}.
		\end{aligned}
	\end{equation}
	For the second part, we have 
	\begin{equation}
		\begin{aligned}
			\sup_{\norm{X-X'}_F\leq 2\xi}\norm{\frac{X}{\norm{X}_F}-\frac{X'}{\norm{X'}_F}}_F&\leq \sup_{\norm{X-X'}_F\leq 2\xi}\norm{\frac{X-X'}{\norm{X}_F}}_F+\norm{\frac{X\left(\norm{X'}_F-\norm{X}_F\right)}{\norm{X}_F\norm{X'}_F}}_F\\
			&\leq \frac{4\xi}{\zeta}.
		\end{aligned}
	\end{equation}
	Therefore, we conclude that $(C)\lesssim \frac{\xi}{\zeta}$.
	
	Combining the derived upper bounds for $(A)$, $(B)$, and $(C)$, we have
	\begin{equation}\label{eq_ABC}
		\begin{aligned}
			\bE[\cG_Y]&\leq (A)+(B)+(C)\\
			&\lesssim \sqrt[]{\frac{dk}{m}\log\left(\frac{R}{\xi}\right)}+\sqrt{\log(m)}\frac{t}{\zeta}+e^{-C\frac{t^2}{\xi^2}}\sqrt{\log\left(m\right)}+\sqrt{\frac{dk\log(m)\log\left(R/\xi\right)}{m}}+\frac{\xi}{\zeta},
		\end{aligned}
	\end{equation}
	provided that $\xi\geq \err$ and $t/\xi\geq \sqrt{d}$. Let $\xi\asymp \zeta\sqrt{{k}/{m}}$ and $t\asymp \zeta\sqrt[]{{dk}\log\left(m\right)/{\measurementnumber}}$. Clearly, these choices of parameters satisfy $t/\xi\geq \sqrt{d}$. Moreover, $\xi\asymp \zeta\sqrt{{k}/{m}}$ together with the assumption $\err\lesssim\zeta\sqrt{k/m}$ implies that $\err\leq\xi$.  Finally, plugging these values in~\eqref{eq_ABC} leads to
	\begin{align}\nonumber
		\bE[\cG_Y]\lesssim \sqrt{\frac{dk}{m}\log^2(m)\log\left(\frac{R}{\zeta}\right)},
	\end{align}
	for every $Y\in\bS$. This in turn implies 
	\begin{align}\nonumber
		\sup_{Y\in\bS_k}\bE[\cG_Y]\leq \sup_{Y\in\bS}\bE[\cG_Y]\lesssim \sqrt{\frac{dk}{m}\log^2(m)\log\left(\frac{R}{\zeta}\right)},
	\end{align}
	which completes the proof.$\hfill\square$
	
	%=================================
	
	\subsection{Proof of Lemma~\ref{lem:appendix-4}}\label{app_lem:appendix-4}
	%\SF{Why is $S_{t+1}S_{t+1}$ invertible?}
	To prove $0.5\leq \norm{S_tS_t^{\top}\left(\tilde{S}_{t+1}\tilde{S}_{t+1}^{\top}\right)^{-1}}\leq 2$, it suffices to show that $0.5\leq\lambda_{\min}\left(\tilde{S}_{t+1}\tilde{S}_{t+1}^{\top}\left(S_tS_t^{\top}\right)^{-1}\right)\leq \norm{\tilde{S}_{t+1}\tilde{S}_{t+1}^{\top}\left(S_tS_t^{\top}\right)^{-1}}\leq 2$. For brevity, we only show $0.5\leq\lambda_{\min}\left(\tilde{S}_{t+1}\tilde{S}_{t+1}^{\top}\left(S_tS_t^{\top}\right)^{-1}\right)$, as the other part of the inequality can be proven in a similar fashion. One can write
	\begin{equation}
		\begin{aligned}
			\tilde{S}_{t+1}\tilde{S}_{t+1}^{\top}&=S_tS_t^{\top}-\bareta S_tS_t^{\top}\left(S_tS_t^{\top}-\Sigma\right)-\bareta \left(S_tS_t^{\top}-\Sigma\right)S_tS_t^{\top}-2\bareta S_tE_tE_t^{\top}S_t^{\top}\\
			&+\bareta^2 \left(S_tS_t^{\top}-\Sigma\right)S_tS_t^{\top}\left(S_tS_t^{\top}-\Sigma\right)+\bareta^2 S_tE_t^{\top}E_tE_t^{\top}E_tS_t^{\top}\\
			&+\bareta^2\left(S_tS_t^{\top}-\Sigma\right)S_tE_tE_t^{\top}S_t^{\top}+\bareta^2 S_tE_tE_t^{\top}S_t^{\top}\left(S_tS_t^{\top}-\Sigma\right).
		\end{aligned}
	\end{equation}
	Note that the eigenvalues of $\tilde{S}_{t+1}\tilde{S}_{t+1}^{\top}\left(S_tS_t^{\top}\right)^{-1}$ are real and nonnegative, due to its similarity to $\left(S_tS_t^{\top}\right)^{-1/2}\tilde{S}_{t+1}\tilde{S}_{t+1}^{\top}\left(S_tS_t^{\top}\right)^{-1/2}$. One the other hand, one can write
	\begin{equation}
		\begin{aligned}
			\tilde{S}_{t+1}\tilde{S}_{t+1}^{\top}\left(S_tS_t^{\top}\right)^{-1}&=I+\bareta\Sigma+\bareta S_tS_t^{\top}\Sigma\left(S_tS_t^{\top}\right)^{-1}-2\bareta S_tS_t^{\top}-2\bareta S_tE_tE_t^{\top}S_t^{\top}\left(S_tS_t^{\top}\right)^{-1}\\
			&+\bareta^2 \left(S_tS_t^{\top}-\Sigma\right)S_tS_t^{\top}\left(S_tS_t^{\top}-\Sigma\right)\left(S_tS_t^{\top}\right)^{-1}+\bareta^2 S_tE_t^{\top}E_tE_t^{\top}E_tS_t^{\top}\left(S_tS_t^{\top}\right)^{-1}\\
			&+\bareta^2\left(S_tS_t^{\top}-\Sigma\right)S_tE_tE_t^{\top}S_t^{\top}\left(S_tS_t^{\top}\right)^{-1}+\bareta^2 S_tE_tE_t^{\top}S_t^{\top}\left(S_tS_t^{\top}-\Sigma\right)\left(S_tS_t^{\top}\right)^{-1}.
		\end{aligned}\nonumber
	\end{equation}
	We will show that every term in the above decomposition, except for the first term, is in the order of $\mathcal{O}(\bareta\sigma_1)$. First note that
	\begin{equation}
		\norm{\bareta\Sigma}=\norm{\bareta S_tS_t^{\top}\Sigma\left(S_tS_t^{\top}\right)^{-1}}= \cO(\bareta \sigma_1).
	\end{equation}
	Similarly, we have
	\begin{equation}
		\begin{aligned}
			\norm{2\bareta S_tS_t^{\top}+2\bareta S_tE_tE_t^{\top}S_t^{\top}\left(S_tS_t^{\top}\right)^{-1}}&\leq2\bareta\norm{S_tS_t^{\top}}+2\bareta\norm{S_tE_tE_t^{\top}S_t^{\top}\left(S_tS_t^{\top}\right)^{-1}}\\
			&=2\bareta\norm{S_tS_t^{\top}}+2\bareta\norm{E_tE_t^{\top}}\norm{S_t^{\top}\left(S_tS_t^{\top}\right)^{-1}S_t}\\
			&\leq 2\bareta\norm{S_tS_t^{\top}}+2\bareta\norm{E_tE_t^{\top}}\\
			&=\cO(\bareta\sigma_1).
		\end{aligned}
	\end{equation}
	It is easy to see that all the remaining terms are in the order of $\cO(\bareta\sigma_1)$; we omit their proofs for brevity.
	Combining the above bounds, we obtain
	\begin{equation}
		\lambda_{\min}\left(\tilde{S}_{t+1}\tilde{S}_{t+1}^{\top}\left(S_tS_t^{\top}\right)^{-1}\right)=1-\cO(\bareta\sigma_1)\geq 1/2,
	\end{equation}
	where the last inequality is due to our choice of $\bareta$. This in turn implies that $\norm{S_tS_t^{\top}\left(\tilde{S}_{t+1}\tilde{S}_{t+1}^{\top}\right)^{-1}}\leq 2$. Similarly, we can show that $\norm{S_tS_t^{\top}\left(S_{t+1}S_{t+1}^{\top}\right)^{-1}}\leq 3$, by proving $\lambda_{\min}\left(S_{t+1}S_{t+1}^{\top}(S_tS_t^\top)^{-1}\right)\geq 1/3$. This can be shown in an analogous fashion, after noting that
	\begin{equation}
		S_{t+1}S_{t+1}^{\top}=\tilde{S}_{t+1}\tilde{S}_{t+1}^{\top}+\underbrace{V^{\top}\del U_t \tilde{S}_{t+1}^{\top}+\tilde{S}_{t+1} U_t^{\top}\del^{\top}V+V^{\top}\del U_t U_t^{\top}\del^{\top}V}_{\text{perturbation}}.
	\end{equation}
	Similar to the proof of Proposition~\ref{lem:dynamics_empirical}, it can shown that the norm of the perturbation term is upper bounded by $4\bareta\delta\norm{\Delta_t}_F\leq 1/6$, due to our choice of $\delta$ and $\eta$, and our assumptions on $\norm{E_tE_t^\top}$ and $\norm{S_tS_t^\top}$.
	
	Finally, we prove the second statement by providing an upper bound for $\norm{\tilde{E}_{t+1}\tilde{S}_{t+1}^{\top}\left(\tilde{S}_{t+1}\tilde{S}_{t+1}^{\top}\right)^{-1}}$. Due to the first part of the lemma, one can write
	\begin{align}
		\norm{\tilde{E}_{t+1}\tilde{S}_{t+1}^{\top}\left(\tilde{S}_{t+1}\tilde{S}_{t+1}^{\top}\right)^{-1}}&\leq \norm{\tilde{E}_{t+1}\tilde{S}_{t+1}^{\top}\left(S_{t}S_{t}^{\top}\right)^{-1}}\norm{S_{t}S_{t}^{\top}\left(\tilde{S}_{t+1}\tilde{S}_{t+1}^{\top}\right)^{-1}}\nonumber\\
		&\leq 2 \norm{\tilde{E}_{t+1}\tilde{S}_{t+1}^{\top}\left(S_{t}S_{t}^{\top}\right)^{-1}}.\nonumber
	\end{align}
	On the other hand, we have
	\begin{equation}
		\begin{aligned}
			\tilde{E}_{t+1}\tilde{S}_{t+1}^{\top}&= E_tS_t^{\top}+\bareta E_tS_t^{\top}(\Sigma-S_tS_t^{\top})-\bareta E_t(S_t^{\top}S_t+E_t^{\top}E_t)S_t^{\top}\\
			&+\bareta^2E_t(S_t^{\top}S_t+E_t^{\top}E_t)S_t^{\top}\left(S_tS_t^{\top}-\Sigma\right)\\
			&-\bareta E_tE_t^{\top}E_tS^{\top}_t+\bareta^2E_t\left(S_t^{\top}S_t+E_t^{\top}E_t\right)E^{\top}_tE_tS_t^{\top}.
		\end{aligned}
	\end{equation}
	Let us define $H_t=E_tS_t^{\top}\left(S_tS_t^{\top}\right)^{-1}$. Based on this definition, one can verify that 
	\begin{equation}
		\begin{aligned}
			\tilde{E}_{t+1}\tilde{S}_{t+1}^{\top}\left(S_tS_t^{\top}\right)^{-1}&= H_t+\bareta H_tS_tS_t^{\top}(\Sigma-S_tS_t^{\top})\left(S_tS_t^{\top}\right)^{-1}-\bareta H_tS_tS_t^{\top}-2\bareta E_tE_t^{\top}H_t\\
			&+\bareta^2H_t\left(S_tS_t^{\top}\right)^2\left(S_tS_t^{\top}-\Sigma\right)\left(S_tS_t^{\top}\right)^{-1}+\bareta^2E_t\left(S_t^{\top}S_t+E_t^{\top}E_t\right)E^{\top}_tH_t\\
			&+\bareta^2 E_tE_t^{\top}H_t S_tS_t^{\top}\left(S_tS_t^{\top}-\Sigma\right)\left(S_tS_t^{\top}\right)^{-1}.
		\end{aligned}
	\end{equation}
	Now, note that
	\begin{equation}
		\norm{\bareta H_t S_tS_t^{\top}(\Sigma-S_tS_t^{\top})\left(S_tS_t^{\top}\right)^{-1}}\leq\bareta \norm{H_t}\norm{\Sigma-S_tS_t^{\top}}\lesssim \bareta\sigma_1\norm{H_t}\leq \frac{1}{12}\norm{H_t}.
	\end{equation}
	Similarly, we have $\norm{\bareta H_tS_tS_t^{\top}}\leq \frac{1}{12}\norm{H_t}$ and $\norm{\bareta E_tE_t^{\top}H_t}\leq \frac{1}{12}\norm{H_t}$.  Moreover, we have
	\begin{equation}
		\norm{\bareta^2H_t\left(S_tS_t^{\top}\right)^2\left(S_tS_t^{\top}-\Sigma\right)\left(S_tS_t^{\top}\right)^{-1}}\leq \bareta^2 \norm{H_t}\norm{S_tS_t^{\top}}\norm{S_tS_t^{\top}-\Sigma}\lesssim (\bareta\sigma_1)^2\norm{H_t}\leq \frac{1}{12}\norm{H_t}.
	\end{equation}
	%    \begin{sloppypar}
	\begin{sloppypar}
		In a similar fashion, it can be shown that $\norm{\bareta^2E_t\left(S_t^{\top}S_t+E_t^{\top}E_t\right)E^{\top}_tH_t}\leq \frac{1}{12}\norm{H_t}$ and $\norm{\bareta^2 E_tE_t^{\top}H_t S_tS_t^{\top}\left(S_tS_t^{\top}-\Sigma\right)\left(S_tS_t^{\top}\right)^{-1}}\leq \frac{1}{12}\norm{H_t}$. Combining the derived bounds completes the proof.
		%    \end{sloppypar}    
		$\hfill\square$
	\end{sloppypar}

	\subsection{Proof of Lemma~\ref{claim::appendix-1}}\label{app_claim::appendix-1}
	To prove this lemma, we show that
	\begin{align}\label{eq_SN}
		S_{t+1}N_t = (I+\Xi)S_tN_t
	\end{align}
	for some matrix $\Xi$ with $\norm{\Xi}<1$. Before proceeding, we show that the above inequality is enough to prove the invertibility of $S_{t+1}N_t$. First note that $S_tN_t = M_t D_t$. Therefore, the matrix $M_t D_t$ is invertible due to the assumption $S_tS_t^\top\succ 0$. On the other hand, $\norm{\Xi}<1$ implies that $I+\Xi$ is invertible, thereby completing the proof.
	To verify~\eqref{eq_SN}, it suffices to show that $\Xi=S_{t+1}N_t\left(S_{t}N_t\right)^{-1}-I$ has norm less than one. To this goal, we write
	\begin{equation}\nonumber
		\begin{aligned}
			\Xi&=S_{t+1}N_t\left(S_{t}N_t\right)^{-1}-I\\
			&=V^{\top}\left(U_t-\bareta\left(U_tU_t^{\top}-X^{\star}\right)U_t+\del U_t\right)N_t\left(V^{\top}U_{t}N_t\right)^{-1}-\left(V^{\top}U_{t}N_t\right)\left(V^{\top}U_{t}N_t\right)^{-1}\\
			&=V^{\top}\left(-\bareta\left(U_tU_t^{\top}-X^{\star}\right)+\del \right)U_tN_t\left(V^{\top}U_{t}N_t\right)^{-1}\\
			&\stackrel{(a)}{=}V^{\top}\left(-\bareta\left(U_tU_t^{\top}-X^{\star}\right)+\del \right)U_t S_t^{\top}\left(S_tS_t^{\top}\right)^{-1}\\
			&=V^{\top}\left(-\bareta\left(U_tU_t^{\top}-X^{\star}\right)+\del \right)V+V^{\top}\left(-\bareta\left(U_tU_t^{\top}-X^{\star}\right)+\del \right)V_{\perp}H_t,
		\end{aligned}
	\end{equation}
	where $H_t$ is defined as $E_tS_t^{\top}\left(S_tS_t^{\top}\right)^{-1}$. Moreover, in (a), we used the following chain of equalities 
	$$
	N_t\left(V^{\top}U_tN_t\right)^{-1}=N_t\left(S_tN_t\right)^{-1}=N_tD_t^{-1}M_t^{\top}=N_tD_t M_t^{\top}M_tD_t^{-2}M_t^{\top}=S_t^{\top}\left(S_tS_t^{\top}\right)^{-1}.    
	$$
	Therefore, we have
	\begin{equation}
		\norm{\Xi}\leq \norm{-\bareta\left(U_tU_t^{\top}-X^{\star}\right)+\del}\left(1+\norm{H_t}\right){\leq 2\bareta\norm{\Delta_t}+6\bareta\delta\norm{\Delta_t}_F\leq 3\bareta\norm{\Delta_t}}\leq 1/2.
		\label{eq::appendix-205}
	\end{equation}
	Here we used {Lemma~\ref{lem::9}}, $\norm{\Delta_t}\lesssim\sigma_1$, $\norm{\del}\leq 3\bareta\delta\norm{\Delta_t}_F$, $\norm{H_t}\leq 1/3$, and our assumption on $\eta$. This completes the proof.$\hfill\square$

	%\subsection{Omitted proof of Lemma~\ref{lem:conditions_empirical}}\label{app_lem_Ht}
	
	\subsection{Proof of Lemma~\ref{lem::H_t}}\label{app_Ht}
	First, note that
	\begin{equation}\label{eq_ES}
		\begin{aligned}
			E_{s+1}S_{s+1}^{\top}&= E_sS_s^{\top}+\bareta E_sS_s^{\top}(\Sigma-S_sS_s^{\top})-\bareta E_s(S_s^{\top}S_s+E_s^{\top}E_s)S_s^{\top}-\bareta E_sE_s^{\top}E_sS^{\top}_s\\
			&+\bareta^2E_s(S_s^{\top}S_s+E_s^{\top}E_s)S_s^{\top}\left(S_sS_s^{\top}-\Sigma\right)+\bareta^2E_s\left(S_s^{\top}S_s+E_s^{\top}E_s\right)E^{\top}_sE_sS_s^{\top}\\
			&+V^{\top}_{\perp}R_s U_s S^{\top}_{s+1}+E_{s+1}U_s^{\top}R_s^{\top}V+V^{\top}_{\perp}R_s U_sU_s^{\top}R_s^{\top}V.
		\end{aligned}
	\end{equation}
	On the other hand, one can write
	\begin{equation}\label{eq_SS}
		\begin{aligned}
			S_{s+1}S_{s+1}^{\top}&=S_sS_s^{\top}+\bareta S_sS_s^{\top}\left(\Sigma-S_sS_s^{\top}\right)+\bareta \left(\Sigma-S_sS_s^{\top}\right)S_sS_s^{\top}-2\bareta S_sE_s^{\top}E_sS_s^{\top}\\
			&+\bareta^2 \left(\Sigma-S_sS_s^{\top}\right)S_sS_s^{\top}\left(\Sigma-S_sS_s^{\top}\right)+\bareta^2 S_sE_s^{\top}E_sE_s^{\top}E_sS_s^{\top}\\
			&-\bareta^2 \left(\Sigma-S_sS_s^{\top}\right)S_sE_s^{\top}E_sS_s^{\top}-\bareta^2 S_sE_s^{\top}E_sS_s^{\top}\left(\Sigma-S_sS_s^{\top}\right)\\
			&+V^{\top}R_s U_sS_{s+1}^{\top}+S_{s+1}U_s^{\top}R_s^{\top}V+V^{\top}R_s U_sU_s^{\top}R_s^{\top}V.
		\end{aligned}
	\end{equation}
	Pre-multiplying~\eqref{eq_SS} with $H_s$ leads to a relationship between $S_{s+1}S_{s+1}^{\top}$ and $E_{s+1}S_{s+1}^{\top}$: 
	\begin{align*}
		& E_{s+1}S_{s+1}^T = H_s S_{s+1}S_{s+1}^{\top} + T\nonumber\\
		\implies & H_{s+1} = H_s +T(S_{s+1}S_{s+1})^{-1},
	\end{align*}
	where simple algebra reveals that
	\begin{equation}
		\begin{aligned}
			T &= -H_s\left(\bareta \Sigma S_sS_s^{\top}-2\bareta S_sE_s^{\top}E_sS_s^{\top}+\bareta^2\Sigma S_sS_s^{\top}\left(\Sigma-S_sS_s^{\top}\right)\right.\\
			&\left.+\bareta^2 S_sE_s^{\top}E_sE_s^{\top}E_sS_s^{\top}-\bareta^2\Sigma S_sE_s^{\top}E_sS_s^{\top}-\bareta^2S_sE_s^{\top}E_sS_s^{\top}\left(\Sigma-S_sS_s^{\top}\right)\right)\\
			&-2\bareta E_sE_s^{\top}E_sS_s^{\top}+\bareta^2 E_sE_s^{\top}E_sS_s^{\top}\left(S_sS_s^{\top}-\Sigma\right)+\bareta^2 E_sE_s^{\top}E_sE_s^{\top}E_sS_s^{\top}\\
			&+V^{\top}_{\perp}R_s U_s S^{\top}_{s+1}+E_{s+1}U_s^{\top}R_s^{\top}V+V^{\top}_{\perp}R_s U_sU_s^{\top}R_s^{\top}V\\
			&-H_s\left(V^{\top}R_s U_sS_{s+1}^{\top}+S_{s+1}U_s^{\top}R_s^{\top}V+V^{\top}R_s U_sU_s^{\top}R_s^{\top}V\right).
		\end{aligned}
	\end{equation}
	For simplicity, we define $D=S_sS_s^{\top}\left(S_{s+1}S_{s+1}^{\top}\right)^{-1}$ in the sequel. Based on the the definition of $T$, one can write
	\begin{align}\label{eq_Hs}
		H_s+T(S_{s+1}S_{s+1})^{-1} = H_sA_s+B_s+C_s,
	\end{align}
	where the matrices $A_s$, $B_s$, and $C_s$ are defined as
	\begin{align*}
		A_s =& I-\bareta\Sigma D-\bareta^2\Sigma \left(S_sS_s^{\top}\right)\Sigma\left(S_sS_s^{\top}\right)^{-1}D+\bareta^2 \Sigma  S_sS_s^{\top}D,\\
		B_s =&\ -2\bareta S_sE_s^{\top}H_sD-\bareta^2 S_sE_s^{\top}E_sE_s^{\top}H_sD-\bareta^2\Sigma S_sE_s^{\top}H_sD-\bareta^2S_sE_s^{\top}E_sS_s^{\top}\Sigma \left(S_sS_s^{\top}\right)^{-1}D\nonumber\\
		&-\bareta^2S_sE_s^{\top}E_sS_s^{\top}D
		-\left(V^{\top}R_s U_sS_{s+1}^{\top}+S_{s+1}U_s^{\top}R_s^{\top}V+V^{\top}R_s U_sU_s^{\top}R_s^{\top}V\right)\left(S_{s+1}S_{s+1}^{\top}\right)^{-1},\\
		C_s = & -2\bareta E_sE_s^{\top}H_sD+\bareta^2 E_sE_s^{\top}H_s S_sS_s^{\top}\left(S_sS_s^{\top}-\Sigma\right)\left(S_{s+1}S_{s+1}^{\top}\right)^{-1}+\bareta^2 E_sE_s^{\top}E_sE_s^{\top}H_sD\\
		&+\left(V^{\top}_{\perp}R_s U_s S^{\top}_{s+1}+E_{s+1}U_s^{\top}R_s^{\top}V+V^{\top}_{\perp}R_s U_sU_s^{\top}R_s^{\top}V\right)\left(S_{s+1}S_{s+1}^{\top}\right)^{-1}.
	\end{align*} 
	To provide an upper bound on $\norm{A_s}$, we define $P=\left(I+\bareta S_sS_s^{\top}\Sigma \left(S_sS_s^{\top}\right)^{-1}-\bareta S_sS_s^{\top}\right)D$ and $Q=I+\bareta S_sS_s^{\top}\Sigma \left(S_sS_s^{\top}\right)^{-1}-\bareta S_sS_s^{\top}$. We have
	\begin{equation}
		\begin{aligned}
			\norm{A_s}^2 = \norm{I-\bareta \Sigma P}^2&=\lambda_{\max}\left(\left(I-\bareta P\Sigma\right)^\top\left(I-\bareta \Sigma P\right)\right)\\
			&\leq 1 + \lambda_{\max}\left(-\bareta \Sigma P-\bareta P^{\top}\Sigma + \bareta^2 P^{\top}\Sigma^2 P\right)\\
			&=1-\bareta\lambda_{\min}\left( \Sigma P+ P^{\top}\Sigma - \bareta P^{\top}\Sigma^2 P\right)\\
			&\leq 1-\bareta\lambda_{\min}\left(\Sigma P+ P^{\top}\Sigma - P^{\top}\Sigma P\right),
		\end{aligned}
	\end{equation}
	\begin{sloppypar}
		\noindent where we used the fact that $\eta\lesssim \frac{1}{\sigma_1}$ in the last inequality. Now it suffices to provide a lower bound for $\lambda_{\min}\left(\Sigma P+ P^{\top}\Sigma - P^{\top}\Sigma P\right)$.
		To this goal, we use the following intermediate lamma.
	\end{sloppypar}
	\begin{lemma}[Theorem 4.1. in \citet{eisenstat1998relative}]
		\label{lem::appendix-12}
		Given a diagonal matrix $\Lambda\in \R^{d\times d}$ and its perturbed variant $\Lambda'=\Lambda R$ for some $R\in\mathbb{R}^{d\times d}$, we have
		\begin{equation}
			\min_{k}{|\lambda_{i}(\Lambda')-\lambda_k(\Lambda)|}\leq |\lambda_{i}(\Lambda')|\norm{I-R^{-1}}.
		\end{equation}
		for every $i=1,2,\cdots,d$.
	\end{lemma}
	To apply this lemma, we choose $\Lambda = \Sigma$ and $R=P+\Sigma^{-1}P^{\top}\Sigma-\Sigma^{-1}P^{\top}\Sigma P$, which leads to the equality $\Sigma P+ P^{\top}\Sigma - P^{\top}\Sigma P = \Sigma R$. Given this definition and Lemma~\ref{lem::appendix-12}, we have
	\begin{align}\label{eq_sigmaR}
		\frac{\sigma_r}{1+\norm{I-R^{-1}}}\leq \lambda_{\min}(\Sigma R).
	\end{align}
	Now, we provide an upper bound for $\norm{I-R^{-1}}$. First note that $\norm{I-R^{-1}}\leq \norm{I-R}\norm{R^{-1}}$. On the other hand
	\begin{equation}\nonumber
		\begin{aligned}
			\norm{I-R}&=\norm{\Sigma^{-1}\left(P^{\top}-I\right)\Sigma\left(I-P\right)}\leq \norm{I-P}^2\leq \left( \norm{I-D}+\norm{\left(I-Q\right)D}\right)^2.
		\end{aligned}
	\end{equation}
	Using a similar approach to the proof of Lemma~\ref{lem:appendix-4}, we have
	\begin{equation}\nonumber
		\begin{aligned}
			\norm{I-D}&=\norm{\left(S_{t+1}S_{t+1}^{\top}-S_tS_t^{\top}\right)\left(S_{t+1}S_{t+1}^{\top}\right)^{-1}}\leq 0.1.
		\end{aligned}
	\end{equation} 
	On the other hand, one can write
	\begin{equation}\nonumber
		\begin{aligned}
			\norm{(I-Q)D}&\leq \norm{I-Q}\norm{D}\\
			&\stackrel{(a)}{\leq} 3\norm{I-Q}\\&=\bareta\norm{S_tS_t^{\top}\left(\Sigma-S_tS_t^{\top}\right)\left(S_tS_t^{\top}\right)^{-1}}\\
			&\leq \bareta \norm{\Sigma-S_tS_t^{\top}}\\
			&\leq 0.1,
		\end{aligned}
	\end{equation}
	where, in (a),  we used Lemma~\ref{lem:appendix-4}. Therefore, we have $\norm{I-R}\leq 0.04$. Next, we provide an upper bound for $\norm{R^{-1}}$. Note that
	\begin{equation}\nonumber
		\begin{aligned}
			R = P + \Sigma^{-1}P^{\top}\Sigma \left(I-P\right).
		\end{aligned}
	\end{equation}
	By Weyl's inequality, we have
	\begin{equation}\nonumber
		\begin{aligned}
			\sigma_{\min}(R)&\geq \sigma_{\min}(P)-\norm{\Sigma^{-1}P^{\top}\Sigma \left(I-P\right)}\\
			&\geq \sigma_{\min}(P)-\norm{P}\norm{I-P}\\
			&\geq 0.8-1.2\times 0.2=0.56.
		\end{aligned}
	\end{equation}
	Here we used the fact that $\norm{I-P}\leq 0.2$. The above inequality implies that $\norm{R^{-1}} = 1/\sigma_{\min}(R)\leq 2$. Combining the above bounds, we have
	\begin{equation}\nonumber
		\norm{I-R^{-1}}\leq \norm{I-R}\norm{R^{-1}}\leq 0.08.
	\end{equation}
	This together with~\eqref{eq_sigmaR} implies that
	\begin{equation}\nonumber
		\lambda_{\min}\left(\Sigma R\right)\geq 0.92\sigma_r.
	\end{equation}
	Therefore, we have
	\begin{equation}
		\norm{A_s}^2\leq 1-0.92\bareta\sigma_r \quad \implies \quad \norm{A_s}\leq 1-0.46\bareta\sigma_r.
		\label{eq_As}
	\end{equation}
	Next, we provide an upper bound for $\norm{B_s}$. Simple algebra reveals that
	\begin{align*}
		&\norm{2\bareta S_sE_s^{\top}H_sD\!+\!\bareta^2 S_sE_s^{\top}E_sE_s^{\top}H_sD\!+\!\bareta^2\Sigma S_sE_s^{\top}H_sD\!+\!\bareta^2S_sE_s^{\top}E_sS_s^{\top}\Sigma \left(\!S_sS_s^{\top}\!\right)^{-1}\!\!D\!+\!\bareta^2S_sE_s^{\top}E_sS_s^{\top}D}\\
		&\lesssim \bareta\norm{S_sE_s^\top}\norm{H_s}.
	\end{align*}
	Next, we provide a bound for the remaining terms in $B_s$. We have
	\begin{align*}
		\norm{V^{\top}_{\perp}R_s U_s S^{\top}_{s+1}\left(S_{s+1}S_{s+1}^{\top}\right)^{-1}}&\leq \norm{R_s}\norm{\left(I-\bareta\left(U_sU_s^{\top}-X^{\star}\right)+R_s\right)^{-1}U_{s+1}S^{\top}_{s+1}\left(S_{s+1}S_{s+1}^{\top}\right)^{-1}}\\
		&\leq \norm{R_s}\norm{\left(I-\bareta\left(U_sU_s^{\top}-X^{\star}\right)+R_s\right)^{-1}}\norm{U_{s+1}S^{\top}_{s+1}\left(S_{s+1}S_{s+1}^{\top}\right)^{-1}}\\
		&\lesssim \bareta\delta\norm{\Delta_s}_F\norm{U_{s+1}S^{\top}_{s+1}\left(S_{s+1}S_{s+1}^{\top}\right)^{-1}}.
	\end{align*}
	To proceed, we provide an upper bound for $\norm{U_{s+1}S^{\top}_{s+1}\left(S_{s+1}S_{s+1}^{\top}\right)^{-1}}$.
	\begin{align*}
		\norm{U_{s+1}S^{\top}_{s+1}\left(S_{s+1}S_{s+1}^{\top}\right)^{-1}}&\leq \norm{VS^{\top}_{s+1}S^{\top}_{s+1}\left(S_{s+1}S_{s+1}^{\top}\right)^{-1}+V_{\perp}E_{s+1}S^{\top}_{s+1}\left(S_{s+1}S_{s+1}^{\top}\right)^{-1}}\\
		&\leq 1+\norm{H_{s+1}}.
	\end{align*}
	Similarly, one can show that 
	\begin{align*}
		\norm{V^{\top}_{\perp}R_s U_sU_s^{\top}R_s^{\top}V\left(S_{s+1}S_{s+1}^{\top}\right)^{-1}}&\lesssim \bareta^2\delta^2\bar{\varphi}^4 \norm{\Delta_s}_F^2\left(1+\norm{H_s}\right)\lesssim \bareta\delta\norm{\Delta_s}_F.
	\end{align*}
	Combining the derived bounds leads to
	\begin{align}\label{eq_Bs}
		\norm{B_s}\lesssim \bareta\norm{S_sE_s^\top}\norm{H_s}+\bareta\delta\norm{\Delta_s}_F(1+\norm{H_{s+1}}).
	\end{align}
	In a similar way, one can show that
	\begin{align}\label{eq_Cs}
		\norm{C_s}\lesssim \bareta\norm{E_sE_s^\top}\norm{H_s}+\bareta\delta\norm{\Delta_s}_F(1+\norm{H_{s+1}}).
	\end{align}
	Substituting~\eqref{eq_As},~\eqref{eq_Bs}, and~\eqref{eq_Cs} in~\eqref{eq_Hs} yields
	
	\begin{align*}
		&(1-c_1\bareta\delta\norm{\Delta_s}_F)\norm{H_{s+1}}\leq \left(1-0.46\bareta\sigma_r+c_2\bareta\left(\norm{S_sE_s^\top}+\norm{E_sE_s^\top}\right)+c_3\bareta\delta\norm{\Delta_s}_F\right)\norm{H_s}\\
		&\hspace{4.5cm}+c_4\bareta\delta\norm{\Delta_s}_F\\
		\implies &\norm{H_{s+1}}\leq \left(\frac{1-0.46\bareta\sigma_r+c_2\bareta\left(\norm{S_sE_s^\top}+\norm{E_sE_s^\top}\right)+c_3\bareta\delta\norm{\Delta_s}_F}{1-c_1\bareta\delta\norm{\Delta_s}_F}\right)\norm{H_{s}}+\frac{c_4\bareta\delta\norm{\Delta_s}_F}{1-c_1\bareta\delta\norm{\Delta_s}_F}\\
		\implies & \norm{H_{s+1}}\leq (1-c_5\bareta\sigma_r)\norm{H_s}+c_6\sqrt{r}\sigma_1\bareta\delta,
	\end{align*}
	where the last inequality follows from the assumed upper bound on $\delta$, as well as~\eqref{eq_F_emp}-\eqref{eq_gen_error_emp}. This completes the proof.$\hfill\square$

	%=====================
	
	\subsection{Proof of Lemma~\ref{lem::appendix-prod-exp}}\label{app_lem::appendix-prod-exp}
	We first prove the upper bound. One can write
	\begin{equation}\nonumber
		\begin{aligned}
			\log\left(\prod_{t=0}^{T}\left(1+\alpha\rho^t\right)\right) & =\sum_{t=0}^{T}\log\left(1+\alpha\rho^t\right)\leq \sum_{t=0}^{\infty}\alpha\rho^t = \frac{\alpha}{1-\rho}.
		\end{aligned}
	\end{equation}
	Hence, we have $\prod_{t=0}^{T}\left(1+\alpha\rho^t\right)\leq \exp\left(\frac{\alpha}{1-\rho}\right)$. For the lower bound, we have
	\begin{equation}\nonumber
		\begin{aligned}
			\log\left(\prod_{t=0}^{\infty}\left(1+\alpha\rho^t\right)\right) & =\sum_{t=0}^{T}\log\left(1+\alpha\rho^t\right) \geq \sum_{t=0}^{T}\frac{\alpha\rho^t}{1+\alpha\rho^t}\geq \sum_{t=0}^{T}\frac{\alpha\rho^t}{1+\alpha}\geq \frac{\alpha}{1+\alpha}T\rho^T.
		\end{aligned}
	\end{equation}
	Hence, we have $\prod_{t=0}^{T}\left(1+\alpha\rho^t\right)\geq \exp\left(\frac{\alpha}{1+\alpha}T\rho^T\right)$.$\hfill\square$
	
	\subsection{Proof of Lemma~\ref{lem_Gt2}}\label{app_lem_Gt2}
	
	We start with the proof of the first statement. We consider two cases:
	\begin{itemize}
		\item[-] Suppose that $\norm{\Delta_t}\leq 0.01\sigma_r\rho^t$. Recall that $U_{t+1}=U_t-\frac{\eta_0\rho^t}{\norm{\Delta_t}}\Delta_t U_t+\del U_t$. Hence, we have
		\begin{equation}\nonumber
			\begin{aligned}
				& \Delta_{t+1} =\left(U_t-\frac{\eta_0\rho^t}{\norm{\Delta_t}}\Delta_t U_t+\del U_t\right)\left(U_t^{\top}-\frac{\eta_0\rho^t}{\norm{\Delta_t}}U_t^{\top}\Delta_t^{\top} +U_t^{\top}\del^{\top}\right)-X^{\star}.
			\end{aligned}
		\end{equation}
		The above equality leads to 
		\begin{equation}
			\begin{aligned}
				\norm{\Delta_{t+1}} & \leq \norm{\Delta_{t}}+\eta\rho^{t}\norm{U_{t}U_{t}^{\top}}\left(2+2\norm{R_{t}}\right)\\
				&+2\norm{U_{t}U_{t}^{\top}}\norm{R_{t}}+\norm{U_{t}U_{t}^{\top}}\norm{R_{t}}^2+\eta^2\rho^{2t}\norm{U_{t}U_{t}^{\top}} \\
				& \stackrel{(a)}{\leq} \norm{\Delta_{t}}+4\sigma_1\eta\rho^{t}+4\sigma_1\eta\delta\rho^{t}\norm{\Delta_t}_F+2\sigma_1\eta^2\rho^{2t} \\                               & \leq 0.02\sigma_r\rho^{t}.
			\end{aligned}
		\end{equation}
		On the other hand, we know that $\gamma_{t+1}\leq 5\norm{\Delta_{t+1}}$, which, together with the above inequality, implies that $\gamma_{t+1}\leq 0.1\sigma_r$.
		\item[-] Suppose that $\norm{\Delta_t}\geq 0.01\sigma_r\rho^t$. Therefore, we have $\frac{\eta_0\rho^t}{\norm{\Delta_t}}\lesssim \frac{1}{\sigma_1}$. On the other hand, since $\gamma_t\leq 0.1\sigma_r$, we have
		\begin{equation}\nonumber
			\lambda_{\min}\left(S_tS_t^{\top}\right)\geq 0.9\sigma_r, \quad \norm{S_tS_t^{\top}}\leq 1.1\sigma_1, \quad \norm{E_tE_t^{\top}}\leq 0.1\sigma_r, \quad \norm{E_tS_t^{\top}\left(S_tS_t^{\top}\right)^{-1}}\leq 0.2.
		\end{equation}
		This implies that the assumptions of Propositions~\ref{prop_min_eig_outlier},~\ref{prop::finite-partially-corrupted}, and~\ref{prop::F-G_outlier} are satisfied at iteration $t$, and we have 
		\begin{equation}\label{eq_gamma}
			\begin{aligned}
				\gamma_{t+1} & \leq
				\norm{\Sigma -S_{t+1}S_{t+1}^{\top}}+2\norm{S_{t+1}E_{t+1}^{\top}}+ \norm{F_{t+1}}^2+\norm{G_{t+1}}^2                                            \\ &\stackrel{(a)}{\leq} \left(1-\Omega(1)\frac{\sigma_r\eta_0\rho^t}{\gamma_t}\right)\gamma_t+\cO(1)\sqrt{r}\sigma_1\eta_0\delta\rho^{t} + \cO\left(1\right)\frac{ \sigma_r\eta_0\rho^t}{\norm{\Delta_t}}\norm{G_t}^2\\
				& \leq \gamma_t-\Omega\left(\sigma_r\eta_0\rho^t\right)+\cO\left(1\right)\sigma_r\eta_0\rho^t\frac{\norm{G_t}^2 }{\norm{\Delta_t}}
				\\ &\stackrel{(b)}{\leq} \gamma_t-\Omega\left(\sigma_r\eta_0\rho^t\right)+\cO\left(\sigma_r\eta_0\rho^t/\sqrt{d}\right)\nonumber\\
				&\leq \gamma_t-\Omega\left(\sigma_r\eta_0\rho^t\right),
			\end{aligned}
		\end{equation}
		where $(a)$ follows from the one-step dynamics of the signal, cross, and residual terms derived in Propositions~\ref{prop_min_eig_outlier},~\ref{prop::finite-partially-corrupted}, and~\ref{prop::F-G_outlier}. Moreover, $(b)$ follows from $\norm{\Delta_t}\geq\sqrt{d}\norm{G_t}^2$. Therefore, we have $\gamma_{t+1}\leq \gamma_t\leq 0.1\sigma_r$.
	\end{itemize}	
	
	To complete the proof of this lemma, it suffices to show that if $\norm{\Delta_{t+1}}\leq\sqrt{d}\norm{G_{t+1}}^2$, then $\norm{\Delta_{t+1}}\leq \sqrt{d}\alpha^{1-\mathcal{O}(\sqrt{r}\kappa\delta)}$. Note that based on our assumption and Phase 1 of the proof of Theorem~\ref{thm::finite-noisy}, the one-step dynamic of $G_s$ holds for every $0\leq s\leq t+1$. Therefore, an analysis similar to Lemma~\eqref{lem:conditions_noisy} leads to $\norm{\Delta_{t+1}}\leq \sqrt{d}\alpha^{1-\mathcal{O}(\sqrt{r}\kappa\delta)}$.$\hfill\square$
	
	\subsection{Proof of Lemma~\ref{lem::appendix-linear-convergence-geometric}}	\label{app_lem::appendix-linear-convergence-geometric}
	Since $\norm{\Delta_{t_0+T_3-1}}\leq 0.02\sigma_r\rho^{t_0-1}$, an argument similar to the proof of Lemma~\ref{lem_Gt2} can be invoked to show that $\norm{\Delta_{t_0+T_3}}\leq 0.03\sigma_r\rho^{t_0}$ and $\gamma_{t_0+T_3}\leq 0.15\sigma_r\rho^{t_0}$. Let $\Delta t$ be the first time that $\norm{\Delta_{t_0+T_3+\Delta t}}\leq 0.02\sigma_r\rho^{t_0+\Delta t}$. Note that since $\norm{\Delta_{t_0+T_3}}> 0.02\sigma_r\rho^{t_0}$, we have $\Delta t\geq 1$. This implies that, for every $0\leq s\leq \Delta t-1$, we have $\norm{\Delta_{t_0+T_3+s}}> 0.02\sigma_r\rho^{t_0+s}$. Therefore,~\eqref{eq_gamma} implies that $\gamma_{t_0+T_3+s+1}\leq \gamma_{t_0+T_3+s}-\Omega(\sigma_r\eta\rho^{t_0+s})$. This in turn leads to
	\begin{align*}
		\gamma_{t_0+T_3+\Delta t}&\leq \gamma_{t_0+T_3}-\Omega\left(\sum_{s=0}^{\Delta t-1}\sigma_r\eta\rho^{t_0+s}\right)\\
		&\leq 0.15\sigma_r\rho^{t_0} -\Omega\left(\sum_{s=0}^{\Delta t-1}\sigma_r\eta\rho^{t_0+s}\right)\\
		&=\sigma_r\rho^{t_0}\left(0.15 - \Omega\left(\sum_{s=0}^{\Delta t-1}\eta\rho^{s}\right)\right).
	\end{align*}
	Let us assume that $\Delta t \lesssim (\kappa/\eta)\log(1/\alpha)$. Under this assumption, we have $\rho^s = \Omega(1)$ for every $s\leq \Delta t$. This implies that $\Omega\left(\sum_{s=0}^{\Delta t-1}\eta\rho^{s}\right) = \Omega(\eta\Delta t)$. Therefore, upon choosing $\Delta t = \Omega(1/\eta)$, we have $\gamma_{t_0+T_3+\Delta t}\leq 0.15 - \Omega\left(\sum_{s=0}^{\Delta t-1}\eta\rho^{s}\right)\leq 0$. This implies that, there must exist $\tilde t\leq \Delta t$ such that $\gamma_{t_0+T_3+\tilde t}\leq 0.02\sigma_r\rho^{t_0+\tilde t}$. This completes the proof.$\hfill\square$

\end{document}